%% file: main.tex
\documentclass[
	twoside,openright,titlepage,numbers=noenddot,headinclude,
	footinclude=true,cleardoublepage=empty,
	BCOR=5mm,paper=letter,fontsize=11pt, 
	ngerman,american, 
]{scrreprt}

\input{classicthesis-config}
\input{helpers}

\begin{document}

\frenchspacing 

\raggedbottom 

\selectlanguage{american} 


\pagenumbering{roman} 

\pagestyle{plain} 


\include{front-back-matter/titlepage} 

\include{front-back-matter/titleback} 

\cleardoublepage
\include{front-back-matter/abstract} 

\cleardoublepage
\include{front-back-matter/signatures} 

\cleardoublepage
\include{front-back-matter/acknowledgments} 

\pagestyle{scrheadings} 

\cleardoublepage
\include{front-back-matter/contents} 

\pagenumbering{arabic} 

\cleardoublepage 


\input{part-svm}
\input{part-pca}


\cleardoublepage\include{front-back-matter/bibliography} 


\end{document}

%% file: classicthesis-config.tex
\usepackage{float}

\PassOptionsToPackage{dottedtoc,eulerchapternumbers,listings,pdfspacing,subfig,beramono,eulermath,parts}{classicthesis}

\newcommand{\myTitle}{Stochastic Optimization for Machine Learning\xspace}
\newcommand{\myName}{Andrew Cotter\xspace}
\newcommand{\myUni}{Toyota Technological Institute at Chicago\xspace}
\newcommand{\myFaculty}{Nathan Srebro\xspace}
\newcommand{\myLocation}{Chicago, Illinois\xspace}
\newcommand{\myTime}{August, 2013\xspace}
\newcommand{\myDefenseTime}{July 19, 2013\xspace}
\newcommand{\myVersion}{version 4.0\xspace}

\newcounter{dummy} 
\providecommand{\mLyX}{L\kern-.1667em\lower.25em\hbox{Y}\kern-.125emX\@}

\usepackage{lipsum} 

\PassOptionsToPackage{latin9}{inputenc} 
\usepackage{inputenc}

\usepackage{babel}

\PassOptionsToPackage{square,numbers}{natbib}
\usepackage{natbib}

\PassOptionsToPackage{fleqn}{amsmath} 
\usepackage{amsmath}

\usepackage{amsthm}
\usepackage{nicefrac}
\usepackage{mathtools}

\PassOptionsToPackage{T1}{fontenc} 
\usepackage{fontenc}

\usepackage{xspace} 

\usepackage{mparhack} 

\usepackage{fixltx2e} 

\PassOptionsToPackage{smaller}{acronym} 
\usepackage{acronym} 

\PassOptionsToPackage{pdftex}{graphicx}
\usepackage{graphicx}

\usepackage{tabularx} 
\usepackage{caption}
\usepackage{subfig}

\usepackage{listings}
\PassOptionsToPackage{pdftex,hyperfootnotes=false,pdfpagelabels}{hyperref}
\usepackage{hyperref}  
\hypersetup{
colorlinks=true, linktocpage=true, pdfstartpage=3, pdfstartview=FitV,
breaklinks=true, pdfpagemode=UseNone, pageanchor=true, pdfpagemode=UseOutlines,
plainpages=false, bookmarksnumbered, bookmarksopen=true, bookmarksopenlevel=1,
hypertexnames=true, pdfhighlight=/O, urlcolor=webbrown, linkcolor=RoyalBlue, citecolor=webgreen,
pdftitle={\myTitle},
pdfauthor={\textcopyright\ \myName, \myUni, Advisor: \myFaculty},
pdfsubject={},
pdfkeywords={},
pdfcreator={pdfLaTeX},
pdfproducer={LaTeX with hyperref and classicthesis}
}

\usepackage{ifthen} 
\newboolean{enable-backrefs} 
\setboolean{enable-backrefs}{false} 

\newcommand{\backrefnotcitedstring}{\relax} 
\newcommand{\backrefcitedsinglestring}[1]{(Cited on page~#1.)}
\newcommand{\backrefcitedmultistring}[1]{(Cited on pages~#1.)}
\ifthenelse{\boolean{enable-backrefs}} 
{
\PassOptionsToPackage{hyperpageref}{backref}
\usepackage{backref} 
\renewcommand*{\backref}[1]{}  
\renewcommand*{\backrefalt}[4]{
\ifcase #1
\backrefnotcitedstring
\or
\backrefcitedsinglestring{#2}
\else
\backrefcitedmultistring{#2}
\fi}
}{\relax}

\makeatletter
\@ifpackageloaded{babel}
{
\addto\extrasamerican{

}
\addto\extrasngerman{

}
}{\relax}
\makeatother

\usepackage{classicthesis}

\usepackage[top=1.2in,bottom=1.2in,textwidth=420pt]{geometry}

\usepackage{textcomp}    
\usepackage{parskip}     
\usepackage{tabularx}    
\usepackage{environ}

%% file: helpers.tex
\newcommand{\algorithmname}{Algorithm}
\newcommand{\listalgorithmname}{List of \algorithmname s}
\floatstyle{ruled}
\newfloat{algorithm}{htbp}{loa}
\floatname{algorithm}{\algorithmname}

\AtBeginDocument{\addtocontents{loa}{\deactivateaddvspace}}
\newcommand{\listofalgorithms}{\listof{algorithm}{\listalgorithmname}}

\numberwithin{equation}{chapter}
\numberwithin{figure}{chapter}
\numberwithin{table}{chapter}
\numberwithin{algorithm}{chapter}

\makeatletter
\let\c@table\c@figure
\let\c@algorithm\c@figure
\makeatother

\newtheorem{theorem}[equation]{Theorem}
\newtheorem{lemma}[equation]{Lemma}
\newtheorem{corollary}[equation]{Corollary}

\makeatletter
\newtheorem*{rep@theorem}{\rep@title}
\newcommand{\newreptheorem}[2]{\newenvironment{rep#1}[1]{\def\rep@title{#2 \ref{##1}}\begin{rep@theorem}}{\end{rep@theorem}}}
\makeatother

\newreptheorem{theorem}{Theorem}
\newreptheorem{lemma}{Lemma}
\newreptheorem{corollary}{Corollary}

\newif\ifproofsection
\proofsectionfalse
\newenvironment{proofs}{\proofsectiontrue}{\proofsectionfalse}
\NewEnviron{splittheorem}[1]{
	\ifproofsection
		\label{prf:#1}\begin{reptheorem}{#1}\BODY\end{reptheorem}
	\else
		\begin{theorem}\label{#1}\BODY\end{theorem}
		\begin{proof}In Section \ref{prf:#1}.\end{proof}
	\fi
}
\NewEnviron{splitlemma}[1]{
	\ifproofsection
		\label{prf:#1}\begin{replemma}{#1}\BODY\end{replemma}
	\else
		\begin{lemma}\label{#1}\BODY\end{lemma}
		\begin{proof}In Section \ref{prf:#1}.\end{proof}
	\fi
}
\NewEnviron{splitcorollary}[1]{
	\ifproofsection
		\label{prf:#1}\begin{repcorollary}{#1}\BODY\end{repcorollary}
	\else
		\begin{corollary}\label{#1}\BODY\end{corollary}
		\begin{proof}In Section \ref{prf:#1}.\end{proof}
	\fi
}
\NewEnviron{splitproof}{
	\ifproofsection
		\begin{proof}\BODY\end{proof}
	\fi
}

\newcolumntype{L}{>{\centering\arraybackslash} m{0.04\textwidth}}
\newcolumntype{T}{>{\centering\arraybackslash} m{0.32\textwidth}}
\newcolumntype{H}{>{\centering\arraybackslash} m{0.48\textwidth}}

\ifthenelse{\boolean{@drafting}} {
	\newcommand{\todo}[1]{\textcolor{BrickRed}{#1}\marginpar{\textcolor{Red}{TODO}}\xspace}
	\newcommand{\margintodo}[1]{\marginpar{\textcolor{Red}{TODO: }\textcolor{BrickRed}{#1}}}
} {
	\newcommand{\todo}[1]{}
	\newcommand{\margintodo}[1]{}
}

\newcommand{\iid}{\emph{i.i.d.}\xspace}

\newcommand{\probability}[2][]{\mathrm{Pr}_{#1}\left\{#2\right\}}
\newcommand{\expectation}[2][]{\mathbb{E}_{#1}\left[#2\right]}
\newcommand{\variance}[2][]{\mathrm{Var}_{#1}\left(#2\right)}
\newcommand{\indicator}[2]{\mathbf{1}_{\left\{#1\right\}}\left(#2\right)}

\newcommand{\loss}[2][]{\mathcal{L}_{#1} \left(#2\right)}
\newcommand{\emploss}[2][]{\hat{\mathcal{L}}_{#1} \left(#2\right)}
\newcommand{\hinge}{\mathrm{hinge}}
\newcommand{\slant}{\mathrm{slant}}
\newcommand{\zeroone}{\mathrm{0/1}}
\newcommand{\smooth}{\mathrm{smooth}}
\newcommand{\hilbert}{\mathcal{H}}
\newcommand{\X}{\mathcal{X}}
\newcommand{\R}{\mathbb{R}}
\newcommand{\N}{\mathbb{N}}

\newcommand{\refw}{u}

\newcommand{\inner}[2]{\left\langle {#1}, {#2} \right\rangle}
\newcommand{\norm}[1]{\left\lVert {#1} \right\rVert}

\newcommand{\abs}[1]{\left\lvert {#1} \right\rvert}
\DeclareMathOperator{\sign}{sign}
\DeclareMathOperator{\argmax}{argmax}
\DeclareMathOperator{\argmin}{argmin}

\newcommand{\spectrum}[2][]{\sigma_{#1}\left( {#2} \right)}
\newcommand{\project}[2][]{\mathcal{P}_{#1}\left(#2\right)}
\DeclareMathOperator{\rank}{rank}
\DeclareMathOperator{\trace}{tr}
\DeclareMathOperator{\diag}{diag}

\newcommand{\code}[1]{\mbox{\texttt{#1}}}
\newcounter{lineno}
\newenvironment{pseudocode}{\begin{small}\begin{tabbing}\textbf{mm}\=mm\=mm\=mm\=mm\=mm\=mm\=mm\=mm\=\kill}{\end{tabbing}\end{small}}
\newcommand{\codename}{\setcounter{lineno}{0}\>}
\newcommand{\codeline}{\>\stepcounter{lineno}\textbf{\arabic{lineno}}\'\>}
\newcommand{\codeskip}{\>\>}

%% file: front-back-matter/titlepage.tex

\begin{titlepage}

\begin{addmargin}[-1cm]{-3cm}
\begin{center}

\hfill
\vfill

\begingroup
\large\color{Maroon}\spacedallcaps{\myTitle} \\ \bigskip 
\endgroup
by \\ \medskip
\spacedallcaps{\myName}

\vfill

\includegraphics[width=0.3\textwidth]{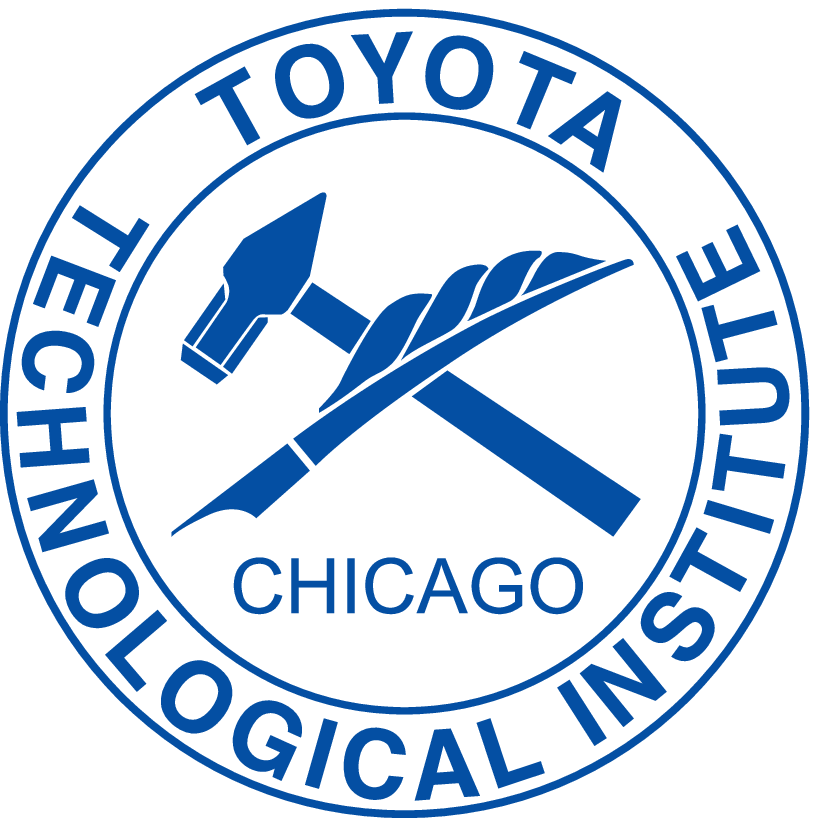} \\ \bigskip 

A thesis submitted in partial fulfillment of the requirements for the degree of \\ \medskip
Doctor of Philosophy in Computer Science \\ \medskip
at the \\ \medskip
\spacedallcaps{\myUni} \\
\myLocation \\ \bigskip
\myTime

\vfill

Thesis committee: \\ \medskip
Yury Makarychev \\
David McAllester \\
Nathan Srebro (thesis advisor) \\
Stephen Wright

\vfill

\end{center}
\end{addmargin}

\end{titlepage}

%% file: front-back-matter/titleback.tex

\thispagestyle{empty}

\begin{center}

\hfill
\vfill

\noindent\myName: \textit{\myTitle,} 
\textcopyright\ \myTime

\end{center}








%% file: front-back-matter/abstract.tex

\pdfbookmark[1]{Abstract}{Abstract} 

\begingroup
\let\clearpage\relax
\let\cleardoublepage\relax
\let\cleardoublepage\relax

\chapter*{Abstract} 

It has been found that stochastic algorithms often find good solutions much
more rapidly than inherently-batch approaches. Indeed, a very useful rule of
thumb is that often, when solving a machine learning problem, an iterative
technique which relies on performing a very large number of
relatively-inexpensive updates will often outperform one which performs a
smaller number of much "smarter" but computationally-expensive updates.

In this thesis, we will consider the application of stochastic algorithms to
two of the most important machine learning problems. Part i is concerned with
the supervised problem of binary classification using kernelized linear
classifiers, for which the data have labels belonging to exactly two classes
(e.g. "has cancer" or "doesn't have cancer"), and the learning problem is to
find a linear classifier which is best at predicting the label. In Part ii, we
will consider the unsupervised problem of Principal Component Analysis, for
which the learning task is to find the directions which contain most of the
variance of the data distribution.

Our goal is to present stochastic algorithms for both problems which are, above
all, practical--they work well on real-world data, in some cases better than
all known competing algorithms. A secondary, but still very important, goal is
to derive theoretical bounds on the performance of these algorithms which are
at least competitive with, and often better than, those known for other
approaches.

\paragraph{Collaborators:} The work presented in this thesis was performed
jointly with Raman Arora, Karen Livescu, Shai Shalev-Shwartz and Nathan Srebro.

\endgroup			

\vfill

%% file: front-back-matter/signatures.tex
\thispagestyle{empty}

\begin{center}

\hfill
\vfill

\begingroup
\large\color{Maroon}\spacedallcaps{\myTitle}
\endgroup

\vfill

A thesis presented by \\ \medskip
\spacedallcaps{\myName} \\ \medskip
in partial fulfillment of the requirements for the degree of \\
Doctor of Philosophy in Computer Science.

\vfill

\myUni \\
\myLocation \\ \bigskip
\myTime

\vfill

\newcommand{\signature}[3]{
	\rule{0pt}{0.1\textheight} \Large \texttt{#2} & & #3 \\
	\hline
	\small #1 & & \small Signature
}

\begin{tabularx}{\textwidth}{lXr}
\signature{Committee member}{Yury Makarychev}{\includegraphics[height=3em]{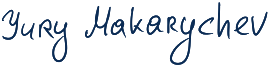}} \\
\signature{Committee member}{David McAllester}{\includegraphics[width=0.5\textwidth]{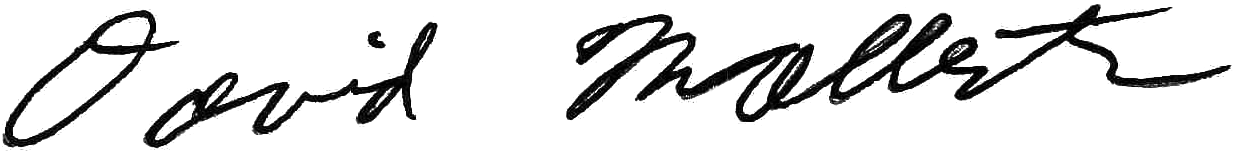}} \\
\small Chief academic officer & & \\
\signature{Committee member}{Nathan Srebro}{\includegraphics[height=3em]{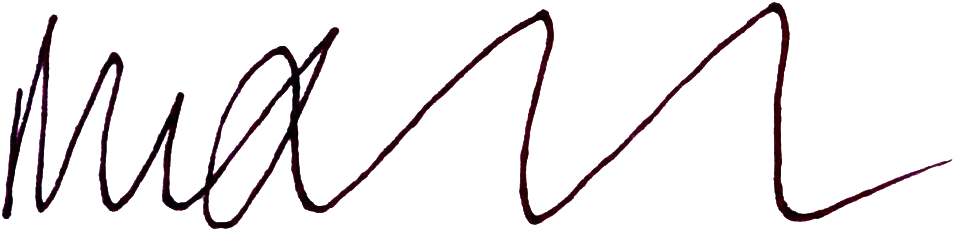}} \\
\small Thesis advisor & & \\
\signature{Committee member}{Stephen Wright}{\includegraphics[height=3em]{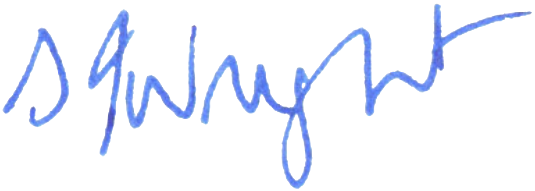}}
\end{tabularx}

\vfill

\myDefenseTime

\vfill

\end{center}

%% file: front-back-matter/acknowledgments.tex

\pdfbookmark[1]{Acknowledgments}{Acknowledgments} 

\begin{flushright}{\slshape
Think not, is my eleventh commandment; and sleep when you can, is my twelfth.} \\ \medskip
--- \defcitealias{Melville51}{Herman Melville}\citetalias{Melville51} \citep{Melville51}
\end{flushright}

\bigskip


\begingroup

\let\clearpage\relax
\let\cleardoublepage\relax
\let\cleardoublepage\relax

\chapter*{Acknowledgments} 

There are four people without whom the completion of this thesis would not have
been possible: Carole Cotter, whose degree of love and support has far exceeded
that which might be expected even from one's mother; my girlfriend, Allie
Shapiro, who has been, if anything, too understanding of my idiosyncrasies;
Umut Acar, who gave me my first opportunity to contribute to a successful
research project; and finally, in the last (and most significant) place, my
advisor, Nati Srebro, whose ability, knowledge and dedication to the discovery
of significant results provides the model of what a ``machine learning
researcher'' should be.

Many other people have contributed, both directly and indirectly, to the work
described in this thesis. Of these, the most significant are my co-authors on
computer science papers (in alphabetical order): Raman Arora, Beno\^{i}t
Hudson, Joseph Keshet, Karen Livescu, Shai Shalev-Shwartz, Ohad Shamir, Karthik
Sridharan and Duru T\"{u}rko\u{g}lu. Others include my fellow graduate
students: Avleen Bijral, Ozgur Sumer and Payman Yadollahpour; my colleagues at
NCAR, most significantly Kent Goodrich and John Williams; and professors,
including Karl Gustafson, Sham Kakade and David McAllester.

I would be remiss to fail to thank the remaining members of my thesis
committee, Yury Makarychev and Stephen Wright, both of whom, along with David
McAllester and Nati Srebro (previously mentioned), have been enormously helpful
in improving not only the writing style of this thesis, but also its content
and points of emphasis. Several others also gave me advice and asked
instructive questions, both while I was practicing for my thesis defense, and
afterwards: these include Francesco Orabona and Matthew Rocklin, as well as
Raman Arora, Allie Shapiro and Payman Yadollahpour (previously mentioned).

I must also thank the creators of the excellent
\href{http://code.google.com/p/classicthesis}{classicthesis} \LaTeX \ style,
which I think is extremely aesthetically pleasing, despite the fact that I
reduced the sizes of the margins in direct contravention of their instructions.

Finally, I apologize to anybody who I have slighted in this brief list of
acknowledgments. One would often say ``you know who you are'', but the people
who have inspired me or contributed to research avenues onto which I have
subsequently embarked are not only too numerous to list, but are often unknown
to me, and I to them. All that I can do is to confirm my debt to all
unacknowledged parties, and beg for their forgiveness.

\endgroup

%% file: front-back-matter/contents.tex

\refstepcounter{dummy}

\pdfbookmark[1]{\contentsname}{tableofcontents} 

\setcounter{tocdepth}{2} 

\setcounter{secnumdepth}{3} 

\manualmark
\markboth{\spacedlowsmallcaps{\contentsname}}{\spacedlowsmallcaps{\contentsname}}
\tableofcontents

\automark[section]{chapter}
\renewcommand{\chaptermark}[1]{\markboth{\spacedlowsmallcaps{#1}}{\spacedlowsmallcaps{#1}}}
\renewcommand{\sectionmark}[1]{\markright{\thesection\enspace\spacedlowsmallcaps{#1}}}

\clearpage

\begingroup
\let\clearpage\relax
\let\cleardoublepage\relax
\let\cleardoublepage\relax


\refstepcounter{dummy}
\addcontentsline{toc}{chapter}{\listtablename} 
\pdfbookmark[1]{\listtablename}{lot} 

\listoftables

\vspace*{8ex}
\newpage


\refstepcounter{dummy}
\addcontentsline{toc}{chapter}{\listfigurename} 
\pdfbookmark[1]{\listfigurename}{lof} 

\listoffigures

\vspace*{8ex}
\newpage


\refstepcounter{dummy}
\addcontentsline{toc}{chapter}{\listalgorithmname} 
\pdfbookmark[1]{\listalgorithmname}{loa} 

\listofalgorithms

\vspace*{8ex}
\newpage

%
%
%
%
%
%
%
%

\endgroup

\cleardoublepage

%% file: part-svm.tex
\ctparttext{
\todo{write part-abstract! include list of contributions, and cite publications}
}

\part{Support Vector Machines}\label{part:svm} 

\include{svm/ch-introduction}

\include{svm/ch-sbp}

\include{svm/ch-sparse}

\cleardoublepage 

%% file: svm/ch-introduction.tex
\chapter{Prior Work}\label{ch:svm-introduction}

\input{svm/introduction/sec-overview}
\input{svm/introduction/sec-objective}
\input{svm/introduction/sec-generalization}
\input{svm/introduction/sec-traditional}
\input{svm/introduction/sec-non-traditional}

\paragraph{Collaborators:} Much of the content of this chapter is not exactly
``novel'', but to the extent that it is, the work was performed jointly with
Shai Shalev-Shwartz and Nathan Srebro.

\clearpage
\input{svm/introduction/sec-proofs}

%% file: svm/introduction/sec-overview.tex
\section{Overview}\label{sec:svm-introduction:overview}

One of the oldest and simplest problems in machine learning is that of
supervised binary classification, in which the goal is, given a set of training
vectors $x_1,\dots,x_n \in \X$ with corresponding labels
$y_1,\dots,y_n \in \{\pm 1\}$ drawn \iid from an unknown distribution
$\mathcal{D}$, to learn a classification function which assigns labels to
previously-unseen samples.

In this and the following chapters, we will consider the application of Support
Vector Machines (SVMs) \citep{CortesVa95} to such problems.  It's important to
clearly distinguish the problem to be solved (in this case, binary
classification) from the tool used to solve it (SVMs). In the years since their
introduction, the use of SVMs has become widespread, and it would not be unfair
to say that it is one of a handful of canonical machine learning techniques
with which nearly every student is familiar and nearly every practitioner has
used. For this reason, it has become increasingly easy to put the cart before
the horse, so to speak, and to view each advancement as \emph{improving the
performance of SVMs}, and not \emph{finding better binary classifiers}.

This is not merely a semantic distinction, since practitioners must measure the
performance of the classifiers which they find, while theoreticians often seek
to bound the amount of computation required to find (and/or use) a good
classifier. For both of these tasks, one requires a metric, a quantifiable way
of answering the question ``just how good is this classifier, anyway?''. For a
SVM, which as we will see in Section \ref{sec:svm-introduction:objective} can be
reduced to a convex optimization problem (indeed, there are several different
such reductions in widespread use), the most convenient measure is often the
suboptimality of the solution. This convenience is an illusion, however, since
unless one has a means of converting a bound on suboptimality into a bound on
the classification error on unseen data, this ``suboptimality metric'' will
tell you nothing at all about the true quality of the solution.

Instead, we follow \citet{BottouBo07,ShalevSr08}, and consider the ``quality''
of a SVM solution to be the expected proportion (with respect to $\mathcal{D}$)
of classification mistakes made by the classifier, also known as the
\emph{generalization error}. The use of generalization error is hardly novel,
neither in this thesis nor in the cited papers---rather, both seek to reverse
the trend of viewing SVMs not as a means to an end, but as an end in
themselves.

Having discussed the ``how'' and ``why'' of measuring the performance of SVMs,
let us move on to the ``what''. Any supervised learning technique consists of
(at least) two phases: training and testing. In the training phase, one uses
the set of provided labeled examples (the training data) to find a classifier.
In the testing phase, one \emph{uses} this classifier on previously-unseen
data. Ideally, this ``use'' would be the application of the learned classifier
in a production system---say, to provide targeted advertisements to web-page
viewers. In research, evaluation of true production systems is rare. Instead,
one simulates such an application by using the learned classifier to predict
the labels of a set of held-out testing data, the true labels of which are
known to the researcher, and then comparing the predicted labels to the true
labels in order to estimate the generalization error of the classifier.

Most work focuses on finding a classifier which generalizes well as quickly as
possible, i.e. jointly minimizing both the testing error and training runtime.
In Chapter \ref{ch:svm-sbp}, a kernel SVM optimization algorithm will be
presented which enjoys the best known bound, measured in these terms. Testing
runtime should not be neglected, however. Indeed, in some applications, it may
be far more important than the training runtime---for example, if one wishes to
use a SVM for face recognition on a camera, then one might not mind a training
runtime of weeks on a cluster, so long as evaluating the eventual solution on
the camera's processor is sufficiently fast. In Chapter \ref{ch:svm-sparse}, an
algorithm for improving the testing performance of kernel SVMs will be
presented, which again enjoys the best known bound on its performance.

The remainder of this chapter will be devoted to laying the necessary
groundwork for the presentation of these algorithms. In Section
\ref{sec:svm-introduction:objective}, the SVM problem will be discussed in detail. In
Section \ref{sec:svm-introduction:generalization}, generalization bounds will
be introduced, and a key result enabling their derivation will be given. The
chapter will conclude, in Sections \ref{sec:svm-introduction:traditional} and
\ref{sec:svm-introduction:non-traditional}, with discussions of several
alternative SVM optimization algorithms, including proofs of generalization
bounds.

%% file: svm/introduction/sec-objective.tex
\section{Objective}\label{sec:svm-introduction:objective}

\input{svm/introduction/figures/tab-notation}

Training a SVM amounts to finding a vector $w$ defining a classifier $x \mapsto
\sign(\inner{w}{\Phi\left(x\right)})$, that on the one hand has low norm, and
on the other has a small training error, as measured through the average hinge
loss on the training sample (see the definition of $\emploss[\hinge]{w}$ in
Table \ref{tab:svm-introduction:notation}). This is captured by the following
bi-criterion optimization problem \citep{HazanKoSr11}:
\begin{align}
\label{eq:svm-introduction:bi-criterion-objective} \mbox{minimize}: & \norm{w},
\emploss[\hinge]{g_w}
\end{align}
We focus on \emph{kernelized} SVMs, for which we assume the existence of a
function $\Phi:\X\rightarrow\hilbert$ which maps elements of $\X$ to elements
of a kernel Hilbert space in which the ``real'' work will be done. The linear
classifier which we seek is an element of this Hilbert space, and is therefore
linear, not with respect to $\X$, but rather with respect to $\hilbert$. It
turns out that we can work in a Hilbert space $\hilbert$ which is defined
implicitly, via a kernel function $K(x,x') = \inner{\Phi(x)}{\Phi(x')}$, in
which case we never need to evaluate $\Phi$, nor do we need to explicitly
represent elements of $\hilbert$. This can be accomplished by appealing to the
representor theorem, which enables us to write $w$ as a linear combination of
the training vectors with coefficients $\alpha$:
\begin{equation}
\label{eq:svm-introduction:representor-theorem} w = \sum_{i=1}^n \alpha_i y_i \Phi( x_i )
\end{equation}
It follows from this representation that we can write some important quantities
in terms of only kernel functions, without explicitly using $\Phi$:
\begin{align*}
\norm{w}^2 = & \sum_{i=1}^n \sum_{j=1}^n \alpha_i \alpha_j y_i y_j K(x_i, x_j)
\\
\inner{w}{\Phi(x)} = & \sum_{i=1}^n \alpha_i y_i K(x_i,x)
\end{align*}
These equations can be simplified further by making use of what we call the
responses (see Table \ref{tab:svm-introduction:notation}---more on these
later).
In the so-called ``linear'' setting, it is assumed that $\X=\hilbert=\R^d$, and
that $\Phi$ is the identity function. In this case, $K(x,x')=\inner{x}{x'}$,
and kernel inner products are nothing but Euclidean inner products on $\X$. It
is generally simpler to understand SVM optimization algorithms in the linear
setting, with the kernel ``extension'' only being given once the underlying
linear algorithm is understood. Here, we will attempt to do the next best
thing, and will generally present our results explicitly, using $\Phi$ instead
of $K$, referring to the kernel only when needed.

When first encountering the formulation of Problem
\ref{eq:svm-introduction:bi-criterion-objective}, two questions naturally
spring to mind: ``why use the hinge loss?'' and ``why seek a low-norm
solution?''. The answer to the second question is related to the first, so let
us first discuss the use of the hinge loss.
As was mentioned in the previous section, the underlying problem which we wish
to solve is binary classification, which would naturally indicate that we
should seek to find a classifier which makes the smallest possible number of
mistakes on the training set---in other words, we should seek to minimize the
training zero-one loss. Unfortunately, doing so is not practical, except for
extremely small problem instances, because the underlying optimization problem
is combinatorial in nature. There is, however, a class of optimization problems
which are widely-studied, and known to be relatively easy to solve: convex
optimization problems. Use of the hinge loss, which is a convex upper-bound on
the zero-one loss, can therefore be justified by the fact that it results in a
convex optimization problem, which we can then solve efficiently. In other
words, we are \emph{approximating} what we ``really want to do'', for the sake
of practicality.

This is, however, only a partial answer. While the hinge loss is indeed a
convex upper bound on the zero-one loss, it is far from the only one, and
others (such as the log loss) are in widespread use. Another more subtle but
ultimately more satisfying reason for using the hinge loss relates to the fact
that we want a solution which experiences low generalization error, coupled
with the observation that a solution which performs well on the training set
does not necessarily perform well with respect to the underlying data
distribution.
In the realizable case (i.e. where there exists a linear classifier which
performs perfectly with respect to the underlying data distribution
$\mathcal{D}$), the classification boundary of a perfect linear classifier must
lie \emph{somewhere} between the set of positive and negative training
instances, so, in order to maximize the chance of getting it nearly-``right'',
it only makes sense for a learning algorithm to place it in the middle. This
intuition can be formalized, and it turns out that large-margin predictors
(i.e. predictors for which all training elements of both classes are ``far''
from the classification boundary) do indeed tend to generalize well.
When applying a SVM to a realizable problem, the empirical hinge loss will be
zero provided that each training vector is at least $1/\norm{w}$-far from the
boundary, so minimizing the norm of $w$ is equivalent to maximizing the margin
of the classifier.
SVMs, which may also be applied to non-realizable problems, can therefore be
interpreted as extending the idea of searching for large-margin classifiers to
the case in which the data are not necessarily linearly separable, and we
should therefore suspect that SVM solutions will tend to generalize well. This
will be formalized in Section \ref{sec:svm-introduction:generalization}.

\subsection{Primal Objective}\label{subsec:svm-introduction:svm-primal}

A typical way in which Problem \ref{eq:svm-introduction:bi-criterion-objective}
is ``scalarized'' (i.e.  transformed from a bi-criterion objective into a
single unified objective taking an additional parameter) is the following:
\begin{align}
\label{eq:svm-introduction:regularized-objective} \mbox{minimize}: &
\frac{\lambda}{2} \norm{w}^{2} +\emploss[\hinge]{g_w}
\end{align}
Here, the parameter $\lambda$ controls the tradeoff between the norm (inverse
margin) and the empirical error.  Different values of $\lambda$ correspond to
different Pareto optimal solutions of Problem
\ref{eq:svm-introduction:bi-criterion-objective}, and the entire Pareto front
can be explored by varying $\lambda$. We will refer to this as the
``regularized objective''.

Alternatively, one could scalarize Problem
\ref{eq:svm-introduction:bi-criterion-objective} not by introducing a
regularization parameter $\lambda$, but instead by enforcing a bound $R$ on the
magnitude of $w$, and minimizing the empirical hinge loss subject to this
constraint:
\begin{align}
\label{eq:svm-introduction:norm-constrained-objective} \mbox{minimize}: &
\emploss[\hinge]{g_w}
\\
\notag \mbox{subject to}: & \norm{w} \le R
\end{align}
We will refer to this as the ``norm-constrained objective''. It and the
regularized objective are extremely similar. Like the regularized objective
(while varying $\lambda$), varying $R$ causes optimal solutions to the
norm-constrained objective to explore the Pareto frontier of Problem
\ref{eq:svm-introduction:bi-criterion-objective}, so the two objectives are
equivalent: for every choice of $\lambda$ in the regularized objective, there
exists a choice of $R$ such that the optimal solution to the norm-constrained
objective is the same as the optimal solution to the regularized objective, and
vice-versa.

\subsection{Dual Objective}

Some SVM training algorithms, of which Pegasos \citep{ShalevSiSr07} is a prime
example, optimize the regularized objective directly. However, particularly for
kernel SVM optimization, it is often more convenient to work on the Lagrangian
dual of Problem \ref{eq:svm-introduction:regularized-objective}:
\begin{align}
\label{eq:svm-introduction:dual-objective} \mbox{maximize}: &
\inner{\mathbf{1}}{\alpha} - \frac{1}{2} \norm{w}^2 \\
\notag \mbox{subject to}: & 0 \le \alpha_i \le \frac{1}{\lambda n}
\end{align}
Here, as in Equation \ref{eq:svm-introduction:representor-theorem}, $w =
\sum_{i=1}^n \alpha_i y_i \Phi(x_i)$, although now $\alpha$ is more properly
considered not as the vector of coefficients resulting from application of the
representor theorem, but rather the dual variables which result from taking the
Lagrangian dual of Problem \ref{eq:svm-introduction:regularized-objective}.
Optimal solutions to the dual objective correspond to optimal solutions of the
primal, and vice-versa. The most obvious practical difference between the two
is that the dual objective is a quadratic programming problem subject to box
constraints, and is thus particularly ``easy'' to solve efficiently, especially
in comparison to the primal, which is not even smooth (due to the ``hinge'' in
the hinge loss).

There is another more subtle difference, however. While optimal solutions to
Problems \ref{eq:svm-introduction:regularized-objective} and
\ref{eq:svm-introduction:dual-objective} both, by varying $\lambda$, explore
the same set of Pareto optimal solutions to Problem
\ref{eq:svm-introduction:bi-criterion-objective}, these alternative objectives are decidedly not
equivalent when one considers the quality of \emph{suboptimal} solutions. We'll
illustrate this phenomenon with a trivial example.

Consider the one-element dataset $\left(x_{1},y_{1}\right)=\left(1,1\right)$
with $\X = \hilbert = \R$ and $\Phi$ being the identity (i.e. this is a
linear SVM), and suppose that we wish to optimize the regularized or dual objective
with $\lambda=\frac{1}{2}$. For this dataset, both the weight vector $w$ and
the vector of dual variables $\alpha$ are one-dimensional scalars, and the
expression $w = \sum_{i=1}^n \alpha_i y_i x_i$ shows that $w=\alpha$. It is
easy to verify that the optimum of both of these objectives occurs at
$\hat{w}^{*}=\hat{\alpha}^{*}=1$.

For suboptimal solutions, however, the correspondence between the two
objectives is less attractive. Consider, for example, suboptimal solutions of
the form $w' = \alpha' = 1 - \delta$, and notice that these solutions are dual
feasible for $0 \le \delta \le 1$. The suboptimality of the dual objective
function for this choice of $\alpha$ is given by:
\begin{equation*}
\left(\hat{\alpha}^* - \frac{1}{2}\left(\hat{\alpha}^*\right)^{2}\right) -
\left(\alpha' - \frac{1}{2}\left(\alpha'\right)^{2}\right) =
\frac{1}{2}\delta^2
\end{equation*}
while the primal suboptimality for this choice of $w$ is:
\begin{equation*}
\left(\frac{\lambda}{2}\left(w'\right)^{2} + \max\left(0,1-w'\right)\right) -
\left(\frac{\lambda}{2}\left(\hat{w}^{*}\right)^{2} +
\max\left(0,1-\hat{w}^{*}\right)\right) = \frac{1}{4}\delta^{2} +
\frac{1}{2}\delta
\end{equation*}
In other words, $O(\epsilon)$-suboptimal solutions to the dual objective may be
only $O(\sqrt{\epsilon})$-suboptimal in the primal. As a consequence, an algorithm
which converges rapidly in the dual may not do so in the primal, and more
importantly, may not do so in terms of the ultimate quantity of interest,
generalization error. Hence, the apparent ease of optimizing the dual objective
may not be as beneficial as it at first appears.

In Chapter \ref{ch:svm-sbp}, an algorithm will be developed based on a
\emph{third} equivalent formulation of the SVM problem, called the
slack-constrained objective. The convergence rate of this algorithm, measured
in terms of suboptimality, is unremarkable. Instead, it performs well because
$\epsilon$-suboptimal solutions to the slack-constrained objective are better,
in terms of known bounds on the generalization error, than
$\epsilon$-suboptimal solutions to the other objectives mentioned in this
section.

\subsection{Unregularized Bias}\label{subsec:svm-introduction:unregularized-bias}

Frequently, SVM problems contain an unregularized bias term---rather than the
classification function being $g_w(x) = \inner{w}{\Phi(x)}$, with the
underlying optimization problem searching for a good low-weight vector $w$, the
classification function is taken to be the sign of $g_{w,b}(x) =
\inner{w}{\Phi(x)} + b$, where $w$ is, as before, low-norm, but
$b\in\R$ has no constraints on its magnitude. Intuitively, the
difference is that, in the former case, the classification boundary is a
hyperplane which passes through the origin (i.e. is a subspace of $\hilbert$),
while in the latter case it does not (i.e. it is an affine subspace).

The optimization problems which we have so far considered for SVMs do not
contain a bias term, and in Chapters \ref{ch:svm-sbp} and \ref{ch:svm-sparse}
our focus will be on such problems. This is purely to simplify the
presentation, however---the algorithms introduced in these chapters can be
extended to work on problems with an unregularized bias fairly easily.

Updating the objectives of Section \ref{subsec:svm-introduction:svm-primal} to
include an unregularized bias amounts to nothing more than inserting $b$ into
the hinge loss term. For example, the regularized objective of Problem
\ref{eq:svm-introduction:regularized-objective} becomes:
\begin{align}
\label{eq:svm-introduction:biased-regularized-objective} \mbox{minimize}: &
\frac{\lambda}{2} \norm{w}^{2} + \emploss[\hinge]{g_{w,b}}
\end{align}
Precisely the same substitution works on the norm-constrained objective of
Problem \ref{eq:svm-introduction:norm-constrained-objective}.

Finding the Lagrangian dual of Problem
\ref{eq:svm-introduction:biased-regularized-objective} gives us the analogue of
the dual objective of Problem \ref{eq:svm-introduction:dual-objective}
\begin{align}
\label{eq:svm-introduction:biased-dual-objective} \mbox{maximize}: &
\inner{\mathbf{1}}{\alpha} - \frac{1}{2} \norm{w}^2 \\
\notag \mbox{subject to}: & 0 \le \alpha_i \le \frac{1}{\lambda n},\ \ 
\sum_{i=1}^{n} y_i \alpha_i = 0
\end{align}
The only difference here is the addition of the constraint $\sum_{i=1}^{n} y_i
\alpha_i = 0$. As before, we can derive $w$ from $\alpha$ via Equation
\ref{eq:svm-introduction:representor-theorem}, but finding $b$ is slightly more
involved. For an \emph{optimal} vector $\alpha$ of dual variables, it will be
the case that $b = y_i - \inner{w}{\Phi(x_i)}$ for any $i$ such that
$0<\alpha_i<\nicefrac{1}{\lambda n}$ (i.e. $\alpha_i$ is bound to neither a
lower nor upper-bound constraint), so the most common method for finding $b$
from a (not necessarily optimal) $\alpha$ is to solve this equation for all
valid $i$, and average the results.

%% file: svm/introduction/figures/tab-notation.tex
\begin{table}

\begin{small}
\begin{center}
\begin{tabular}{llll}
\hline
& Definition & Name \\
\hline
$\mathcal{D}$ & & Data distribution \\
$n\in\N$ & & Training size \\
$x_1,\dots,x_n\in\X$ & & Training vectors \\
$y_1,\dots,y_n\in\{\pm 1\}$ & & Training labels \\
$\Phi:\X\rightarrow\hilbert$ & & Kernel map \\
$K:\X\times\X\rightarrow\R$ & $K(x,x') = \inner{\Phi(x)}{\Phi(x')}$ & Kernel function such that $K(x,x)\le 1$ \\
$g_w:\X\rightarrow\R$ & $g_w(x)=\inner{w}{\Phi(x)}$ & Linear classifier \\
$g_{w,b}:\X\rightarrow\R$ & $g_{w,b}(x)=\inner{w}{\Phi(x)}+b$ & Linear classifier with bias \\
$c_1,\dots,c_n\in\R$ & $c_i = y_i \inner{w}{\Phi(x_i)}$ & Training responses \\
$\ell_\zeroone:\R\rightarrow\R$ & $\ell_\zeroone(z) = \indicator{z\le 0}{z}$ & Zero-one loss \\
$\ell_\hinge:\R\rightarrow\R$ & $\ell_\hinge(z) = \max(0,1-z)$ & Hinge loss \\
$\hat{\mathcal{L}}:(\X\rightarrow\R)\rightarrow\R$ & $\emploss{g} = \frac{1}{n}\sum_{i=1}^{n}\ell\left(y_i g(x_i)\right)$ & Empirical loss \\
$\mathcal{L}:(\X\rightarrow\R)\rightarrow\R$ & $\loss{g} = \expectation[x,y\sim\mathcal{D}]{\ell\left(y g(x)\right)}$ & Expected loss \\
\hline
\end{tabular}
\end{center}
\end{small}

\caption{
Summary of common notation across Chapters \ref{ch:svm-introduction},
\ref{ch:svm-sbp} and \ref{ch:svm-sparse}. The training responses $c$ are
calculated and kept up-to-date during optimization, so the $w$ parameterizing
them should be taken to be the weight vector at the current optimization
step---see Section \ref{sec:svm-introduction:traditional} for details.  The
empirical and expected losses are generally subscripted with the particular
loss which they use---for example, $\emploss[\hinge]{g_w}$ is the empirical
hinge loss of the linear classifier defined by $w\in\hilbert$.
}

\label{tab:svm-introduction:notation}

\end{table}

%% file: svm/introduction/sec-generalization.tex
\section{SVM Generalization Bounds}\label{sec:svm-introduction:generalization}

\input{svm/introduction/figures/tab-runtimes}

Before introducing and analyzing particular SVM optimization algorithms, we must
first describe how these analyses will performed, and what the results of such
analyses should be.
In this section, we will give a bound on the sample size $n$ required to
guarantee good generalization performance (in terms of the 0/1 loss) for a
low-norm classifier which is $\epsilon$-suboptimal in terms of the empirical
hinge loss.
This bound (Lemma \ref{lem:svm-introduction:generalization-from-expected-loss})
essentially enables us to transform a bound on the empirical hinge loss into a
bound on generalization performance.
This lemma is a vital building block of the bounds derived in the following
sections of this chapter (see Table \ref{tab:svm-introduction:runtimes}), as
well as in Chapters \ref{ch:svm-sbp} and \ref{ch:svm-sparse}.

In order to simplify the presentation of our generalization bounds, throughout
this chapter, as well as Chapters \ref{ch:svm-sbp} and \ref{ch:svm-sparse}, we
make the simplifying assumption that $K(x,x') \le 1$ with probability $1$ (with
respect to the underlying data distribution $\mathcal{D}$).

\input{svm/introduction/theorems/lem-generalization-from-expected-loss}

The bounds which we derive from this Lemma are \emph{optimistic}, in the sense
that they will look like $n=O\left(\nicefrac{1}{\epsilon^2}\right)$ when the
problem is difficult (i.e. $\epsilon$ is small relative to
$\loss[\hinge]{g_u}$), and be better when the problem is easier, approaching
$n=O\left(\nicefrac{1}{\epsilon}\right)$ when it is linearly separable (i.e.
$\loss[\hinge]{g_u}=0$).

Applying Lemma \ref{lem:svm-introduction:generalization-from-expected-loss} to
a particular algorithm follows two steps. First, we determine, on an
algorithm-by-algorithm basis, what is required to find a solution $w$ which
satisfies the following:
\begin{equation*}
\norm{w} \le B,\ \ 
\emploss[\hinge]{g_w} \le \emploss[\hinge]{g_u} + \epsilon
\end{equation*}
Here, $u$ is an arbitrary reference classifier with $\norm{u} \le B$
(typically, $u$ will be taken to be the optimum). Next, we use Lemma
\ref{lem:svm-introduction:generalization-from-expected-loss} to find the sample
size $n$ required such that $\loss[\zeroone]{g_w} \le \loss[\hinge]{g_u} +
\epsilon$ is satisfied.

%% file: svm/introduction/figures/tab-runtimes.tex
\begin{table}

\begin{small}
\begin{center}
\begin{tabular}{r|cc|cc}
\hline
& \multicolumn{2}{c|}{Overall} & \multicolumn{2}{c}{$\epsilon = \Omega\left( L^* \right)$} \\
& Training & Testing & Training & Testing \\
\hline
SGD on Problem \ref{eq:svm-introduction:norm-constrained-objective} & $\left( \frac{L^* + \epsilon}{\epsilon} \right) \frac{R^4}{\epsilon^3}$ & $\frac{R^2}{\epsilon^2}$ & $\frac{R^4}{\epsilon^3}$ & $\frac{R^2}{\epsilon^2}$ \\
Dual Decomposition & $\left( \frac{L^* + \epsilon}{\epsilon} \right)^2 \frac{R^4}{\epsilon^2}$ & $\left( \frac{L^* + \epsilon}{\epsilon} \right) \frac{R^2}{\epsilon}$ & $\frac{R^4}{\epsilon^2}$ & $\frac{R^2}{\epsilon}$ \\
Perceptron + Online-to-Batch & $\left( \frac{L^* + \epsilon}{\epsilon} \right)^3 \frac{R^4}{\epsilon}$ & $\left( \frac{L^* + \epsilon}{\epsilon} \right)^2 R^2$ & $\frac{R^4}{\epsilon}$ & $R^2$ \\
Random Fourier Features & $\left( \frac{L^* + \epsilon}{\epsilon} \right) \frac{d R^4}{\epsilon^3}$ & $\frac{d R^2}{\epsilon^2}$ & $\frac{d R^4}{\epsilon^3}$ & $\frac{d R^2}{\epsilon^2}$ \\
\hline
\end{tabular}
\end{center}
\end{small}

\caption{
Summary of results from Sections \ref{sec:svm-introduction:traditional} and
\ref{sec:svm-introduction:non-traditional}. First three rows: upper bounds on
the number of kernel evaluations required to train a kernelized classifier /
evaluate a solution on a single previously-unseen example, such that
$\loss[\zeroone]{g_w}\leq L^* +\epsilon$. Last row: number of $d$-dimensional
inner products (random Fourier feature constructions) required.  All quantities
are given up to constant and log factors, and hold with probability $1-\delta$
(although the dependence on $\delta$ is not shown in this table). It is assumed
that $K(x,x) \le 1$ with probability one for $x\sim\mathcal{D}$, and
furthermore, in the last column, that $\mathcal{X} = \R^d$. The
quantity $L^*$ can be read as the ``optimal hinge loss'', and is the expected
hinge loss (with respect to $\mathcal{D}$) suffered by an arbitrary reference
classifier $g_u$ with $\norm{u}\le R$. For the last row, $\norm{u} \le
\norm{u}_1 \le R$, with the first inequality following from the fact that
$K(x,x') \le 1$. The ``Overall'' columns show the standard optimistic bounds,
while the ``$\epsilon = \Omega\left( L^* \right)$'' columns give versions of
these bounds in the most typical machine learning setting, in which the desired
level of suboptimality $\epsilon$ is comparable to the optimal hinge loss.
}

\label{tab:svm-introduction:runtimes}

\end{table}

%% file: svm/introduction/theorems/lem-generalization-from-expected-loss.tex
\medskip
\begin{splitlemma}{lem:svm-introduction:generalization-from-expected-loss}

Let $u$ be an arbitrary linear classifier, and suppose that we sample a
training set of size $n$, with $n$ given by the following equation, for
parameters $B \ge \norm{u}$, $\epsilon>0$ and
$\delta\in\left(0,1\right)$:
\begin{equation}
\ifproofsection
\label{eq:svm-introduction:generalization-from-expected-loss-bound}
\else
\notag
\fi
n = \tilde{O}\left( \left(
\frac{\loss[\hinge]{g_u} + \epsilon}{\epsilon} \right) \frac{ \left( B +
\log\frac{1}{\delta} \right)^{2} }{\epsilon}
\right)
\end{equation}
Then, with probability $1-\delta$ over the \iid training sample
$x_{i},y_{i}:i\in\left\{ 1,\dots,n\right\}$, we have that
$\emploss[\hinge]{g_u} \le \loss[\hinge]{g_u} + \epsilon$ and
$\loss[\zeroone]{g_w} \le \emploss[\hinge]{g_u} + \epsilon$, and in particular
that
$\loss[\zeroone]{g_w} \le \loss[\hinge]{g_u} + 2\epsilon$ uniformly for all $w$
satisfying:
\begin{equation*}
\norm{w} \le B,\ \ 
\emploss[\hinge]{g_w} \le \emploss[\hinge]{g_u} + \epsilon
\end{equation*}
\end{splitlemma}
\begin{splitproof}

Lemma \ref{lem:svm-introduction:empirical-loss-from-expected-loss} gives that
$\emploss[\hinge]{g_w} \le \loss[\hinge]{u} + 2\epsilon$ provided that Equation
\ref{eq:svm-introduction:empirical-loss-from-expected-loss-bound} is satisfied.
Taking $L = \loss[\hinge]{u} + 2\epsilon$ in Lemma
\ref{lem:svm-introduction:generalization-from-empirical-loss} gives the desired
result, provided that Equation
\ref{eq:svm-introduction:generalization-from-empirical-loss-bound} is also
satisfied. Equation
\ref{eq:svm-introduction:generalization-from-expected-loss-bound} is what
results from combining these two bounds. The fact that this argument holds with
probability $1-2\delta$, while our claim is stated with probability $1-\delta$,
is due to the fact that scaling $\delta$ by $1/2$ only changes the bound by a
constant.
\end{splitproof}

%% file: svm/introduction/sec-traditional.tex
\section{Traditional Optimization Algorithms}\label{sec:svm-introduction:traditional}

\input{svm/introduction/figures/alg-traditional}

In this section, algorithms for optimizing a SVM based on traditional convex
optimization techniques will be described. One significant feature of many such
algorithms is that they can be represented as instances of a common
``outline'', in which a weight vector $w$ is found iteratively according to
Algorithm \ref{alg:svm-introduction:traditional}.
The various algorithms which we will consider in this section, as well as
Chapters \ref{ch:svm-sbp} and \ref{ch:svm-sparse}, differ in how they perform
the ``\emph{some}'' portions of this algorithm, but their presentation (and
analysis, to some extent) will be simplified by recognizing them as instances
of this outline. The existence of this general outline also has important
consequences for the implementation of these algorithms, since it enables one
to create an optimized shared implementation of those portions of the update
which are identical across all traditional optimizers, with specialized code
only being necessary for the small number of differences between them.
ISSVM\footnote{\url{http://ttic.uchicago.edu/~cotter/projects/SBP}}, which
contains the reference implementations of the algorithms of Chapters
\ref{ch:svm-sbp} and \ref{ch:svm-sparse}, is a project designed according to
these principles.

For a linear kernel, as was discussed in Section
\ref{sec:svm-introduction:objective}, $\Phi$ is the identity function, and one
may implement a traditional optimizer by essentially following exactly
Algorithm \ref{alg:svm-introduction:traditional}. For a kernelized SVM
optimizer, however, we cannot explicitly depend on $\Phi$, and must instead
appeal to Equation \ref{eq:svm-introduction:representor-theorem} to represent
$w$ as a linear combination of $\Phi\left(x_i\right)$ with coefficients
$\alpha_i$, which enables us to write the algorithm in terms of only the
coefficients $\alpha$ and kernel evaluations.

\input{svm/introduction/figures/alg-traditional-kernel}

For all of the algorithms which we will consider, calculating the responses on
line $4$ will be the most computationally expensive step, with a cost of $kn$
kernel evaluations, where $k=\abs{\mathcal{I}}$ is the number of responses
which we need to calculate. Alternatively, we can modify the algorithm to keep
track of a complete vector of $n$ responses at all times, updating them
whenever we make a change to $\alpha$---this is done in Algorithm
\ref{alg:svm-introduction:traditional-kernel}. Updating the responses in line
$5$ costs $kn$ kernel evaluations, as before. However, this version has the
advantage that the complete vector of responses is available at every point. In
particular, we can use $c$ to choose the working set $\mathcal{I}$ on line $3$.

The existence of this ``improved'' optimization outline demonstrates why the
best linear SVM optimization algorithms tend to be different from the best
kernel SVM algorithms. In the linear setting, calculating $k$ responses only
costs $k$ inner products per iteration, while keeping a complete up-to-date
vector $c$ would cost $n$ inner products per iteration. As a result, good
algorithms tend to be ``fast but dumb'', in that each iteration is
computationally very inexpensive, but cannot include the use of advanced
working-set-selection heuristics because the responses are unavailable (and
computing them would lead to an unacceptable increase to the cost of each
iteration). In the kernel setting, however, we must perform $kn$ kernel
evaluations per iteration \emph{anyway}, and therefore prefer ``slower but
smarter'' updates, for two reasons. First, there is no additional cost for
keeping a complete up-to-date vector of responses available throughout, which
enables one to use a better working-set-selection heuristic; second, because
every iteration is already fairly expensive (at least $O(kn)$ as opposed to
$O(d)$ in the linear setting, where $\mathcal{X}=\R^d$), all of the steps of
the update, most importantly the subproblem optimization, can be made more
computationally expensive without increasing the asymptotic per-iteration cost
of the algorithm.

The work described in this thesis focuses on the kernel setting, so the
canonical optimization algorithm outline will be Algorithm
\ref{alg:svm-introduction:traditional-kernel}. With this said, it is important
to keep in mind that switching between the linear and kernel settings can
dramatically change the performance profile of each algorithms under
consideration. In particular, algorithms based on stochastic gradient descent
(Section \ref{subsec:svm-introduction:primal}), as well as stochastic dual
coordinate ascent (Section \ref{subsec:svm-introduction:dual}), which do not
make use of the complete vector of responses when finding a working set,
operate much more efficiently in the linear setting.

\subsection{Stochastic Gradient Descent}\label{subsec:svm-introduction:primal}

Probably the simplest algorithm which can be applied to SVM optimization is
stochastic gradient descent (SGD). We will present two SGD algorithms, the
first being SGD applied to the norm-constrained SVM objective of Problem
\ref{eq:svm-introduction:norm-constrained-objective}, and the second to the
regularized objective of Problem
\ref{eq:svm-introduction:regularized-objective}. Due to the similarity of the
objectives, the two optimization algorithms are almost the same, although, as
we will see, the former enjoys a slightly superior convergence rate, despite
the fact that the latter algorithm, also known as Pegasos \citep{ShalevSiSr07},
is far more well-known.

\subsubsection{Norm-Constrained Objective}

Formally differentiating the norm-constrained SVM objective (Problem
\ref{eq:svm-introduction:norm-constrained-objective}---refer to the definition
of the empirical hinge loss in Table \ref{tab:svm-introduction:notation})
yields that:
\begin{equation*}
\frac{\partial}{\partial w} = \frac{1}{n}\sum_{i=1}^{n} \indicator{y_i
\inner{w}{\Phi\left(x_i\right)} < 1}{y_i \Phi\left(x_i\right)}
\end{equation*}
We could perform projected gradient descent (GD) by iteratively subtracting a
multiple of the above quantity from $w$, and then projecting $w$ onto the
feasible set $\left\{w:\norm{w}\le R\right\}$. This, however, would be an
extremely computationally expensive update, because it requires knowledge of
all $n$ values of the responses $c_i = y_i \inner{w}{\Phi\left(x_i\right)}$,
and could potentially add a multiple of every training element to $w$ at each
iteration. As a result, regardless of whether we calculate the responses only
as-needed, or keep an up-to-date set of all of the responses, the cost of each
iteration is potentially as high as $n^2$ kernel evaluations in the kernel
setting.

It has been observed \citep{Zhang04,ShalevSiSr07,ShalevSr08,ShalevSiSrCo10}
that performing a large number of simple updates generally leads to faster
convergence than performing fewer computationally expensive (but ``smarter'')
updates, which motivates the idea of using SGD, instead of GD, by iteratively
performing the following steps on the $t$th iteration:
\begin{enumerate}
\item Sample $i$ uniformly from $\{1,\dots,n\}$.
\item If $y_i \inner{w}{\Phi\left(x_i\right)} < 1$ then subtract $\eta_t y_i
\Phi\left(x_i\right)$ from $w$.
\item If $\norm{w} > R$ then multiply $w$ by $R / \norm{w}$.
\end{enumerate}
Notice that the quantity added to $w$ in step 2 is equal in expectation (over
$i$) to the gradient of the objective with respect to $w$, and therefore that
this is a SGD algorithm. Furthermore, these three steps together constitute an
instance of the ``traditional'' algorithm outline, showing that they can be
performed using an efficient modular implementation of Algorithm
\ref{alg:svm-introduction:traditional-kernel}. The following lemma gives the
theoretical convergence rate of this algorithm, in the kernel setting:

\input{svm/introduction/theorems/thm-norm-constrained-runtime}

\subsubsection{Regularized Objective}

Applying SGD to the regularized objective results in a similar update, except
that instead of explicitly projecting onto the feasible region, it is the
regularization term of the objective which ensures that the weight vector $w$
does not become too large. Differentiating Problem
\ref{eq:svm-introduction:regularized-objective} with respect to $w$ yields:
\begin{equation*}
\frac{\partial}{\partial w} = \lambda w + \frac{1}{n}\sum_{i=1}^{n}
\indicator{y_i \inner{w}{\Phi\left(x_i\right)} < 1}{y_i \Phi\left(x_i\right)}
\end{equation*}
Which shows that the following is a SGD update:
\begin{enumerate}
\item Sample $i$ uniformly from $\{1,\dots,n\}$.
\item Multiply $w$ by $1-\eta_t \lambda$.
\item If $y_i \inner{w}{\Phi\left(x_i\right)} < 1$ then subtract $\eta_t y_i
\Phi\left(x_i\right)$ from $w$.
\end{enumerate}
With an appropriate choice of the step size $\eta_t$, this algorithm is known
as Pegasos \citep{ShalevSiSr07}. Here, steps $2$ and $3$ together add $\eta$
times $\lambda w + \indicator{y_i \inner{w}{\Phi\left(x_i\right)} < 1}{y_i
\Phi\left(x_i\right)}$ to $w$, which is equal in expectation to the gradient of
the regularized objective with respect to $w$. Once more, this is an instance
of the ``traditional'' algorithm outline (except that the adding and scaling
steps are swapped, which makes no difference).

The only difference between SGD on the norm-constrained and regularized
objectives is that the first projects $w$ onto the Euclidean ball of radius
$R$, while the second shrinks $w$ by a factor proportional to the step size at
every iteration. At first glance, this is insignificant, but once we move on to
analyzing the performance of these algorithms we find that there are, in fact,
pronounced differences. 

The first difference results from the fact that, when one applies SGD to
non-strongly-convex objectives, the use of a step size $\eta_t \propto
1/\sqrt{t}$ typically yields the best convergence rates, whereas $\eta_t
\propto 1/t$ is better for strongly-convex objectives. These two algorithms are
not exceptions---the norm-constrained objective is not strongly convex, while
the regularized objective \emph{is} (with convexity parameter $\lambda$), and
these two algorithms perform best, at least in theory, with $O(1/\sqrt{t})$ and
$O(1/t)$ step sizes, respectively.

The second main difference is a result of our use of Lemma
\ref{lem:svm-introduction:generalization-from-expected-loss}, which gives a
bound on the generalization error suffered by weight vectors $w$ with some
maximum norm $B$. Because this lemma is stated in terms of a norm-constrained
weight vector, it is more convenient to apply it to the result of applying SGD
to the norm-constrained objective. On the regularized objective, a bound on the
norm of $w$ can only be derived from a bound on the suboptimality, which
ultimately causes the convergence rate to not be quite as good as we would
hope, as we can see in the following lemma:

\input{svm/introduction/theorems/thm-regularized-runtime}

Notice that the bound on $\mbox{\#K}$ is worse than that of Theorem
\ref{thm:svm-introduction:norm-constrained-runtime}. As a result, despite the
fact that Pegasos is better-known, SGD on the norm-constrained objective
appears to be the better algorithm. For this reason, in this and subsequent
chapters, we will not consider Pegasos when we compare the bounds of the
various algorithms under consideration.

\subsection{Dual Decomposition}\label{subsec:svm-introduction:dual}

Probably the most commonly-used methods for optimizing kernel SVMs are what we
call ``dual decomposition'' approaches. These methods work not by performing
SGD on a (primal) objective, but instead by applying coordinate ascent to the
dual objective (Problem \ref{eq:svm-introduction:dual-objective}).
Interestingly, although the principles justifying these techniques are very
different from those of Section \ref{subsec:svm-introduction:primal}, the
resulting algorithms are strikingly similar, and also fall under the umbrella
of the traditional algorithm outline of Algorithm
\ref{alg:svm-introduction:traditional-kernel}.

The various dual-decomposition approaches which we will discuss differ in how
they select the working set (line $3$ of Algorithm
\ref{alg:svm-introduction:traditional-kernel}), but once the working set has
been chosen, they all do the same thing: holding $\alpha_i : i\notin
\mathcal{I}$ fixed, they solve for the values of $\alpha_i : i \in \mathcal{I}$
which maximize the dual objective of Problem
\ref{eq:svm-introduction:dual-objective}.
Because the dual objective is a quadratic programming program subject to box
constraints, this $k$-dimensional subproblem is, likewise. When $k$ is sufficiently
small, a closed-form solution can be found, rendering the computational cost of
each dual-decomposition iteration comparable to the cost of a single
iteration of SGD. In particular, for $k=1$, if we take $\mathcal{I} = \{ i \}$,
then the problem is to maximize:
\begin{equation*}
\inner{\mathbf{1}}{\alpha} + \delta - \frac{1}{2} \left( \alpha + \delta e_i
\right)^T Q \left( \alpha + \delta e_i \right)
\end{equation*}
subject to the constraint that $0 \le \alpha_i + \delta \le
\nicefrac{1}{\lambda n}$, where $e_i$ is the $i$th standard unit basis vector
and $Q_{ij} = y_i y_j K\left( x_i, x_j \right)$. Simplifying this expression
gives:
\begin{equation*}
\delta \left( 1 - c_i \right) - \frac{1}{2} \delta^2 K\left( x_i, x_i \right)
\end{equation*}
Finally, differentiating with respect to $\delta$ and setting the result equal
to zero gives that:
\begin{equation}
\label{eq:svm-introduction:dual-decomposition-unbiased-update} \delta = \frac{1 - c_i}{K\left( x_i, x_i \right)}
\end{equation}
Because this equation results from maximizing a quadratic objective, we may
apply the constraints by simply clipping $\alpha_i + \delta$ to the set $\left[
0, \nicefrac{1}{\lambda n} \right]$. This establishes what to do on the
subproblem optimization step of line $4$ of Algorithm
\ref{alg:svm-introduction:traditional-kernel}, at least for $k=1$. The scaling
step of line $6$ is unnecessary, so the only step left undefined is the
working-set selection step of line $3$.

The simplest approach to choosing a working set is probably to sample $i$
uniformly at random from $\{ 1, \dots, n \}$. This is the Stochastic Dual
Coordinate Ascent (SDCA) algorithm of \citet{HsiehChLiKeSu08}. As was
previously mentioned, because this approach does not require the complete set
of responses to be available for working-set selection, it works particularly
well in the linear setting.
In the kernel setting, the most widely-used algorithm is probably SMO
(Sequential Minimal Optimization~\citep{Platt98,FanChLi05}), which
deterministically chooses $i$ to be the index which, after an update, results
in the largest increase in the dual objective.  SMO is implemented by the
popular LibSVM package \citep{LinLi03}. The algorithm used by SVM-Light is
similar in spirit, except that it heuristically chooses not a single training
example, but rather a comparatively large working set ($k=10$ is the default),
and solves the resulting subproblem using a built-in QP solver (the gradient
projection algorithm \citep[Chapter 16.6]{NocedalWr99} is a good choice),
instead of Equation
\ref{eq:svm-introduction:dual-decomposition-unbiased-update}.

\subsubsection{Unregularized Bias}

So far, for both SGD and dual-decomposition algorithms, we have only considered
SVM instances without an unregularized bias. For SGD, it turns out that
handling a bias is somewhat difficult, and in fact no ``clean'' method of doing
so is known. Most dual decomposition algorithms, however, can be extended to
handle the presence of an unregularized bias quite naturally. The reason for
this is that adding a bias to the primal problem is equivalent to adding the
constraint $\sum_{i=1}^n y_i \alpha_i = 0$ to the dual problem, as we saw in
Section \ref{subsec:svm-introduction:unregularized-bias}. If we choose a
working set of size $k=2$, this constraint can be enforced during subproblem
optimization---more rigorously, if $\mathcal{I} = \{ i, j \}$, then finding the
values of $\alpha_i$ and $\alpha_j$ which maximize the dual objective amounts
to finding a $\delta$ such that adding $y_i\delta$ to $\alpha_i$ and
subtracting $y_j\delta$ from $\alpha_j$ (which maintains the constraint that
$\sum_{i=1}^n y_i \alpha_i = 0$) maximizes:
\begin{equation*}
y_i\delta - y_j\delta - \frac{1}{2}\left( \alpha + y_i\delta e_i - y_j \delta
e_j \right)^T Q \left( \alpha + y_i\delta e_i - y_j \delta e_j \right)
\end{equation*}
Simplifying, differentiating with respect to $\delta$ and setting the result
equal to zero yields that:
\begin{equation*}
\delta = \frac{y_i \left(1-c_i\right) - y_j\left(1 - c_j\right) + 2 K\left(
x_i, x_j \right) }{K\left( x_i, x_i \right) + K\left( x_j, x_j \right)}
\end{equation*}
Once more, we must clip $\delta$ to the constraints $0 \le \alpha_i + y_i
\delta \le \nicefrac{1}{\lambda n}$ and $0 \le \alpha_j + y_j \delta \le
\nicefrac{1}{\lambda n}$ before applying it.
More generally, for larger subproblem sizes, such as those used by SVM-Light,
the constraint $\sum_{i=1}^n y_i \alpha_i = 0$ can be imposed during subproblem
optimization.

\subsubsection{Analysis}

The use of heuristics to choose the working set makes SMO very difficult to
analyze. Although it is known to converge linearly after some number of
iterations \cite{ChenFaLi06}, the number of iterations required to reach
this phase can be very large. To the best of our knowledge, the most satisfying
analysis for a dual decomposition method is the one given in
\citet{HushKeScSt06}, showing a total number of kernel evaluations of $O(n^2)$.
In terms of learning runtime, combining this result with Lemma
\ref{lem:svm-introduction:generalization-from-expected-loss} yields a runtime
and support size of:
\begin{align}
\label{eq:svm-introduction:dual-decomposition-runtime} \mbox{\#K} = &
\tilde{O}\left( \left( \frac{\loss[\hinge]{g_u}+\epsilon}{\epsilon} \right)^2 \frac{\norm{u}^4}{\epsilon^2} \right)
\\
\notag \mbox{\#S} = &
\tilde{O}\left( \left( \frac{\loss[\hinge]{g_u}+\epsilon}{\epsilon} \right) \frac{\norm{u}^2}{\epsilon} \right)
\end{align}
to guarantee $\loss[\zeroone]{g_w} \leq \loss[\hinge]{g_u}+\epsilon$.

Until very recently, there was no satisfying analysis of SDCA, although it is
now known \citep{ShalevZh13} that SDCA finds a $\epsilon$-suboptimal solution
(in terms of duality gap, which upper bounds both the primal and dual
suboptimality) in $\tilde{O}\left( n + \nicefrac{1}{\lambda \epsilon} \right)$
iterations. Each iteration costs $n$ kernel evaluations, so this analysis is
slightly worse than that of SMO, although the simplicity of the algorithm, and
its excellent performance in the linear SVM setting, still make it a worthwhile
addition to a practitioner's toolbox.

%% file: svm/introduction/figures/alg-traditional.tex
\begin{algorithm}[t]

\begin{pseudocode}
\codename $\code{optimize}\left( n:\N, d:\N, x_{1},\dots,x_{n}:\R^d, y_{1},\dots,y_{n}:\left\{\pm 1\right\}, T:\N, K:\R^d\times\R^d\rightarrow\R \right)$\\
\codeline $w := 0$;\\
\codeline $\code{for } t=1 \code{ to } T$\\
\codeline \>\emph{Somehow} choose a working set $\mathcal{I}\subseteq\{1,\dots,n\}$
of training indices;\\
\codeline \>Calculate the responses $c_i = y_i \inner{w}{\Phi\left(x_i\right)}$ and
use them to optimize \emph{some}\\
\codeskip \>subproblem on the working set $\mathcal{I}$;\\
\codeline \>Update $w$ by adding \emph{some} linear combination of $\Phi\left(x_i\right)
: i\in\mathcal{I}$ to $w$;\\
\codeline \>Scale $w$ by multiplying it by \emph{some} quantity $\gamma$;\\
\codeline $\code{return } w$;
\end{pseudocode}

\caption{
Outline of a ``traditional'' SVM optimization algorithm. Different algorithms
will do different things in the ``\emph{some}'' portions.
}

\label{alg:svm-introduction:traditional}

\end{algorithm}

%% file: svm/introduction/figures/alg-traditional-kernel.tex
\begin{algorithm}[t]

\begin{pseudocode}
\codename $\code{optimize}\left( n:\N, d:\N, x_{1},\dots,x_{n}:\R^d, y_{1},\dots,y_{n}:\left\{\pm 1\right\}, T:\N, K:\R^d\times\R^d\rightarrow\R \right)$\\
\codeline $\alpha := 0$; $c := 0$;\\
\codeline $\code{for } t=1 \code{ to } T$\\
\codeline \>\emph{Somehow} choose a working set $\mathcal{I}\subseteq\{1,\dots,n\}$
of training indices;\\
\codeline \>Optimize \emph{some} subproblem on the working set $\mathcal{I}$;\\
\codeline \>Update $w$ by adding \emph{some} $\delta_i$ to $\alpha_i$ for every
$i\in\mathcal{I}$, and update the responses by\\
\codeskip \>adding $y_i y_j \delta_i K(x_i,x_j)$ to $c_j$ for every
$i\in\mathcal{I}$ and $j\in\{1,\dots,n\}$;\\
\codeline \>Scale $w$ by multiplying $\alpha$ and $c$ by \emph{some} quantity
$\gamma$;\\
\codeline $\code{return } \alpha$;
\end{pseudocode}

\caption{
Outline of a ``traditional'' SVM optimization algorithm in the kernel setting
which keeps track of a complete up-to-date vector of $n$ responses throughout.
Different algorithms will do different things in the ``\emph{some}'' portions.
}

\label{alg:svm-introduction:traditional-kernel}

\end{algorithm}

%% file: svm/introduction/theorems/thm-norm-constrained-runtime.tex
\medskip
\begin{splittheorem}{thm:svm-introduction:norm-constrained-runtime}

Let $u\in\hilbert$ be an arbitrary reference classifier with $R \ge \norm{u}$.
Suppose that we perform $T$ iterations of SGD on the norm-constrained
objective, with $w_t$ the weight vector after the $t$th iteration. If we define
$\bar{w} = \frac{1}{T} \sum_{t=1}^T w_t$ as the average of these iterates, then
there exist values of the training size $n$ and iteration count $T$ such that
SGD on the norm-constrained objective with step size $\eta_t = R
\sqrt{\nicefrac{2}{t}}$ finds a solution $\bar{w} = \sum_{i=1}^{n}
\bar{\alpha}_{i}y_{i}\Phi\left(x_{i}\right)$ satisfying:
\begin{equation*}
\loss[\zeroone]{g_{\bar{w}}} \le \loss[\hinge]{g_u} + \epsilon
\end{equation*}
after performing the following number of kernel evaluations:
\begin{equation*}
\mbox{\#K} = \tilde{O}\left( \left( \frac{\loss[\hinge]{g_u} +
\epsilon}{\epsilon} \right) \frac{ R^{4} }{\epsilon^{3}} \log^3\frac{1}{\delta}
\right)
\end{equation*}
with the size of the support set of $\bar{w}$ (the number nonzero elements in
$\bar{\alpha}$) satisfying:
\begin{equation*}
\mbox{\#S} = O\left( \frac{R^{2}}{\epsilon^2} \log\frac{1}{\delta} \right)
\end{equation*}
the above statements holding with probability $1-\delta$.

\end{splittheorem}
\begin{splitproof}

We may bound the online regret of SGD using \citet[Theorem 1]{Zinkevich03},
which gives the following, if we let $i_t$ be the index sampled at the $t$th
iteration:
\begin{equation*}
\frac{1}{T} \sum_{t=1}^T \ell_{\hinge}\left(g_{w_t}\left(x_{i_t}\right)\right)
\le \frac{1}{T} \sum_{t=1}^T \ell_{\hinge}\left(g_u\left(x_{i_t}\right)\right)
+ R\sqrt{\frac{1}{2T}}
\end{equation*}
In order to change the above regret bound into a bound on the empirical hinge
loss, we use the online-to-batch conversion of \citet[Theorem 1]{CesaCoGe01},
which gives that with probability $1-\delta$:
\begin{equation*}
\emploss[\hinge]{g_{\bar{w}}} \le \emploss[\hinge]{g_u} + R\sqrt{\frac{1}{2T}}
+ R\sqrt{\frac{2}{T}\log\frac{1}{\delta}}
\end{equation*}
Which gives that $\bar{w}$ will be $\epsilon$-suboptimal after performing the
following number of iterations:'
\begin{equation*}
T = O\left( \frac{R^2}{\epsilon^2} \log\frac{1}{\delta} \right)
\end{equation*}
Because $w$ is projected onto the Euclidean ball of radius $R$ at the end of
every iteration, we must have that $\norm{\bar{w}} \le R$, so we may apply
Lemma \ref{lem:svm-introduction:generalization-from-expected-loss} to bound $n$
as (Equation
\ref{eq:svm-introduction:generalization-from-expected-loss-bound}):
\begin{equation*}
n = \tilde{O}\left( \left( \frac{\loss[\hinge]{g_u} + \epsilon}{\epsilon}
\right) \frac{ \left( R + \log\frac{1}{\delta} \right)^{2} }{\epsilon} \right)
\end{equation*}
The number of kernel evaluations performed over $T$ iterations is $Tn$, and the
support size is $T$, which gives the claimed result.
\end{splitproof}

%% file: svm/introduction/theorems/thm-regularized-runtime.tex
\medskip
\begin{splittheorem}{thm:svm-introduction:regularized-runtime}

Let $u\in\hilbert$ be an arbitrary reference classifier. Suppose that we
perform $T$ iterations of Pegasos, with $w_t$ the weight vector after the $t$th
iteration. If we define $\bar{w} = \frac{1}{T} \sum_{t=1}^T w_t$ as the average
of these iterates, then there exist values of the training size $n$, iteration
count $T$ and regularization parameter $\lambda$ such that Pegasos with step
size $\eta_t = \nicefrac{1}{\lambda t}$ finds a solution $\bar{w} =
\sum_{i=1}^{n} \bar{\alpha}_{i}y_{i}\Phi\left(x_{i}\right)$ satisfying:
\begin{equation*}
\loss[\zeroone]{g_w} \le \loss[\hinge]{g_u} + \epsilon
\end{equation*}
after performing the following number of kernel evaluations:
\begin{equation*}
\mbox{\#K} = \tilde{O}\left( \min\left( \frac{1}{\epsilon}, \left(
\frac{\loss[\hinge]{g_u} + \epsilon}{\epsilon} \right)^{2} \right) \frac{
	\norm{u}^{4} }{\epsilon^{3}} \log^3\frac{1}{\delta} \right)
\end{equation*}
with the size of the support set of $w$ (the number nonzero elements in
$\bar{\alpha}$) satisfying:
\begin{equation*}
\mbox{\#S} = \tilde{O}\left( \frac{\norm{u}^{2}}{\epsilon^2}
\log\frac{1}{\delta} \right)
\end{equation*}
the above statements holding with probability $1-\delta$.

\end{splittheorem}
\begin{splitproof}

The analysis of \citet[Corollary 7]{KakadeTe09} permits us to bound the
suboptimality relative to the reference classifier $u$, with probability
$1-\delta$, as:
\begin{equation*}
\left( \frac{\lambda}{2}\norm{\bar{w}}^{2} +
\emploss[\hinge]{g_{\bar{w}}} \right) - \left( \frac{\lambda}{2}\norm{u}^{2} +
\emploss[\hinge]{g_u} \right) \le \frac{84 \log T}{\lambda T}
\log\frac{1}{\delta}
\end{equation*}
Solving this equation for $T$ gives that, if one performs the following number
of iterations, then the resulting solution will be $\epsilon/2$-suboptimal in
the \emph{regularized objective}, with probability $1-\delta$:
\begin{equation*}
T = \tilde{O}\left( \frac{1}{\epsilon} \cdot \frac{r^2}{\lambda}
\log\frac{1}{\delta} \right)
\end{equation*}
We next follow \citet{ShalevSr08} by decomposing the suboptimality in the
empirical hinge loss as:
\begin{align*}
\emploss[\hinge]{g_{\bar{w}}} - \emploss[\hinge]{g_u} & = \frac{\epsilon}{2} -
\frac{\lambda}{2}\norm{\bar{w}}^2 + \frac{\lambda}{2}\norm{u}^2 \\
& \le \frac{\epsilon}{2} + \frac{\lambda}{2}\norm{u}^2
\end{align*}
In order to have both terms bounded by $\epsilon/2$, we choose $\lambda =
\epsilon / \norm{u}^{2}$, which reduces the RHS of the above to $\epsilon$.
Continuing to use this choice of $\lambda$, we next decompose the squared norm
of $\bar{w}$ as:
\begin{align*}
\frac{\lambda}{2}\norm{\bar{w}}^2 & = \frac{\epsilon}{2} - \emploss[\hinge]{g_{\bar{w}}} +
\emploss[\hinge]{g_u} + \frac{\lambda}{2}\norm{u}^2 \\
& \le \frac{\epsilon}{2} + \emploss[\hinge]{g_u} + \frac{\lambda}{2}\norm{u}^2
\\
\norm{\bar{w}}^2 & \le 2 \left( \frac{ \emploss[\hinge]{g_u} + \epsilon }{ \epsilon}
\right) \norm{u}^2
\end{align*}
Hence, we will have that:
\begin{align*}
\norm{\bar{w}}^{2} & \le 2 \left(
\frac{ \emploss[\hinge]{g_u} + \epsilon }{ \epsilon} \right) \norm{u}^{2} \\
\emploss[\hinge]{g_{\bar{w}}} - \emploss[\hinge]{g_u} & \le \epsilon
\end{align*}
with probability $1-\delta$, after performing the following number of
iterations:
\begin{equation}
\label{eq:svm-introduction:regularized-time} T = \tilde{O}\left(
\frac{\norm{u}^2}{\epsilon^{2}} \log\frac{1}{\delta} \right)
\end{equation}

There are two ways in which we will use this bound on $T$ to find bound on the
number of kernel evaluations required to achieve some desired
error. The easiest is to note that the bound of Equation
\ref{eq:svm-introduction:regularized-time} exceeds that of Lemma
\ref{lem:svm-introduction:generalization-from-expected-loss}, so that if we
take $T=n$, then with high probability, we'll achieve generalization error $2
\epsilon$ after $Tn=T^{2}$ kernel evaluations, in which case:
\begin{equation*}
\mbox{\#K} = \tilde{O}\left(
\frac{\norm{u}^4}{\epsilon^{4}} \log^{2}\frac{1}{\delta} \right)
\end{equation*}
Because we take the number of iterations to be precisely the same as the number
of training examples, this is essentially the online stochastic setting.

Alternatively, we may combine our bound on $T$ with Lemma
\ref{lem:svm-introduction:generalization-from-expected-loss}. This yields the
following bound on the generalization error of Pegasos in the data-laden batch
setting.
\begin{equation*}
\mbox{\#K} = \tilde{O}\left( \left( \frac{\loss[\hinge]{g_u} +
\epsilon}{\epsilon} \right)^{2} \frac{ \norm{u}^{4} }{\epsilon^{3}}
\log^3\frac{1}{\delta} \right)
\end{equation*}
Combining these two bounds on $\mbox{\#K}$ gives the claimed result.
\end{splitproof}

%% file: svm/introduction/sec-non-traditional.tex
\section{Other Algorithms}\label{sec:svm-introduction:non-traditional}

To this point, we have only discussed the optimization of various equivalent
SVM objectives which work in the ``black-box kernel'' setting---i.e.  for which
we need no special knowledge of the kernel function $K$ in order to perform the
optimization, as long as we may evaluate it. In this section, we will consider
two additional algorithms which do not fall under this umbrella, but
nonetheless may be compared to the SVM optimization algorithms which we have so
far considered.

The first of these---the online Perceptron algorithm---does not optimize any
version of the SVM objective, but nevertheless achieves a learning guarantee of
the same form as those of the previous section. The second approach is based on
the idea of approximating a kernel SVM problem with a linear SVM problem,
which, as a linear SVM, may be optimized very rapidly, but suffers from the
drawback that special knowledge of the kernel function $K$ is required.

\subsection{Perceptron}\label{subsec:svm-introduction:non-traditional-perceptron}

The online Perceptron algorithm considers each training example in sequence,
and if $w$ errs on the point under consideration (i.e.~$y_i
\inner{w}{\Phi(x_i)}\leq 0$), then $y_i \Phi(x_i)$ is added into $w$ (i.e. $w$
is updated as $w \leftarrow w + y_i \Phi(x_i)$). Let $M$ be the number of
mistakes made by the Perceptron on the sequence of examples. Support vectors
are added only when a mistake is made, and so each iteration of the Perceptron
involves at most $M$ kernel evaluations. The total runtime is therefore $Mn$,
and the total support size $M$.

Analysis of the online Perceptron algorithm is typically presented as a bound
(e.g. \citet[Corollary 5]{Shalev07}) on the number of mistakes $M$ made by the
algorithm in terms of the hinge loss of the best classifier. While it is an
online learning algorithm, the Perceptron can also be used in the batch setting
using an online-to-batch conversion (e.g. \cite{CesaCoGe01}), in which case we
may derive a generalization guarantee. The proof of the following lemma follows
this outline, and bounds the generalization performance of the Perceptron:

\input{svm/introduction/theorems/thm-perceptron-runtime}

Although this result has a $\delta$-dependence of $1/\delta$, this is merely a
relic of the simple online-to-batch conversion which we use in the analysis.
Using a more complex algorithm (e.g. \citet{CesaCoGe01}) would
likely improve this term to $\log\frac{1}{\delta}$.

The flexibility of being able to apply the Perceptron in both the batch and
online settings comes at a cost---unlike algorithms which rely on regularization
to ensure that the classifier does not overfit, the correctness of the
online-to-batch conversion which we use is predicated on the assumption that
the training examples considered during training are sampled independently from
the underlying distribution $\mathcal{D}$. For this reason, we may only make a
single ``pass'' over the training data, in the batch setting---if we considered
each sample multiple times, then each could no longer be regarded as an
independent draw from $\mathcal{D}$. As a result, unless training samples are
plentiful, the Perceptron may not be able to find a good solution before
``running out'' of training examples.

\subsection{Random Projections}\label{subsec:svm-introduction:non-traditional-random-projections}

We have so far relied only on ``black box'' kernel accesses. However, there are
specialized algorithms which work only for particular kernels (algorithms which
work well on linear kernels are by far the most common), or for a particular
class of kernels. One such is the random Fourier feature approach of
\citet{RahimiRe07}, which works for so-called ``shift-invariant'' kernels, each
of which can be written in terms of the difference between its arguments (i.e.
$K(x,x') = K'(x-x')$). The Gaussian kernel, which is probably the most
widely-used nonlinear kernel, is one such.

The fundamental idea behind this approach is, starting from a kernel SVM
problem using a shift invariant kernel, to find a linear SVM problem which
approximates the original problem, and may therefore be solved using any of the
many fast linear SVM optimizers in existence (e.g. Pegasos
\citep{ShalevSiSr07}, SDCA \citep{HsiehChLiKeSu08} or SIMBA
\citep{HazanKoSr11}).
Supposing that $\mathcal{X} = \R^d$, this is accomplished by observing
that we can write $K'$ in terms of its Fourier transform as \citep[Theorem
1]{RahimiRe07}:
\begin{align*}
K'(x - x') = & Z \int_{\R^d} p(w) \cos \inner{w}{x-x'} dw \\
= & Z \int_{\R^d} p(w) \left( \cos \inner{w}{x} \cos \inner{w}{x'} + \sin
\inner{w}{x} \sin \inner{w}{x'} \right) dw
\end{align*}
where $p$, the Fourier transform of $K'$, can be taken to be a probability
density function by factoring out the normalization constant $Z$. We may therefore find an unbiased estimator of $K(x,x')$ by
sampling $v_1,\dots,v_D \in \R^d$ independently from $p$ and defining
the explicit feature map:
\begin{equation*}
\tilde{\Phi}(x) = \sqrt{\frac{Z}{D}} \left[ \begin{array}{c} \cos
\inner{v_1}{x} \\ \sin \inner{v_1}{x} \\ \vdots \\ \cos \inner{v_D}{x} \\ \sin
\inner{v_D}{x} \end{array} \right]
\end{equation*}
so that $\expectation[v_1,\dots,v_D]{\inner{\tilde{\Phi}(x)}{\tilde{\Phi}(x')}}
= K(x,x')$. Drawing more samples from $p$ (increasing the dimension $D$ of the
explicit features $\tilde{\Phi}(x)$) increases the accuracy of this
approximation, with the number of features required to achieve good
generalization error being bounded by the following theorem:
\input{svm/introduction/theorems/thm-projection-runtime}
Computing each random Fourier feature requires performing a $d$-dimensional
inner product, costing $O(d)$ operations. This is the same cost as performing a
kernel evaluation for many common kernels with $\mathcal{X} = \R^d$, so
one may consider the computational cost of finding a single random Fourier
feature to be roughly comparable to the cost of performing a single kernel
evaluation. Hence, if we ignore the cost of optimizing the approximate linear
SVM entirely, and only account for the cost of \emph{constructing} it, then the
above bound on $nD$ can be compared directly to the bounds on the number of
kernel evaluations $\mbox{\#K}$ which were found elsewhere in this chapter.

%% file: svm/introduction/theorems/thm-perceptron-runtime.tex
\medskip
\begin{splittheorem}{thm:svm-sbp:perceptron-runtime}

Let $u\in\hilbert$ be an arbitrary reference classifier. There exists a value
of the training size $n$ such that when the Perceptron algorithm is run for a
single ``pass'' over the dataset, the result is a solution $w = \sum_{i=1}^{n}
\alpha_{i}y_{i}\Phi\left(x_{i}\right)$ satisfying:
\begin{equation*}
\loss[\zeroone]{g_w} \le \loss[\hinge]{g_u} + \epsilon
\end{equation*}
after performing the following number of kernel evaluations:
\begin{equation*}
\mbox{\#K} = \tilde{O}\left( \left( \frac{\loss[\hinge]{g_u} +
\epsilon}{\epsilon} \right)^{3} \frac{ \norm{u}^{4} }{\epsilon}
\frac{1}{\delta} \right)
\end{equation*}
with the size of the support set of $w$ (the number nonzero elements in
$\alpha$) satisfying:
\begin{equation*}
\mbox{\#S} = O\left( \left( \frac{\loss[\hinge]{g_u} + \epsilon}{\epsilon}
\right)^{2} \norm{u}^{2} \frac{1}{\delta} \right)
\end{equation*}
the above statements holding with probability $1-\delta$.

\end{splittheorem}
\begin{splitproof}

If we run the online Perceptron algorithm for a single pass over the dataset,
then Corollary 5 of \cite{Shalev07} gives the following mistake bound, for
$\mathcal{M}$ being the set of iterations on which a mistake is made:
\begin{align}
\label{eq:svm-sbp:perceptron-starting-bound} \MoveEqLeft \abs{\mathcal{M}} \le
\sum_{i\in\mathcal{M}}\ell\left(y_{i}\inner{u}{\Phi\left(x_{i}\right)}\right) \\
\nonumber & + \norm{u}
\sqrt{\sum_{i\in\mathcal{M}}\ell\left(y_{i}\inner{u}{\Phi\left(x_{i}\right)}\right)} +
\norm{u}^{2} \\
\MoveEqLeft \nonumber \sum_{i=1}^{n}\ell_{0/1}\left(y_{i}\inner{w_{i}}{\Phi\left(x_{i}\right)}\right)
\le \sum_{i=1}^{n}\ell\left(y_{i}\inner{u}{\Phi\left(x_{i}\right)}\right) + \\
\nonumber & + \norm{u} \sqrt{\sum_{i=1}^{n}\ell\left(y_{i}\inner{u}{\Phi\left(x_{i}\right)}\right)} +
r^{2} \norm{u}^{2}
\end{align}
Here, $\ell$ is the hinge loss and $\ell_{0/1}$ is the 0/1 loss. Dividing
through by $n$:
\begin{align*}
\MoveEqLeft \frac{1}{n}\sum_{i=1}^{n}\ell_{0/1}\left(y_{i}\inner{w_{i}}{\Phi\left(x_{i}\right)}\right)
\le \frac{1}{n}\sum_{i=1}^{n}\ell\left(y_{i}\inner{u}{\Phi\left(x_{i}\right)}\right) \\
\notag & + \frac{r
\norm{u}}{\sqrt{n}}\sqrt{\frac{1}{n}\sum_{i=1}^{n}\ell\left(y_{i}\inner{u}{\Phi\left(x_{i}\right)}\right)}
+ \frac{\norm{u}^{2}}{n}
\end{align*}
If we suppose that the $x_{i},y_{i}$s are \iid, and that
$w\sim\mbox{Unif}\left(w_{1},\dots,w_{n}\right)$ (this is a ``sampling''
online-to-batch conversion), then:
\begin{equation*}
\expectation{\loss[\zeroone]{g_w}} \le \loss[\hinge]{g_u} + \frac{r
\norm{u}}{\sqrt{n}}\sqrt{\loss[\hinge]{g_u}} + \frac{
\norm{u}^{2}}{n}
\end{equation*}
Hence, the following will be satisfied:
\begin{equation}
\label{eq:svm-sbp:perceptron-bound} \expectation{\loss[\zeroone]{g_w}} \le
\loss[\hinge]{g_u}+\epsilon
\end{equation}
when:
\begin{equation*}
n\le
O\left(\left(\frac{\loss[\hinge]{g_u}+\epsilon}{\epsilon}\right)\frac{
\norm{u}^{2}}{\epsilon}\right)
\end{equation*}
The expectation is taken over the random sampling of $w$. The number of kernel
evaluations performed by the $i$th iteration of the Perceptron will be equal to
the number of mistakes made before iteration $i$.  This quantity is upper
bounded by the total number of mistakes made over $n$ iterations, which is
given by the mistake bound of equation \ref{eq:svm-sbp:perceptron-starting-bound}:
\begin{align*}
\abs{\mathcal{M}}\le & n\loss[\hinge]{g_u} +
r\norm{u}\sqrt{n\loss[\hinge]{g_u}} + \norm{u}^{2}\\
\le &
O\left(\left(\frac{1}{\epsilon}\left(\frac{\loss[\hinge]{g_u}+\epsilon}{\epsilon}\right)\loss[\hinge]{g_u} \right. \right. \\
& \left. \left. + \sqrt{\frac{1}{\epsilon}\left(\frac{\loss[\hinge]{g_u}+\epsilon}{\epsilon}\right)\loss[\hinge]{g_u}}+1\right)\norm{u}^{2}\right)\\
\le &
O\left(\left(\left(\frac{\loss[\hinge]{g_u}+\epsilon}{\epsilon}\right)^{2}-\left(\frac{\loss[\hinge]{g_u}+\epsilon}{\epsilon}\right) \right. \right. \\
& \left. \left. + \sqrt{\left(\frac{\loss[\hinge]{g_u}+\epsilon}{\epsilon}\right)^{2}-\left(\frac{\loss[\hinge]{g_u}+\epsilon}{\epsilon}\right)}+1\right)
r^{2}\norm{u}^{2}\right)\\
\le &
O\left(\left(\frac{\loss[\hinge]{g_u}+\epsilon}{\epsilon}\right)^{2}\norm{u}^{2}\right)
\end{align*}
The number of mistakes $\abs{\mathcal{M}}$ is necessarily equal to the size of
the support set of the resulting classifier. Substituting this bound into the
number of iterations:
\begin{align*}
\mbox{\#K}= & n\abs{\mathcal{M}}\\
\le &
O\left(\left(\frac{\loss[\hinge]{g_u}+\epsilon}{\epsilon}\right)^{3}\frac{\norm{u}^{4}}{\epsilon}\right)
\end{align*}
This holds in expectation, but we can turn this into a high-probability result
using Markov's inequality, resulting in in a $\delta$-dependence of
$\frac{1}{\delta}$.
\end{splitproof}

%% file: svm/introduction/theorems/thm-projection-runtime.tex
\medskip
\begin{splittheorem}{thm:svm-sbp:projection-runtime}

Let $u\in\hilbert$ be an arbitrary reference classifier which is written in
terms of the training data as $u = \sum_{i=1}^n \alpha_i y_i
\Phi\left(x_i\right)$. Suppose that we perform the approximation procedure of
\citet{RahimiRe07} to yield a $2D$-dimensional linear approximation to the
original kernelized classification problem, and define $\tilde{u} =
\sum_{i=1}^n \alpha_i y_i \tilde{\Phi}\left(x_i\right)$. There exists a value
of the approximate linear SVM dimension $D$:
\begin{equation*}
%
%
D = \tilde{O}\left( \frac{ d \norm{\alpha}_1^2 }{\epsilon^2} \log\frac{1}{\delta} \right)
\end{equation*}
as well as a training size $n$, such that, if the resulting number of features
satisfies:
\begin{equation*}
%
%
nD = \tilde{O}\left( \left( \frac{\loss[\hinge]{g_u} + \epsilon}{\epsilon}
\right) \frac{ d \norm{\alpha}_1^4 }{\epsilon^3} \log^3\frac{1}{\delta} \right)
\end{equation*}
Then for any $w\in\R^{2D}$ solving the resulting approximate linear
problem such that:
\begin{equation*}
\norm{w} \le \left( 1 + \sqrt{\frac{\epsilon}{\norm{\alpha}_1}}
\right)\norm{\alpha}_1,\ \ 
\emploss[\hinge]{g_w} \le \emploss[\hinge]{g_{\tilde{u}}} + \epsilon
\end{equation*}
we will have that $\loss[\zeroone]{g_w} \le \loss[\hinge]{g_u} + 3\epsilon$,
with this result holding with probability $1-\delta$.
Notice that we have abused notation slightly here---losses of $g_u$ are for the
original kernel problem, while losses of $g_w$ and $g_{\tilde{u}}$ are of the
approximate linear SVM problem.

\end{splittheorem}
\begin{splitproof}

Hoeffding's inequality yields \citep[Claim 1]{RahimiRe07} that, in order to
ensure that $\abs{ K(x,x') - 
\inner{\tilde{\Phi}(x)}{\tilde{\Phi}(x')} } < \tilde{\epsilon}$ uniformly for
all $x,x' \in \R^d$ with probability $1-\delta$, the following number
of random Fourier features are required:
\begin{equation*}
D = O\left( \frac{d}{\tilde{\epsilon}^2} \log\left( \frac{1}{\delta
\tilde{\epsilon}^2} \right) \right)
\end{equation*}
Here, the proportionality constant hidden inside the big-Oh notation varies
depending on the particular kernel in which we are interested, and a uniform
bound on the $2$-norm of $x\in\mathcal{X}$. This result is not immediately
applicable, because the quantities which we need in order to derive a
generalization bound are the norm of the predictor, and the error in each
\emph{classification} (rather than the error in each kernel evaluation).
Applying the above bound yields that:
\begin{align*}
\MoveEqLeft \abs{\ell_{\hinge}\left( y \inner{u}{\Phi\left(x\right)} \right) -
\ell_{\hinge}\left( y \inner{ \tilde{u} }{\tilde{\Phi}\left(x\right)} \right)}
\\
& \le \abs{ \inner{u}{\Phi\left(x\right)} -
\inner{\tilde{u}}{\tilde{\Phi}\left(x\right)} } \\
& = \abs{\sum_{i=1}^n \alpha_i y_i \left( K\left(x_i,x\right) -
\inner{\tilde{\Phi}\left(x_i\right)}{\tilde{\Phi}\left(x\right)} \right) } \\
& \le \tilde{\epsilon} \norm{\alpha}_1
\end{align*}
and:
\begin{align*}
\abs{\norm{u}^2 - \norm{\tilde{u}}^2} = & \abs{\sum_{i=1}^n \sum_{j=1}^n
\alpha_i \alpha_j y_i y_j \left( K\left(x_i,x_j\right) -
\inner{\tilde{\Phi}\left(x_i\right)}{\tilde{\Phi}\left(x_j\right)} \right) } \\
\le & \tilde{\epsilon} \norm{\alpha}_1^2
\end{align*}
Observe that, because $K(x,x) \le 1$ by assumption, $\norm{u}_2 \le
\norm{u}_1$. Hence, if we take $B = \left( 1 + \sqrt{\tilde{\epsilon}}
\right)\norm{\alpha}_1$ so that $B^2 \ge \norm{u}^2 +
\tilde{\epsilon}\norm{\alpha}_1^2$ in Lemma
\ref{lem:svm-introduction:generalization-from-expected-loss}, we will have that
all $\tilde{w} \in \R^{2D}$ satisfying $\norm{w} \le B$ and
$\emploss[\hinge]{g_w} \le \emploss[\hinge]{g_u} + \epsilon$ generalize as
$\loss[\zeroone]{g_w} \le \loss[\hinge]{g_{\tilde{u}}} + 2\epsilon$, as long
as:
\begin{equation*}
n = \tilde{O}\left( \left( \frac{\loss[\hinge]{g_u} + \epsilon}{\epsilon}
\right) \frac{ \left( \left( 1 + \sqrt{\tilde{\epsilon}} \right)
\norm{\alpha}_1 + \log\frac{1}{\delta} \right)^2 }{\epsilon} \right)
\end{equation*}
Taking $\tilde{\epsilon} = \nicefrac{\epsilon}{\norm{\alpha}_1}$, multiplying
the bounds on $n$ and $D$, and observing that $\loss[\hinge]{g_{\tilde{u}}} \le
\loss[\hinge]{g_u} + \epsilon$ completes the proof.
\end{splitproof}

%% file: svm/introduction/sec-proofs.tex
\section{Proofs for Chapter \ref{ch:svm-introduction}}

\begin{proofs}

\subsection{Proof of Lemma \ref{lem:svm-introduction:generalization-from-expected-loss}}\label{subsec:svm-introduction:generalization-proof}

In order to prove Lemma
\ref{lem:svm-introduction:generalization-from-expected-loss}, we must first
prove two intermediate results. The first of these, Lemma
\ref{lem:svm-introduction:generalization-from-empirical-loss}, will give a
bound on $n$ which guarantees good generalization performance for all
classifiers which achieve some target empirical hinge loss $L$. The second,
Lemma \ref{lem:svm-introduction:empirical-loss-from-expected-loss}, gives a
bound on $n$ which guarantees that the empirical hinge loss is close to the
expected hinge loss, for a single arbitrary classifier. Combining these results
yields the sample size required to guarantee that a suboptimal classifier is
within $\epsilon$ of the optimum, giving Lemma
\ref{lem:svm-introduction:generalization-from-expected-loss}.

We begin with the following lemma, which follows immediately from
\citet[Theorem 1]{SrebroSrTe10}---indeed, all of the results of this section may
be viewed as nothing but a clarification of results proved in this paper:

\input{svm/introduction/theorems/lem-generalization-from-empirical-loss}

While Lemma \ref{lem:svm-introduction:generalization-from-empirical-loss}
bounds the generalization error of classifiers with empirical hinge loss less
than the target value of $L$, the question of how to set $L$ remains
unanswered. The following lemma shows that we may take it to be
$\loss[\hinge]{w^{*}} + \epsilon$, where $w^{*}$ minimizes the expected hinge
loss:

\input{svm/introduction/theorems/lem-empirical-loss-from-expected-loss}

Combining lemmas \ref{lem:svm-introduction:generalization-from-empirical-loss}
and \ref{lem:svm-introduction:empirical-loss-from-expected-loss} gives us a
bound of the desired form, which is given by the following lemma:

\input{svm/introduction/theorems/lem-generalization-from-expected-loss}
\subsection{Other Proofs}

\input{svm/introduction/theorems/thm-norm-constrained-runtime}
\input{svm/introduction/theorems/thm-regularized-runtime}
\input{svm/introduction/theorems/thm-perceptron-runtime}
\input{svm/introduction/theorems/thm-projection-runtime}
\end{proofs}

%% file: svm/introduction/theorems/lem-generalization-from-empirical-loss.tex
\input{svm/introduction/figures/fig-01-hinge-bound}
\medskip
\begin{lemma}
\label{lem:svm-introduction:generalization-from-empirical-loss}

Suppose that we sample a training set of size $n$, with $n$ given by the
following equation, for parameters $L,B,\epsilon>0$ and
$\delta\in\left(0,1\right)$:
\begin{equation}
\label{eq:svm-introduction:generalization-from-empirical-loss-bound} n = \tilde{O}\left( \left(
\frac{L}{\epsilon} \right) \frac{\left( B + \sqrt{\log\frac{1}{\delta}}
\right)^2}{\epsilon} \right)
\end{equation}
Then, with probability $1-\delta$ over the \iid training sample
$x_{i},y_{i}:i\in\left\{ 1,\dots,n\right\}$, uniformly for all $w$ satisfying:
\begin{equation}
\label{eq:svm-introduction:generalization-from-empirical-loss-conditions}
\norm{w} \le B,\ \ 
\emploss[\hinge]{g_w} \le L
\end{equation}
we have that $\loss[\zeroone]{g_w} \le L + \epsilon$.

\end{lemma}
\begin{proof}

For a \emph{smooth} loss function, Theorem 1 of \citet{SrebroSrTe10} bounds the
expected loss in terms of the empirical loss, plus a factor depending on (among
other things) the sample size. Neither the 0/1 nor the hinge losses are smooth,
so we will define a bounded and smooth loss function which upper bounds the 0/1
loss and lower-bounds the hinge loss. The particular function which we use
doesn't matter, since its smoothness parameter and upper bound will ultimately
be absorbed into the big-Oh notation---all that is needed is the
\emph{existence} of such a function. One such is:
\begin{equation*}
\psi\left(x\right)=\left\{
\begin{array}{lcr}
\nicefrac{5}{4} & \cdots & x<-\nicefrac{1}{2}\\
-x^{2}-x+1 & \cdots & -\nicefrac{1}{2}\le x<0\\
x^{3}-x^{2}-x+1 & \cdots & 0\le x<1\\
0 & \cdots & x\ge1
\end{array}
\right.
\end{equation*}
This function, illustrated in Figure \ref{fig:svm-introduction:01-hinge-bound}, is $4$-smooth
and $\nicefrac{5}{4}$-bounded. If we define $\loss[\smooth]{g_w}$ and
$\emploss[\smooth]{g_w}$ as the expected and empirical $\psi$-losses,
respectively, then the aforementioned theorem gives that, with probability
$1-\delta$ uniformly over all $w$ such that $\norm{w}\le B$:
\begin{align*}
\MoveEqLeft \loss[\smooth]{g_w} \le \emploss[\smooth]{g_w} \\
& + O\left( \frac{B^2 \log^3 n}{n} + \frac{\log{\frac{1}{\delta}}}{n}
+ \sqrt{\emploss[\smooth]{g_w}}\left( \sqrt{\frac{B^2 \log^3 n}{n}} +
\sqrt{\frac{\log{\frac{1}{\delta}}}{n}} \right) \right)
\end{align*}
Because $\psi$ is lower-bounded by the 0/1 loss and upper-bounded by the hinge
loss, we may replace $\loss[\smooth]{g_w}$ with $\loss[\zeroone]{g_w}$ on the LHS of
the above bound, and $\emploss[\smooth]{g_w}$ with $L$ on the RHS.  Setting the
big-Oh expression to $\epsilon$ and solving for $n$ then gives the desired
result.
\end{proof}

%% file: svm/introduction/figures/fig-01-hinge-bound.tex
\begin{figure}[t]

\begin{center}
\includegraphics[width=0.45\textwidth]{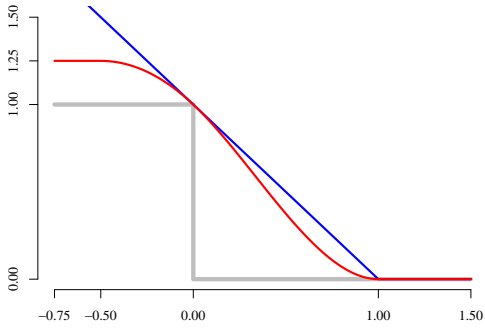}
\end{center}

\caption{
Plot of a smooth and bounded function (red) which upper bounds the 0/1 loss and lower
bounds the hinge loss.
}

\label{fig:svm-introduction:01-hinge-bound}

\end{figure}

%% file: svm/introduction/theorems/lem-empirical-loss-from-expected-loss.tex
\medskip
\begin{lemma}
\label{lem:svm-introduction:empirical-loss-from-expected-loss}

Let $u$ be an arbitrary linear classifier, and suppose that we sample a
training set of size $n$, with $n$ given by the following equation, for
parameters $\epsilon>0$ and $\delta\in\left(0,1\right)$:
\begin{equation}
\label{eq:svm-introduction:empirical-loss-from-expected-loss-bound} n = 2 \left(
\frac{\loss[\hinge]{g_u} + \epsilon}{\epsilon} \right)
\frac{\norm{u}\log\frac{1}{\delta}}{\epsilon}
\end{equation}
Then, with probability
$1-\delta$ over the \iid training sample $x_{i},y_{i}:i\in\left\{
1,\dots,n\right\}$, we have that $\emploss[\hinge]{g_u} \le \loss[\hinge]{g_u} +
\epsilon$.

\end{lemma}
\begin{proof}

The hinge loss is upper-bounded by $\norm{u}$ (by assumption, $K(x,x) \le 1$
with probability $1$, when $x\sim\mathcal{D}$), from which it follows that
$\variance[x,y]{\ell\left(y\inner{u}{x}\right)} \le
\norm{u}\loss[\hinge]{g_u}$.  Hence, by Bernstein's inequality:
\begin{align*}
\probability{ \emploss[\hinge]{g_u} > \loss[\hinge]{g_u} + \epsilon } \le &
\exp\left( - \frac{n}{\norm{u}}\left(
\frac{\epsilon^{2}/2}{\loss[\hinge]{g_u}+\epsilon/3} \right) \right) \\
\le & \exp\left( - \frac{n}{2\norm{u}}\left(
\frac{\epsilon^{2}}{\loss[\hinge]{g_u}+\epsilon} \right) \right)
\end{align*}
Setting the LHS to $\delta$ and solving for $n$ gives the desired result.
\end{proof}

%% file: svm/ch-sbp.tex
\chapter{The Kernelized Stochastic Batch Perceptron}\label{ch:svm-sbp}

\input{svm/sbp/sec-overview}
\input{svm/sbp/sec-problem}

\input{svm/sbp/sec-algorithm}
\input{svm/sbp/sec-comparison}
\input{svm/sbp/sec-experiments}

\paragraph{Collaborators:} The work presented in this chapter was performed
jointly with Shai Shalev-Shwartz and Nathan Srebro.

\clearpage
\input{svm/sbp/sec-proofs}

%% file: svm/sbp/sec-overview.tex
\section{Overview}\label{sec:svm-sbp:overview}

In this chapter, we will present a novel algorithm for training kernel Support
Vector Machines (SVMs), and establish learning runtime guarantees which are
better then those for any other known kernelized SVM optimization approach. We
also show experimentally that our method works well in practice compared to
existing SVM training methods. The content of this chapter was originally
presented in the 29th International Conference on Machine Learning (ICML 2012)
\citep{CotterShSr12}.

Our method is a stochastic gradient
\footnote{we are a bit loose in often using ``gradient'' when we actually refer
to subgradients of a convex function, or equivalently supergradients of a
concave function.}
method on a non-standard scalarization of the bi-criterion SVM objective of
Problem \ref{eq:svm-introduction:bi-criterion-objective} in Chapter \ref{ch:svm-introduction}. In particular, we use
the ``slack constrained'' scalarized optimization problem introduced by
\citet{HazanKoSr11}, where we seek to maximize the classification margin,
subject to a constraint on the total amount of ``slack'', i.e.~sum of the
violations of this margin. Our approach is based on an efficient method for
computing unbiased gradient estimates on the objective. Our algorithm can be
seen as a generalization of the ``Batch Perceptron'' to the non-separable case
(i.e.~when errors are allowed), made possible by introducing stochasticity, and
we therefore refer to it as the ``Stochastic Batch Perceptron'' (SBP).
%
%

The SBP is fundamentally different from other stochastic gradient approaches to
the problem of training SVMs (such as SGD on the norm-constrained and
regularized objectives of Section \ref{sec:svm-introduction:objective} in
Chapter \ref{ch:svm-introduction}, see Section
\ref{sec:svm-introduction:traditional} for details), in that calculating each
stochastic gradient estimate still requires considering the {\em entire} data
set. In this regard, despite its stochasticity, the SBP is very much a
``batch'' rather than ``online'' algorithm. For a linear SVM, each iteration
would require runtime linear in the training set size, resulting in an
unacceptable overall runtime. However, in the kernel setting, essentially all
known approaches already require linear runtime per iteration. A more careful
analysis reveals the benefits of the SBP over previous kernel SVM optimization
algorithms.
%
%

As was done in Chapter \ref{ch:svm-introduction}, we follow
\citet{BottouBo07,ShalevSr08} and compare the runtimes required to ensure a
generalization error of $L^*+\epsilon$, assuming the existence of some unknown
predictor $u$ with norm $\norm{u}$ and expected hinge loss $L^*$. We derive an
``optimistic'' bound, as described in Section
\ref{sec:svm-introduction:generalization} of Chapter \ref{ch:svm-introduction},
for which the main advantage of the SBP over competing algorithms is in the
``easy'' regime: when $\epsilon = \Omega(L^*)$ (i.e. we seek a constant factor
approximation to the best achievable error, such as if we desire an error of
$1.01 L^*$). In such a case, the overall SBP runtime is $\norm{u}^4/\epsilon$,
compared with $\norm{u}^4/\epsilon^3$ for Pegasos and $\norm{u}^4/\epsilon^2$
for the best known dual decomposition approach.
%
%

%% file: svm/sbp/sec-problem.tex
\section{Setup and Formulations}

In Section \ref{sec:svm-introduction:objective} of Chapter
\ref{ch:svm-introduction}, the bi-criterion SVM objective was presented, along
with several of the equivalent scalarizations of this objective which are
optimized by various popular algorithms.
The SBP relies on none of these, but rather on the ``slack constrained''
scalarization \cite{HazanKoSr11}, for which we maximize the ``margin'' subject
to a constraint of $\nu$ on the total allowed ``slack'', corresponding to the
average error. That is, we aim at maximizing the margin by which all points are
correctly classified (i.e.~the minimal distance between a point and the
separating hyperplane), after allowing predictions to be corrected by a total
amount specified by the slack constraint:
\begin{align}
\label{eq:svm-sbp:slack-constrained-objective} \max_{w\in\R^d}
\max_{\xi\in\R^n} \min_{i\in\left\{1,\dots,n\right\}} & \left(
y_{i}\inner{w}{\Phi\left(x_{i}\right)}+\xi_{i} \right) \\
\nonumber \mbox{subject to: } & \norm{w} \le 1,~ \xi \succeq 0, ~
\mathbf{1}^{T}\xi\le n\nu
\end{align}
This scalarization is equivalent to the original bi-criterion objective in that
varying $\nu$ explores different Pareto optimal solutions of Problem
\ref{eq:svm-introduction:bi-criterion-objective}. This is captured by the
following Lemma, which also quantifies how suboptimal solutions of the
slack-constrained objective correspond to Pareto suboptimal points:

\input{svm/sbp/theorems/lem-slack-constrained-suboptimality}

%% file: svm/sbp/theorems/lem-slack-constrained-suboptimality.tex
\medskip
\begin{lemma}
\label{lem:svm-sbp:slack-constrained-suboptimality}

For any $u \ne 0$, consider Problem \ref{eq:svm-sbp:slack-constrained-objective} with
$\nu = \emploss[\hinge]{g_u}/\norm{u}$. Let $\bar{w}$ be an
$\bar{\epsilon}$-suboptimal solution to this problem with objective value
$\gamma$, and consider the rescaled solution $w=\bar{w}/\gamma$. Then:
\begin{align*}
\norm{w} \le & \frac{1}{1-\bar{\epsilon}\norm{u}}\norm{u} ~~,~~
\emploss[\hinge]{g_w} \le  \frac{1}{1-\bar{\epsilon}\norm{u}}
\emploss[\hinge]{g_u}
\end{align*}

\end{lemma}
\begin{proof}

This is Lemma 2.1 of \citet{HazanKoSr11}.
\end{proof}

%% file: svm/sbp/sec-algorithm.tex
\section{The Stochastic Batch Perceptron}\label{sec:svm-sbp:algorithm}

In this section, we will develop the Stochastic Batch Perceptron. We consider
Problem \ref{eq:svm-sbp:slack-constrained-objective} as optimization of the variable
$w$ with a single constraint $\norm{w}\leq 1$, with the objective being to
maximize:
\begin{equation}
\label{eq:svm-sbp:fw-definition} f\left(w\right) ~= \max_{\xi\succeq 0,
\mathbf{1}^{T}\xi\le n\nu} ~~\min_{p\in\Delta^{n}} ~~\sum_{i=1}^{n} p_{i}
\left( y_{i}\inner{w}{\Phi\left(x_{i}\right)}+\xi_{i} \right)
\end{equation}
Notice that we replaced the minimization over training indices $i$ in Problem
\ref{eq:svm-sbp:slack-constrained-objective} with an equivalent minimization over the
probability simplex, $\Delta^n = \{p \succeq 0 : \mathbf{1}^T p = 1\}$, and
that we consider $p$ and $\xi$ to be a part of the objective, rather than
optimization variables. The objective $f(w)$ is a concave function of $w$, and
we are maximizing it over a convex constraint $\norm{w}\leq 1$, and so this is
a convex optimization problem in $w$.

Our approach will be to perform a stochastic gradient update on $w$ at each
iteration: take a step in the direction specified by an unbiased estimator of a
(super)gradient of $f(w)$, and project back to $\norm{w}\leq 1$. To this end,
we will need to identify the (super)gradients of $f(w)$ and understand how to
efficiently calculate unbiased estimates of them.

\subsection{Warmup: The Separable Case}\label{subsec:svm-sbp:warmup}

As a warmup, we first consider the separable case, where $\nu=0$ and no errors
are allowed. The objective is then:
\begin{equation}
\label{eq:svm-sbp:sep-fw} f(w) = \min_i y_{i}\inner{w}{\Phi\left(x_{i}\right)},
\end{equation}
This is simply the ``margin'' by which all points are correctly classified,
i.e.~$\gamma$ s.t.~$\forall_i~ y_i \inner{w}{\Phi(x_i)}\geq\gamma$. We seek a
linear predictor $w$ with the largest possible margin. It is easy to see that
(super)gradients with respect to $w$ are given by $y_i \Phi(x_i)$ for any index
$i$ attaining the minimum in Equation \ref{eq:svm-sbp:sep-fw}, i.e.~by the ``most poorly
classified'' point(s). A gradient ascent approach would then be to iteratively
find such a point, update $w \leftarrow w+\eta y_i \Phi(x_i)$, and project back
to $\norm{w}\leq 1$. This is akin to a ``batch Perceptron'' update, which at
each iteration searches for a violating point and adds it to the predictor.

In the separable case, we could actually use \emph{exact} supergradients of the
objective. As we shall see, it is computationally beneficial in the
non-separable case to base our steps on unbiased gradient estimates. We
therefore refer to our method as the ``Stochastic Batch Perceptron'' (SBP), and
view it as a generalization of the batch Perceptron which uses stochasticity
and is applicable in the non-separable setting. In the same way that the
``batch Perceptron'' can be used to maximize the margin in the separable case,
the SBP can be used to obtain any SVM solution along the Pareto front of the
bi-criterion Problem \ref{eq:svm-introduction:bi-criterion-objective}.

\subsection{Supergradients of $f(w)$}\label{subsec:svm-sbp:minimax-optimality}

\input{svm/sbp/figures/fig-water}

Recall that, in Chapter \ref{ch:svm-introduction}, we defined the vector of
``responses'' $c \in \R^n$ to be, for a fixed $w$:
\begin{equation}
\label{eq:svm-sbp:responses-definition} c_{i} =
y_{i}\inner{w}{\Phi\left(x_{i}\right)}
\end{equation}
The objective of the max-min optimization problem in the definition of $f(w)$
can be written as $p^T(c+\xi)$. 
Supergradients of $f(w)$ at $w$ can be characterized explicitly in terms of
minimax-optimal pairs $p^*$ and $\xi^*$ such that $p^* =
\argmin_{p\in\Delta^n} p^t (c+\xi^*)$ and $\xi^* =
\argmax_{\xi\succeq 0, \mathbf{1}^{T}\xi\le n\nu} (p^*)^T (c+\xi)$.

\input{svm/sbp/theorems/lem-slack-constrained-supergradient}

This suggests a simple method for obtaining unbiased estimates of
supergradients of $f(w)$: sample a training index $i$ with probability
$p^{*}_{i}$, and take the stochastic supergradient to be
$y_{i}\Phi\left(x_{i}\right)$. The only remaining question is how one
finds a minimax optimal $p^{*}$.

For any $\xi$, a solution of $\min_{p \in \Delta^n} p^T(x+\xi)$ must put all of
the probability mass on those indices $i$ for which $c_{i}+\xi_{i}$ is
minimized. Hence, an optimal $\xi^{*}$ will maximize the minimal value of
$c_{i}+\xi^{*}_{i}$. This is illustrated in Figure \ref{fig:svm-sbp:water}. The
intuition is that the total mass $n\nu$ available to $\xi$ is distributed among
the indices as if this volume of water were poured into a basin with height
$c_{i}$. The result is that the indices $i$ with the lowest responses have
columns of water above them such that the common surface level of the water is
$\gamma$.

Once the ``water level'' $\gamma$ has been determined, we may find the
probability distribution $p$. Note first that, in order for $p$ to be optimal
at $\xi^{*}$, it must be supported on a subset of those indices $i$ for which
$c_{i}+\xi^{*}_{i}=\gamma$, since any such choice results in
$p^{T}\left(c+\xi^{*}\right)=\gamma$, while any choice supported on another set
of indices must have $p^{T}\left(c+\xi^{*}\right)>\gamma$.

However, merely being supported on this set is insufficient for minimax
optimality. If $i$ and $j$ are two indices with $\xi^{*}_{i},\xi^{*}_{j}>0$,
and $p_{i}>p_{j}$, then $p^{T}\left(x+\xi^{*}\right)$ could be made larger by
increasing $\xi^{*}_{i}$ and decreasing $\xi^{*}_{j}$ by the same amount.
Hence, we must have that $p^{*}$ takes on the constant value $q$ on all indices
$i$ for which $\xi^{*}_{i}>0$.
What about if $c_{i}=\gamma$ (and therefore $\xi^{*}_{i}=0$)? For such indices,
$\xi^{*}_{i}$ cannot be decreased any further, due to the nonnegativity
constraint on $\xi$, so we may have that $p_{i}<q$. If $p_{i}>q$, however, then
a similar argument to the above shows that $p$ is not minimax optimal.

The final characterization of minimax optimal probability distributions is that
$p^{*}_{i}\le q$ for all indices $i$ such that $c_{i}=\gamma$, and that
$p^{*}_{i}=q$ if $c_{i}<\gamma0$. This is illustrated in the lower portion of
Figure \ref{fig:svm-sbp:water}. In particular, the uniform distribution over all
indices such that $c_{i}<\gamma$ is minimax optimal.

\input{svm/sbp/figures/alg-water}
It is straightforward to find the water level $\gamma$ in linear time once the
responses $c_i$ are sorted (as in Figure \ref{fig:svm-sbp:water}), i.e.~with a total
runtime of $O(n \log n)$ due to sorting. It is also possible to find the water
level $\gamma$ in linear time, without sorting the responses, using a
divide-and-conquer algorithm (Algorithm \ref{alg:svm-sbp:water}).
This algorithm works by subdividing the set of responses into those less than,
equal to and greater than a pivot value (if one uses the median, which can be
found in linear time using e.g. the median-of-medians algorithm
\citep{BlumFlPrRiTa73}, then the overall will be linear in $n$). Then, it
calculates the size, minimum and sum of each of these subsets, from which the
total volume of the water required to cover the subsets can be easily
calculated. It then recurses into the subset containing the point at which a
volume of $n\nu$ just suffices to cover the responses, and continues until
$\gamma$ is found.

\subsection{Putting it Together}\label{subsec:svm-sbp:slack-constrained-sgd}

\input{svm/sbp/figures/alg-slack-constrained-sgd}
We are now ready to summarize the SBP algorithm. Starting from $w^{(0)}=0$ (so
both $\alpha^{(0)}$ and all responses are zero), each iteration proceeds as
follows:
\begin{enumerate}
\item Find $p^*$ by finding the ``water level'' $\gamma$ from the responses
(Section \ref{subsec:svm-sbp:minimax-optimality}), and taking $p^*$ to be uniform on
those indices for which $c_i\leq \gamma$.
\item Sample $j \sim p^*$.
\item Update $w^{(t+1)} \leftarrow
\mathcal{P}\left(w^{(t)}+\eta_{t}y_{j}\Phi\left(x_{j}\right)\right)$, where
$\mathcal{P}$ projects onto the unit ball and $\eta_{t} = \frac{1}{\sqrt{t}}$.
This is done by first increasing $\alpha \leftarrow\alpha + \eta_{t}$, updating
the responses $c$ accordingly, and projecting onto the set $\norm{w} \le 1$.
\end{enumerate}
Detailed pseudo-code may be found in Algorithm
\ref{alg:svm-sbp:slack-constrained-sgd}---observe that it is an instance of the
traditional SVM optimization outline of Section
\ref{sec:svm-introduction:traditional} of Chapter \ref{ch:svm-introduction}
(compare to Algorithm \ref{alg:svm-introduction:traditional-kernel}).  Updating
the responses requires $O(n)$ kernel evaluations (the most computationally
expensive part) and all other operations require $O(n)$ scalar arithmetic
operations.

Since at each iteration we are just updating using an unbiased
estimator of a supergradient, we can rely on the standard analysis of
stochastic gradient descent to bound the suboptimality after $T$
iterations: 

\input{svm/sbp/theorems/lem-zinkevich-batch}

Since each iteration is dominated by $n$ kernel evaluations, and thus takes
linear time (we take a kernel evaluation to require $O(1)$ time), the overall
runtime to achieve $\epsilon$ suboptimality for Problem
\ref{eq:svm-sbp:slack-constrained-objective} is $O(n/\epsilon^2)$.

\subsection{Learning Runtime}\label{subsec:svm-sbp:slack-constrained-runtime}

The previous section has given us the runtime for obtaining a certain
suboptimality of Problem \ref{eq:svm-sbp:slack-constrained-objective}. However, since
the suboptimality in this objective is not directly comparable to the
suboptimality of other scalarizations, e.g. Problem
\ref{eq:svm-introduction:regularized-objective}, we follow \citet{BottouBo07,ShalevSr08}, and
analyze the runtime required to achieve a desired generalization performance,
instead of that to achieve a certain optimization accuracy on the empirical
optimization problem.

Recall that our true learning objective is to find a predictor with low
generalization error $\loss[\zeroone]{g_w} = \probability[(x,y)]{ y\inner{w}{\Phi(x)} \leq
0 }$ with respect to some unknown distribution over $x,y$ based on a training
set drawn \iid from this distribution. We assume that there exists some
(unknown) predictor $u$ that has norm $\norm{u}$ and low expected hinge loss
$L^* = \loss[\hinge]{g_u} = \expectation{\ell(y\inner{u}{\Phi(x)})}$
(otherwise, there is no point in training a SVM), and analyze the runtime to
find a predictor $w$ with generalization error $\loss[\zeroone]{g_w} \leq
L^*+\epsilon$.

In order to understand the SBP runtime, we will follow \citet{HazanKoSr11} by
optimizing the empirical SVM bi-criterion Problem
\ref{eq:svm-introduction:bi-criterion-objective} such that:
\begin{align}
\label{eq:svm-sbp:slack-constrained-target} \norm{w} & \le 2\norm{u} ~~;~~ 
 \emploss[\hinge]{g_w} - \emploss[\hinge]{g_u} \le \epsilon/2
\end{align}
which suffices to ensure $\loss[\zeroone]{g_w} \leq L^*+\epsilon$ with high
probability.
Referring to Lemma \ref{lem:svm-sbp:slack-constrained-suboptimality}, Equation
\ref{eq:svm-sbp:slack-constrained-target} will be satisfied for
$\bar{w}/\gamma$ as long as $\bar{w}$ optimizes the objective of Problem
\ref{eq:svm-sbp:slack-constrained-objective} to within:
\begin{equation}
\label{eq:svm-sbp:reqbarepsilon} \bar{\epsilon} = \frac{\epsilon/2}{\norm{u}
(\emploss[\hinge]{g_u}+\epsilon/2)} \geq \Omega\left(\frac{\epsilon}{\norm{u}
(\emploss[\hinge]{g_u}+\epsilon)}\right)
\end{equation}
The following theorem performs this analysis, and combines it with a bound on
the required sample size from Chapter \ref{ch:svm-introduction} to yield a
generalization bound:

\medskip
\input{svm/sbp/theorems/thm-slack-constrained-runtime}

In the realizable case, where $L^*=0$, or more generally when we would like to
reach $L^*$ to within a small constant multiplicative factor, we have
$\epsilon=\Omega(L^*)$, the first factors in Equation
\ref{eq:svm-sbp:sbp-runtime} is a constant, and the runtime simplifies to
$\tilde{O}(\norm{u}^4/\epsilon)$. As we will see in Section
\ref{sec:svm-sbp:comparison}, this is a better guarantee than that enjoyed by
any other SVM optimization approach.

\subsection{Including an Unregularized Bias}\label{subsec:svm-sbp:unregularized-bias}

\input{svm/sbp/figures/fig-bias}
\input{svm/sbp/figures/alg-water-bias}
We have so far considered only homogeneous SVMs, without an unregularized bias
term.  It is possible to use the SBP to train SVMs with a bias term, i.e.~where
one seeks a predictor of the form $x \mapsto (\inner{w}{\Phi(x)}+b)$.
This is done by including the optimization over $b$ inside the objective
function $f(w)$, together with the optimization over the slack variables
$\xi$. That is, we now take stochastic gradient steps on:
\begin{equation}
\label{eq:svm-sbp:fw-definition-bias} f(w) = \max_{\begin{array}{c}
\scriptstyle b\in\R, \xi\succeq 0 \\ \scriptstyle \mathbf{1}^{T}\xi\le n\nu
\end{array}}\min_{p\in\Delta^{n}} \sum_{i=1}^{n} p_{i} \left(
y_{i}\inner{w}{\Phi(x_{i})} + y_{i} b + \xi_{i} \right)
\end{equation}
Lemma \ref{lem:svm-sbp:slack-constrained-supergradient} still holds, but we
must now find minimax optimal $p^*$,$\xi^*$ and $b^*$. This can be accomplished
using a modified ``water filling'' involving two basins, one containing the
positively-classified examples, and the other the negatively-classified ones.

As before, finding the water level $\gamma$ reduces to finding minimax-optimal
values of $p^{*}$, $\xi^{*}$ and $b^{*}$. The characterization of such
solutions is similar to that in the case without an unregularized bias.  In
particular, for a fixed value of $b$, we may still think about ``pouring water
into a basin'', except that the height of the basin is now $c_{i}+y_{i}b$,
rather than $c_{i}$.

When $b$ is not fixed it is easier to think of \emph{two} basins, one
containing the positive examples, and the other the negative examples. These
basins will be filled with water of a total volume of $n\nu$, to a common water
level $\gamma$. The relative heights of the two basins are determined by $b$:
increasing $b$ will raise the basin containing the positive examples, while
lowering that containing the negative examples by the same amount. This is
illustrated in Figure \ref{fig:svm-sbp:bias}.

It remains only to determine what characterizes a minimax-optimal value of $b$.
Let $k^{+}$ and $k^{-}$ be the number of elements covered by water in the
positive and negative basins, respectively, for some $b$. If $k^{+}>k^{-}$,
then raising the positive basin and lowering the negative basin by the same
amount (i.e.  increasing $b$) will raise the overall water level, showing that
$b$ is not optimal. Hence, for an optimal $b$, water must cover an equal number
of indices in each basin. Similar reasoning shows that an optimal $p^{*}$ must
place equal probability mass on each of the two classes.

Once more, the resulting problem is amenable to a divide-and-conquer approach.
The water level $\gamma$ and bias $b$ will be found in $O(n)$ time by Algorithm
\ref{alg:svm-sbp:water-bias}, provided that the $\code{partition}$ function
chooses the median as the pivot.

%% file: svm/sbp/figures/fig-water.tex
\begin{figure}

\begin{center}
\includegraphics[width=0.45\textwidth]{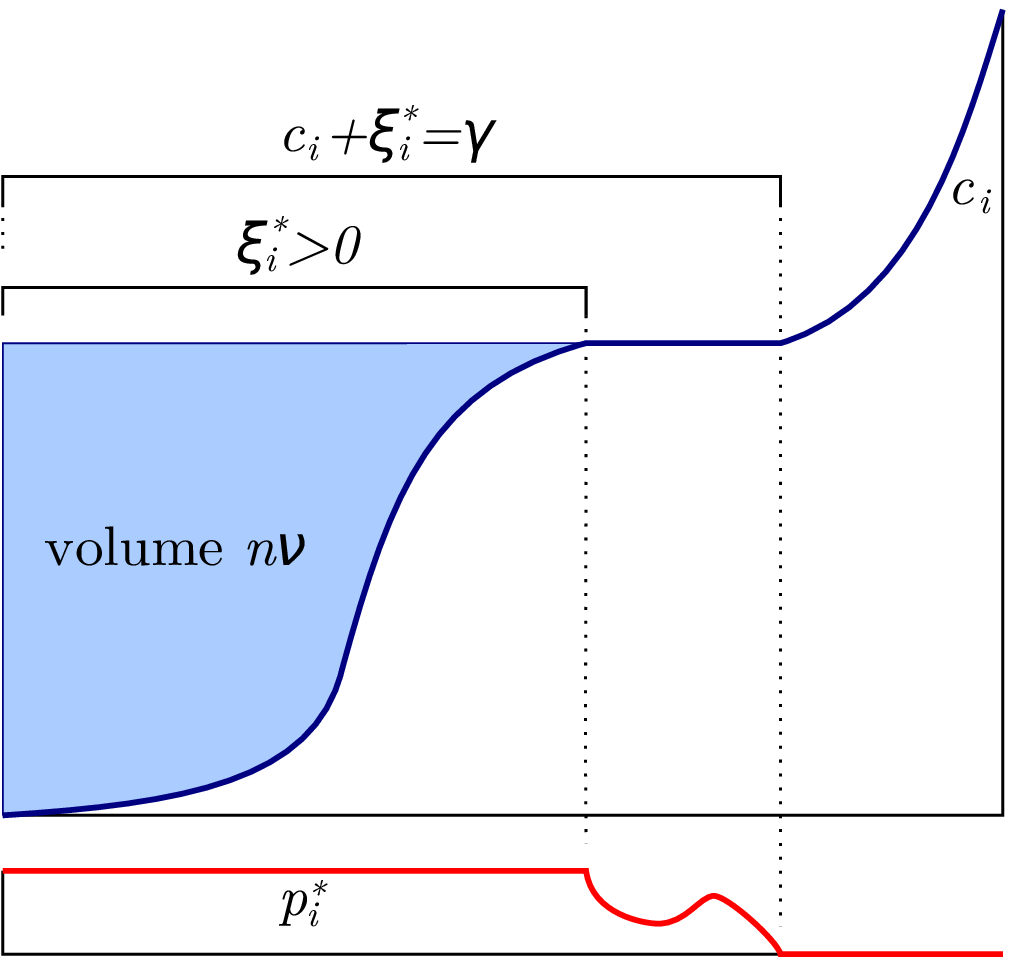}
\end{center}

\caption{
Illustration of how one finds $\xi^{*}$ and $p^{*}$. The upper curve
represents the values of the responses $c_{i}$, listed in order of increasing
magnitude.
The lower curve illustrates a minimax optimal probability distribution $p^{*}$.
%
}

\label{fig:svm-sbp:water}

\end{figure}

%% file: svm/sbp/theorems/lem-slack-constrained-supergradient.tex
\medskip
\begin{splitlemma}{lem:svm-sbp:slack-constrained-supergradient}

For any $w$, let $p^*,\xi^*$ be minimax optimal for Equation
\ref{eq:svm-sbp:fw-definition}. Then $\sum_{i=1}^{n}
p^{*}_{i}y_{i}\Phi\left(x_{i}\right)$ is a supergradient of $f(w)$ at $w$.

\end{splitlemma}
\begin{splitproof}

By the definition of $f$, for any $v\in\R^d$:
\begin{equation*}
f\left(w+v\right) =
\max_{\xi\succeq 0, \mathbf{1}^{T} \xi \le n\nu}\min_{p\in\Delta^{n}}
\sum_{i=1}^{n} p_{i} \left( y_{i}\inner{w+v}{\Phi\left(x_{i}\right)} + \xi_{i}
\right)
\end{equation*}
Substituting the particular value $p^{*}$ for $p$ can only increase the RHS,
so:
\begin{align*}
f\left(w+v\right) \le & \max_{\xi\succeq 0, \mathbf{1}^{T} \xi \le n\nu}
\sum_{i=1}^{n} p^{*}_{i} \left( y_{i}\inner{w+v}{\Phi\left(x_{i}\right)} +
\xi_{i} \right) \\
\le & \max_{\xi\succeq 0, \mathbf{1}^{T} \xi \le n\nu} \sum_{i=1}^{n} p^{*}_{i}
\left( y_{i}\inner{w}{\Phi\left(x_{i}\right)} + \xi_{i} \right)
+ \sum_{i=1}^{n} p^{*}_{i}y_{i}\inner{v}{\Phi\left(x_{i}\right)}
\end{align*}
Because $p^{*}$ is minimax-optimal at $w$:
\begin{align*}
f\left(w+v\right) \le & f\left(w\right) + \sum_{i=1}^{n}
p^{*}_{i}y_{i}\inner{v}{\Phi\left(x_{i}\right)} \\
\le & f\left(w\right) +
\inner{v}{\sum_{i=1}^{n}p^{*}_{i}y_{i}\Phi\left(x_{i}\right)}
\end{align*}
So $\sum_{i=1}^{n} p^{*}_{i}y_{i}\Phi\left(x_{i}\right)$ is a supergradient of
$f$.
\end{splitproof}

%% file: svm/sbp/figures/alg-water.tex
\begin{algorithm}[t]

\begin{pseudocode}
\codename $\code{find\_gamma}\left( C:\R^n, n\nu:\R \right)$\\
\codeline $low := 1$; $up := n$;\\
\codeline $low\_max := -\infty$; $low\_sum := 0$;\\
\codeline $\code{while } low < up$\\
\codeline \>$mid := \code{partition}( C\left[ low:up \right] )$;\\
\codeline \>$mid\_max := \max\left( low\_max, C\left[ low:\left( mid - 1 \right) \right] \right)$;\\
\codeline \>$mid\_sum := low\_sum + \sum C\left[ low:\left( mid - 1 \right) \right]$;\\
\codeline \>$\code{if } mid\_max \cdot \left( mid - 1 \right) - mid\_sum \ge n\nu \code{ then}$\\
\codeline \>\>$up := mid - 1$;\\
\codeline \>$else$\\
\codeline \>\>$low := mid$; $low\_max := mid\_ max$; $low\_sum := mid\_sum$;\\
\codeline $\code{return } \left( n\nu - low\_max \cdot \left( low - 1 \right) + low\_sum \right) / \left( low - 1 \right) + low\_max$;
\end{pseudocode}

\caption{
Divide-and-conquer algorithm for finding the ``water level'' $\gamma$ from an
array of responses $C$ and total volume $n\nu$. The $\code{partition}$ function
chooses a pivot value from the array it receives as an argument (the median
would be ideal), places all values less than the pivot at the start of the
array, all values greater at the end, and returns the index of the pivot in the
resulting array.
}

\label{alg:svm-sbp:water}

\end{algorithm}

%% file: svm/sbp/figures/alg-slack-constrained-sgd.tex
\begin{algorithm}[t]

\begin{pseudocode}
\codename $\code{optimize}\left( n:\N, d:\N, x_{1},\dots,x_{n}:\R^d, y_{1},\dots,y_{n}:\left\{\pm 1\right\}, T:\N, \nu:\R_{+}, K:\R^d\times\R^d\rightarrow\R_{+} \right)$\\
\codeline $\eta_{0} := 1 / \sqrt{\max_{i}K\left(x_{i},x_{i}\right)}$;\\
\codeline $\alpha^{(0)} := 0^{n}$; $c^{(0)} := 0^{n}$; $r_{0} := 0$;\\
\codeline $\code{for } t := 1 \code{ to } T$\\
\codeline \>$\eta_{t} := \eta_{0}/\sqrt{t}$;\\
\codeline \>$\gamma := \code{find\_gamma}\left( c^{(t-1)}, n\nu \right)$;\\
\codeline \>$\code{sample } i \sim \code{uniform} \left\{ j : c^{(t-1)}_{j} < \gamma \right\}$;\\
\codeline \>$\alpha^{(t)} := \alpha^{(t-1)} + \eta_{t} e_{i}$;\\
\codeline \>$r_{t}^2 := r_{t-1}^2 + 2 \eta_{t} c^{(t-1)}_{i} + \eta_{t}^{2} K\left(x_{i},x_{i}\right)$;\\
\codeline \>$\code{for } j=1 \code{ to } n$\\
\codeline \>\>$c^{(t)}_{j} := c^{(t-1)}_{j} + \eta_{t} y_{i} y_{j} K\left(x_{i},x_{j}\right)$;\\
\codeline \>$\code{if } \left( r_{t} > 1 \right) \code{ then}$\\
\codeline \>\>$\alpha^{(t)} := \left(1/r_{t}\right) \alpha^{(t)}$; $c^{(t)} := \left(1/r_{t}\right) c^{(t)}$; $r_{t} := 1$;\\
\codeline $\bar{\alpha} := \frac{1}{T} \sum_{t=1}^{T} \alpha^{(t)}$; $\bar{c} := \frac{1}{T} \sum_{t=1}^{T} c^{(t)}$; $\gamma := \code{find\_gamma}\left( \bar{c}, n\nu \right)$;\\
\codeline $\code{return } \bar{\alpha}/\gamma$;
\end{pseudocode}

\caption{
Stochastic gradient ascent algorithm for optimizing the kernelized version of
Problem \ref{eq:svm-sbp:slack-constrained-objective}. Here, $e_{i}$ is the
$i$th standard unit basis vector. The $\code{find\_gamma}$ subroutine finds the
``water level'' $\gamma$ from the vector of responses $c$ and total volume
$n\nu$.
}

\label{alg:svm-sbp:slack-constrained-sgd}

\end{algorithm}

%% file: svm/sbp/theorems/lem-zinkevich-batch.tex
\medskip
\begin{splitlemma}{lem:svm-sbp:zinkevich-batch}

For any $T,\delta>0$, after $T$ iterations of the Stochastic Batch
Perceptron, with probability at least $1-\delta$, the average iterate
$\bar{w} = \frac{1}{T}\sum_{t=1}^{T} w^{(t)}$ (corresponding to
$\bar{\alpha} = \frac{1}{T} \sum_{t=1}^{T} \alpha^{(t)}$), satisfies:
%
%
$
f\left(\bar{w}\right) \geq 
\sup_{\norm{w}\leq 1} f\left(w \right) -
O\left(\sqrt{\frac{1}{T}\log\frac{1}{\delta}}\right).
$
%

\end{splitlemma}
\begin{splitproof}

Define $h=-\frac{1}{r}f$, where $f$ is as in Equation \ref{eq:svm-sbp:fw-definition}.
Then the stated update rules constitute an instance of Zinkevich's algorithm,
in which steps are taken in the direction of stochastic subgradients $g^{(t)}$
of $h$ at $w^{(t)}=\sum_{i=1}^{n}\alpha_{i}y_{i}\Phi\left(x_{i}\right)$.

The claimed result follows directly from \citet[Theorem 1]{Zinkevich03}
combined with an online-to-batch conversion analysis in the style of
\citet[Lemma 1]{CesaCoGe01}.
\end{splitproof}

%% file: svm/sbp/theorems/thm-slack-constrained-runtime.tex
\medskip
\begin{splittheorem}{thm:svm-sbp:slack-constrained-runtime}

Let $u$ be an arbitrary linear classifier in the RKHS and let $\epsilon>0$ be
given. There exist values of the training size $n$, iteration count $T$ and parameter
$\nu$ such that Algorithm \ref{alg:svm-sbp:slack-constrained-sgd} finds a solution $w =
\sum_{i=1}^{n} \alpha_{i}y_{i}\Phi\left(x_{i}\right)$ satisfying:
\begin{equation*}
\loss[\zeroone]{g_w} \le \loss[\hinge]{g_u} + \epsilon
\end{equation*}
where $\loss[\zeroone]{g_w}$ and $\loss[\hinge]{g_u}$ are the expected 0/1 and hinge
losses, respectively, after performing the following number of kernel
evaluations:
\begin{equation}
\ifproofsection
\notag
\else
\label{eq:svm-sbp:sbp-runtime}
\fi
\mbox{\#K} = \tilde{O}\left( \left( \frac{\loss[\hinge]{g_u} +
\epsilon}{\epsilon} \right)^{3} \frac{ \norm{u}^{4} }{\epsilon}
\log^{2}\frac{1}{\delta} \right)
\end{equation}
with the size of the support set of $w$ (the number nonzero elements in
$\alpha$) satisfying:
\begin{equation}
\ifproofsection
\notag
\else
\label{eq:svm-sbp:sbp-support}
\fi
\mbox{\#S} = O\left( \left( \frac{\loss[\hinge]{g_u} +
\epsilon}{\epsilon} \right)^{2} \norm{u}^{2} \log\frac{1}{\delta} \right)
\end{equation}
the above statements holding with probability $1-\delta$.

\end{splittheorem}
\begin{splitproof}

For a training set of size $n$, where:
\begin{equation*}
n = \tilde{O}\left( \left( \frac{\loss[\hinge]{g_u} +
\epsilon}{\epsilon} \right) \frac{ B^{2} }{\epsilon} \log\frac{1}{\delta}
\right)
\end{equation*}
taking $B=2\norm{u}$ in Lemma \ref{lem:svm-introduction:generalization-from-expected-loss} gives
that $\emploss[\hinge]{g_u} \le \loss[\hinge]{g_u} + \epsilon$
and $\loss[\zeroone]{g_w} \le \loss[\hinge]{g_u} + 2\epsilon$
with probability $1-\delta$ over the training sample, uniformly for all linear
classifiers $w$ such that $\norm{w} \le B$ and $\emploss[\hinge]{g_w}
- \emploss[\hinge]{g_ u } \le \epsilon$. We will now show that these inequalities are
satisfied by the result of Algorithm \ref{alg:svm-sbp:slack-constrained-sgd}. Define:
\begin{equation*}
\hat{w}^{*} = \underset{w:\norm{w}\le\norm{u}}{\argmin} \emploss[\hinge]{g_w}
\end{equation*}
Because $\hat{w}^{*}$ is a Pareto optimal solution of the bi-criterion
objective of Problem \ref{eq:svm-introduction:bi-criterion-objective}, if we choose the
parameter $\nu$ to the slack-constrained objective (Problem
\ref{eq:svm-sbp:slack-constrained-objective}) such that $\norm{\hat{w}^{*}}\nu =
\emploss[\hinge]{g_{\hat{w}^*}}$, then the optimum of the
slack-constrained objective will be equivalent to $\hat{w}^{*}$ (Lemma
\ref{lem:svm-sbp:slack-constrained-suboptimality}). As was discussed in Section
\ref{subsec:svm-sbp:slack-constrained-runtime}, We will use Lemma
\ref{lem:svm-sbp:zinkevich-batch} to find the number of iterations $T$ required to
satisfy Equation \ref{eq:svm-sbp:reqbarepsilon} (with $u=\hat{w}^{*}$).  This yields
that, if we perform $T$ iterations of Algorithm
\ref{alg:svm-sbp:slack-constrained-sgd}, where $T$ satisfies the following:
\begin{equation}
\label{eq:svm-sbp:slack-constrained-time} T \ge O\left(
\left(\frac{\emploss[\hinge]{g_{\hat{w}^*}} + \epsilon }{\epsilon}\right)^2
\norm{\hat{w}^{*}}^2 \log\frac{1}{\delta} \right)
\end{equation}
then the resulting solution $w=\bar{w}/\gamma$ will satisfy:
\begin{align*}
\norm{w} & \le 2\norm{\hat{w}^{*}} \\
\notag \emploss[\hinge]{g_w} - \emploss[\hinge]{g_{\hat{w}^*}} & \le \epsilon
\end{align*}
with probability $1-\delta$. That is:
\begin{align*}
\norm{w} & \le 2 \norm{\hat{w}^{*}} \\
& \le B
\end{align*}
and:
\begin{align*}
\emploss[\hinge]{g_w} & \le \emploss[\hinge]{g_{\hat{w}^*}} + \epsilon \\
& \le \emploss[\hinge]{g_ u } + \epsilon
\end{align*}
These are precisely the bounds on $\norm{w}$ and $\emploss[\hinge]{g_w}$ which
we determined (at the start of the proof) to be necessary to permit us to apply
Lemma \ref{lem:svm-introduction:generalization-from-expected-loss}. Each of the
$T$ iterations requires $n$ kernel evaluations, so the product of the bounds on
$T$ and $n$ bounds the number of kernel evaluations (we may express Equation
\ref{eq:svm-sbp:slack-constrained-time} in terms of $\loss[\hinge]{g_u}$ and
$\norm{u}$ instead of $\emploss[\hinge]{g_{\hat{w}^*}}$ and $\norm{\hat{w}^*}$,
since $\emploss[\hinge]{g_{\hat{w}^*}} \le \emploss[\hinge]{g_u} \le
\loss[\hinge]{g_u}+\epsilon$ and $\norm{\hat{w}^*} \le \norm{u}$).

Because each iteration will add at most one new element to the support set, the
size of the support set is bounded by the number of iterations, $T$.

This discussion has proved that we can achieve suboptimality $2\epsilon$ with
probability $1-2\delta$ with the given $\mbox{\#K}$ and $\mbox{\#S}$. Because
scaling $\epsilon$ and $\delta$ by $1/2$ only changes the resulting bounds by
constant factors, these results apply equally well for suboptimality $\epsilon$
with probability $1-\delta$.
\end{splitproof}

%% file: svm/sbp/figures/fig-bias.tex
\begin{figure}

\begin{center}
\includegraphics[width=0.9\textwidth]{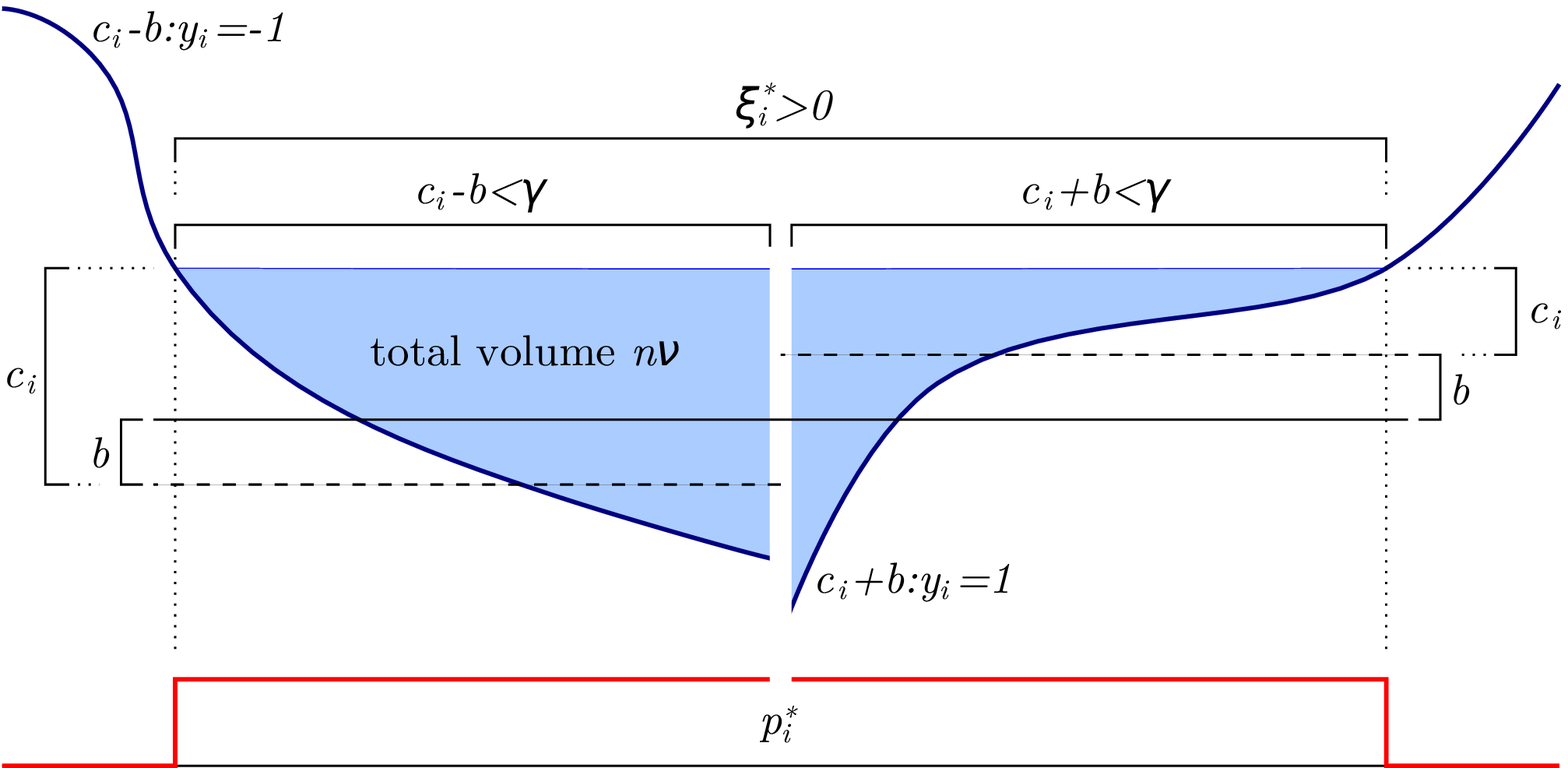}
\end{center}

\caption{
Illustration of how one finds the ``water level'' in a problem with an
unregularized bias. The two curves represent the heights of two basins of
heights $c_{i}-b$ and $c_{i}+b$, corresponding to the negative and positive
examples, respectively, with the bias $b$ determining the relative heights of
the basins. Optimizing over $\xi$ and $p$ corresponds to filling these two
basins with water of total volume $n\nu$ and common water level $\gamma$, while
optimizing $b$ corresponds to ensuring that water covers the same number of
indices in each basin.
}

\label{fig:svm-sbp:bias}

\end{figure}

%% file: svm/sbp/figures/alg-water-bias.tex
\begin{algorithm}[t]

\begin{pseudocode}
\codename $\code{find\_gamma\_and\_bias}\left( y:\left\{\pm1\right\}^n, C:\R^n, n\nu:\R \right)$\\
\codeline $C^+ := \{ C[i] : y[i] = +1 \}$; $n^+ := \abs{C^+}$; $low^+ := 1$; $up^+ := n^+$;\\
\codeline $C^- := \{ C[i] : y[i] = -1 \}$; $n^- := \abs{C^-}$; $low^- := 1$; $up^- := n^-$;\\
\codeline $low\_max^+ := -\infty$; $low\_sum^+ := 0$; $low\_max^- := -\infty$; $low\_sum^- := 0$;\\
\codeline $mid^+ := \code{partition}( C^+\left[ low^+:up^+ \right] )$;\\
\codeline $mid^- := \code{partition}( C^-\left[ low^-:up^- \right] )$;\\
\codeline $mid\_max^+ := \max\left( C\left[ low^+:\left( mid^+ - 1 \right) \right] \right)$; $mid\_sum^+ := \sum C\left[ low^+:\left( mid^+ - 1 \right) \right]$;\\
\codeline $mid\_max^- := \max\left( C\left[ low^-:\left( mid^- - 1 \right) \right] \right)$; $mid\_sum^- := \sum C\left[ low^-:\left( mid^- - 1 \right) \right]$;\\
\codeline $\code{while } \left( low^+ < up^+ \right) \code{ or } \left( low^- < up^- \right)$\\
\codeline \>$direction^+ := 0$; $direction^- := 0$;\\
\codeline \>$\code{if } mid^+ < low^- \code{ then } direction^+ = 1$;\\
\codeline \>$\code{else if } mid^+ > up^- \code{ then } direction^+ = -1$;\\
\codeline \>$\code{if } mid^- < low^+ \code{ then } direction^- = 1$;\\
\codeline \>$\code{else if } mid^- > up^+ \code{ then } direction^- = -1$;\\
\codeline \>$\code{if } direction^+ = direction^- = 0 \code{ then}$\\
\codeline \>\>$volume^+ := mid\_max^+ \cdot \left( mid^+ - 1 \right) - mid\_sum^+$;\\
\codeline \>\>$volume^- := mid\_max^- \cdot \left( mid^- - 1 \right) - mid\_sum^-$;\\
\codeline \>\>$\code{if } volume^+ + volume^- \ge n\nu \code{ then}$\\
\codeline \>\>\>$\code{if } mid^+ > mid^- \code{ then } direction^+ = -1$;\\
\codeline \>\>\>$\code{else if } mid^- > mid^+ \code{ then } direction^- = -1$;\\
\codeline \>\>\>$\code{else if } up^+ - low^+ > up^- - low^- \code{ then } direction^+ = -1$;\\
\codeline \>\>\>$\code{else } direction^- = -1$;\\
\codeline \>\>$\code{else}$\\
\codeline \>\>\>$\code{if } mid^+ < mid^- \code{ then } direction^+ = 1$;\\
\codeline \>\>\>$\code{else if } mid^- < mid^+ \code{ then } direction^- = 1$;\\
\codeline \>\>\>$\code{else if } up^+ - low^+ > up^- - low^- \code{ then } direction^+ = 1$;\\
\codeline \>\>\>$\code{else } direction^- = 1$;\\
\codeline \>$\code{if } direction^+ \ne 0 \code{ then}$\\
\codeline \>\>$\code{if } direction^+ > 0 \code{ then } up^+ := mid^+ - 1$;\\
\codeline \>\>$\code{else } low^+ := mid^+$; $low\_max^+ := mid\_max^+$; $low\_sum^+ := mid\_sum^+$;\\
\codeline \>\>$mid^+ := \code{partition}( C^+\left[ low^+:up^+ \right] )$;\\
\codeline \>\>$mid\_max^+ := \max\left( low\_max^+, C\left[ low^+:\left( mid^+ - 1 \right) \right] \right)$;\\
\codeline \>\>$mid\_sum^+ := low\_sum^+ + \sum C\left[ low^+:\left( mid^+ - 1 \right) \right]$;\\
\codeline \>$\code{if } direction^- \ne 0 \code{ then}$\\
\codeline \>\>$\code{if } direction^- > 0 \code{ then } up^- := mid^- - 1$;\\
\codeline \>\>$\code{else } low^- := mid^-$; $low\_max^- := mid\_max^-$; $low\_sum^- := mid\_sum^-$;\\
\codeline \>\>$mid^- := \code{partition}( C^-\left[ low^-:up^- \right] )$;\\
\codeline \>\>$mid\_max^- := \max\left( low\_max^-, C\left[ low^-:\left( mid^- - 1 \right) \right] \right)$;\\
\codeline \>\>$mid\_sum^- := low\_sum^- + \sum C\left[ low^-:\left( mid^- - 1 \right) \right]$;\\
\codeline // at this point $low^+ = low^- = up^+ = up^-$\\
\codeline $\Delta\gamma := \left( n\nu + low\_sum^+ + low\_sum^- \right) / \left( low^+ - 1 \right) - low\_max^+ - low\_max^-$;\\
\codeline $\code{if } low^+ < n^+ \code{ then } \Delta\gamma^+ := \min\left( \Delta\gamma, C^+[ low^+ ] - low\_max^+ \right) \code{ else } \Delta\gamma^+ := \Delta\gamma$;\\
\codeline $\code{if } low^- < n^- \code{ then } \Delta\gamma^- := \min\left( \Delta\gamma, C^-[ low^- ] - low\_max^- \right) \code{ else } \Delta\gamma^- := \Delta\gamma$;\\
\codeline $\gamma^+ := low\_max^+ + 0.5 \cdot \left( \Delta\gamma + \Delta\gamma^+ -\Delta\gamma^- \right)$;\\
\codeline $\gamma^- := low\_max^- + 0.5 \cdot \left( \Delta\gamma - \Delta\gamma^+ +\Delta\gamma^- \right)$;\\
\codeline $\gamma := 0.5 \cdot \left( \gamma^+ + \gamma^- \right)$; $b := 0.5 \cdot \left( \gamma^- - \gamma^+ \right)$;\\
\codeline $\code{return } \left( \gamma, b \right)$;
\end{pseudocode}

\caption{
Divide-and-conquer algorithm for finding the ``water level'' $\gamma$ and bias
$b$ from an array of labels $y$, array of responses $C$ and total volume
$n\nu$, for a problem with an unregularized bias. The $\code{partition}$
function is as in Algorithm \ref{alg:svm-sbp:water}.
}

\label{alg:svm-sbp:water-bias}

\end{algorithm}

%% file: svm/sbp/sec-comparison.tex
\section{Relationship to Other Methods}\label{sec:svm-sbp:comparison}

\input{svm/sbp/figures/tab-bounds}

We will now discuss the relationship between the SBP and several other SVM
optimization approaches, highlighting similarities and key differences, and
comparing their performance guarantees.

In Table \ref{tab:svm-sbp:bounds}, we compare the best known upper bounds on
the number of kernel evaluations required to achieve 0/1 generalization error
which is within $\epsilon$ of the hinge loss achieved by the best predictor,
under the assumption that one makes optimal choices of the training set size
$n$, iteration count $T$, and other algorithm parameters (such as $\lambda$ or
$\nu$).

\subsection{Traditional Optimization Algorithms}

The SBP is an instance of the ``traditional'' SVM optimization algorithm
outline described in Section \ref{sec:svm-introduction:traditional} of Chapter
\ref{ch:svm-introduction}. As can be seen in Table \ref{tab:svm-sbp:bounds},
the upper bound on the amount of computation required to find a solution which
generalizes well is better for the SBP than for any other known traditional SVM
optimizer. In particular, when $\epsilon = \Omega\left(L^*\right)$, the SBP
converges at a $\nicefrac{1}{\epsilon}$ rate, as compared to the
$\nicefrac{1}{\epsilon^2}$ rate enjoyed by the best alternative SVM optimizer
(the dual decomposition approach).

\subsection{Perceptron}

While the Perceptron is an online learning algorithm, it can also be used for
obtaining guarantees on the generalization error using an online-to-batch
conversion (e.g. \citet{CesaCoGe01}), as was described in Section
\ref{subsec:svm-introduction:non-traditional-perceptron} of Chapter
\ref{ch:svm-introduction}. Although the generalization bound which we derive
for the Perceptron is of the same order as that for the SBP (see Table
\ref{tab:svm-sbp:bounds}), the Perceptron does {\em not} converge to a Pareto
optimal solution to the bi-criterion SVM Problem
\ref{eq:svm-introduction:bi-criterion-objective}, and therefore cannot be
considered a SVM optimization procedure.

Furthermore, the online Perceptron generalization analysis relies on an
online-to-batch conversion, and is therefore valid only for a \emph{single}
pass over the data. If we attempt to run the Perceptron for multiple passes,
then it might begin to overfit uncontrollably. Hence, when applied to a fixed
dataset, the online Perceptron will occasionally ``run out of data'', in that
performing a single pass over the dataset will result in a poor classifier,
while too many passes will result in overfitting. Although the worst-case
theoretical guarantee obtained after a single pass is indeed similar to that
for an optimum of the SVM objective, in practice an optimum of the empirical
SVM optimization problem does seem to have significantly better generalization
performance.

With this said, the lack of explicit regularization may also be regarded as an
advantage, because the online Perceptron is \emph{parameter free}. In order to
make a full accounting of the cost of optimizing a SVM on a dataset with which
one is unfamiliar, one should not only consider the cost of optimizing a
particular instance of the SVM objective, but also that of performing a
parameter search to find a good value of the regularization parameter. Because
no such search is necessary for the online Perceptron, it is likely to
sometimes be preferable to a SVM optimizer, particularly when training
examples are abundant.

\subsection{Random Projections}

The random Fourier projection approach of \citet{RahimiRe07} can be used to
transform a kernel SVM problem into an approximately-equivalent linear SVM.
This algorithm was described in Section
\ref{subsec:svm-introduction:non-traditional-random-projections} of Chapter
\ref{ch:svm-introduction}. Unlike the other methods considered, which rely only
on black-box kernel accesses, Rahimi and Recht's projection technique can only
be applied on a certain class of kernel functions (shift-invariant kernels), of
which the Gaussian kernel is a member.

For $d$-dimensional feature vectors, and using a Gaussian kernel with parameter
$\sigma^2$, Rahimi and Recht's approach is to sample $v_{1},\dots,v_{D}\in\R^d$
independently according to $v_{i}\sim\mathcal{N}\left(0,I\right)$, and then
define the mapping $\tilde{\Phi}:\R^d\rightarrow\R^{2D}$ as:
\begin{align*}
{\tilde{\Phi}\left(x\right)}_{2i} = & \frac{1}{\sqrt{D}} \cos\left(
\frac{1}{\sigma} \inner{v_{i}}{x} \right) \\
{\tilde{\Phi}\left(x\right)}_{2i+1} = & \frac{1}{\sqrt{D}} \sin\left(
\frac{1}{\sigma} \inner{v_{i}}{x} \right)
\end{align*}
Then $\inner{\tilde{\Phi}\left(x_i\right)}{\tilde{\Phi}\left(x_j\right)} \approx
K\left(x_i,x_j\right)$, with the quality of this approximation improving with
increasing $D$.

Notice that computing each pair of Fourier features requires computing the
$d$-dimensional inner product $\inner{v}{x}$. For comparison, let us write the
Gaussian kernel in the following form:
\begin{align*}
K\left(x_i,x_j\right) =& \exp\left( -\frac{1}{2\sigma^{2}} \norm{x_i - x_j}^2
\right) \\ =& \exp\left( -\frac{1}{2\sigma^{2}} \left( \norm{x_i}^2 +
\norm{x_j}^2 - 2\inner{x_i}{x_j} \right) \right)
\end{align*}
The norms $\norm{x_i}$ may be cheaply precomputed, so the dominant cost of
performing a single Gaussian kernel evaluation is, likewise, that of the
$d$-dimensional inner product $\inner{x_i}{x_j}$.

This observation suggests that the computational cost of the use of Fourier
features may be directly compared with that of a kernel-evaluation-based SVM
optimizer in terms of $d$-dimensional inner products. The last row of Table
\ref{tab:svm-sbp:bounds} contains the upper bounds on the number of random
feature computations required to achieve $\epsilon$-sized generalization error,
and neglects the cost of optimizing the resulting linear SVM \emph{entirely}.
Despite this advantage, the upper bound on the performance of the SBP is still
far superior.

\subsection{SIMBA}

Recently, \citet{HazanKoSr11} presented SIMBA, a method for training {\em
linear} SVMs based on the same ``slack constrained'' scalarization (Problem
\ref{eq:svm-sbp:slack-constrained-objective}) we use here. SIMBA also fully optimizes
over the slack variables $\xi$ at each iteration, but differs in that, instead
of fully optimizing over the distribution $p$ (as the SBP does), SIMBA updates
$p$ using a stochastic mirror descent step. The predictor $w$ is then updated,
as in the SBP, using a random example drawn according to $p$. A SBP iteration
is thus in a sense more ``thorough'' then a SIMBA iteration. The SBP
theoretical guarantee
(Lemma \ref{lem:svm-sbp:zinkevich-batch}) is correspondingly better by a logarithmic
factor (compare to \citet[Theorem 4.3]{HazanKoSr11}). All else being equal, we
would prefer performing a SBP iteration over a SIMBA iteration.

For linear SVMs,
a SIMBA iteration can be performed in time $O(n+d)$. However, fully optimizing
$p$ as described in Section \ref{subsec:svm-sbp:minimax-optimality} requires the
responses $c_i$, and calculating or updating all $n$ responses would
require time $O(nd)$. In this setting, therefore, a SIMBA iteration is
much more efficient than a SBP iteration.

In the kernel setting, as was discussed in Section
\ref{sec:svm-introduction:traditional} of Chapter \ref{ch:svm-introduction},
calculating even a single response requires $O(n)$ kernel evaluation, which is
the same cost as updating \emph{all} responses after a change to a single
coordinate $\alpha_{i}$. This makes the responses essentially ``free'', and
gives an advantage to methods such as the SBP (and the dual decomposition
methods discussed below) which make use of the responses.

Although SIMBA is preferable for linear SVMs,
the SBP is preferable for kernelized SVMs.
It should also be noted that SIMBA relies heavily on having direct access to
features, and that it is therefore not obvious how to apply it directly in the
kernel setting.

%% file: svm/sbp/figures/tab-bounds.tex
\begin{table}

\begin{small}
\begin{center}
\begin{tabular}{l|cc}
\hline
& Overall & $\epsilon = \Omega\left( L^* \right)$ \\
\hline
SBP & $\left( \frac{L^* + \epsilon}{\epsilon} \right)^3 \frac{R^4}{\epsilon}$ & $\frac{R^4}{\epsilon}$ \\
SGD on Problem \ref{eq:svm-introduction:norm-constrained-objective} & $\left( \frac{L^* + \epsilon}{\epsilon} \right) \frac{R^4}{\epsilon^3}$ & $\frac{R^4}{\epsilon^3}$ \\
Dual Decomposition & $\left( \frac{L^* + \epsilon}{\epsilon} \right)^2 \frac{R^4}{\epsilon^2}$ & $\frac{R^4}{\epsilon^2}$  \\
Perceptron + Online-to-Batch & $\left( \frac{L^* + \epsilon}{\epsilon} \right)^3 \frac{R^4}{\epsilon}$ & $\frac{R^4}{\epsilon}$ \\
Random Fourier Features & $\left( \frac{L^* + \epsilon}{\epsilon} \right) \frac{d R^4}{\epsilon^3}$ & $\frac{d R^4}{\epsilon^3}$ \\
\hline
\end{tabular}
\end{center}
\end{small}

\caption{
Upper bounds, up to constant and log factors, on the runtime (number of kernel
evaluations, or random Fourier feature constructions in the last row) required
to achieve $\loss[\zeroone]{g_w}\leq L^* +\epsilon$, where $R$ bounds the norm
of a reference classifier achieving hinge loss $L^*$. See Chapter
\ref{ch:svm-introduction} (in particular Sections
\ref{sec:svm-introduction:traditional} and
\ref{sec:svm-introduction:non-traditional}, as well as Table
\ref{tab:svm-introduction:runtimes}) for derivations of the non-SBP bounds.
}

\label{tab:svm-sbp:bounds}

\end{table}

%% file: svm/sbp/sec-experiments.tex
\section{Experiments}\label{sec:svm-sbp:experiments}

\input{svm/sbp/figures/tab-datasets}

We compared the SBP to other ``traditional'' SVM optimization approaches on the
datasets in Table \ref{tab:svm-sbp:datasets}. We compared to Pegasos
\citep{ShalevSiSrCo10}, SDCA \citep{HsiehChLiKeSu08}, and SMO \citep{Platt98}
with a second order heuristic for working point selection \cite{FanChLi05}.
All of these algorithms were discussed in Section
\ref{sec:svm-introduction:traditional} Chapter \ref{ch:svm-introduction}. These
approaches work on the regularized formulation of Problem
\ref{eq:svm-introduction:regularized-objective} or its dual (Problem
\ref{eq:svm-introduction:dual-objective}). To enable comparison, the parameter
$\nu$ for the SBP was derived from $\lambda$ as
$\norm{\hat{w}^{*}}\nu=\frac{1}{n}\sum_{i=1}^{n}\ell\left(y_{i}\inner{w^{*}}{\Phi\left(x_{i}\right)}\right)$,
where $\hat{w}^*$ is the known (to us) optimum. 

\input{svm/sbp/figures/fig-experiments}
We first compared the methods on a SVM formulation \emph{without} an
unregularized bias, since Pegasos and SDCA do not naturally handle one. So that
this comparison would be implementation-independent, we measure performance in
terms of the number of kernel evaluations. As can be seen in Figure
\ref{fig:svm-sbp:experiments}, the SBP outperforms Pegasos and SDCA, as predicted by
the upper bounds. The SMO algorithm has a dramatically different performance
profile, in line with the known analysis: it makes relatively little progress,
in terms of generalization error, until it reaches a certain critical point,
after which it converges rapidly. Unlike the other methods, terminating SMO
early in order to obtain a cruder solution does not appear to be advisable.

To give a sense of actual runtime, we compared our implementation of the
SBP\footnote{Source code is available from
\url{http://ttic.uchicago.edu/~cotter/projects/SBP}} to the SVM package LibSVM,
running on an Intel E7500 processor. We allowed an unregularized bias (since
that is what LibSVM uses), and used the parameters in Table \ref{tab:svm-sbp:datasets}.
For these experiments, we replaced the Reuters dataset with the version of the
Forest dataset used by \citet{NguyenMaTaHa10}, using their parameters. LibSVM
(with default optimization parameters, except that ``shrinking'' was turned
off) converged to a solution with $14.9$\% error in $195$s on Adult, $0.44$\%
in $1980$s on MNIST, and $1.8$\% in $35$ hours on Forest. In \emph{one-quarter}
of each of these runtimes, SBP obtained $15.0$\% error on Adult, $0.46$\% on
MNIST, and $1.6$\% on Forest. These results of course depend heavily on the
specific stopping criterion used by LibSVM, and do not directly compare the
runtimes required to reach solutions which generalize well. We refer to the
experiments discussed above for a more controlled comparison. When comparing
these SBP results to those of our own SMO implementation, it appears to us that
the SBP converges to an acceptable solution more rapidly on Adult, but more
slowly on Reuters and MNIST.

\subsection{Perceptron}

\input{svm/sbp/figures/fig-experiments-perceptron}
We also compared to the online Perceptron algorithm. As can be seen in Figure
\ref{fig:svm-sbp:experiments-perceptron}, the Perceptron's generalization
performance is similar to that of the SBP for the first epoch (light purple
curve), but the SBP continues improving over additional passes.
Although use of the Perceptron is justified for non-separable data only if run
for a single pass over the training set, we did continue running for multiple
passes (dark purple curve). This is unsafe, since it might overfit after the
first epoch, an effect which is clearly visible on the Adult dataset.

\subsection{Random Projections}

\input{svm/sbp/figures/fig-experiments-fourier}
Our final set of experiments compares to the random Fourier projection approach
of \citet{RahimiRe07}. Figure \ref{fig:svm-sbp:experiments-fourier} compares
the computational cost of the use of Fourier features to that of the SBP in
terms of $d$-dimensional inner products. In this figure, the computational cost
of a $2k$-dimensional Fourier linearization is taken to be the cost of
computing $\mathcal{P}\left(x_i\right)$ on the entire training set ($kn$ inner
products, where $n$ is the number of training examples)---we ignore the cost of
optimizing the resulting linear SVM entirely.  The plotted testing error is
that of the \emph{optimum} of the resulting linear SVM problem, which
approximates the original kernel SVM. We can see that at least on Reuters and
MNIST, the SBP is preferable to (i.e. faster than) approximating the kernel
with random Fourier features.

%% file: svm/sbp/figures/tab-datasets.tex
\begin{table}

\begin{small}
\begin{center}
\begin{tabular}{lcc|ccc|ccc}
\hline
& & & \multicolumn{3}{c|}{Without unreg. bias} & \multicolumn{3}{c}{With unreg. bias} \\
Dataset & Train size $n$ & Test size & $\sigma^2$ & $\lambda$ & $\nu$ & $\sigma^2$ & $\lambda$ & $\nu$ \\
\hline
Reuters & $7770$   & $3229$  & $0.5$ & $\nicefrac{1}{n}$  & $6.34\times{10}^{-4}$ &        &                        &                        \\
Adult   & $31562$  & $16282$ & $10$  & $\nicefrac{1}{n}$  & $1.10\times{10}^{-2}$ & $100$  & $\nicefrac{1}{100n}$   & $5.79\times{10}^{-4}$  \\
MNIST   & $60000$  & $10000$ & $25$  & $\nicefrac{1}{n}$  & $2.21\times{10}^{-4}$ & $25$   & $\nicefrac{1}{1000n}$  & $6.42\times{10}^{-11}$ \\
Forest  & $522910$ & $58102$ &       &                    &                       & $5000$ & $\nicefrac{1}{10000n}$ & $7.62\times{10}^{-10}$ \\
\hline
\end{tabular}
\end{center}
\end{small}

\caption{
Datasets, downloaded from \url{http://leon.bottou.org/projects/lasvm}, and
parameters used in the experiments, in which we use a Gaussian kernel with
bandwidth $\sigma$. Reuters is also known as the "money\_fx" dataset. For the
multiclass MNIST dataset, we took the digit '8' to be the positive class, with
the other nine digits together being the negative class.
}

\label{tab:svm-sbp:datasets}

\end{table}

%% file: svm/sbp/figures/fig-experiments.tex
\begin{figure}

\begin{center}
\begin{tabular}{ @{} L @{} H @{} H @{} }
& \large{Reuters} & \large{Adult} \\
\rotatebox{90}{\scriptsize{Test error}} &
\includegraphics[width=0.45\textwidth]{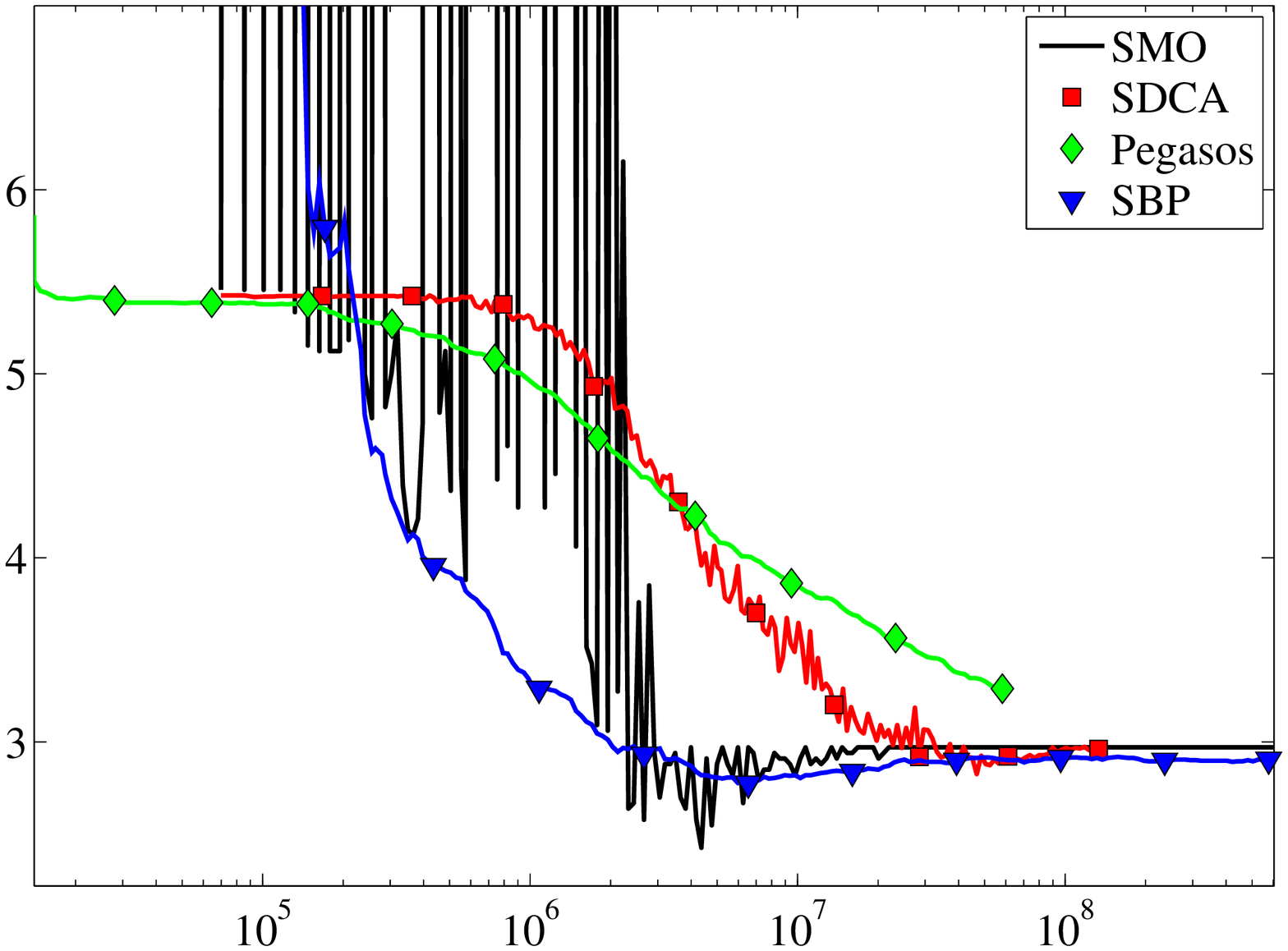} &
\includegraphics[width=0.45\textwidth]{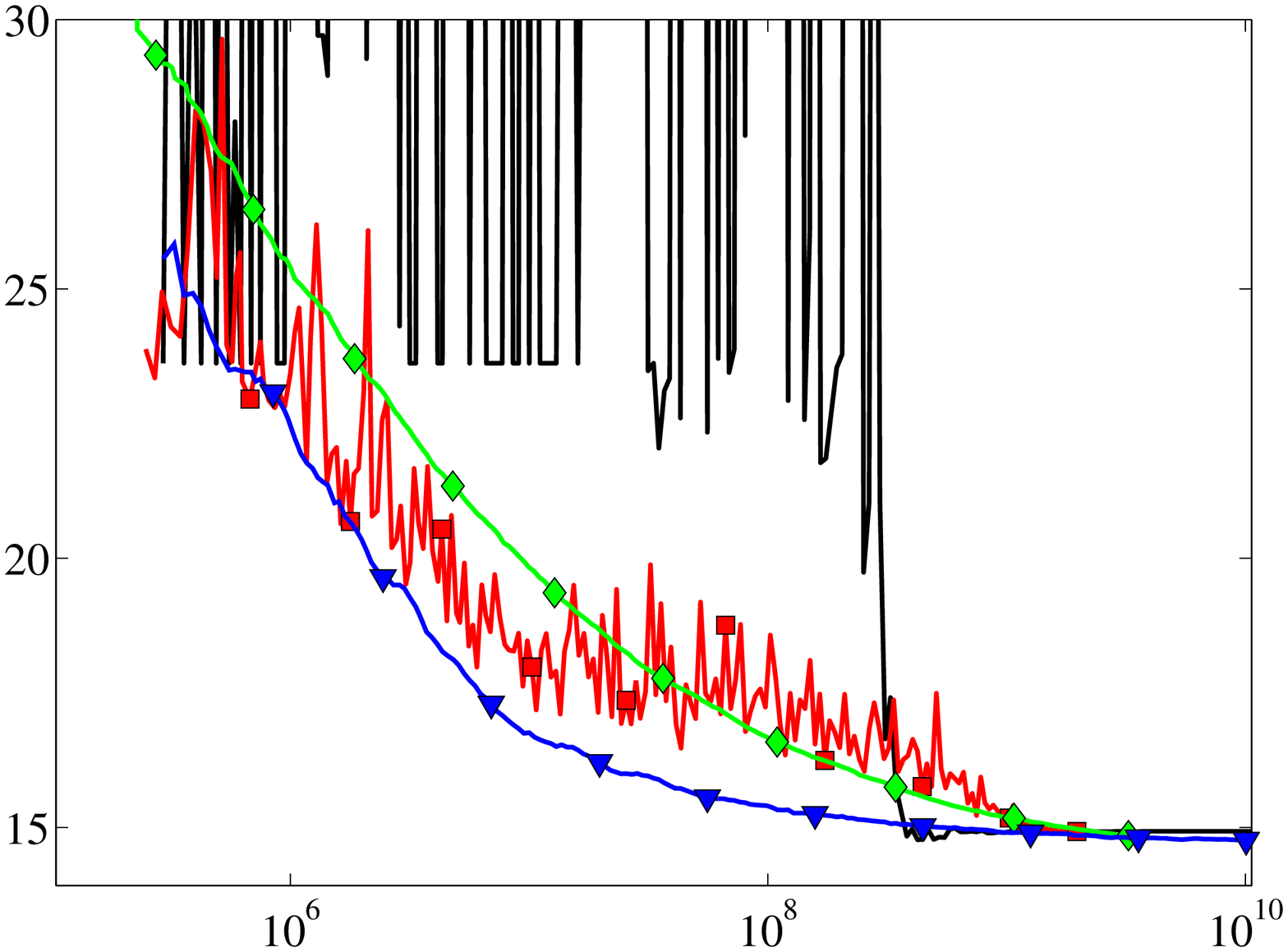} \\
& \scriptsize{Kernel Evaluations} & \scriptsize{Kernel Evaluations}
\end{tabular}
\begin{tabular}{ @{} L @{} H @{} }
& \large{MNIST} \\
\rotatebox{90}{\scriptsize{Test error}} &
\includegraphics[width=0.45\textwidth]{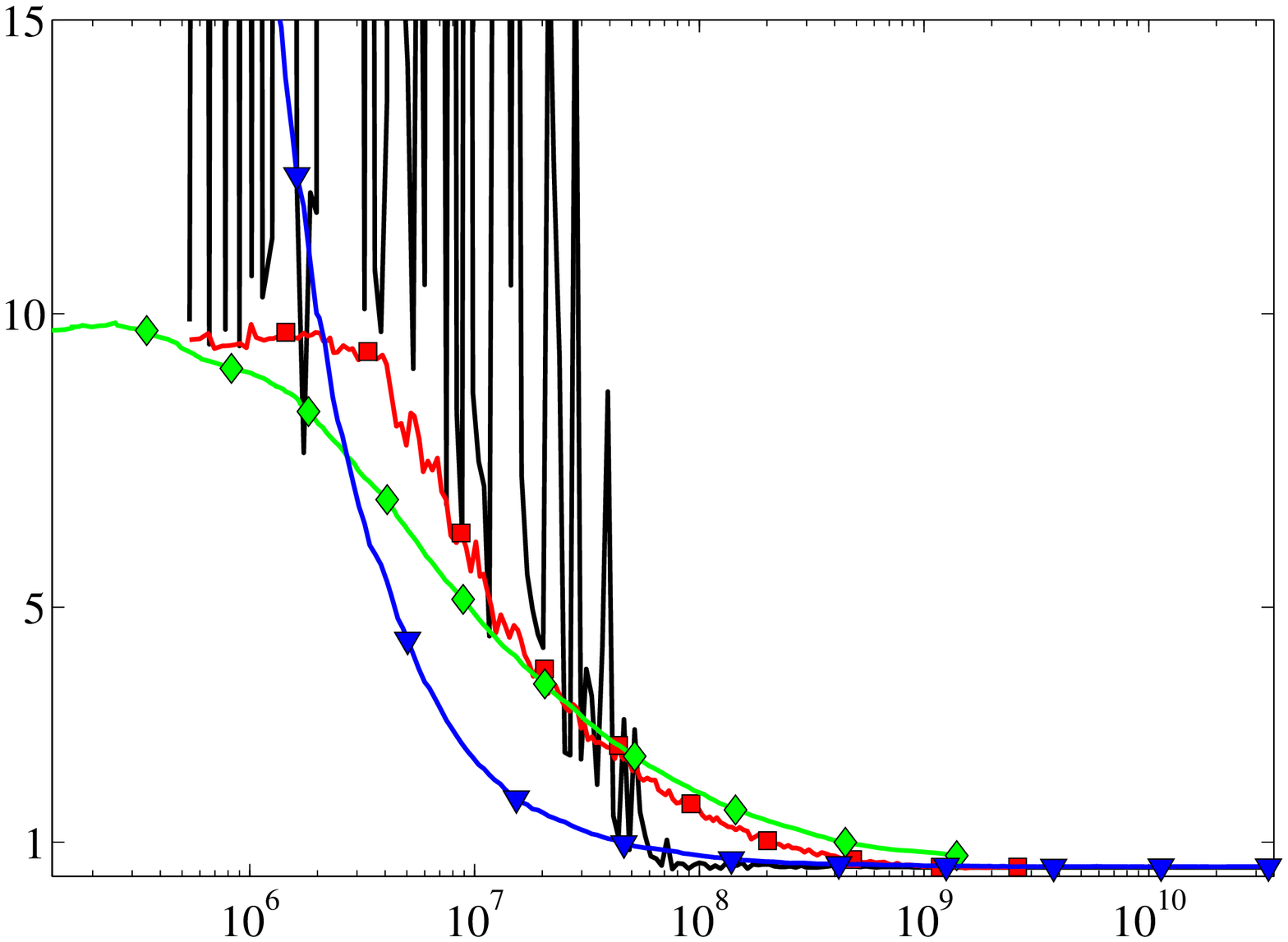} \\
& \scriptsize{Kernel Evaluations}
\end{tabular}
\end{center}

\caption{
Classification error on the held-out testing set (linear scale) vs. the number
of kernel evaluations performed during optimization (log scale), averaged over
ten runs All algorithms were run for ten epochs.
}

\label{fig:svm-sbp:experiments}

\end{figure}

%% file: svm/sbp/figures/fig-experiments-perceptron.tex
\begin{figure}

\begin{center}
\begin{tabular}{ @{} L @{} H @{} H @{} }
& \large{Reuters} & \large{Adult} \\
\rotatebox{90}{\scriptsize{Test error}} &
\includegraphics[width=0.45\textwidth]{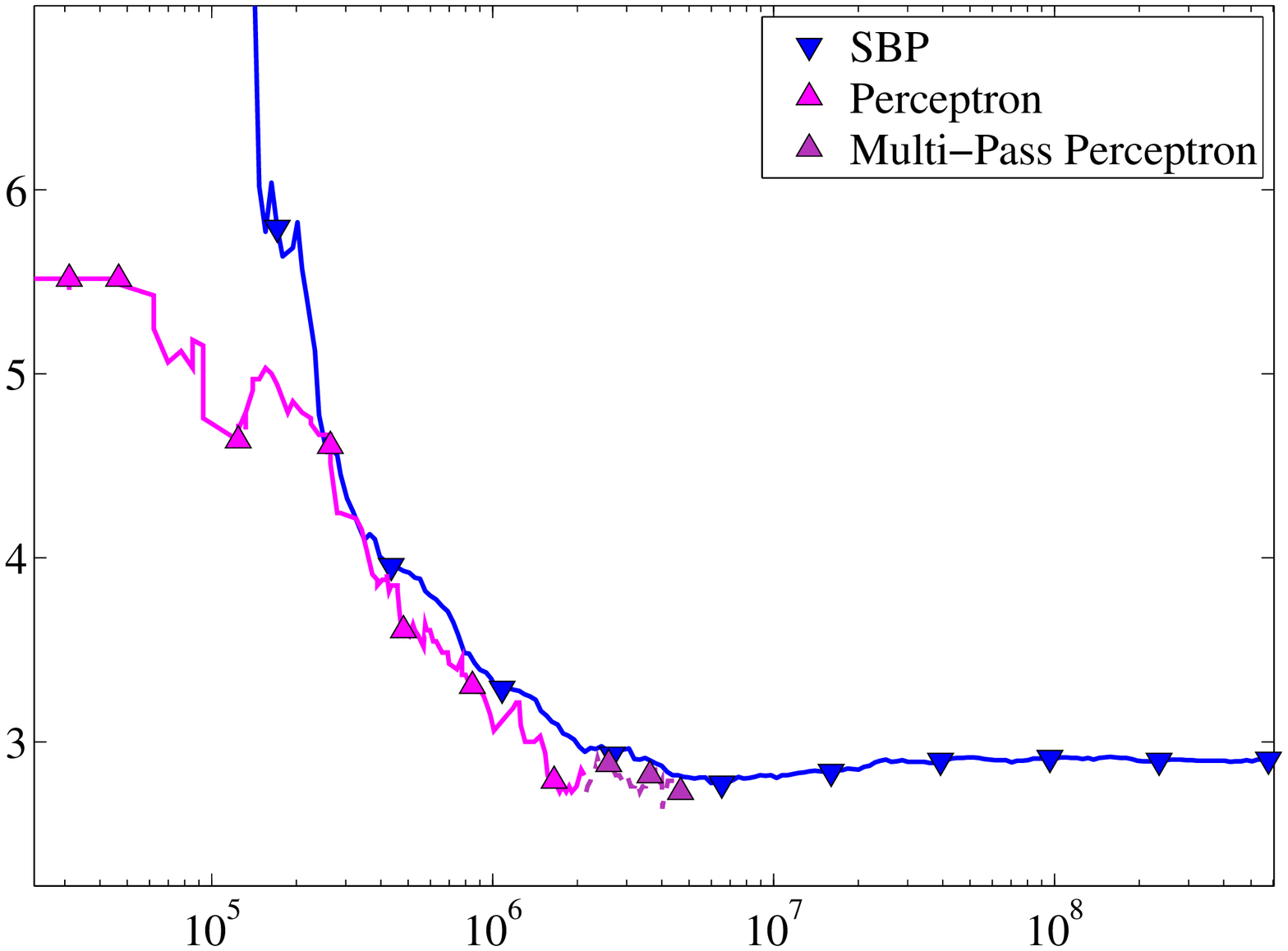} &
\includegraphics[width=0.45\textwidth]{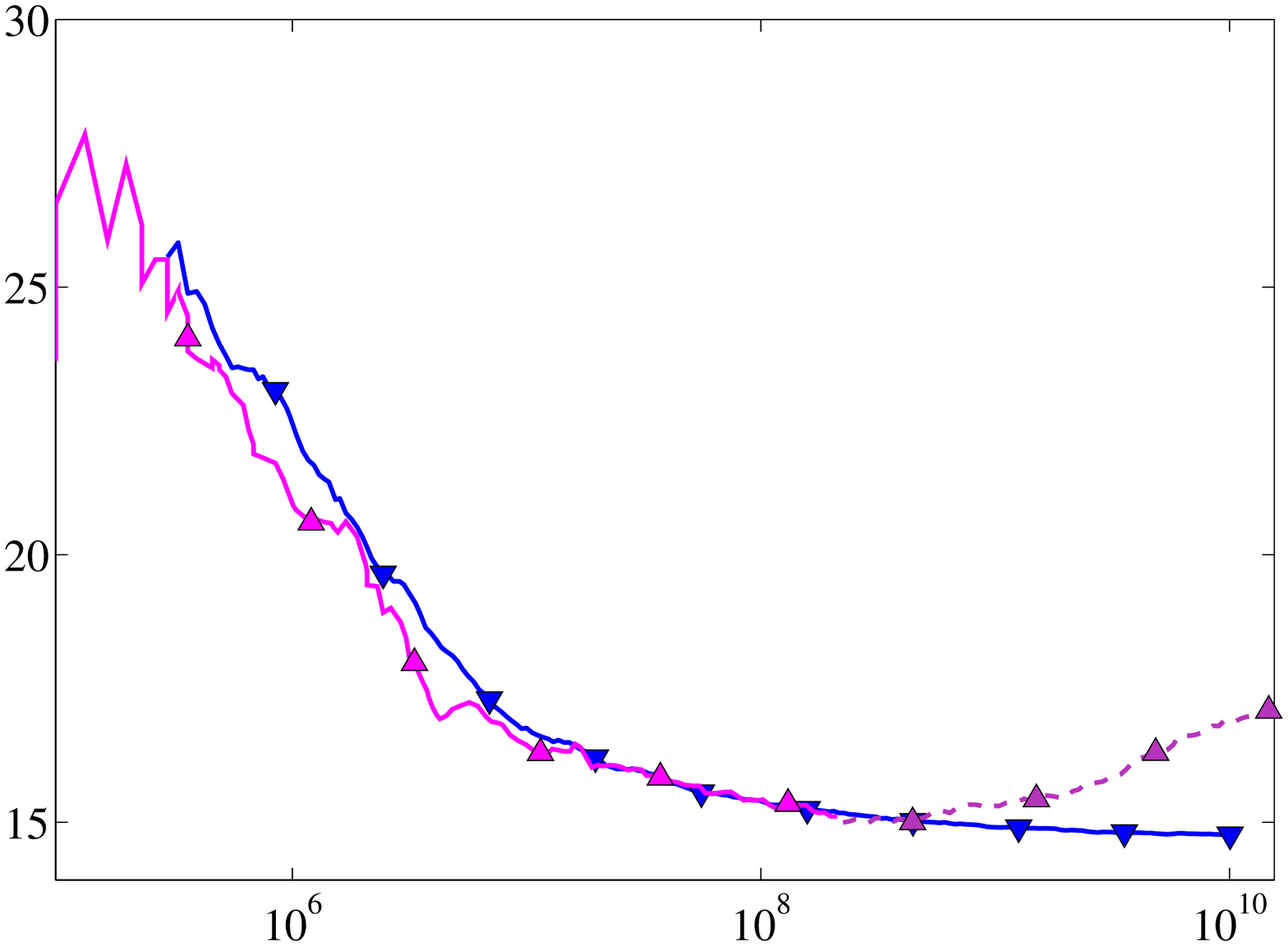} \\
& \scriptsize{Inner Products} & \scriptsize{Inner Products}
\end{tabular}
\begin{tabular}{ @{} L @{} H @{} }
& \large{MNIST} \\
\rotatebox{90}{\scriptsize{Test error}} &
\includegraphics[width=0.45\textwidth]{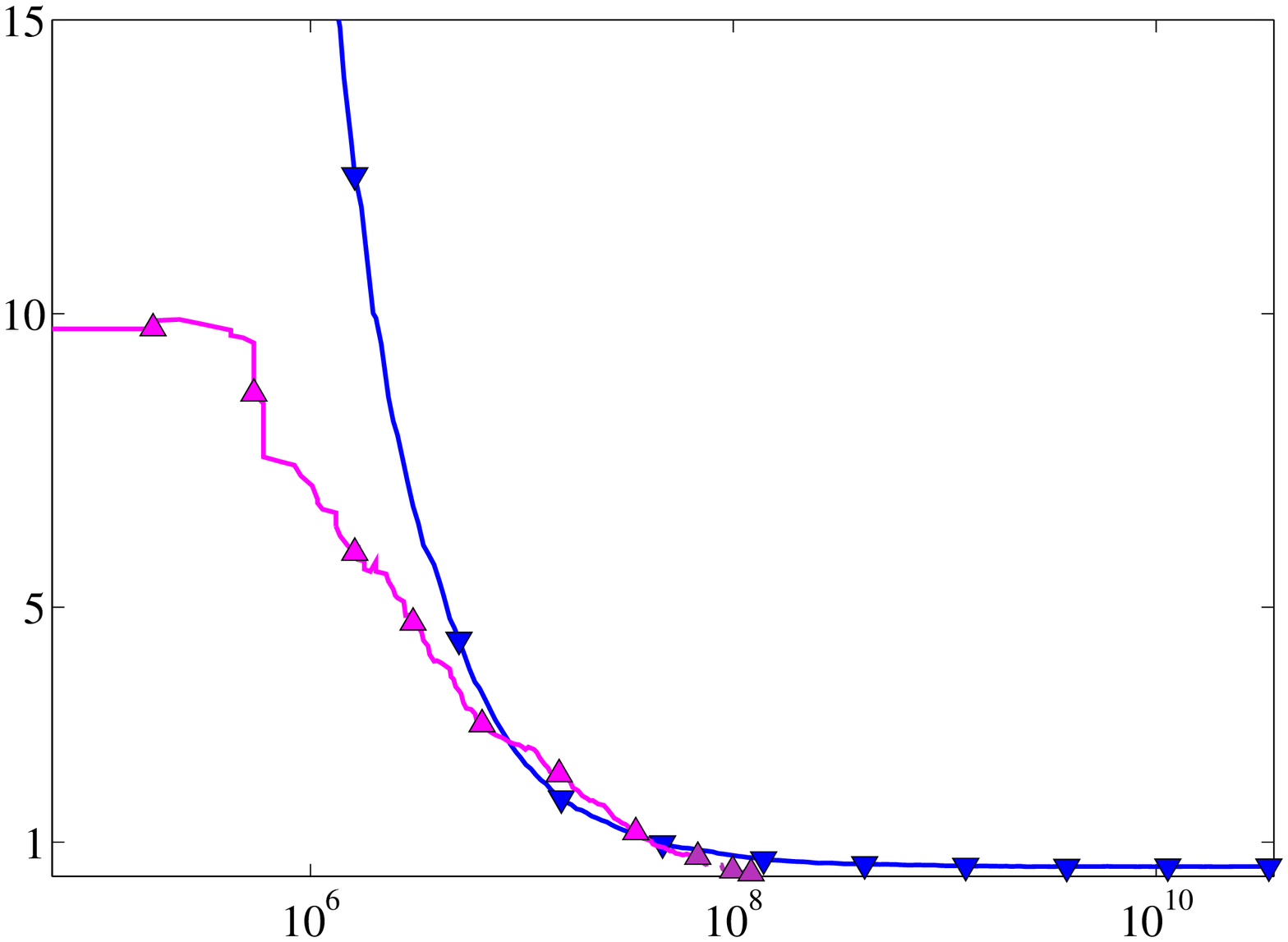} \\
& \scriptsize{Inner Products}
\end{tabular}
\end{center}

\caption{
Classification error on the held-out testing set (linear scale) vs. the number
of kernel evaluations performed during optimization (log scale), averaged over
ten runs. The Perceptron was run for multiple passes over the data---its curve
becomes dashed after the first epoch ($n$ iterations). All algorithms were run
for ten epochs, \emph{except} for Perceptron on Adult, which we ran for $100$
epochs to better illustrate its overfitting.
}

\label{fig:svm-sbp:experiments-perceptron}

\end{figure}

%% file: svm/sbp/figures/fig-experiments-fourier.tex
\begin{figure}

\begin{center}
\begin{tabular}{ @{} L @{} H @{} H @{} }
& \large{Reuters} & \large{Adult} \\
\rotatebox{90}{\scriptsize{Test error}} &
\includegraphics[width=0.45\textwidth]{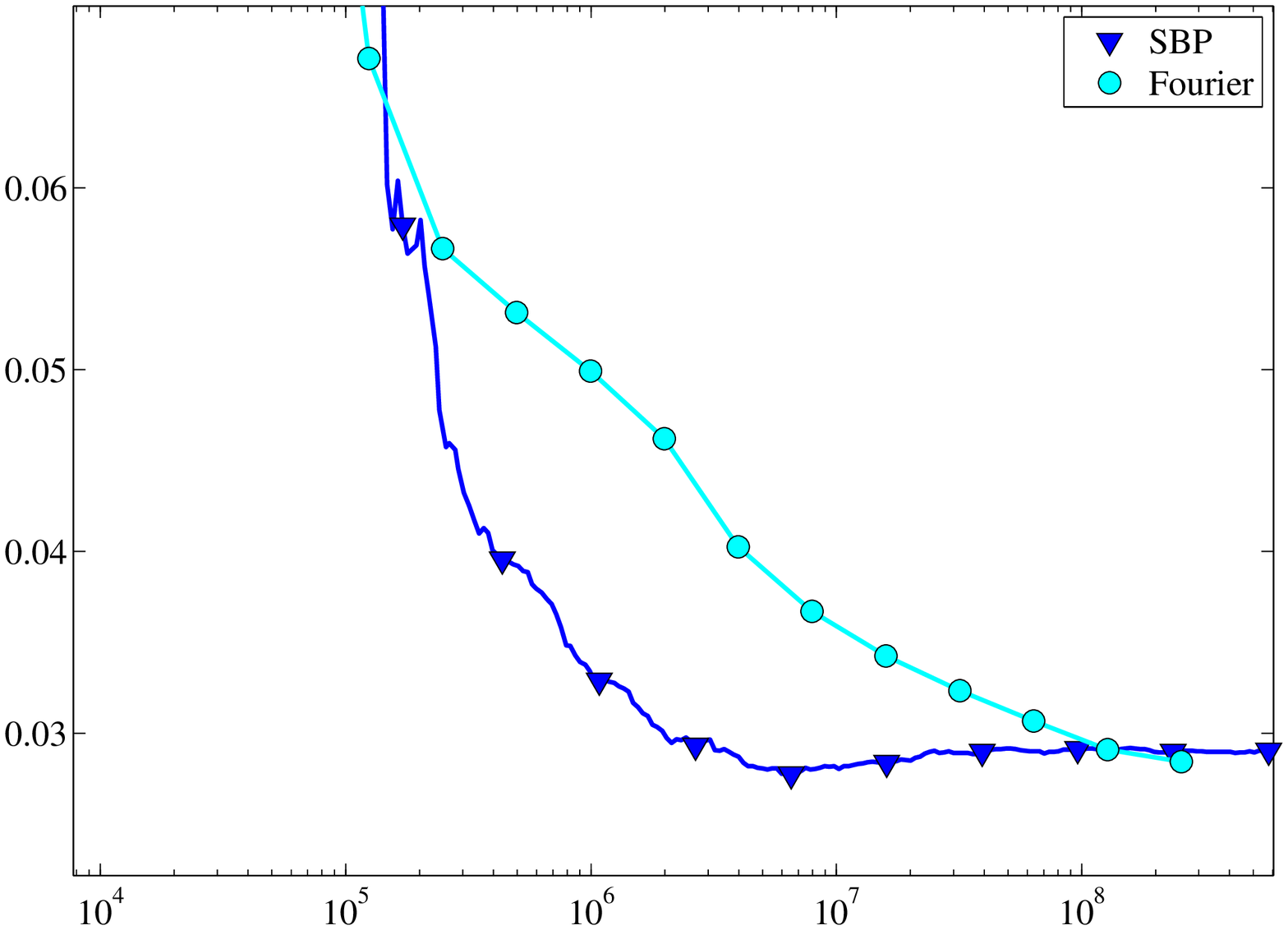} &
\includegraphics[width=0.45\textwidth]{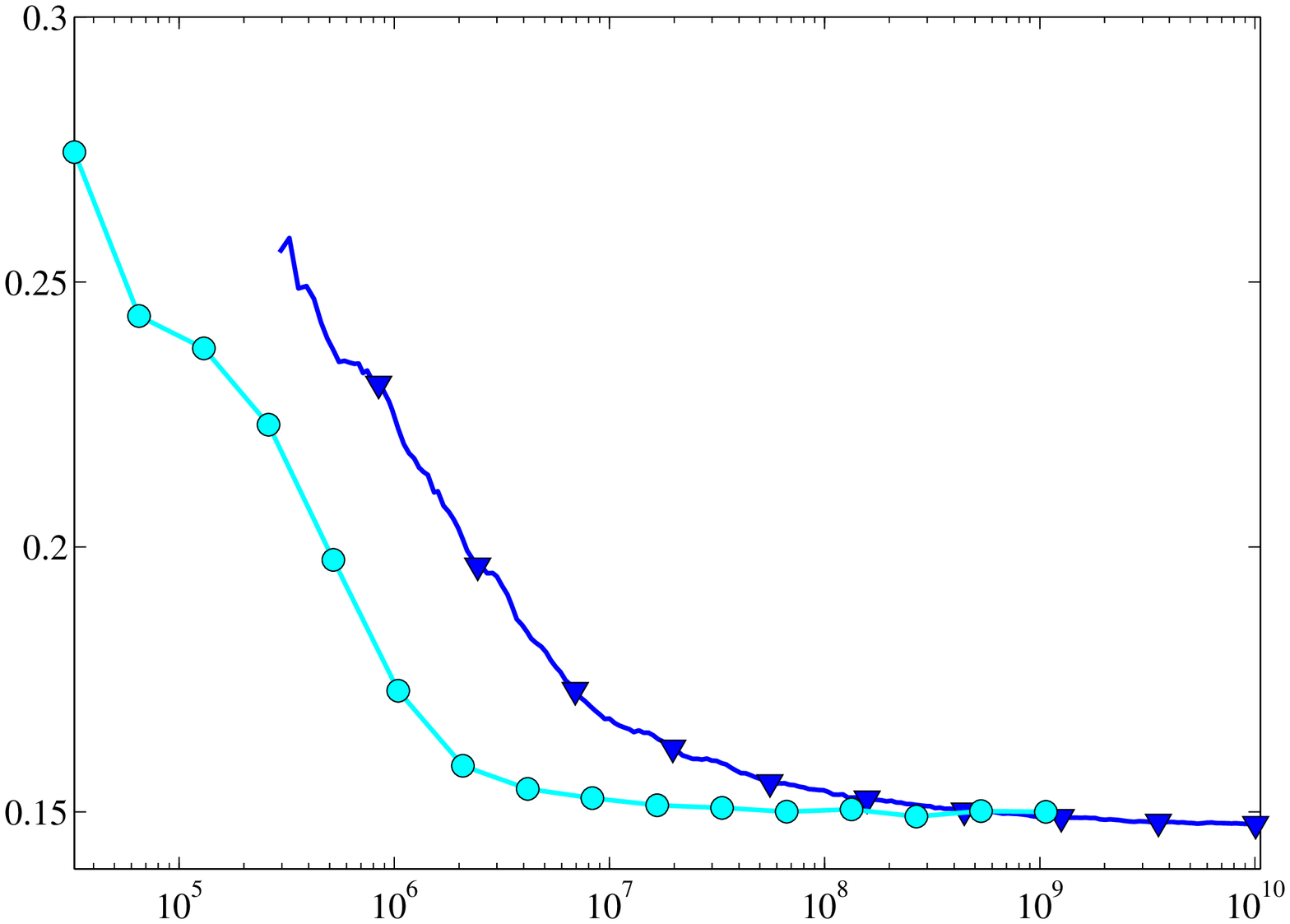} \\
& \scriptsize{Inner Products} & \scriptsize{Inner Products}
\end{tabular}
\begin{tabular}{ @{} L @{} H @{} }
& \large{MNIST} \\
\rotatebox{90}{\scriptsize{Test error}} &
\includegraphics[width=0.45\textwidth]{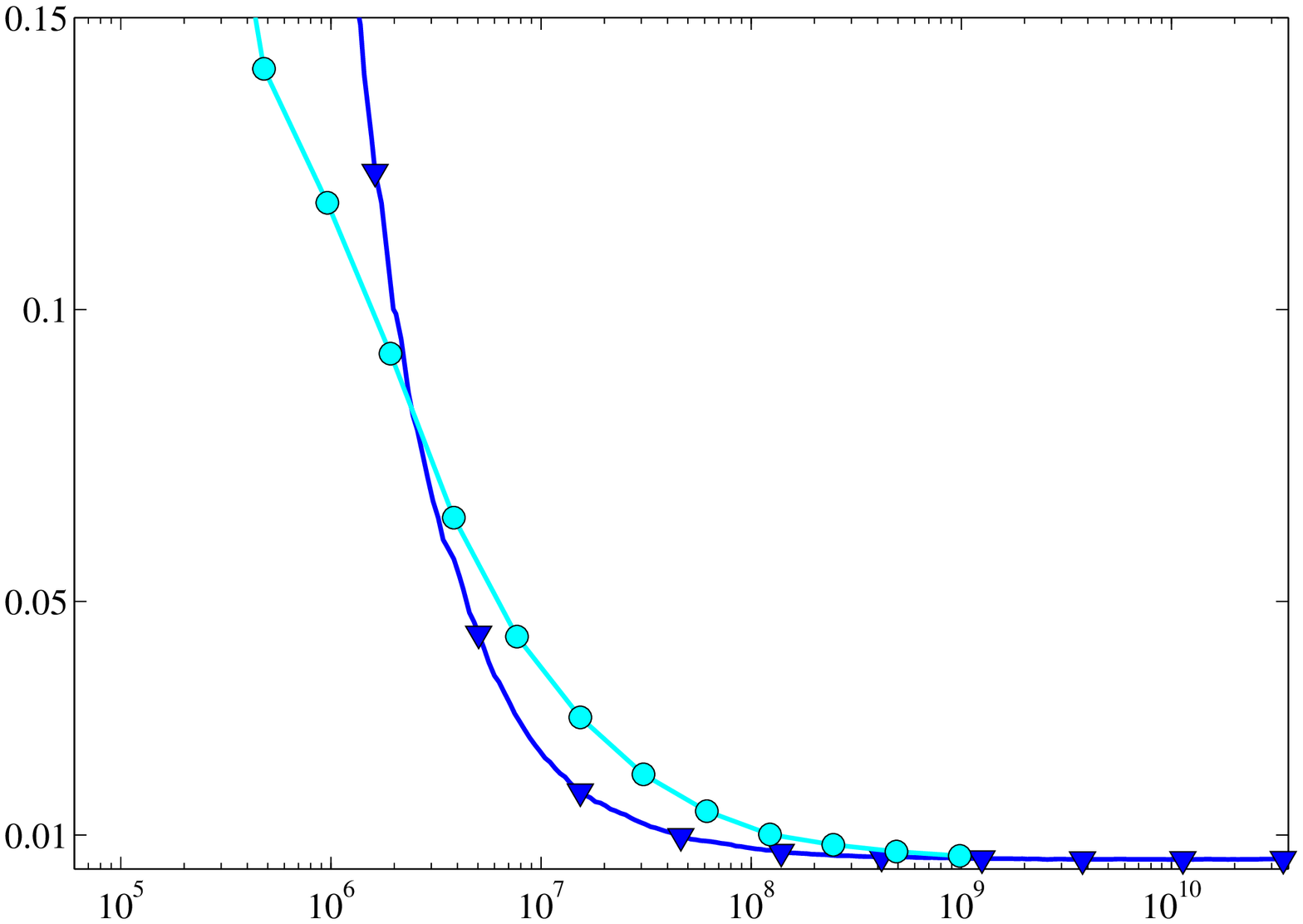} \\
& \scriptsize{Inner Products}
\end{tabular}
\end{center}

\caption{
Classification error on the held-out testing set (linear scale) vs.
computational cost measured in units of $d$-dimensional inner products (where
the training vectors satisfy $x\in\R^d$) (log scale), and averaged
over ten runs. For the Fourier features, the computational cost (horizontal
axis) is that of computing $k\in\left\{1,2,4,8,\dots\right\}$ pairs of Fourier
features over the entire training set, while the test error is that of the
\emph{optimal} classifier trained on the resulting linearized SVM objective.
}

\label{fig:svm-sbp:experiments-fourier}

\end{figure}

%% file: svm/sbp/sec-proofs.tex
\section{Proofs for Chapter \ref{ch:svm-sbp}}

\begin{proofs}
\input{svm/sbp/theorems/lem-slack-constrained-supergradient}
\input{svm/sbp/theorems/lem-zinkevich-batch}
\input{svm/sbp/theorems/thm-slack-constrained-runtime}
\end{proofs}

%% file: svm/ch-sparse.tex
\chapter{Learning Optimally Sparse Support Vector Machines}\label{ch:svm-sparse}

\input{svm/sparse/sec-overview}
\input{svm/sparse/sec-sparsity}

\input{svm/sparse/sec-algorithm}
\input{svm/sparse/sec-comparison}
\input{svm/sparse/sec-practical}
\input{svm/sparse/sec-experiments}

\paragraph{Collaborators:} The work presented in this chapter was performed
jointly with Shai Shalev-Shwartz and Nathan Srebro.

\clearpage
\input{svm/sparse/sec-proofs}

%% file: svm/sparse/sec-overview.tex
\section{Overview}\label{sec:svm-sparse:overview}

In the previous chapter, we presented a kernel SVM optimization algorithm with
the best known bound on its training runtime. In this chapter, we will consider
the related problem of testing runtime. The approach which we describe, which
was originally presented in the 29th International Conference on Machine
Learning (ICML 2013) \citep{CotterShSr13}, results in the best known bound on
the testing runtime of the learned predictor. This is possible because the
predictor can be expressed in terms of only a subset of the training points,
known as ``support vectors''.  The number of support vectors determines the
memory footprint of the learned predictor, as well as the computational cost of
using it (i.e.~of classifying query points).  In order for SVMs to be practical
in large scale applications, it is therefore important to have only a small
number of support vectors.
This is particularly true when, as is the case in many applications, the
training is done only once, on a powerful compute cluster that can handle large
data, but the predictor then needs to be used many times, possibly in real
time, perhaps on a small low-powered device.

However, when training a SVM in the non-separable setting, all
incorrectly classified points will be support vectors---e.g.~with 10\%
error, the solution of the SVM empirical optimization problem will
necessarily have at least 10\% of the training points as support vectors. For
data sets with many thousands or even millions of points, this results in very
large predictors that are expensive to store and use. Even for some separable
problems, the number of support vectors in the SVM solution (the
margin-maximizing classifier) may increase linearly with the training set size
(e.g.~when all the points are on the margin). And so, even though minimizing
the empirical SVM objective might be good in terms of generalization ability
(and this is very well studied), as we argue here, it might be bad in terms of
obtaining sparse solutions.

In this paper, we ask how sparse a predictor we can expect, and show how to
learn a SVM predictor with an optimal guarantee on its sparsity, without
increasing the required sample complexity nor (asymptotically) the training
runtime, and for which the worst-case generalization error has the same bound
as the best known for (non-sparse) kernel SVM optimization.

%% file: svm/sparse/sec-sparsity.tex
\section{Sparse SVMs}\label{sec:svm-sparse:sparsity}

In this work, we're interested in the problem of finding \emph{sparse} SVM
solutions. Hence, the question is whether, given that there exists a good
reference classifier $\refw$ which does not necessarily have a small support
size, it is possible to efficiently find a $w$ based on a training sample which
not only generalizes well, but \emph{also} has a small support set.

If the reference classifier $\refw$ separates the data with a margin, namely
$\loss[\hinge]{g_{\refw}} = 0$,
then one can run the kernelized Perceptron algorithm (see for example
\citet{FreundSc99}).
%
%
The Perceptron processes the training examples one by one and adds a support
vector only when it makes a prediction mistake. Therefore, a bound on the
number of prediction mistakes (i.e. a mistake bound) translates to a bound on
the sparsity of the learned predictor. A classic result (the mistake bound used
to prove Lemma \ref{thm:svm-sbp:perceptron-runtime} in Chapter
\ref{ch:svm-introduction}) shows that if the data is separable with a margin
$1$ by some vector $\refw$, then the Perceptron will make at most
$\norm{\refw}^2$ prediction mistakes. Combining this with a generalization
bound based on compression of representation (or with an online-to-batch
conversion) we can conclude that with $n \ge
\tilde{O}(\norm{\refw}^2/\epsilon)$, the generalization error of $w$ will be at
most $\epsilon$.

The non-separable case is more tricky. If we somehow obtained a vector $v$
which makes a small number of \emph{margin} violations on the training set,
i.e. $\epsilon_{v} = \frac{1}{n} \sum_{i=1}^n \mathbf{1}( y_i
\inner{v}{\Phi(x_i)} < 1)$ is small, then we can find a $w$ with $\norm{v}^2$
support vectors which satisfies $\emploss[\zeroone]{g_w} \le \epsilon_{v}$ by simply
ignoring the examples on which $y_i \inner{v}{\Phi(x_i)} < 1$ and running the
Perceptron on the remainder. Again using a compression bound, we can show that
$\loss[\zeroone]{g_w}$ is little larger than $\emploss[\zeroone]{g_w}$.

However, we cannot, in general, efficiently find a predictor $v$ with a low
margin error rate, even if we know that such a predictor exists.
%
%
Instead, in learning SVMs, we minimize the empirical \emph{hinge} loss. It is
not clear how to relate the margin error rate $\epsilon_{v}$ to the hinge loss of
$\refw$. One option is to note that $\mathbf{1}(z < 1) \le 2[1 - z/2]_+$, hence
$\epsilon_v \le 2\emploss[\hinge]{g_{v/2}}$. Since $2\emploss[\hinge]{g_{v/2}}$
is a convex function of $v$, we can minimize it as a convex surrogate to
$\epsilon_v$. Unfortunately, this approach would lead to a dependence on the
quantity $2\loss[\hinge]{g_{\refw/2}}$, which might be significantly larger than
$\loss[\hinge]{g_{\refw}}$. This is the main issue we address in Section
\ref{sec:svm-sparse:algorithm}, in which we show how to construct an efficient
sparsification procedure which depends on $\loss[\hinge]{g_{\refw}}$ and has the
same error guarantees and sample complexity as the vanilla SVM predictor.

Before moving on, let us ask whether we could hope for sparsity {\em less} then
$\Theta(\norm{\refw}^2)$. As the following Lemma establishes, we cannot:

\input{svm/sparse/theorems/lem-lower-bound}

%% file: svm/sparse/theorems/lem-lower-bound.tex
\medskip
\begin{splitlemma}{lem:svm-sparse:lower-bound}

Let $R,\mathcal{L}^{*},\epsilon\ge0$ be given, with
$\mathcal{L}^{*}+\epsilon\le\nicefrac{1}{4}$ and with $R^2$ being an
integer. There exists a data distribution $\mathcal{D}$ and a reference vector
$\refw$ such that $\norm{\refw} = R$,
$\loss[\hinge]{g_{\refw}}=\mathcal{L}^{*}$, and any $w$ which satisfies:
\begin{equation*}
\loss[\zeroone]{g_w} \le \mathcal{L}^{*}+\epsilon
\end{equation*}
must necessarily be supported on at least $R^{2}/2$ vectors. Furthermore, the
claim also holds for randomized classification rules that predict $1$ with
probability $\psi(g_{\refw}(x))$ for some $\psi : \R \to [0,1]$.

\end{splitlemma}
\begin{splitproof}

We define $\mathcal{D}$ such that $i$ is sampled uniformly at random from the
set $\left\{ 1,\dots,d\right\}$, with $d=R^2$, and the feature vector is taken
to be $x=e_i$ (the $i$th standard unit basis vector) with corresponding label
distributed according to $\probability{y=z} = 1-\mathcal{L}^*/2$. The value of
$z \in \{\pm 1\}$ will be specified later. Choose ${\refw}_i = z$ for all $i$,
so that $\norm{\refw} = R$ and $\loss[\hinge]{g_{\refw}} = \mathcal{L}^*$.

Take $w$ to be a linear combination of $k < d/2 = R^2 / 2$ vectors. Then
$g_w(x)=0$ on any $x$ which is not in its support set. Suppose that whenever
$g_w(x_i)=0$ the algorithm predicts the label $1$ with probability $p \in
[0,1]$ ($p=\psi(0)$ for a randomized classifier). If $p \ge 1/2$ we'll set
$z=-1$, and if $p<1/2$ we'll set $z=1$. This implies that:
\begin{equation*}
\loss[\zeroone]{g_w} \ge \frac{d-k}{2d} > \frac{1}{4} \ge \mathcal{L}^{*} + \epsilon
\end{equation*}
which concludes the proof.
\end{splitproof}

%% file: svm/sparse/sec-algorithm.tex
\section{Learning Sparse SVMs}\label{sec:svm-sparse:algorithm}

In the previous section we showed that having a good low-norm predictor $\refw$
often implies there exists also a good sparse predictor, but the existence
proof required low \emph{margin error}. We will now consider the problem of
efficiently finding a good sparse predictor, given the existence of low-norm
reference predictor $\refw$ which suffers low \emph{hinge loss} on a finite
training sample.

Our basic approach will be broadly similar to that of Section
\ref{sec:svm-sparse:sparsity}, but instead of relying on an unknown $\refw$, we will start
by using any SVM optimization approach to learn a (possibly dense) $w$. We will
then
learn a sparse classifier $\tilde{w}$ which mimics $w$.

In Section \ref{sec:svm-sparse:sparsity} we removed margin violations and dealt with an
essentially separable problem. But when one considers not margin error, but
hinge loss, the difference between ``correct'' and ``wrong'' is more nuanced,
and we must take into account the numeric value of the loss:
\begin{enumerate}
\item If $y\inner{w}{\Phi(x)} \le 0$ (i.e. $w$ is wrong), then we can ignore
the example, as in Section \ref{sec:svm-sparse:sparsity}.
\item If $0 < y\inner{w}{\Phi(x)} < 1$ (i.e. $w$ is correct, but there is a
margin violation), then we allow $\tilde{w}$ to make a margin violation at most
$\nicefrac{1}{2}$ worse than the margin violation made by $w$, i.e. $y
\inner{\tilde{w}}{\Phi(x)} \ge y\inner{w}{\Phi(x)} - \nicefrac{1}{2}$.
%
\item If $y\inner{w}{\Phi(x)} \ge 1$ (i.e. $w$ is correct and classifies this
example outside the margin), we would like $\tilde{w}$ to also be correct,
though we require a smaller margin: $y\inner{\tilde{w}}{\Phi(x)} \ge
\nicefrac{1}{2}$.
\end{enumerate}
These are equivalent to finding a solution with value at most
$\nicefrac{1}{2}$ to the following optimization problem:
\begin{align}
\label{eq:svm-sparse:objective-approximation} \text{minimize}: & f(\tilde{w}) =
\max_{i:h_i > 0} \left( h_i - y_i \inner{\tilde{w}}{\Phi(x_i)} \right)\\
\notag \text{where}: & h_i=\min\left( 1, y_i \inner{w}{\Phi(x_i)} \right)
\end{align}
We will show that a randomized classifier based on a solution to Problem
\ref{eq:svm-sparse:objective-approximation} with $f\left(\tilde{w}\right) \le
\nicefrac{1}{2}$ has empirical 0/1 error bounded by the empirical hinge loss of
$w$; that we can efficiently find such solution based on at most $4\norm{w}^2$
support vectors; and that such a sparse solution generalizes as well as $w$
itself.

\subsection{The Slant-loss}

\input{svm/sparse/figures/fig-slant-loss}

The key to our approach is our use of the \emph{randomized} classification rule
$\tilde{g}_{\tilde{w}}\left(x\right) = \inner{\tilde{w}}{\Phi\left(x\right)} +
Z$ for $Z\sim\mbox{Unif}\left[-\nicefrac{1}{2},\nicefrac{1}{2}\right]$, rather
than the standard linear classification rule $g_w\left(x\right) =
\inner{w}{\Phi\left(x\right)}$.
The effect of the randomization is to ``spread out'' the expected loss of the
classifier. We define the loss function:
\begin{equation}
\ell_{slant}\left( z \right) = \min\left( 1, \max\left( 0, \nicefrac{1}{2} - z
\right) \right)
\end{equation}
which we call the ``slant-loss'' (Figure \ref{fig:svm-sparse:slant-loss}),
using $\loss[\slant]{g}$ and $\emploss[\slant]{g}$ to denote its expectation
and empirical average, analogously to the 0/1 and hinge losses. It is easy to
see that $\expectation[Z]{\ell_{0/1}(\tilde{g}_{\tilde{w}})} =
\ell_{slant}(g_{\tilde{w}})$---hence, the $0/1$ loss of $\tilde{g}_{\tilde{w}}$
is the same as the slant-loss of $g_{\tilde{w}}$.
Equally importantly, $\ell_{slant}( z - \nicefrac{1}{2} ) \le \ell_{hinge}(
z)$, from which the following Lemma follows:

\input{svm/sparse/theorems/lem-slant-loss}

\subsection{Finding Sparse Solutions}\label{subsec:pca-sparse:subgradient-descent}

\input{svm/sparse/figures/alg-subgradient-descent}
To find a sparse $\tilde{w}$ with value $f(\tilde{w})\le 1/2$, we apply
subgradient descent to Problem \ref{eq:svm-sparse:objective-approximation}. The algorithm
is extremely straightforward to understand and implement. We initialize
$\tilde{w}^{(0)} = 0$, and then proceed iteratively, performing the following
at the $t$th iteration:
\begin{enumerate}
\item Find the training index $i : y_{i} \inner{w}{\Phi(x_{i})} > 0$
which maximizes $h_{i} - y_{i} \inner{\tilde{w}^{(t-1)}}{\Phi(x_{i})}$.
\item Take the subgradient step $\tilde{w}^{(t)} \leftarrow \tilde{w}^{(t-1)} +
\eta y_{i} \Phi(x_{i})$.
\end{enumerate}

To this point, we have worked in the explicit kernel Hilbert space
$\hilbert$, even though we are interested primarily in the kernelized case,
where $\hilbert$ and $\Phi(\cdot)$ are specified only implicitly through
$K(x,x')=\inner{\Phi(x)}{\Phi(x')}$. However, like the SBP (Chapter
\ref{ch:svm-sbp}), the above algorithm can be interpreted as an instance of the
traditional SVM optimization outline of Section
\ref{sec:svm-introduction:traditional} in Chapter \ref{ch:svm-introduction}, by
relying on the representations $w=\sum_{i=1}^{n} \alpha_i y_i \Phi(x_i)$ and
$\tilde{w}=\sum_{i=1}^{n} \tilde{\alpha}_i y_i \Phi(x_i)$, keeping track of the
coefficients $\alpha$ and $\tilde{\alpha}$. We will also maintain an up-to-date
vector of responses $\tilde{c}$:
\begin{equation*}
\tilde{c}_j = y_j \inner{\tilde{w}}{\Phi(x_j)} = \sum_{i=1}^{n}
\tilde{\alpha}_i y_i y_j K( x_i, x_j )
\end{equation*}
Notice that the values of the responses provide sufficient knowledge to find
the update index $i$ at every iteration. We can then perform the subgradient
descent update $\tilde{w} \leftarrow \tilde{w} + \eta y_{i} \Phi(x_{i})$ by
adding $\eta$ to $\tilde{\alpha}_i$, and updating the responses as $\tilde{c}_j
\leftarrow \tilde{c}_j + \eta y_i y_j K( x_i, x_j )$, at a total cost of $n$
kernel evaluations.

Algorithm \ref{alg:svm-sparse:subgradient-descent} gives a detailed version of
this algorithm, as an instance of the outline of Algorithm
\ref{sec:svm-introduction:traditional}. Its convergence rate is characterized
in the following lemma:

\input{svm/sparse/theorems/lem-convergence-rate}
Because each iteration adds at most one new element to the support set, the
support size of the solution will likewise be bounded by $4\norm{w}^2$.
We must calculate the objective function value in the course of each
subgradient descent iteration (while finding $i$), so one can identify an
iterate with $f(\tilde{w}^{(t)}) \le \nicefrac{1}{2}$ at no additional
cost---this is the termination condition on line $4$ of Algorithm
\ref{alg:svm-sparse:subgradient-descent}.

\subsection{Generalization Guarantee - Compression}

The fact that the optimization procedure outlined in the previous section
results in a sparse predictor of a particularly simple structure enables us to
bound its generalization error using a compression bound. Before giving our
generalization bound, we begin by presenting the compression bound which we use:

\input{svm/sparse/theorems/thm-compression-bound}

With this result in place, it is straightforward to bound the generalization
performance of the classifier which is found by our subgradient descent procedure:

\input{svm/sparse/theorems/lem-generalization-from-compression}

\subsection{Generalization Guarantee - Smoothness}

The main disadvantage of using a compression bound to prove generalization is
that the result bounds the performance of a particular \emph{algorithm} (in
this case, subgradient descent), instead of the performance of all solutions
satisfying some conditions. One can address this drawback by using uniform
concentration arguments to obtain an almost identical guarantee (up to log
factors) which holds for any $\tilde{w}$ (not necessarily sparse) with norm
$\norm{\tilde{w}} \le \norm{w}$ and value $f(\tilde{w}) \le \nicefrac{1}{3}$.
The result is a bound which is slightly worse than that of Lemma
\ref{lem:svm-sparse:convergence-rate}, but is more flexible. In order derive
it, we must first modify the objective of Problem
\ref{eq:svm-sparse:objective-approximation} by adding a norm-constraint:
\begin{align}
\label{eq:svm-sparse:objective-approximation-smooth} \mbox{minimize}: &
f(\tilde{w}) = \max_{i:y_i \inner{w}{\Phi(x_i)} > 0} \left( h_i - y_i
\inner{\tilde{w}}{\Phi(x_i)} \right) \\
\notag \mbox{subject to}: & \norm{\tilde{w}} \le \norm{w}
\end{align}
Here, as before, $h_i=\min\left( 1, y_i \inner{w}{\Phi(x_i)} \right)$. Like
Problem \ref{eq:svm-sparse:objective-approximation}, this objective can be
optimized using subgradient descent, although one must add a step in which the
current iterate is projected onto the ball of radius $\norm{w}$ after every
iteration.  Despite this change, an $\epsilon$-suboptimal solution can still be
found in $\norm{w}^2 /\epsilon^2$ iterations.

\input{svm/sparse/figures/fig-smooth-loss}

The concentration-based version of our main theorem follows:

\input{svm/sparse/theorems/thm-main-smooth}

It's worth pointing out that the addition of a norm-constraint to the objective
function (Problem \ref{eq:svm-sparse:objective-approximation-smooth}) is only necessary
because we want the theorem to apply to any $\tilde{w}$ with $f(\tilde{w}) \le
\nicefrac{1}{3}$. If we restrict ourselves to $\tilde{w}$ which are found using
subgradient descent with the suggested step size and iteration count, then
applying the triangle inequality to the sequence of steps yields that
$\norm{\tilde{w}} \le O\left(\norm{w}\right)$, and the above bound still holds
(albeit with a different constant hidden inside the big-Oh notation).

\subsection{Putting it Together}\label{sec:svm-sparse:together}

Now that all of the pieces are in place, we can state our final procedure,
start-to-finish:
\begin{enumerate}
\item Train a SVM to obtain $w$ with norm $\norm{w} \le O\left( \norm{\refw} \right)$ and
$\emploss[\hinge]{g_w} \le \emploss[\hinge]{g_{\refw}} + O\left( \epsilon \right)$.
\item Run subgradient descent on Problem
\ref{eq:svm-sparse:objective-approximation} until we find a predictor
$\tilde{w}$ with value $f(\tilde{w}) \le \nicefrac{1}{2}$ (see Algorithm
\ref{alg:svm-sparse:subgradient-descent}).
\item Predict using $\tilde{g}_{\tilde{w}}$.
\end{enumerate}

\input{svm/sparse/theorems/thm-main-compression}

The procedure is efficient, and aside from initial SVM optimization, requires
at most $O( \norm{\refw}^2 )$ iterations.

\subsection{Kernelization}

As we saw in Section \ref{subsec:pca-sparse:subgradient-descent}, each
iteration of Algorithm \ref{alg:svm-sparse:subgradient-descent} requires $n$
kernel evaluations. These kernel evaluations dominate the computational cost of
the gradient descent procedure (all other operations can be performed in $O(n)$
operations per iteration). As is standard for kernel SVM training, we will
therefore analyze runtime in terms of the number of kernel evaluations
required.

With $O(\|\refw\|^2)$ iterations, and $O(n)$ kernel evaluations per iteration,
the overall number of required kernel evaluations for the gradient descent
procedure is (ignoring the $\delta$ dependence):
\begin{equation*}
O\left(n \norm{\refw}^2\right) = \tilde{O}\left(  \left( \frac{
	\emploss[\hinge]{ g_{\refw} } + \epsilon }{ \epsilon } \right)
\frac{\norm{\refw}^4}{\epsilon} \right)
\end{equation*}
This is {\em less} then the best known runtime bound for kernel SVM
optimization, so we do not expect the sparsification step to be computationally
dominant (i.e.  it is in a sense ``free''). In order to complete the picture
and understand the entire runtime of our method, we must also consider the
runtime of the SVM training (Step 1). The best kernelized SVM optimization
guarantee of which we are aware is achieved by the Stochastic Batch Perceptron
(SBP, \citet{CotterShSr12}). Using the SBP, we can find $w$ with $\norm{w} \leq
2\norm{u}$ and $\emploss[\hinge]{g_w} \le \emploss[\hinge]{g_{\refw}} + \epsilon$
using:
\begin{equation*}
O\left( \left( \frac{ \emploss[\hinge]{ g_{\refw} } + \epsilon }{ \epsilon }
\right)^2 \norm{\refw}^2  n \right)
\end{equation*}
kernel evaluations, yielding (with the $\delta$-dependence):
\begin{corollary}
If using the SBP for Step 1 and the sample size required by Theorem
\ref{thm:svm-sparse:main-compression}, the procedure in Section
\ref{sec:svm-sparse:together} can be performed with:
\begin{equation*}
\tilde{O}\left( \left( \frac{ \emploss[\hinge]{ g_{\refw} } + \epsilon }{
	\epsilon } \right)^3 \frac{ \norm{\refw}^4}{\epsilon} \log\frac{1}{\delta}
	\right)
\end{equation*}
kernel evaluations.
\end{corollary}
Because the SBP finds a $w$ with $\norm{w} \leq 2\norm{u}$, and our subgradient
descent algorithm finds a $\tilde{w}$ supported on $4\norm{w}^2$ training
vectors, it follows that the support size of $\tilde{w}$ is bounded by
$16\norm{u}^2$. The runtime of Step 2 (the sparsification procedure) is
asymptotically negligible compared to Step 1 (initial SVM training), so the
overall runtime is the same as for stand-alone SBP. Overall, our procedure
finds an optimally sparse SVM predictor, and at the same time matches the best
known sample and runtime complexity guarantees for SVM learning (up to small
constant factors).

\subsection{Unregularized Bias}\label{subsec:svm-sparse:bias}

Frequently, SVM problems contain an unregularized bias term---rather
than the classification function being the sign of $g_w(x) =
\inner{w}{\Phi(x)}$, it is the sign of $g_{w,b}(x) =
\left(\inner{w}{\Phi(x)}+b\right)$ for a weight vector $w$ and a bias
$b$, where the bias is unconstrained, being permitted to take on the
value of any real number.

When optimizing SVMs, the inclusion of an unregularized bias introduces some
additional complications which typically require special treatment. Our
subgradient descent procedure, however, is essentially unchanged by the
inclusion of a bias (although the SVM solver which we use to find $w$ and $b$
must account for it). Indeed, we need only redefine:
\begin{equation*}
h_i=\min\left( 1, y_i \left( \inner{w}{\Phi(x_i)} + b \right) \right) - y_i b 
\end{equation*}
in Problem \ref{eq:svm-sparse:objective-approximation}, and then find $\tilde{w}$ as
usual. The resulting sparse classifier is parameterized by $\tilde{w}$ and $b$,
with $b$ being that of the initial SVM solution.

Alternatively, instead of taking the bias of the sparse predictor to be the
same as that of the original SVM predictor, we may optimize over $\tilde{b}$
during our subgradient descent procedure. The resulting optimization procedure
is more complex, although it enjoys the same performance guarantee, and may
result in better solutions, in practice. The relevant optimization problem
(analogous to Problem \ref{eq:svm-sparse:objective-approximation}) is:
\begin{align}
\label{eq:svm-sparse:objective-approximation-bias} \mbox{minimize}: &
f(\tilde{w},\tilde{b}) = \max_{i:y_i \inner{w}{\Phi(x_i)} > 0} \left( h_i - y_i
\left( \inner{\tilde{w}}{\Phi(x_i)} + \tilde{b} \right) \right) \\
\notag \mbox{with}: & h_i = \min\left( 1, y_i ( \inner{w}{\Phi(x_i)} + b)
\right)
\end{align}
A $\nicefrac{1}{2}$-approximation may once more be found using subgradient
descent. The difference is that, before finding a subgradient, we will
implicitly optimize over $\tilde{b}$. It can be easily observed that the
optimal $\tilde{b}$ will ensure that:
\begin{equation}
\label{eq:svm-sparse:bias-calculation} \max_{i:y_i > 0 \wedge
\inner{w}{\Phi(x_i)} > 0} \left( h_i - ( \inner{\tilde{w}}{\Phi(x_i)} +
\tilde{b}) \right) = \max_{i:y_i < 0 \wedge \inner{w}{\Phi(x_i)} < 0} \left(
h_i + ( \inner{\tilde{w}}{\Phi(x_i)} + \tilde{b}) \right)
\end{equation}
In other words, $\tilde{b}$ will be chosen such that the maximal violation
among the set of positive examples will equal that among the negative examples.
Hence, during optimization, we may find the most violating \emph{pair} of one
positive and one negative example, and then take a step on both elements. The
resulting subgradient descent algorithm is:
\begin{enumerate}
\item Find the training indices $i_{+}:y_i > 0 \wedge \inner{w}{\Phi(x_i)} + b
> 0$ and $i_{-}:y_i < 0 \wedge \inner{w}{\Phi(x_i)} + b < 0$ which maximize
$h_i - y_i \inner{\tilde{w}^{(t-1)}}{\Phi(x_i)}$
\item Take the subgradient step $\tilde{w}^{(t)} \leftarrow \tilde{w}^{(t-1)} +
\eta( \Phi(x_{i_{+}}) - \Phi(x_{i_{-}}))$.
\end{enumerate}
Once optimization has completed, $\tilde{b}$ may be computed from Equation
\ref{eq:svm-sparse:bias-calculation}. As before, this algorithm will find a
$\nicefrac{1}{2}$-approximation in $4 \norm{w}^2$ iterations.

%% file: svm/sparse/figures/fig-slant-loss.tex
\begin{figure}

\begin{center}
\includegraphics[width=0.45\textwidth]{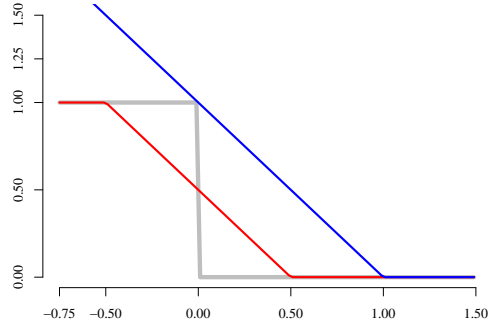}
\end{center}

\caption{
Illustration of how the slant-loss (red) relates to the 0/1 (gray) and
hinge (blue) losses. Notice that, if the slant-loss is shifted by
$\nicefrac{1}{2}$, then it is still upper bounded by the hinge loss.
}

\label{fig:svm-sparse:slant-loss}

\end{figure}

%% file: svm/sparse/theorems/lem-slant-loss.tex
\medskip
\begin{splitlemma}{lem:svm-sparse:slant-loss}

For any $w$, and any $\tilde{w}$ for which Problem
\ref{eq:svm-sparse:objective-approximation} has value $f\left(\tilde{w}\right) \le
\nicefrac{1}{2}$, we have that
\begin{equation*}
\expectation[Z]{\emploss[0/1]{\tilde{g}_{\tilde{w}}}} =
\emploss[slant]{g_{\tilde{w}}} \leq \emploss[hinge]{g_w}
\end{equation*}

\end{splitlemma}
\begin{splitproof}

It remains only to establish that $\emploss[slant]{g_{\tilde{w}}} \leq
\emploss[hinge]{g_w}$. For every $x_i,y_i$, consider the following three cases:
\begin{enumerate}
\item If $y_i \inner{w}{\Phi(x_i)} \le 0$, then $\ell_{slant}( y_i
g_{\tilde{w}}(x_i)) \le 1 \le \ell_{hinge}( y_i g_w(x_i) )$.
\item If $0 < y_i \inner{w}{\Phi(x_i)} < 1$, then $\ell_{slant}( y_i
g_{\tilde{w}}(x_i) ) \le \ell_{slant}( y_i g_w(x_i) - \nicefrac{1}{2} ) \le
\ell_{hinge}( y_i g_w(x_i) )$.
\item If $y_i \inner{w}{\Phi(x_i)} \ge 1$, then $\ell_{slant}( y_i
g_{\tilde{w}}(x_i)) \le \ell_{slant}( \nicefrac{1}{2} ) = 0 = \ell_{hinge}( y_i
g_w(x_i) )$.
\end{enumerate}
Hence, $\ell_{slant}( y_i g_{\tilde{w}}(x_i) ) \le \ell_{hinge}( y_i
g_w(x_i))$, which completes the proof.
\end{splitproof}

%% file: svm/sparse/figures/alg-subgradient-descent.tex
\begin{algorithm}[t]

\begin{pseudocode}
\codename $\code{optimize}\left( n:\N, d:\N, x_{1},\dots,x_{n}:\R^d, y_{1},\dots,y_{n}:\left\{\pm 1\right\}, h:\R^n, K:\R^d\times\R^d\rightarrow\R_{+} \right)$\\
\codeline $\eta := 1 / 2$; $\tilde{\alpha}^{(0)} := 0^{n}$; $\tilde{c}^{(0)} := 0^{n}$;\\
\codeline $\code{do}$\\
\codeline \>$i := \argmax_{i:h_i>0}\left(h_i - \tilde{c}_i\right)$;\\
\codeline \>$\code{if } h_i - \tilde{c}_i \le 1/2 \code{ then}$\\
\codeline \>\>$\code{return } \tilde{\alpha}$;\\
\codeline \>$\tilde{\alpha}^{(t)} := \tilde{\alpha}^{(t-1)} + \eta e_{i}$;\\
\codeline \>$\code{for } j=1 \code{ to } n$\\
\codeline \>\>$\tilde{c}^{(t)}_{j} := \tilde{c}^{(t-1)}_{j} + \eta y_{i} y_{j} K\left(x_{i},x_{j}\right)$;
\end{pseudocode}

\caption{
Subgradient ascent algorithm for optimizing the kernelized version of Problem
\ref{eq:svm-sparse:objective-approximation}. Here, $\tilde{\alpha}$ is the
vector of coefficients representing $\tilde{w}$ as a linear combination of the
training data: $\tilde{w} = \sum_{i=1}^n \tilde{\alpha}_i y_i
\Phi\left(x_i\right)$ (this is the representor theorem---see Section
\ref{sec:svm-introduction:objective} of Chapter \ref{ch:svm-introduction}). The
vector $h$ is derived from the initial SVM solution as in Problem
\ref{eq:svm-sparse:objective-approximation}, and $e_{i}$ is the $i$th standard
unit basis vector. Lemma \ref{lem:svm-sparse:convergence-rate} shows that the
termination condition on line $4$ will be satisfied after at most $4\norm{w}^2$
iterations.
}

\label{alg:svm-sparse:subgradient-descent}

\end{algorithm}

%% file: svm/sparse/theorems/lem-convergence-rate.tex
\medskip
\begin{splitlemma}{lem:svm-sparse:convergence-rate}

After $T \le 4\norm{w}^2$ iterations of subgradient descent with $\eta =
\nicefrac{1}{2}$, we obtain a solution of the form $\tilde{w} = \frac{1}{2}
\sum_{t=1}^T y_{i_t} \Phi( x_{i_t} )$ which has value $f\left(\tilde{w}\right)
\le \nicefrac{1}{2}$.

\end{splitlemma}
\begin{splitproof}

First note that $f(w) \leq 0$. Relying on this possible solution $w$, the Lemma
follows from standard convergence bounds of subgradient descent (see e.g.
Section 1.2 of \citet{Nesterov09}): with the step size $\eta = \epsilon$, after
performing $\norm{w}^2 / \epsilon^2$ iterations, at least one iterate
$\tilde{w}^{(t)}$ will have an objective function value no greater than
$\epsilon$. Choosing $\epsilon=\nicefrac{1}{2}$ gives the desired result.
\end{splitproof}

%% file: svm/sparse/theorems/thm-compression-bound.tex
\medskip
\begin{splittheorem}{thm:svm-sparse:compression-bound}

This is Theorem 2 of \citet{ShalevHandouts2010}. Let $k$ and $n$ be fixed,
with $n\ge2k$, and let $A:\left(\R^d\times\left\{ \pm1\right\}
\right)^{k}\rightarrow \hilbert$ be a mapping which receives a list of $k$
labeled training examples, and returns a classification vector $w\in \hilbert$.
Use $S \in \left[n\right]^k$ to denote a list of $k$ training indices, and let
$w_{S}$ be the result of applying $A$ to the training elements indexed by $S$.
Finally, let $\ell:\R \rightarrow \left[0,1\right]$ be a loss function bounded
below by $0$ and above by $1$, with
$\loss{g_w}$ and $\emploss{g_w}$ the expected loss, and empirical loss on the
training set, respectively. Then, with probability $1-\delta$, for all $S$:
\begin{equation*}
\loss{g_{w_{S}}}\le \emploss{g_{w_{S}}}+
\sqrt{\frac{32\emploss{g_{w_{S}}}\left(k\log
n+\log\frac{1}{\delta}\right)}{n}}+\frac{8\left(k\log
n+\log\frac{1}{\delta}\right)}{n}
\end{equation*}

\end{splittheorem}
\begin{splitproof}

Consider, for some fixed $\delta'$, the probability that there exists a $S
\subseteq \left\{1,\dots,n\right\}$ of size $k$ such that:
\begin{equation*}
\loss{g_{w_{S}}}\ge\emploss[\mbox{test}]{g_{w_{S}}}+\sqrt{\frac{2\emploss[\mbox{test}]{g_{w_{S}}}\log\frac{1}{\delta'}}{\left(n-k\right)}}+\frac{4\log\frac{1}{\delta'}}{n-k}
\end{equation*}
where $\emploss[\mbox{test}]{g_w}=\frac{1}{n-k}\sum_{i\notin S}\ell\left( y_i
g_w\left( x_i \right) \right)$ is the empirical loss on the \emph{complement}
of $S$. It follows from Bernstein's inequality that, for a \emph{particular}
$S$, the above holds with probability at most $\delta'$. By the union bound:
\begin{equation*}
n^k\delta' \ge \mathrm{Pr} \left\{ \exists S \in [n]^k :
\loss{g_{w_{S}}}\ge\emploss[\mbox{test}]{g_{w_{S}}} +
\sqrt{\frac{2\emploss[\mbox{test}]{g_{w_{S}}}\log\frac{1}{\delta'}}{\left(n-k\right)}}+\frac{4\log\frac{1}{\delta'}}{n-k}
\right\}
\end{equation*}
Let $\delta=n^k\delta'$. Notice that
$\left(n-k\right)\emploss[\mbox{test}]{g_{w_{S}}}\le n\emploss{g_{w_{S}}}$, so:
\begin{equation*}
\delta \ge \mathrm{Pr} \left\{ \exists S \in [n]^k :
\loss{g_{w_{S}}}\ge\frac{n\emploss{g_{w_{S}}}}{n-k} +
\sqrt{\frac{2n\emploss{g_{w_{S}}}\log\frac{n^k}{\delta}}{\left(n-k\right)^{2}}}+\frac{4\log\frac{n^k}{\delta}}{n-k}
\right\}
\end{equation*}
Because $\emploss{g_{w_{S}}}\le1$ and $k\le n$, it follows that
$k\emploss{g_{w_{S}}}\le2n\log n$, and therefore that
$\frac{k}{n-k}\emploss{g_{w_{S}}}\le\sqrt{\frac{2nk\emploss{g_{w_{S}}}\log
n}{\left(n-k\right)^{2}}}\le\sqrt{\frac{2n\emploss{g_{w_{S}}}\log\frac{n^{k}}{\delta}}{\left(n-k\right)^{2}}}$.
Hence:
\begin{equation*}
\delta \ge \mathrm{Pr} \left\{ \exists S \in [n]^k :
\loss{g_{w_{S}}}\ge\emploss{g_{w_{S}}} +
\sqrt{\frac{8n\emploss{g_{w_{S}}}\log\frac{n^{k}}{\delta}}{\left(n-k\right)^{2}}}+\frac{4\log\frac{n^{k}}{\delta}}{n-k}
\right\}
\end{equation*}
Using the assumption that $n\ge2k$ completes the proof.
\end{splitproof}

%% file: svm/sparse/theorems/lem-generalization-from-compression.tex
\medskip
\begin{splitlemma}{lem:svm-sparse:generalization-from-compression}

With probability at least $1-\delta$ over the training set, if $\tilde{w}$ is a
solution with value $f(\tilde{w})\le\nicefrac{1}{2}$ to Problem
\ref{eq:svm-sparse:objective-approximation} found by performing $T = 4
\norm{w}^2$ iterations of subgradient descent (see Algorithm
\ref{alg:svm-sparse:subgradient-descent} and Lemma
\ref{lem:svm-sparse:convergence-rate}), then:
\begin{equation*}
\loss[\zeroone]{\tilde{g}_{\tilde{w}}}\le
\emploss[\hinge]{\tilde{g}_{\tilde{w}}} +
O\left(
\sqrt{\emploss[\hinge]{g_{w}}\frac{\norm{w}^2\log n+\log\frac{1}{\delta}}{n}} +
\frac{\norm{w}^2\log n+\log\frac{1}{\delta}}{n}
\right)
\end{equation*}
provided that $n \ge 2 T$.

\end{splitlemma}
\begin{splitproof}
%
%
Because $\tilde{w}$ is found via subgradient descent,it can be written in the
form $\tilde{w}=\eta \sum_{i \in S} y_i x_i = A(S)$ with $S \in
\left[n\right]^k$, and $k=T=4\norm{w}^{2}$.  Hence, by the compression bound of
Theorem \ref{thm:svm-sparse:compression-bound}, we have that with probability
$1-\delta$:
\begin{equation*}
\loss[\slant]{g_{\tilde{w}}}\le \emploss[\slant]{g_{\tilde{w}}} +
\sqrt{\frac{32\emploss[\slant]{g_{\tilde{w}}}\left(k\log
n+\log\frac{1}{\delta}\right)}{n}} + \frac{8\left(k\log
n+\log\frac{1}{\delta}\right)}{n}
\end{equation*}
The randomized classification rule $\tilde{g}_{\tilde{w}}$ was defined in such
a way that $\ell_{slant}\left( y g_{\tilde{w}}\left( x \right) \right) =
\expectation{\ell_{0/1}\left( y \tilde{g}_{\tilde{w}}\left( x \right)
\right)}$. Furthermore, as was shown in Lemma \ref{lem:svm-sparse:slant-loss},
$\emploss[\slant]{g_{\tilde{w}}}\le\emploss[\hinge]{g_w}$. Plugging these
results into the above equation:
\begin{equation*}
\loss[\zeroone]{\tilde{g}_{\tilde{w}}}\le \emploss[\hinge]{g_w} +
\sqrt{\frac{32\emploss[\hinge]{g_w}\left(k\log
n+\log\frac{1}{\delta}\right)}{n}} + \frac{8\left(k\log
n+\log\frac{1}{\delta}\right)}{n}
\end{equation*}
The above holds for all $w$ such that $\tilde{w}$ is a
$\nicefrac{1}{2}$-suboptimal solution to the corresponding instance of Problem
\ref{eq:svm-sparse:objective-approximation}. Plugging the assumption that
$k=4\norm{w}^{2}$ completes the proof.
\end{splitproof}

%% file: svm/sparse/figures/fig-smooth-loss.tex
\begin{figure}

\begin{center}
\includegraphics[width=0.45\textwidth]{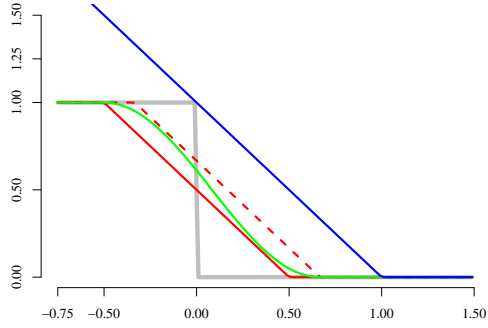}
\end{center}

\caption{
Illustration of the how our smooth loss relates to the slant and hinge losses.
Our smooth loss (green) upper bounds the slant-loss, and lower bounds the
slant-loss when shifted by $\nicefrac{1}{6}$, and the hinge-loss when shifted
by $\nicefrac{1}{3}$.
}

\label{fig:svm-sparse:smooth-loss}

\end{figure}

%% file: svm/sparse/theorems/thm-main-smooth.tex
\medskip
\begin{splittheorem}{thm:svm-sparse:main-smooth}

Let $R\in\R_{+}$ be fixed.
With probability $1-\delta$ over the training sample, uniformly over all pairs
$w,\tilde{w}\in \hilbert$ such that $\norm{w}\le R$ and $\tilde{w}$ has
objective function $f(\tilde{w}) \le \nicefrac{1}{3}$ in Problem
\ref{eq:svm-sparse:objective-approximation-smooth}:
\begin{align*}
\MoveEqLeft \loss[\zeroone]{\tilde{g}_{\tilde{w}}} \le \emploss[\hinge]{g_w} \\
& + O\left(\sqrt{\frac{\emploss[\hinge]{g_w}R^{2}\log^{3}n}{n}} +
\sqrt{\frac{\emploss[\hinge]{g_w}\log\frac{1}{\delta}}{n}}+\frac{R^{2}\log^{3}n}{n}+\frac{\log\frac{1}{\delta}}{n}\right)
\end{align*}
%

\end{splittheorem}
\begin{splitproof}

Because our bound is based on a smooth loss, we begin by defining the bounded
$4$-smooth loss $\ell_{smooth}(z)$ to be $1$ if $z < -\nicefrac{1}{2}$, $0$ if
$z > \nicefrac{2}{3}$, and $\nicefrac{1}{2}\left( 1+\cos\left(
\nicefrac{\pi}{2}\left( 1+\nicefrac{1}{7}\left( 12z-1 \right) \right) \right)
\right)$ otherwise.
This function is illustrated in Figure \ref{fig:svm-sparse:smooth-loss}---notice that it
upper-bounds the slant-loss, and lower-bounds the hinge loss even when shifted
by $\nicefrac{1}{3}$. Applying Theorem 1 of \citet{SrebroSrTe10} to this smooth
loss yields that, with probability $1-\delta$, uniformly over all $\tilde{w}$
such that $\norm{\tilde{w}}\le R$:
\begin{align*}
\MoveEqLeft \loss[\smooth]{g_{\tilde{w}}}\le\emploss[\smooth]{g_{\tilde{w}}} + \\
& O\left(\sqrt{\frac{\emploss[\smooth]{g_{\tilde{w}}}R^{2}\log^{3}n}{n}} +
\sqrt{\frac{\emploss[\smooth]{g_{\tilde{w}}}\log\frac{1}{\delta}}{n}}+\frac{R^{2}\log^{3}n}{n}+\frac{\log\frac{1}{\delta}}{n}\right)
\end{align*}
Just as the empirical slant-loss of a $\tilde{w}$ with $f(\tilde{w}) \le
\nicefrac{1}{2}$ is upper bounded by the empirical hinge loss of $w$, the
empirical smooth loss of a $\tilde{w}$ with $f(\tilde{w}) \le \nicefrac{1}{3}$
is upper-bounded by the same quantity. As was argued in the proof of Lemma
\ref{lem:svm-sparse:slant-loss}, this follows directly from Problem
\ref{eq:svm-sparse:objective-approximation-smooth}, and the definition of the smooth loss.
Combining this with the facts that the slant-loss lower bounds the smooth loss,
and that $\loss[\slant]{g_{\tilde{w}}} = \loss[\zeroone]{\tilde{g}_{\tilde{w}}}$,
completes the proof.
\end{splitproof}

%% file: svm/sparse/theorems/thm-main-compression.tex
\medskip
\begin{splittheorem}{thm:svm-sparse:main-compression}

For an arbitrary (unknown) reference classifier $\refw$, with probability at
least $1-2\delta$ over a training set of size:
\begin{equation*}
n = \tilde{O}\left( \left( \frac{ \loss[\hinge]{ g_{\refw} } + \epsilon }{
\epsilon } \right) \frac{\norm{\refw}^2}{\epsilon} \log \frac{1}{\delta}
\right)
\end{equation*}
the procedure above finds a predictor $\tilde{w}$ supported on at most $O(
\norm{\refw}^2 )$ training vectors and error
$\loss[\zeroone]{\tilde{g}_{\tilde{w}}} \le \loss[\hinge]{g_{\refw}} +
O(\epsilon)$

\end{splittheorem}
\begin{splitproof}

First, note that with the specified sample size, applying Bernstein's
inequality to the fixed predictor $\refw$, we have that with probability at
least $1-\delta$, 
\begin{equation} \label{eq:svm-sparse:refwemp}
\emploss[\hinge]{\refw} \leq \loss[\hinge]{\refw} + \epsilon.
\end{equation}
Combining Equation \ref{eq:svm-sparse:refwemp} with the SVM training goal (Step 1) and
Lemma \ref{lem:svm-sparse:slant-loss}, we have that $\emploss[\slant]{g_{\tilde{w}}} \leq
\loss[\hinge]{g_{\refw}} + O(\epsilon)$.  Following Lemma \ref{lem:svm-sparse:convergence-rate}
we can apply Lemma \ref{lem:svm-sparse:generalization-from-compression} with $T = O(\|u\|^2)$, and plugging in the
specified sample complexity, we have $\loss[\slant]{g_{\tilde{w}}} \leq
\emploss[\slant]{g_{\tilde{w}}}$. Combining the two inequalities, and recalling
that the slant-loss of $g_{\tilde{w}}$ is the same as the expected $0/1$ error
of $\tilde{g}_{\tilde{w}}$, we obtain $\loss[\zeroone]{\tilde{g}_{\tilde{w}}} \le
\loss[\hinge]{g_{\refw}} + O(\epsilon)$. Lemma \ref{lem:svm-sparse:convergence-rate} also
establishes the desired bound on the number of support vectors.
\end{splitproof}

%% file: svm/sparse/sec-comparison.tex
\section{Related Algorithms}\label{sec:svm-sparse:comparison}

\input{svm/sparse/figures/tab-bounds}
The tendency of SVM training algorithms to find solutions with large numbers of
support vectors has been recognized as a shortcoming of SVMs since their
introduction, and many approaches for finding solutions with smaller support
sizes have been proposed, of varying levels of complexity and effectiveness.

We group these approaches into two categories: those which, like ours, start
with a non-sparse solution to the SVM problem, and then find a sparse
approximation; and those which either modify the SVM objective so as to result
in sparse solutions, or optimize it using an algorithm specifically designed to
maximize sparsity.

In this section, we will discuss previous work of both of these types. None of
these algorithms have performance guarantees which can be compared to that of
Theorem \ref{thm:svm-sparse:main-compression}, so we will also discuss some
algorithms which do not optimize the SVM objective (even approximately), but do
find sparse solutions, and for some of which generalization bounds have been
proven. In section \ref{sec:svm-sparse:experiments}, we also report on
empirical comparisons to some of the methods discussed here.

There is a third class of algorithms which are worthy of mention: those SVM
optimizers which are not explicitly designed to limit the support size of the
solution, but for which we can prove bound son the support size which are
comparable to that of Theorem \ref{thm:svm-sparse:main-compression}. Table
\ref{tab:svm-sparse:bounds} contains such bounds for the algorithms of Chapters
\ref{ch:svm-introduction} and \ref{ch:svm-sbp}. One can see that none perform
as well as our novel sparsification procedure (first row), although both SBP
and the online Perceptron perform quite well. Significantly, our procedure is
the only one for which the support size has \emph{no dependence} on $\epsilon$.

\subsection{Post-hoc approximation approaches}

One of the earliest proposed methods for finding sparse SVM solutions was that
of \citet{OsunaGi98}, who suggest that one first solve the kernel SVM
optimization problem to find $w$, and then, as a post-processing step, find a
sparse approximation $\tilde{w}$ using support vector regression (SVR),
minimizing the average $\epsilon$-insensitive loss, plus a regularization
penalty:
\begin{align}
\label{eq:svm-sparse:objective-osuna-girosi} \MoveEqLeft \mbox{minimize}:
f_{OG}(\tilde{w}) = \frac{1}{2}\norm{\tilde{w}}^2 + \\
\notag & \tilde{C} \sum_{i=1}^{n} \max\left( 0, \abs{ \inner{w}{\Phi(x_i)} -
\inner{\tilde{w}}{\Phi(x_i)} } - \epsilon \right)
\end{align}
Optimizing this problem results in a $\tilde{w}$ for which the numerical values
of $\inner{w}{\Phi(x)}$ and $\inner{\tilde{w}}{\Phi(x)}$ must be similar, even
for examples which $w$ misclassifies. This is an unnecessarily stringent
condition---because the underlying problem is one of classification, we need
only find a solution which gives roughly the same classifications as $w$,
without necessarily matching the value of the classification function $g_w$.
It is in this respect that our objective function, Problem
\ref{eq:svm-sparse:objective-approximation}, differs.

Osuna and Girosi's work was later used as a key component of the work of
\citet{ZhanSh05}, who first solve the SVM optimization problem, and then
exclude a large number of support vectors from the training set based on a
``curvature heuristic''. They then retrain the SVM on this new, smaller,
training set, and finally apply the technique of Osuna and Girosi to the
result.

\subsection{Alternative optimization strategies}

Another early method for finding sparse approximate SVM solutions is RSVM
\citep{LeeMa01}, which randomly samples a subset of the training set, and then
searches for a solution supported only on this sample, minimizing the loss on
the \emph{entire} training set.

So-called ``reduced set'' methods \citep{BurgesSc97,WuScBa05} address the
problem of large support sizes by removing the constraint that the SVM solution
be supported on the training set. Instead it is now supported on a set of
``virtual training vectors'' $z_1,\dots,z_k$ with $k\ll n$, while having the
same form as the standard SVM solution: $\sign\left( \sum_{i=1}^k \beta_i y_i
K(x,z_i) \right)$. One must optimize over not just the coefficients $\beta_i$,
but \emph{also} the virtual training vectors $z_i$. Because the support set is
not a subset of the training set, our lower bound (Lemma \ref{lem:svm-sparse:lower-bound})
does not apply. However, the resulting optimization problem is non-convex, and
is therefore difficult to optimize.

More recently, techniques such as those of \citet{JoachimsYu09} and
\citet{NguyenMaTaHa10} have been developed which, rather than explicitly
including the search for good virtual training vectors in the objective
function, instead find such vectors heuristically during optimization.
These approaches have the significant advantage of not explicitly relying on
the optimization of a non-convex problem, although in a sense this difficulty
is being ``swept under the rug'' through the use of heuristics.

Another approach is that of \citet{KeerthiChDe06}, who optimize the standard
SVM objective function while explicitly keeping track of a support set $S$. At
each iteration, they perform a greedy search over all training vectors
$x_i\notin S$, finding the $x_i$ such that the optimal solution supported on $S
\cup \{x_i\}$ is best. They then add $x_i$ to $S$, and repeat. This is another
extremely well-performing algorithm, but while the authors propose a clever
method for improving the computational cost of their approach, it appears that
it is still too computationally expensive to be used on very large datasets.

\subsection{Non-SVM algorithms}

\citet{CollobertSiWeBo06} modify the SVM objective to minimize not the convex
hinge loss, but rather the non-convex ``ramp loss'', which differs from our
slant-loss only in that the ramp covers the range $[-1,1]$ instead of
$[-\nicefrac{1}{2},\nicefrac{1}{2}]$. Because the resulting objective function
is non-convex, it is difficult to find a global optimum, but the experiments of
\citet{CollobertSiWeBo06} show that local optima achieve essentially the same
performance with smaller support sizes than solutions found by ``standard''
SVM optimization.

Another approach for learning kernel-based classifiers is to use online
learning algorithms such as the Perceptron (e.g. \citet{FreundSc99}). The
Perceptron processes the training example one by one and adds a support vector
only when it makes a prediction mistake. Therefore, a bound on the number of
prediction mistakes (i.e. a mistake bound) translates to a bound on the
sparsity of the learned predictor.

As was discussed in Section \ref{sec:svm-sparse:sparsity}, if the data are separable with
margin $1$ by some vector $\refw$, then the Perceptron can find a very sparse
predictor with low error.
%
%
However, in the non-separable case, the Perceptron might make a number of
mistakes that grows linearly with the size of the training sample, although, as
can be seen in Table \ref{tab:svm-sparse:bounds}, the Perceptron shares the
second-best bound on the sparsity of its solution with the SBP algorithm of
Chapter \ref{ch:svm-sbp}.

To address this linear growth in the support size, online learning algorithms
for which the support size is bounded by a \emph{budget parameter} have been
proposed. Notable examples include the Forgetron \citep{DekelShSi05} and the
Randomized Budget Perceptron (RBP, \citet{CavallantiCeGe07}). Such algorithms
discard support vectors when their number exceeds the budget parameter---for
example, the RBP discards an example chosen uniformly at random from the set of
support vectors, whenever needed.

Both of these algorithms have been analyzed, but the resulting mistake bounds
are inferior to that of the Perceptron, leading to worse generalization bounds
than the one we achieve for our proposed procedure, for the same support size.
For example, the generalization bound of the Forgetron is at least
$4\emploss[hinge]{g_{\refw}}$. The bound of the RBP is more involved, but it is
possible to show that in order to obtain a support size of $16\norm{u}^2$, the
generalization bound would depend on at least
$(5/3)\emploss[hinge]{g_{\refw}}$. In contrast, the bound we obtain only
depends on $\emploss[hinge]{g_{\refw}}$.

%% file: svm/sparse/figures/tab-bounds.tex
\begin{table}

\begin{small}
\begin{center}
\begin{tabular}{l|c}
\hline
& Support Size \\
\hline
GD on Problem \ref{eq:svm-sparse:objective-approximation} & $R^2$ \\
SBP & $\left( \frac{L^* + \epsilon}{\epsilon} \right)^2 R^2$ \\
SGD on Problem \ref{eq:svm-introduction:norm-constrained-objective} & $\frac{R^2}{\epsilon^2}$ \\
Dual Decomposition & $\left( \frac{L^* + \epsilon}{\epsilon} \right) \frac{R^2}{\epsilon}$ \\
Perceptron + Online-to-Batch & $\left( \frac{L^* + \epsilon}{\epsilon} \right)^2 R^2$ \\
Random Fourier Features & $\frac{d R^2}{\epsilon^2}$ \\
\hline
\end{tabular}
\end{center}
\end{small}

\caption{
Upper bounds, up to constant and log factors, on the support size of the
solution found by various algorithms, where the solution satisfies
$\loss[\zeroone]{g_w}\leq L^* +\epsilon$, where $R$ bounds the norm of a
reference classifier achieving hinge loss $L^*$. See Chapter
\ref{ch:svm-introduction} (in particular Sections
\ref{sec:svm-introduction:traditional} and
\ref{sec:svm-introduction:non-traditional}, as well as Table
\ref{tab:svm-introduction:runtimes}) for derivations of the non-SBP bounds, and
Chapter \ref{ch:svm-sbp} (in particular Section
\ref{subsec:svm-sbp:slack-constrained-runtime}) for the SBP bound.
}

\label{tab:svm-sparse:bounds}

\end{table}

%% file: svm/sparse/sec-practical.tex
\section{Practical Variants}\label{sec:svm-sparse:practical}

While our ``basic algorithm'' (Algorithm
\ref{alg:svm-sparse:subgradient-descent}) gives asymptotically optimal
theoretical performance, slight variations of it give better empirical
performance.

The analysis of Theorem \ref{thm:svm-sparse:main-compression} bounds the
performance of the \emph{randomized} classifier $\tilde{g}_{\tilde{w}}$, but we
have found that randomization \emph{hurts} performance in practice, and that
one is better off predicting using $\sign(g_{\tilde{w}})$.
Our technique relies on finding an approximate solution $\tilde{w}$ to Problem
\ref{eq:svm-sparse:objective-approximation} with $f(\tilde{w}) \le
\nicefrac{1}{2}$, but this $\nicefrac{1}{2}$ threshold is a relic of our use of
randomization. Since randomization does not help in practice, there is little
reason to expect there to be anything ``special'' about $\nicefrac{1}{2}$---one
may achieve a superior sparsity/generalization tradeoff at different levels of
convergence, and with values of the step-size $\eta$ other than the suggested
value of $\nicefrac{1}{2}$. For this reason, we suggest experimenting with
different values of these parameters, and choosing the best based on
cross-validation.

Another issue is the handling of an unregularized bias. In Section
\ref{subsec:svm-sparse:bias}, we give two alternatives: the first is to take
the bias associated with $\tilde{w}$ to be the same as that associated with
$w$, while the second is to learn $\tilde{b}$ during optimization The latter
approach results in a slightly more complicated subgradient descent algorithm,
but its use may result in a small boost in performance---hence, our reference
implementation uses this procedure.

\subsection{Aggressive Variant}\label{subsec:svm-sparse:aggressive}

A more substantial deviation from our basic algorithm is to try to be more
aggressive about maintaining sparsity by re-using existing support vectors when
optimizing Problem \ref{eq:svm-sparse:objective-approximation}. This can be
done in the the following way: at each iteration, check if there is a support
vector (i.e.~a training point already added to the support set) for which $h_i
- \inner{\tilde{w}}{\Phi(x_i)} \le \epsilon$ (where $\epsilon$ is the
termination threshold, $\nicefrac{1}{2}$ in the analysis of Section
\ref{sec:svm-sparse:algorithm}). If there is such a support vector, increase
its coefficient $\alpha_i$---only take a step on index $i$ which is not
currently in the support set if all current support vectors satisfy the
constraint. This yields a potentially sparser solution at the cost of more
iterations.

%

%% file: svm/sparse/sec-experiments.tex
\section{Experiments}\label{sec:svm-sparse:experiments}

\input{svm/sparse/figures/tab-datasets}
\input{svm/sparse/figures/fig-experiments}

Basing our experiments on recent comparisons between sparse SVM optimizers
\citep{KeerthiChDe06,NguyenMaTaHa10}, we compare our
implementation\footnote{\url{http://ttic.uchicago.edu/~cotter/projects/SBP}} to
the following methods\footnote{We were unable to find a reduced set
implementation on which we could successfully perform our experiments}:
\begin{enumerate}
\item SpSVM \citep{KeerthiChDe06}, using Olivier Chapelle's
implementation\footnote{\url{http://olivier.chapelle.cc/primal}}.
%
\item CPSP \citep{JoachimsYu09}, using SVM-Perf.
%
\item Osuna \& Girosi's algorithm \citep{OsunaGi98}, using LibSVM
\citep{ChangLi01} to optimize the resulting SVR problems.
\item RSVM \citep{LeeMa01}, using the LibSVM Tools
implementation \citep{LinLi03}.
\item CSVM \citep{NguyenMaTaHa10}. We did not perform these experiments
ourselves, and instead present the results reported in the CSVM paper.
\end{enumerate}
Our comparison was performed on the datasets listed in Table
\ref{tab:svm-sparse:datasets}.
Adult and IJCNN are the ``a8a'' and ``ijcnn1'' datasets from LibSVM Tools. Web
and Forest are from the LibCVM
Toolkit\footnote{\url{http://c2inet.sce.ntu.edu.sg/ivor/cvm.html}}. We also use
a multiclass
dataset\footnote{\url{http://ttic.uchicago.edu/~cotter/projects/gtsvm}} derived
from the TIMIT speech corpus, on which we perform one-versus-rest
classification, with class number $3$ (the phoneme \texttt{/k/}) providing the
``positive'' instances. Both Adult and TIMIT have relatively high error rates,
making them more challenging for sparse SVM solvers.
Both our algorithm and that of Osuna \& Girosi require a reference classifier
$w$, found using GTSVM \citep{CotterSrKe11}.

We experimented with two versions of our algorithm, both incorporating the
modifications of Section \ref{sec:svm-sparse:practical}, differing only in whether they
include the aggressive variation. For the ``basic'' version, we tried $\eta =
\{ 4^{-4}, 4^{-3}, \dots, 4^2 \}$, keeping track of the progress of the
algorithm throughout the course of optimization. For each support size, we
chose the best $\eta$ based on a validation set (half of the original test set)
reporting errors on an independent test set (the other half). This was then
averaged over $100$ random test/validation splits.

For the aggressive variant (Section \ref{subsec:svm-sparse:aggressive}), we experimented
not only with multiple choices of $\eta$, but also termination thresholds
$\epsilon \in \left\{ 2^{-4}, 2^{-3}, \dots, 1 \right\}$, running until this
threshold was satisfied. Optimization over the parameters was then performed
using the same validation approach as for the ``basic'' algorithm.
%

%

In our CPSP experiments, the target numbers of basis functions were taken to be
powers of two. For Osuna \& Girosi's algorithm, we took the SVR regularization
parameter $\tilde{C}$ to be that of Table \ref{tab:svm-sparse:datasets} (i.e. $\tilde{C} =
C$), and experimented with $\epsilon \in \left\{ 2^{-16}, 2^{-15}, \dots, 2^4
\right\}$. For RSVM, we tried subset ratios $\nu \in \left\{ 2^{-16}, 2^{-15},
\dots, 1 \right\}$---however, the implementation we used was unable to find a
support set of size larger than $200$, so many of the larger values of $\nu$
returned duplicate results.


The results are summarized in Figure \ref{fig:svm-sparse:experiments}. Our
aggressive variant achieved a test error / support size tradeoff comparable to
or better than the best competing algorithms, except on the Adult and TIMIT
datasets, on the latter of which performance was fairly close to that of CPSP.
Even the basic variant achieved very good results, often similar to or better
than other, more complicated, methods.  On the Adult data set, the test errors
(reported) are significantly higher then the validation errors, indicating our
methods are suffering from parameter overfitting due to too small a validation
set (this is also true, to a lesser degree, on TIMIT).
Note that SpSVM and CPSP, both of which perform very well, failed to find good
solutions on the forest dataset within a reasonable timeframe, illustrating the
benefits of the simplicity of our approach.

To summarize, not only does our proposed method achieve optimal theoretical
guarantees (the best possible sparseness guarantee with the best known sample
complexity and runtime for kernelized SVM learning), it is also computationally
inexpensive, simple to implement, and performs well in practice.

%% file: svm/sparse/figures/tab-datasets.tex
\begin{table}

\begin{small}
\begin{center}
\begin{tabular}{l|cc|cc}
\hline
& Training & Testing & $\gamma$ & $C$ \\
\hline
Adult  & $22696$  & $9865$  & $0.1$    & $1$     \\
IJCNN  & $35000$  & $91701$ & $1$      & $10$    \\
Web    & $49749$  & $14951$ & $0.1$    & $10$    \\
TIMIT  & $63881$  & $22257$ & $0.025$  & $1$     \\
Forest & $522910$ & $58102$ & $0.0001$ & $10000$ \\
\hline
\end{tabular}
\end{center}
\end{small}

\caption{
Datasets used in our experiments. Except for TIMIT, these are a subset of the
datasets, with the same parameters, as were compared in \citet{NguyenMaTaHa10}.
We use a Gaussian kernel $K(x,x') = \exp(-\gamma \norm{x-x'})$ with parameter
$\gamma$, and regularization tradeoff parameter $C$.
}

\label{tab:svm-sparse:datasets}

\end{table}

%% file: svm/sparse/figures/fig-experiments.tex
\begin{figure}

\begin{center}
\begin{tabular}{ @{} L @{} H @{} H @{} }
& \large{Adult} & \large{IJCNN} \\
\rotatebox{90}{\scriptsize{Test error}} &
\includegraphics[width=0.45\textwidth]{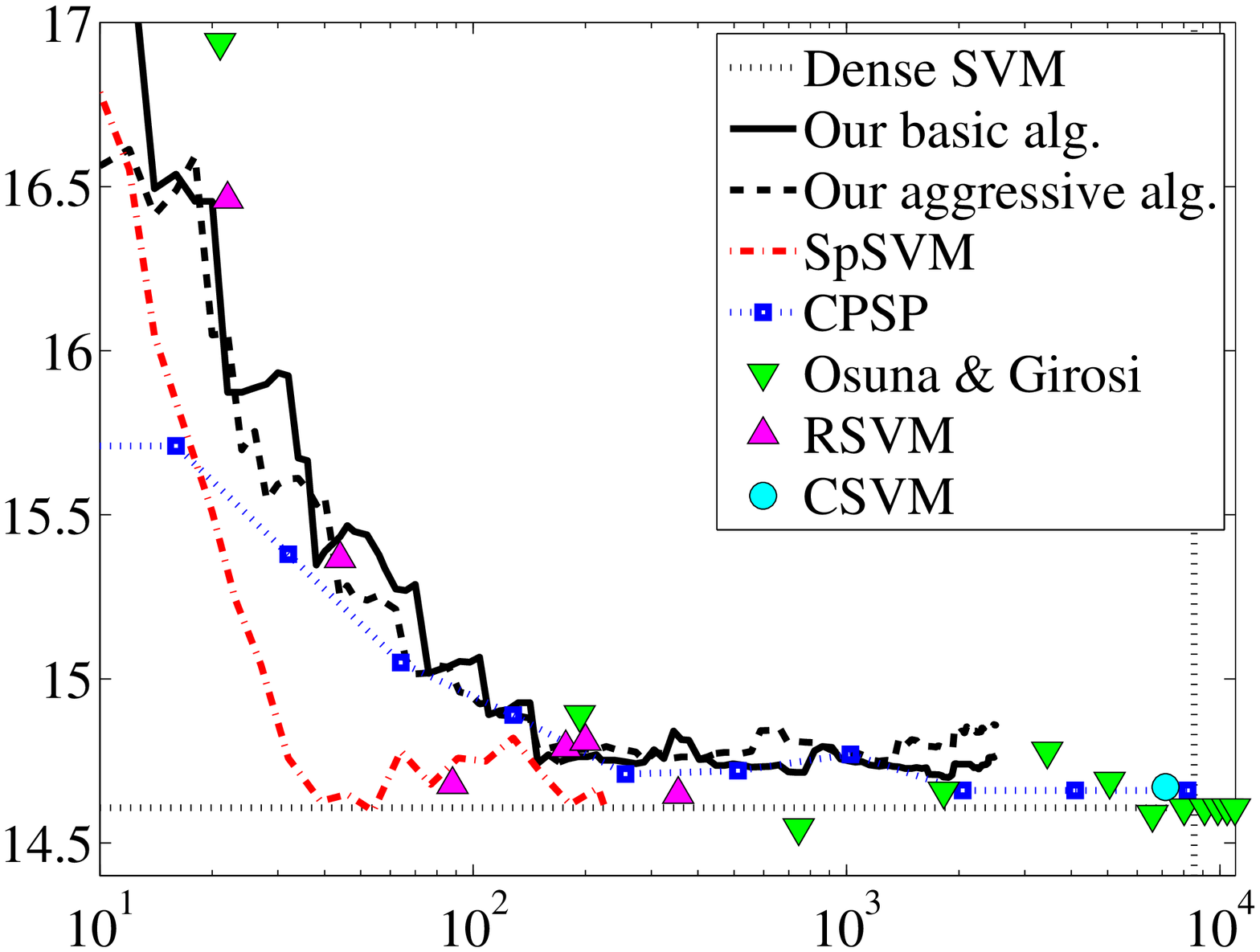} &
\includegraphics[width=0.45\textwidth]{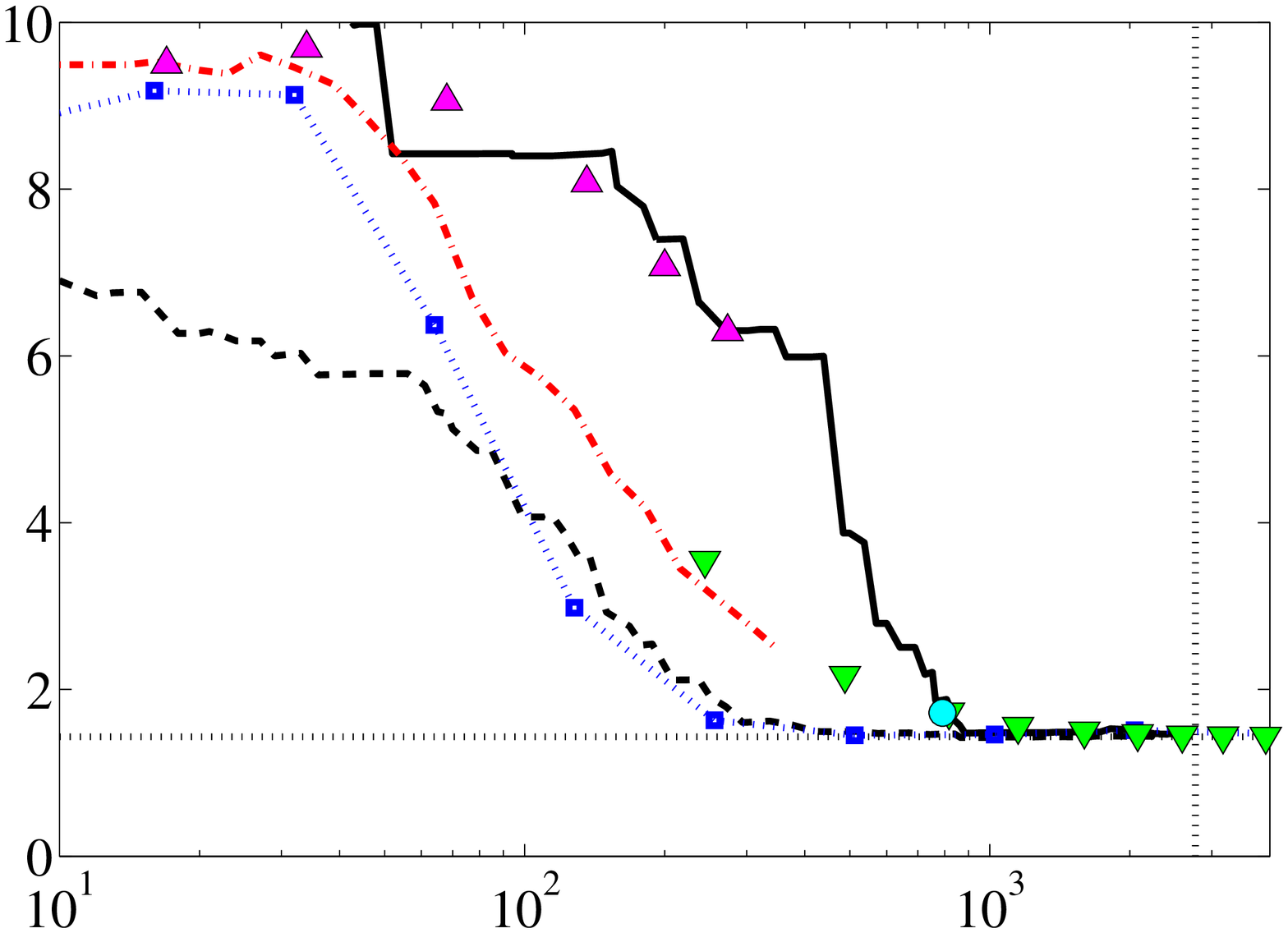} \\
& \scriptsize{Support Size} & \scriptsize{Support Size}
\end{tabular}
\begin{tabular}{ @{} L @{} H @{} H @{} }
& \large{Web} & \large{TIMIT} \\
\rotatebox{90}{\scriptsize{Test error}} &
\includegraphics[width=0.45\textwidth]{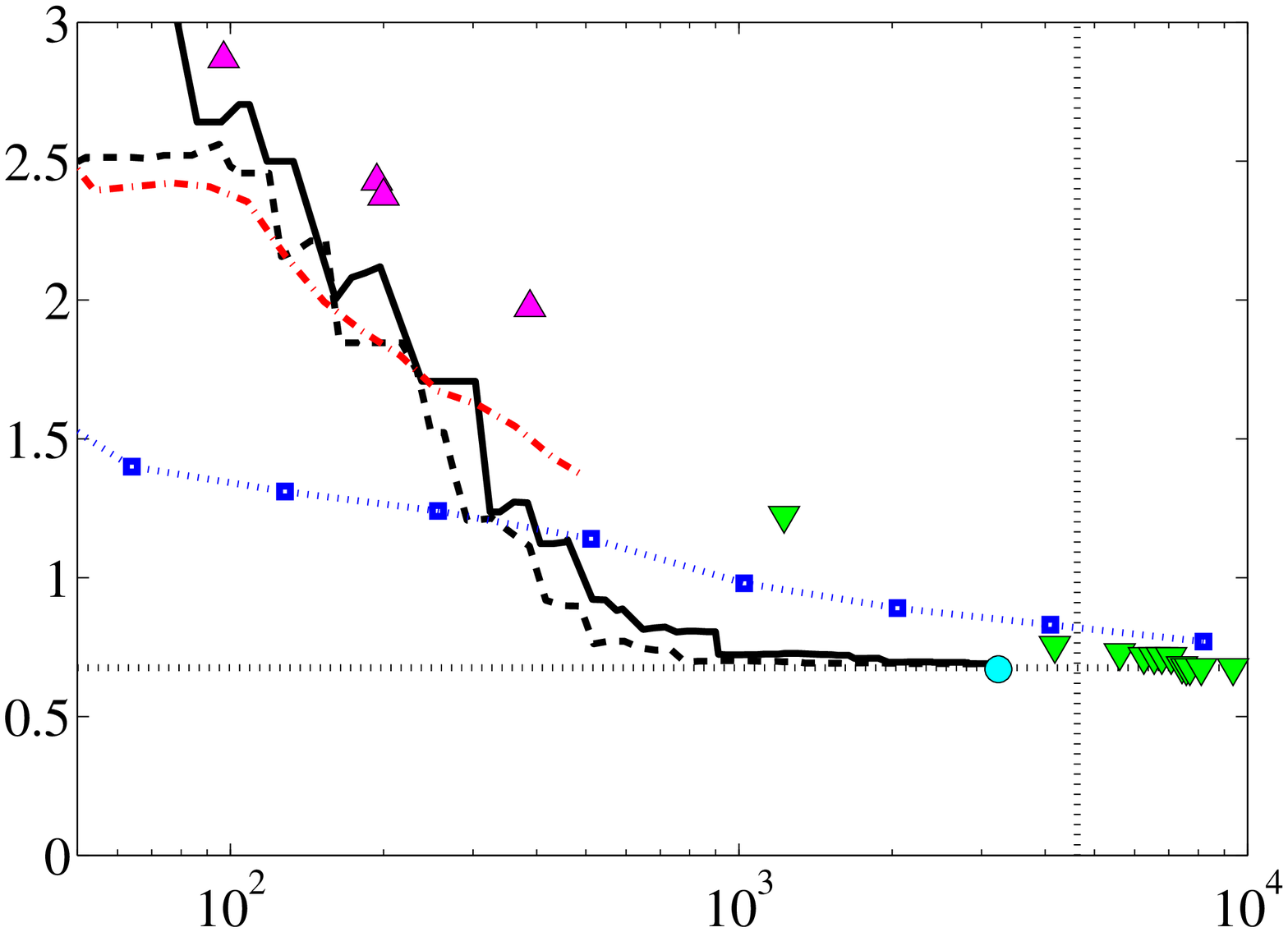} &
\includegraphics[width=0.45\textwidth]{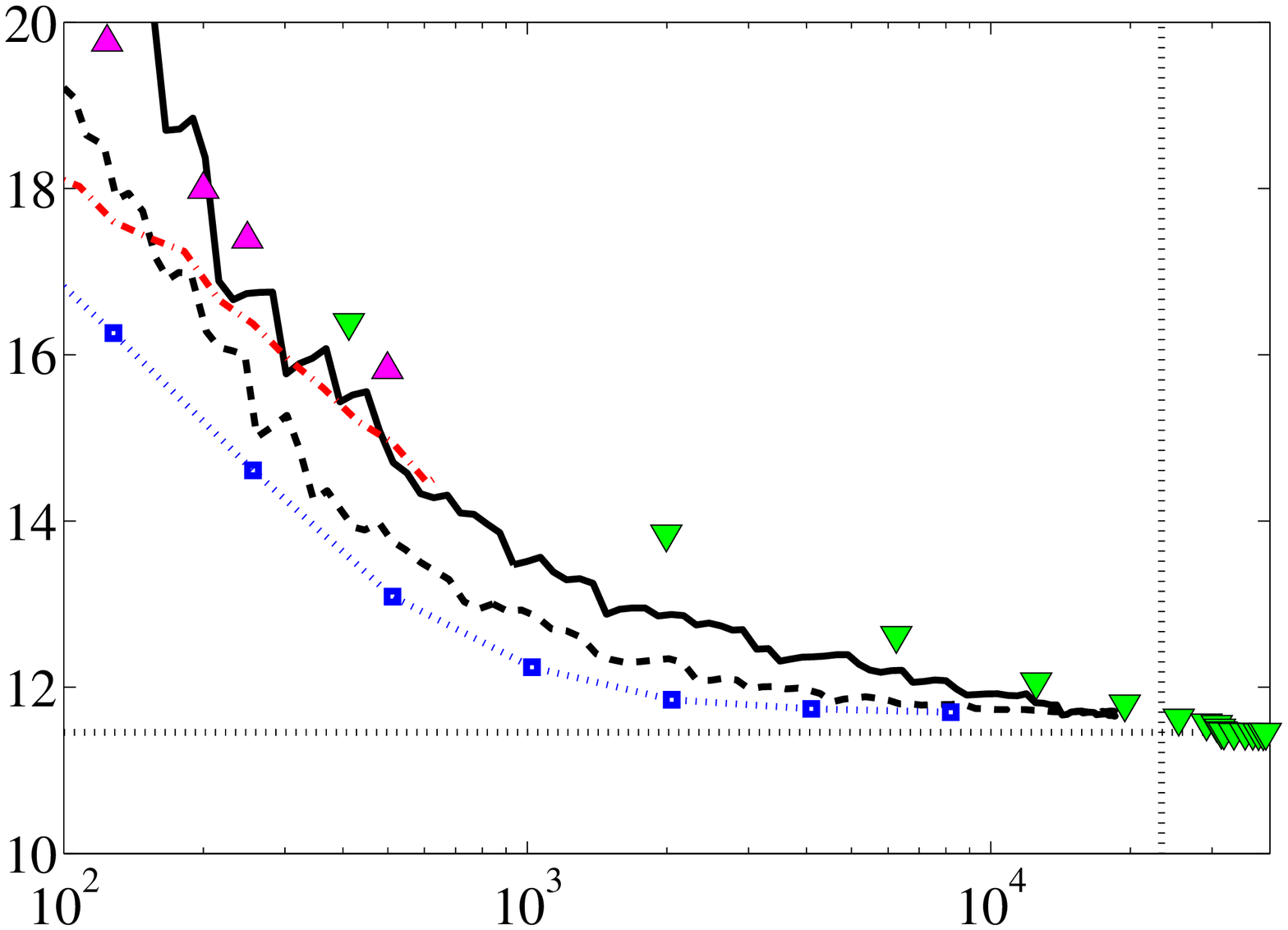} \\
& \scriptsize{Support Size} & \scriptsize{Support Size}
\end{tabular}
\begin{tabular}{ @{} L @{} H @{} }
& \large{Forest} \\
\rotatebox{90}{\scriptsize{Test error}} &
\includegraphics[width=0.45\textwidth]{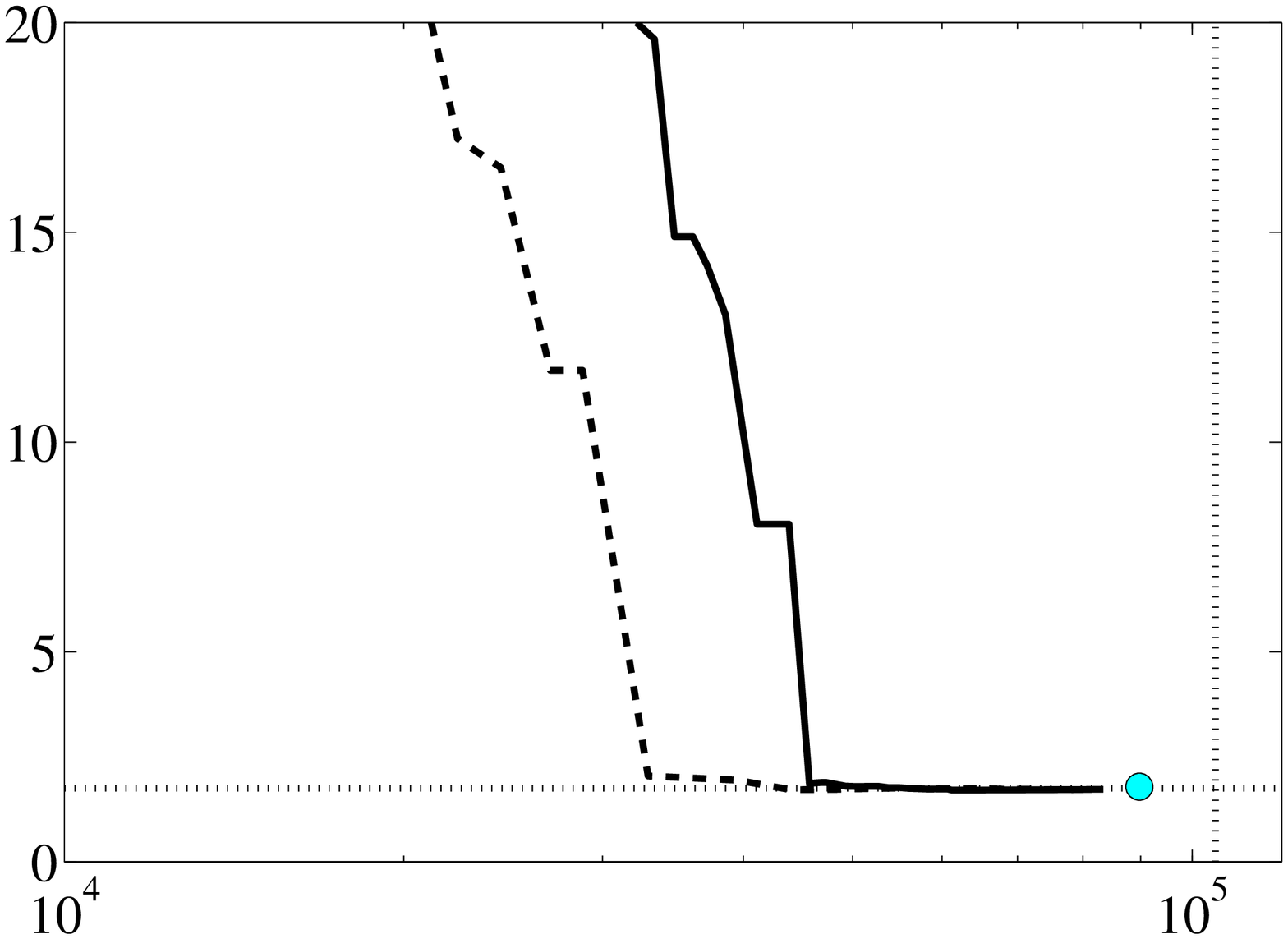} \\
& \scriptsize{Support Size}
\end{tabular}
\end{center}

\caption{
Plots of test error (linear scale) versus support size (log scale). The
horizontal and vertical dotted lines are the test error rate and support size
of the classifier found by GTSVM. TIMIT was not included in the experiments of
\citet{NguyenMaTaHa10}. On Forest, SpSVM ran out of memory, CPSP failed to
terminate in one week for 4096 or more basis functions, LibSVM failed to
optimize the SVR problem (Problem \ref{eq:svm-sparse:objective-osuna-girosi}) in $4$ days
for $\epsilon < 1$, and RSVM's solutions were limited to $200$ support vectors,
far too few to perform well on this dataset.
}

\label{fig:svm-sparse:experiments}

\end{figure}

%% file: svm/sparse/sec-proofs.tex
\section{Proofs for Chapter \ref{ch:svm-sparse}}

\begin{proofs}
\input{svm/sparse/theorems/lem-lower-bound}
\input{svm/sparse/theorems/lem-slant-loss}
\input{svm/sparse/theorems/lem-convergence-rate}
\input{svm/sparse/theorems/thm-compression-bound}
\input{svm/sparse/theorems/lem-generalization-from-compression}
\input{svm/sparse/theorems/thm-main-smooth}
\input{svm/sparse/theorems/thm-main-compression}
\end{proofs}

%% file: part-pca.tex
\ctparttext{
\todo{write part-abstract! include list of contributions, and cite publications}
}

\part{Principal Component Analysis}\label{part:pca} 

\include{pca/ch-introduction}

\include{pca/ch-warmuth}

\include{pca/ch-capped-msg}

\cleardoublepage 

%% file: pca/ch-introduction.tex
\chapter{Basic Algorithms}\label{ch:pca-introduction}

\input{pca/introduction/sec-overview}
\input{pca/introduction/sec-objective}
\input{pca/introduction/sec-saa}
\input{pca/introduction/sec-power}
\input{pca/introduction/sec-incremental}

\paragraph{Collaborators:} The novel content of this chapter (particularly
Section \ref{sec:pca-introduction:incremental}) was performed jointly with
Raman Arora, Karen Livescu and Nathan Srebro.

\clearpage
\input{pca/introduction/sec-proofs}

%% file: pca/introduction/sec-overview.tex
\section{Overview}\label{sec:pca-introduction:overview}

Principal Component Analysis (PCA) is a ubiquitous tool used in many data
analysis, machine learning and information retrieval applications. It is used
for obtaining a lower dimensional representation of a high dimensional signal
that still captures as much as possible of the original signal. Such a low
dimensional representation can be useful for reducing storage and computational
costs, as complexity control in learning systems, or to aid in visualization.

Uncentered PCA is typically phrased as a question about a fixed dataset: given
$n$ vectors in $\R^d$, what is the $k$-dimensional subspace that captures most
of the variance in the dataset? (or equivalently, that is best in
reconstructing the vectors, minimizing the sum squared distances, or residuals,
to the subspace). It is well known that this subspace is given by the leading
$k$ components of the singular value decomposition of the data matrix (or
equivalently the top $k$ eigenvectors of the empirical second moment matrix).
And so, the study of computational approaches for PCA has mostly focused on
methods for finding the SVD (or leading components of the SVD) of a given $n
\times d$ matrix.

Our treatment of this problem departs from the traditional statistical approach
in that we view PCA as a stochastic optimization problem, where the goal is to
optimize a ``population objective'' based on \iid draws from the population.
That is, we have some unknown source (``population'') distribution
$\mathcal{D}$ over $\R^d$, and the goal is to find the $k$-dimensional subspace
maximizing the (uncentered) variance of $\mathcal{D}$ inside the subspace (or
equivalently, minimizing the average squared residual in the population), based
on \iid samples from $\mathcal{D}$. The main point here is that the true
objective does not measure how well the subspace captures the {\em sample}
(i.e.~the ``training error''), but rather how well the subspace captures the
underlying source distribution (i.e.~the ``generalization error'').
Furthermore, we are not concerned here with capturing some ``true'' subspace
(in which case one might quantify success as e.g. the ``angle'' between the
found subspace and the ``true'' subspace), but rather at finding a ``good''
subspace, which has a near-optimal value of the PCA objective. This will be
formalized more in Section \ref{sec:pca-introduction:objective}.

The straightforward approach is ``Sample Average Approximation'' (SAA)
(i.e.~``Empirical Risk Minimization''), in which one collects a sample of data
points, and then optimizes an empirical version of the objective \emph{on the
sample} using standard deterministic techniques (in this case linear algebra).
In the case of uncentered PCA, this amounts to computing the empirical
second-moment matrix of the sample, and then seeking the best rank-$k$
approximation to it, e.g.~by computing the leading components of its
eigendecomposition. The success of this approach is measured not by how well we
approximate the {\em empirical} second-moment matrix, but rather how well the
subspace we obtain captures the unknown source distribution (i.e.~the {\em
population} second-moment matrix). This approach will be considered in greater
detail in Section \ref{sec:pca-introduction:saa}.

The alternative, which we advocate here, is a ``Stochastic Approximation'' (SA)
approach. A SA algorithm is iterative---in each iteration a single sampled
point is used to perform an update, as in Stochastic Gradient Descent (SGD, the
canonical stochastic approximation algorithm). In the context of PCA, one
iteratively uses vectors sampled from $\mathcal{D}$ to update the subspace
being considered.

Stochastic approximation has been shown to be computationally preferable to
statistical average approximation (i.e.~to ``batch'' methods) both
theoretically and empirically for learning~\citep{BottouBo07,ShalevSr08} and
more broadly for stochastic optimization~\citep{NemirovskiJuLaSh09}.
Accordingly, SA approaches, mostly variants of SGD, are often the methods of
choice for many learning problems, especially when very large datasets are
available~\citep{ShalevSiSr07,CollinsGlKoCaBa08,ShalevTe09}.

This chapter will begin, in Section \ref{sec:pca-introduction:objective}, with
a description of two equivalent versions of the PCA objective: one of which
represents the underlying PCA subspace as a set of vectors spanning it,
represented as the orthonormal columns of a rectangular matrix $U$; the other
as a projection matrix $M$ which projects onto the PCA subspace. We call these
the $U$-based and $M$-based objectives, respectively. Section
\ref{sec:pca-introduction:saa} will discuss the SAA approach, while, in Section
\ref{sec:pca-introduction:power}, we will present the ``stochastic power
method'', a catch-all algorithm of which many variants may be found in
literature dating back several decades, and highlight its relationship to the
well-known power method for finding the maximum eigenvectors of a matrix. In
Section \ref{sec:pca-introduction:incremental}, an ``incremental algorithm''
will be discussed. Both of the latter two algorithms have been found to work
well in practice, but suffer from serious theoretical limitations---the former
is known to converge with probability one, but the rate of convergence is
unknown, while the latter fails to converge entirely on certain problem
instances. Portions of this chapter, in particular Sections
\ref{sec:pca-introduction:power} and \ref{sec:pca-introduction:incremental},
were originally presented at the 50th Allerton Conference on Communication,
Control and Computing \citep{AroraCoLiSr12}.

These basic algorithms form the foundation for Chapters \ref{ch:pca-warmuth}
and \ref{ch:pca-capped-msg}, which discuss theoretically justified and
analyzable algorithms for solving PCA problems, the first originally due to
\citet{WarmuthKu06}, and the second novel.

%% file: pca/introduction/sec-objective.tex
\section{Objective}\label{sec:pca-introduction:objective}

\input{pca/introduction/figures/tab-notation}
The goal of uncentered PCA is to find, for a distribution $\mathcal{D}$ over
vectors $x\in\R^d$, the subspace of dimension $k$ for which the projections of
$x$ onto this subspace have maximal uncentered second sample moments. To
simplify the presentation, we will assume, in this chapter as well as Chapters
\ref{ch:pca-warmuth} and \ref{ch:pca-capped-msg}, that $\norm{x}\le 1$ with
probability $1$ for $x \sim \mathcal{D}$---this is a relatively weak
assumption, since any bounded data distribution can be adjusted to satisfy it
through scaling. There are a number of alternative perspectives which one can
take, which ultimately lead to the same problem formulation. For example, one
may wish to find a lower-dimensional representation of the data which minimizes
the $L^{2}$ reconstruction error:
\begin{align*}
\underset{U^{T}U=I}{\argmin} \expectation[x\sim\mathcal{D}]{\norm{ x - U U^T x
}_{2}^{2} }
\end{align*}
where $U\in\R^{d\times k}$ has orthonormal columns, which, at the optimum, will
be the $k$ maximal eigenvectors of the second moment matrix $\Sigma =
\expectation[x\sim\mathcal{D}]{ x x^T }$. Alternatively, perhaps one wishes to
find the projection $U U^T$ of $\Sigma$ onto a $k$-dimensional subspace such
that the result is closest to $\Sigma$ in terms of the trace norm:
\begin{equation*}
\underset{U^{T}U=I}{\argmin} \trace \left( U^T \Sigma U - \Sigma \right)
\end{equation*}
Both of these examples may be simplified to the problem of finding a $U$
satisfying:
\begin{equation*}
\underset{U^{T}U=I}{\argmin} \expectation[x\sim\mathcal{D}]{ x^T U U^T x }
\end{equation*}
The orthonormality constraint on the columns of $U$ may be weakened to the
constraint that all eigenvalues of $U^{T}U$ must be at most one, because if any
eigenvalue of this matrix is less than one, then increasing it (in the same
basis) will only increase the above objective function. Hence, for the optimal
$U^*$, all of the eigenvalues of $\left(U^*\right)^T U^*$ will be $1$, implying
that $\left(U^*\right)^T U^*=I$, and therefore that the columns of $U$ are
orthonormal. This allows us to state the PCA objective in the language of
optimization, as: 
\begin{align}
\label{eq:pca-introduction:objective-u} \mbox{maximize}: &
\expectation[x\sim\mathcal{D}]{ x^T U U^T x } \\
\notag \mbox{subject to}: & U \in \R^{d\times k}, U^{T}U \preceq 1
\end{align}

Some of the algorithms which we will consider work not by optimizing $U$, but
instead by changing the optimization variables in such as way as to optimize
over a rank-$k$ positive semidefinite matrix $M\in\R^{d\times d}$, which we can
think of as satisfying $M = UU^T$, giving the objective:
\begin{align}
\label{eq:pca-introduction:objective-m} \mbox{maximize}: &
\expectation[x\sim\mathcal{D}]{x^T M x} \\
\notag \mbox{subject to} : & M \in \R^{d\times d}, \spectrum[i]{M} \in \left\{
	0, 1 \right\}, \rank M = k
\end{align}
Here, $\spectrum[i]{M}$ is the $i$th eigenvalue of $M$, so the constraint on
the eigenvalues, combined with the rank constraint, forces $M$ to have exactly
$k$ eigenvalues equal to $1$, and $n-k$ equal to $0$---in other words, $M$ is
a rank-$k$ projection matrix.

Problems \ref{eq:pca-introduction:objective-u} and
\ref{eq:pca-introduction:objective-m} are both {\em stochastic}, in that their
objective functions are expectations over the distribution $\mathcal{D}$. In
the situation we consider, and which we argue is typical in practice, we do not
have direct knowledge of the distribution $\mathcal{D}$, and so cannot exactly
calculate the (population) objective, let alone optimize it. Instead, we only
have access to \iid samples from $\mathcal{D}$---these can be thought of as
``training examples''. The regime we are mostly concerned with is that in which
we have an essentially unlimited supply of training examples, and would like to
obtain an $\epsilon$-suboptimal solution in the least possible runtime. That
is, one can think of access to an ``example oracle'' that generates an example
on-demand, at the cost of reading the sample. We refer to such a regime, where
data is abundant and the resources of interest are runtime and perhaps memory
consumption, as the ``data laden'' regime.

The algorithms which we consider are all \emph{iterative}, in that they
consider training examples $x_t$ one-at-a-time. At every step, an estimate
$U^{(t)}$ or $M^{(t)}$ is produced based on some internal state of the
algorithm, a new example is obtained, and a loss is incurred in terms of the
component of the example not explained by the current iterate. This loss
(residual) is then used to update the internal state of the algorithm. 

Problem \ref{eq:pca-introduction:objective-u} is a quadratic objective
subject to quadratic constraints, but because it is a \emph{maximization}
problem, it is \emph{not} a convex optimization problem. Likewise, Problem
\ref{eq:pca-introduction:objective-m} is a linear objective subject to
non-convex constraints. As a result, we may not immediately appeal to the vast
literature on convex optimization to efficiently optimize either of these
objectives.

The fact that the optimization problems corresponding to PCA are {\em not}
convex is a major complication in designing and studying stochastic
approximation methods for PCA. The empirical optimization problem is still
tractable due to algebraic symmetry, and the eigendecomposition can be computed
either through algebraic elimination, or through iterative local search methods
such as the power method. However standard methods and analyses for stochastic
convex optimization are not immediately applicable. This
non-convex-but-tractable situation poses a special challenge, and we are not
aware of other similar situations where stochastic optimization has been
studied. Rather, in most learning applications of stochastic optimization,
either stochastic gradient descent, stochastic mirror descent, or variants
thereof are directly applicable, or global optimality cannot be ensured even
for the deterministic empirical optimization problem (e.g.~when using
stochastic approximation to train deep networks).

One method for addressing this difficulty is to essentially ignore it, and use
standard optimization algorithms (e.g. stochastic gradient ascent), while
attempting to prove whatever bounds can be found. In Section
\ref{sec:pca-introduction:power}, we will see an algorithm which takes this
approach. Alternatively, one may continue manipulating this objective until one
finds an equivalent relaxed convex optimization problem, which may then be
optimized efficiently using conventional techniques. Chapters
\ref{ch:pca-warmuth} and \ref{ch:pca-capped-msg} will give examples of this
approach.

%% file: pca/introduction/figures/tab-notation.tex
\begin{table}

\begin{small}
\begin{center}
\begin{tabularx}{\linewidth}{lX}
\hline
& Description \\
\hline
$\mathcal{D}$ & Data distribution such that $\norm{x}\le 1$ for $x \sim \mathcal{D}$\\
$n,T\in\N$ & Training size ($n$ for a ``batch'' algorithm, $T$ for a stochastic algorithm) \\
$d\in\N$ & Data dimension \\
$k\in\N$ & PCA subspace dimension \\
$x_1,\dots,x_n\in\R^d$ & Training samples \\
$U\in\R^{n\times d}$ & PCA solution (columns span the maximal space) \\
$M\in\R^{n\times n}$ & PCA solution (projection matrix onto the maximal space, or ``relaxed'' projection matrices in Chapter \ref{ch:pca-capped-msg}) \\
$W\in\R^{n\times n}$ & PCA solution (``relaxed'' projection matrix onto the \emph{minimal} space in Chapter \ref{ch:pca-warmuth}) \\
$k_t'\in\N$ & Number of nontrivial eigenvalues of the $t$th iterate found by the algorithms of Chapters \ref{ch:pca-warmuth} and \ref{ch:pca-capped-msg} \\
$K\in\N$ & Upper bound on the rank of the iterates found by the capped MSG algorithm of Chapter \ref{ch:pca-capped-msg} \\
\hline
\end{tabularx}
\end{center}
\end{small}

\caption{
Summary of common notation across Chapters \ref{ch:pca-introduction},
\ref{ch:pca-warmuth} and \ref{ch:pca-capped-msg}. In the two latter chapters,
$M$ will be relaxed to not be a projection matrix onto the PCA space, but
rather a (potentially full-rank) PSD matrix for which the magnitude of each
eigenvalue represents the likelihood that the corresponding eigenvector is one
of the principal components.
}

\label{tab:pca-introduction:notation}

\end{table}

%% file: pca/introduction/sec-saa.tex
\section{Sample Average Approximation (SAA)}\label{sec:pca-introduction:saa}

\input{pca/introduction/figures/tab-bounds}
In Section \ref{sec:pca-introduction:overview}, we mentioned the Sample Average
Approximation algorithm, which is nothing more than the ``traditional''
technique for solving stochastic PCA problems: draw $n$ samples from
$\mathcal{D}$, calculate their empirical second moment matrix, and find its
eigendecomposition to derive the top-$k$ eigenvectors.

Because we're working in the stochastic setting, this is not, as may at first
be assumed, an exact solution---we want to capture most of the variance in the
\emph{true} second moment matrix $\Sigma$ in a $k$-dimensional subspace, while
SAA gives us only the maximal subspace based on an \emph{empirical} second
moment matrix $\hat{\Sigma}$. The quality of its solution depends on how
accurately $\hat{\Sigma}$ approximates the unknown ground truth $\Sigma$, which
depends on the number of samples upon which $\hat{\Sigma}$ is based. One can
derive just such a bound with a Rademacher complexity based
analysis~\citep{BartlettMe03}:

\input{pca/introduction/theorems/lem-saa-bound.tex}

This bound tells only part of the story, since the computational cost of the
SAA procedure is so high. Merely calculating the empirical second moment
matrix, to say nothing of finding its eigendecomposition, costs $O(n d^2)$
operations for a na\"{i}ve implementation, $O\left(n d^{\log_2 7 - 1}\right)$
using Strassen's algorithm \citep{WikipediaMatrixMultiplication,Strassen69} to
multiply the matrix of samples with its transpose (thereby calculating
$\hat{\Sigma}$), or $O\left( n d^{1.3727} \right)$ using
\citet{CoppersmithWi90}, the fastest-known matrix multiplication
algorithm~\citep{WikipediaMatrixMultiplication}. Both of the latter two
algorithms (particularly the second) are widely-regarded as impractical except
in highly specialized circumstances, so we will treat the computational cost as
$O(n d^2)$. The memory requirements of this algorithm are also relatively high:
$d^2$ to store $\hat{\Sigma}$.
As we will see in Chapters \ref{ch:pca-warmuth} and \ref{ch:pca-capped-msg},
there are stochastic algorithms which have similar sample complexity bounds to
that of Lemma \ref{lem:pca-introduction:saa-bound}, and a lower
computational cost. Even in this chapter, we will consider algorithms which are
much ``cheaper'', in that they perform less work-per-sample, and work very well
in practice, although there are no known bounds on their rate of convergence.

%% file: pca/introduction/figures/tab-bounds.tex
\begin{table}

\begin{small}
\begin{center}
\begin{tabular}{r|cc|c}
\hline
& Computation & Memory & Convergence \\
\hline
SAA & $nd^2$ & $d^2$ & $\sqrt{\frac{k}{n}}$ \\
SAA (Coppersmith-Winograd) & $nd^{1.3727}$ & $d^2$ & $\sqrt{\frac{k}{n}}$ \\
Stochastic Power Method & $Tkd$ & $kd$ & w.p. $1$ \\
Incremental & $T k^2 d$ & $kd$ & no \\
\hline
\end{tabular}
\end{center}
\end{small}

\caption{
Summary of results from Sections \ref{sec:pca-introduction:saa},
\ref{sec:pca-introduction:power} and \ref{sec:pca-introduction:incremental}.  All
bounds are given up to constant factors. The ``Convergence'' column contains
bounds on the suboptimality of the solution---i.e. the difference between the
total variance captured by the rank-$k$ subspace found by the
algorithm, and the best rank-$k$ subspace with respect to the data distribution
$\mathcal{D}$. The stochastic power method converges with probability $1$, but
at an unknown rate, while there exist data distributions for which the
incremental algorithm fails to converge entirely. The ``Coppersmith-Winograd''
variant of SAA uses an asymptotically fast matrix multiplication
algorithm~\citep{CoppersmithWi90} to calculate the empirical second moment
matrix, and is unlikely to be useful in practical applications.
}

\label{tab:introduction:bounds}

\end{table}

%% file: pca/introduction/theorems/lem-saa-bound.tex
\medskip
\begin{splitlemma}{lem:pca-introduction:saa-bound}
Suppose that $U\in\R^{d\times k}$ has orthonormal columns spanning the maximal
subspace of an empirical covariance matrix $\hat{\Sigma} = (1/n) \sum_{i=1}^n
x_i x_i^T$ over $n$ samples drawn \iid from $\mathcal{D}$. Likewise, let
$U^*$ be the corresponding matrix with $k$ orthogonal columns spanning the
maximal subspace of the true covariance $\Sigma = \expectation[x \sim
\mathcal{D}]{xx^T}$. Then, with probability $1-\delta$:
\begin{equation*}
\trace \left(U^*\right)^T \Sigma U^* - \trace U^T \Sigma U \le
O\left( \sqrt{\frac{k \log\frac{1}{\delta}}{n}} \right)
\end{equation*}
\end{splitlemma}
\begin{splitproof}
Observe that, while Lemma \ref{lem:pca-introduction:saa-rademacher-bound} holds
uniformly for all $U$, we are now considering only a \emph{particular} $U$,
albeit a random variable depending on the sample (hence the need for a uniform
bound).  Because $U$ is the empirical optimum, it follows that $\trace \left( I
- U U^T \right)\hat{\Sigma} \le \trace \left( I - U^* \left(U^*\right)^T
\right)\hat{\Sigma}$. Combining this fact with Lemmas
\ref{lem:pca-introduction:saa-rademacher-bound} and
\ref{lem:pca-introduction:saa-rademacher} yields that, with probability
$1-\delta$:
\begin{equation*}
\trace \left(I - U U^T \right) \Sigma - \trace \left( I - U^*
\left(U^*\right)^T \right) \hat{\Sigma} \le \sqrt{\frac{k}{n}} + \sqrt{\frac{8
\log \frac{2}{\delta}}{n}}
\end{equation*}
To complete the proof, we need only prove a bound on the difference between
$\trace \left( I - U^* \left(U^*\right)^T \right)\hat{\Sigma}$ and $\trace
\left( I - U^* \left(U^*\right)^T \right)\Sigma$. By Hoeffding's inequality:
\begin{equation*}
\probability{\trace \left(I-U^* \left(U^*\right)^T\right)\left( \hat{\Sigma} -
\Sigma \right) \ge \epsilon} \le \exp\left( -\frac{n\epsilon}{2} \right)
\end{equation*}
Setting the RHS to $\delta$ and solving for $\epsilon$ yields that:
\begin{equation*}
\epsilon = \frac{2 \log\frac{1}{\delta}}{n}
\end{equation*}
Hence, with probability $1-2\delta$:
\begin{equation*}
\trace \left(I - U U^T \right) \Sigma - \trace \left( I - U^*
\left(U^*\right)^T \right) \Sigma \le \sqrt{\frac{k}{n}} +
\sqrt{\frac{2\log\frac{1}{\delta}}{n}} + \sqrt{\frac{8 \log
\frac{2}{\delta}}{n}}
\end{equation*}
Canceling the two $\trace \Sigma$ terms on the LHS, negating the inequality
and simplifying yields the claimed result.
\end{splitproof}

%% file: pca/introduction/sec-power.tex
\section{The Stochastic Power Method}\label{sec:pca-introduction:power}

For convex optimization problems, stochastic gradient descent is a simple and
often highly efficient optimization technique. As was previously mentioned, the
$U$-optimizing PCA objective function of Problem
\ref{eq:pca-introduction:objective-u} is convex, as is the constraint, but as
the goal is \emph{maximization} of this objective, the formulation of Equation
\ref{eq:pca-introduction:objective-u} is \emph{not} convex as an optimization problem.
However, stochastic gradient descent is still a viable algorithm.

\subsection{The Power Method}

If $\Sigma=\expectation[x\sim\mathcal{D}]{x x^T}$ were known exactly, then the
gradient of the PCA objective function $\expectation[x\sim\mathcal{D}]{x^T UU^T
x} = \trace ( U^T \Sigma U )$ with respect to $U$ would be $2 \Sigma U$,
leading one to consider updates of the form:
\begin{equation}
\label{eq:pca-introduction:power-method} U^{(t+1)} = \project[\textrm{orth}]{
	U^{(t)} + \eta \Sigma U^{(t)} }
\end{equation}
where $\project[\textrm{orth}]{U}$ performs a projection with respect to the
spectral norm of $U U^T$ onto the set of $d \times d$ matrices with $k$
eigenvalues equal to $1$ and the rest $0$ (calling this a ``projection'' is a
slight abuse of terminology, since it is $U U^T$ which is projected, not $U$
itself).

One advantage which we have in the non-stochastic setting (i.e. when $\Sigma$
is known) is that we may analytically determine the optimal value of the step
size $\eta$. To this end, let's consider only the one-dimensional case (i.e.
$U$ is a column vector $u$). With this simplification in place, projection onto
the constraint of Problem \ref{eq:pca-introduction:objective-u} can be
accomplished through normalization, giving the equivalent problem:
\begin{equation}
\label{eq:pca-introduction:objective-step-size} \mbox{maximize}: \frac{u^T \Sigma u}{u^T u} 
\end{equation}
Because this objective function is invariant to the scale of $U$, we may, for
reasons which will become clear shortly, rewrite the update of Equation
\ref{eq:pca-introduction:power-method} with the step size $\eta$ applied to
the \emph{first} term, instead of the second:
\begin{equation}
\label{eq:pca-introduction:power-method-step-size} u^{(t+1)} = \eta u^{(t)} + \Sigma u^{(t)}
\end{equation}
Notice that we have removed the projection step, since this is now handled as a
part of the objective function. Assume without loss of generality that $\Sigma
= \diag\left( \sigma_1, \sigma_2, \dots, \sigma_d \right)$ is diagonal.
Substituting Equation \ref{eq:pca-introduction:power-method-step-size} into
Equation \ref{eq:pca-introduction:objective-step-size} and maximizing over
$\eta$ will give the optimal step size:
\begin{equation*}
\mbox{maximize}: \frac{\sum_{i=1}^d u_i^2 \sigma_i \left( \eta + \sigma_i
\right)^2}{\sum_{i=1}^d u_i^2 \left( \eta + \sigma_i \right)^2}
\end{equation*}
Differentiating with respect to $\eta$:
\begin{align*}
\frac{\partial}{\partial \eta} = & 2 \frac{\sum_{i=1}^d u_i^2 \sigma_i \left(
\eta + \sigma_i \right)} {\sum_{i=1}^d u_i^2 \left( \eta + \sigma_i \right)^2}
- 2 \frac{\left( \sum_{i=1}^d u_i^2 \sigma_i \left( \eta + \sigma_i \right)^2
\right)\left( \sum_{i=1}^d u_i^2 \left( \eta + \sigma_i \right) \right)}
{\left( \sum_{i=1}^d u_i^2 \left( \eta + \sigma_i \right)^2 \right)^2} \\
= & -2 \frac{ \sum_{i=1}^d \sum_{j=1}^d u_i^2 u_j^2 \sigma_i \left( \eta +
\sigma_i \right) \left( \eta + \sigma_j \right) \left( \sigma_i - \sigma_j
\right) } {\left( \sum_{i=1}^d u_i^2 \left( \eta + \sigma_i \right)^2
\right)^2} \\
= & -2 \frac{ \sum_{i=1}^d \sum_{j=1}^{i-1} u_i^2 u_j^2 \left( \eta + \sigma_i
\right) \left( \eta + \sigma_j \right) \left( \sigma_i - \sigma_j \right)^2 }
{\left( \sum_{i=1}^d u_i^2 \left( \eta + \sigma_i \right)^2 \right)^2} \\
\end{align*}
Since this derivative is always negative for $\eta \ge 0$ (note that it will be
positive for some negative choices of $\eta$, depending on the spectrum of
$\Sigma$), we see that $\eta=0$ is the optimal nonnegative choice for the
learning rate, yielding the widely-used power method~\citep{Gu00}:
\begin{equation*}
u^{(t+1)} = \Sigma u^{(t)}
\end{equation*}
This shows that the power method can be viewed as an instance of the gradient
ascent algorithm with an exact line search. For this reason, we refer to
\emph{stochastic} gradient ascent on Problem
\ref{eq:pca-introduction:objective-u} as the ``stochastic power method''.

\subsection{Stochastic Gradient Ascent}

Since $2 x x^T U$ is equal in expectation to $2\Sigma U$, which is the gradient
of the objective of Problem \ref{eq:pca-introduction:objective-u}, we may
perform stochastic gradient ascent by iteratively sampling $x_t\sim\mathcal{D}$
at the $t$th iteration, and performing the following update:
\begin{equation}
\label{eq:pca-introduction:stochastic-power-method} U^{(t+1)} =
\project[\textrm{orth}]{ U^{(t)} + \eta_t x_t x_t^T U^{(t)}}
\end{equation}
This is the ``stochastic power method''. Notice that finding $x x^T U$ requires
only $O( kd )$ operations (two matrix-vector multiplies). The renormalization
step represented by $\mathcal{P}_{\textrm{orth}}$ can be performed in $O( k^2
d)$ operations using, e.g., the Gram-Schmidt procedure. However, it turns out
that it is not necessary to renormalize, except for numerical reasons. To see
this, suppose that we \emph{do} renormalize $U^{(t)}$ after each iteration.  We
may then write $U^{(t)} = Q^{(t)} R^{(t)}$ with $Q^{(t)}$ having orthonormal
columns and $R^{(t)}$ being a nonsingular $k\times k$ matrix (this is not
necessarily a QR factorization, although it may be, if one renormalizes using
Gram-Schmidt). The matrix $Q^{(t)}$ is then the renormalized version of
$U^{(t)}$. With this representation of renormalization, the $1$-step SGD update
of Equation \ref{eq:pca-introduction:stochastic-power-method} is:
\begin{align*}
U^{(t+1)} = & Q^{(t)} + \eta_t x_t x_t^T Q^{(t)}, \\
U^{(t+1)} R^{(t)} = & U^{(t)} + \eta_t x_t x_t^T U^{(t)}
\end{align*}
From this equation, it is easy to prove by induction on $t$ that if $V^{(t)}$
is the sequence of iterates which would result if renormalization was
\emph{not} performed, then $V^{(t)} = Q^{(t)} R^{(t)} R^{(t-1)} \cdots
R^{(1)}$. Because $R^{(t)} R^{(t-1)} \cdots R^{(1)}$ is a product of
nonsingular matrices, it is nonsingular, showing that $V^{(t)}$ and $Q^{(t)}$
span the same subspace.

As a result of this observation, renormalization may be performed for purely
numerical reasons, and only very infrequently. Hence, the computational cost of
renormalization may be ignored, showing that performing $T$ iterations of SGD
costs only $O( Tkd )$ operations and $O\left( kd \right)$ memory (to store
$U$), both of which are better by a factor of $d/k$ than the cost of
``na\"{i}ve'' SAA (see Section \ref{sec:pca-introduction:saa}), if $T=n$. For
small $k$ and large $d$, this represents an enormous potential performance
difference over non-stochastic linear algebra-based methods.

Although not presented as instances of SGD, there are a number of algorithms in
the literature that perform precisely the above SGD update, differing only in
how they renormalize. For example, \citet{OjaKa85} perform Gram-Schmidt
orthonormalization after every iteration, while the popular generalized Hebbian
algorithm \citep{Sanger89}, which was later generalized to the kernel PCA
setting by \citet{KimFrSc05}, performs a partial renormalization. Both of these
algorithms converge with probability $1$ (under certain conditions on the
distribution $\mathcal{D}$ and step sizes $\eta_t$). However, the \emph{rate} of
convergence is not known.

%% file: pca/introduction/sec-incremental.tex
\section{The Incremental Algorithm}\label{sec:pca-introduction:incremental}

One of the most straightforward ways to perform PCA on the $M$-based objective
of Problem \ref{eq:pca-introduction:objective-m} is empirical risk
minimization (ERM): at every step $t$, take $C^{(t)} = \frac{1}{t}
\sum_{s=1}^{t} x_s x_s^T$ to be the empirical second-moment matrix of all of
the samples seen so far, calculate its eigendecomposition, compute the top $k$
eigenvectors, say $U^{(t)}$, of $C^{(t)}$ and take $M^{(t)} = U^{(t)} \left(
U^{(t)} \right)^T$.

Despite being perfectly sensible, this is far from being a practical solution.
Calculating $C^{(t+1)}$ from $C^{(t)}$ requires $O(d^2)$ operations, to say
nothing of then finding its eigendecomposition---this algorithm is simply far
too expensive.  One can, however, perform \emph{approximate} ERM at a much
lower computational cost by explicitly constraining the rank of the
second-moment estimates $C^{(t)}$, and updating these estimates incrementally,
as each new sample is observed \citep{AroraCoLiSr12}. Rather than defining
$(t+1) C^{(t+1)} = t C^{(t)} + x_t x_t^T$, one instead takes:
\begin{equation*}
(t+1) \tilde{C}^{(t+1)} = \project[\textrm{rank-$k$}]{t \tilde{C}^{(t)} + x_t
x_t^T}
\end{equation*}
where $\project[\textrm{rank-$k$}]{\cdot}$ projects its argument onto the set
of rank-$k$ matrices with respect to the Frobenius norm (i.e. sets all but the
top $k$ eigenvalues to zero).

\input{pca/introduction/figures/alg-rank1-update}
This update can be performed efficiently by maintaining an up-to-date
eigendecomposition of $\tilde{C}^{(t)}$ which is updated at every iteration.
Take $t\tilde{C}^{(t)} = U \code{diag}(\sigma) U^T$ to be an eigendecomposition
of $t\tilde{C}^{(t)}$, where $U\in \R^{d \times k'}$ has orthonormal columns,
and the rank $k'$ of $\tilde{C}^{(t)}$ satisfies $k' \le k$. In order to find
an eigendecomposition of $t\tilde{C}^{(t)} + x_t x_t^T$, we will consider the
component of $x_t$ which lies in the span of the columns of $U$ separately from
the orthogonal component $x_\perp = (I-U)(I-U)^T x_t$ with norm
$r=\norm{x_\perp}$. This gives that, if $r>0$ (the $r=0$ case is trivial):
\begin{equation*}
t C^{(t)} + x_t x_t^T = \left[ \begin{array}{cc} U & \frac{x_\perp}{r}
\end{array} \right] \left[ \begin{array}{cc} \mbox{diag}(\sigma) + U^T x_t
x_t^T U & r U U^T x_t \\ r x_t^T U U^T & 1 \end{array} \right] \left[
\begin{array}{cc} U & \frac{x_\perp}{r} \end{array} \right]^T
\end{equation*}
Taking the eigendecomposition of the rank-$k'+1$ matrix in the above
expression:
\begin{equation*}
\left[ \begin{array}{cc} \mbox{diag}(\sigma) + U^T x_t x_t^T U & r U U^T x_t \\
r x_t^T U U^T & 1 \end{array} \right] = V' \mbox{diag}\left( \sigma' \right)
\left(V'\right)^T
\end{equation*}
gives that:
\begin{align*}
tC^{(t)} + x_t x_t^T =& \left( \left[ \begin{array}{cc} U & \frac{x_\perp}{r}
\end{array} \right] V' \right) \mbox{diag}\left(\sigma'\right) \left( \left[
\begin{array}{cc} U & \frac{x_\perp}{r} \end{array} \right] V' \right)^T \\
\end{align*}
Hence, the new vector of nonzero eigenvalues is $\sigma'$, with the
corresponding eigenvectors being:
\begin{equation*}
U' = \left[ \begin{array}{cc} U & \frac{x_\perp}{r} \end{array} \right] V'
\end{equation*}
Algorithm \ref{alg:pca-introduction:rank1-update} contains pseudocode which
implements this operation, and will find an eigendecomposition of
$t\tilde{C}^{(t)} + x_t x_t^T$ of rank at most $k'+1$, as $U',\sigma' =
\code{rank1-update}( d, k', U, \sigma, 1, x_t )$ (this algorithm takes an
additional parameter $\eta$ which will not be needed until Chapter
\ref{ch:pca-capped-msg}). Projecting onto the set of rank $k$ matrices amounts
to removing all but the top $k$ elements of $\sigma'$, along with the
corresponding columns of $U'$. The next iterate then satisfies $t
\tilde{C}^{(t+1)} = U' \code{diag}(\sigma') \left(U'\right)^T$.  The total
computational cost of performing this update $O((k')^2 d) \le O(k^2 d)$
operations, which is superior to the $d^2$ computational cost of ``true'' ERM,
since for most applications there would be little point in performing PCA
unless $k\ll d$. The memory usage is, likewise, better than that of SAA, since
the dominant storage requirement is that of the matrix $U$, which contains $kd$
elements.

In \citet{AroraCoLiSr12}, we found that this ``incremental algorithm'' performs
extremely well on real datasets---it was the best, in fact, among the compared
algorithms. However, there exist somewhat-contrived cases in which this
algorithm entirely fails to converge. For example, If the data are drawn from
a discrete distribution $\mathcal{D}$ which samples $[\sqrt{3},0]^T$ with
probability $1/3$ and $[0,\sqrt{2}]^T$ with probability $2/3$, and one runs the
incremental algorithm with $k=1$, then it will converge to $[1,0]^T$ with
probability $5/9$, despite the fact that the maximal eigenvector is $[0,1]^T$.
The reason for this failure is essentially that the orthogonality of the data
interacts poorly with the low-rank projection: any update which does not
entirely displace the maximal eigenvector in one iteration will be removed
entirely by the projection, causing the algorithm to fail to make progress.

The incremental algorithm is of interest only because of its excellent
empirical performance. In Chapter \ref{ch:pca-capped-msg}, we will develop a
similar algorithm which is more theoretically justified, less likely to fail,
and still performs well in practice.

%% file: pca/introduction/figures/alg-rank1-update.tex
\begin{algorithm}[t]

\begin{pseudocode}
\codename $\code{rank1-update}\left( d,k':\N, U:\R^{d \times k'}, \sigma:\R^{k'}, \eta \in \R, x:\R^d \right)$\\
\codeline $\hat{x} \leftarrow U^T x$; $x_\perp \leftarrow x - U \hat{x}$; $r \leftarrow \norm{x_\perp}$;\\
\codeline $\code{if }r > 0$\\
\codeline \>$V', \sigma' \leftarrow \code{eig}( [ \code{diag}( \sigma ) + \eta \hat{x} \hat{x}^T, \eta r \hat{x} ; \eta r \hat{x}^T, \eta r^2 ] )$;\\
\codeline \>$U' \leftarrow [ U, x_\perp / r ] V'$;\\
\codeline $\code{else}$\\
\codeline \>$V', \sigma' \leftarrow \code{eig}( \code{diag}( \sigma ) + \eta \hat{x} \hat{x}^T )$;\\
\codeline \>$U' \leftarrow U V'$;\\
\codeline $\code{return } U', \sigma'$;
\end{pseudocode}

\caption{
Routine which computes an eigendecomposition of $M + \eta x x^T$ from a
rank-$k'$ eigendecomposition $M = U \code{diag}(\sigma) U^T$. The
computational cost of this algorithm is dominated by the matrix multiplication
defining $U'$ (line $4$ or $7$) costing $O((k')^2 d)$ operations.
}

\label{alg:pca-introduction:rank1-update}

\end{algorithm}

%% file: pca/introduction/sec-proofs.tex
\section{Proof of Lemma \ref{lem:pca-introduction:saa-bound}}

\begin{proofs}
Our proof of Lemma \ref{lem:pca-introduction:saa-bound} will follow a similar
outline to the proof of Lemma
\ref{lem:svm-introduction:generalization-from-expected-loss} in Chapter
\ref{ch:svm-introduction}. The first step is to determine the how many samples
we must draw in order to ensure that, uniformly over all $U\in\R^{d\times k}$
with orthonormal columns, the expected loss suffered by $U$ is close to the
empirical loss:
\input{pca/introduction/theorems/lem-saa-rademacher-bound}
The bound of the above Lemma is expressed in terms of the Rademacher complexity
of the function class $\mathcal{F}_\mathcal{W}$, so the next step is to bound
this quantity:
\input{pca/introduction/theorems/lem-saa-rademacher}
We may now prove Lemma \ref{lem:pca-introduction:saa-bound} by combining Lemmas
\ref{lem:pca-introduction:saa-rademacher-bound} and
\ref{lem:pca-introduction:saa-rademacher}, and apply Hoeffding's inequality to
bound the empirical loss of the optimal set of eigenvectors $U^*$ in terms of
its expected loss:
\input{pca/introduction/theorems/lem-saa-bound}
\end{proofs}

%% file: pca/introduction/theorems/lem-saa-rademacher-bound.tex
\medskip
\begin{lemma}\label{lem:pca-introduction:saa-rademacher-bound}
Define $\hat{\Sigma} = (1/n) \sum_{i=1}^n x_i x_i^T$ to be the empirical
covariance matrix over $n$ samples drawn \iid from $\mathcal{D}$, and $\Sigma =
\expectation[x \sim \mathcal{D}]{xx^T}$ the true covariance matrix. Then, with
probability $1-\delta$, every $U\in\R^{d\times k}$ with orthonormal columns
satisfies:
\begin{equation*}
\trace \left(I - U U^T \right) \Sigma - \trace \left( I - U U^T \right)
\hat{\Sigma} \le R_n\left(\mathcal{F}_\mathcal{W}\right) + \sqrt{\frac{8 \log
\frac{2}{\delta}}{n}}
\end{equation*}
Here, $R_n$ is the Rademacher complexity:
\begin{equation*}
R_n\left(\mathcal{F}\right) = \expectation{\sup_{f\in\mathcal{F}}
\abs{\frac{2}{n}\sum_{i=1}^n \sigma_i f\left(x_i x_i^T\right)}}
\end{equation*}
where the expectation is taken jointly with respect to the samples
$x_1,\dots,x_n$ and \iid Rademacher random variables
$\sigma_1,\dots,\sigma_n\in\{\pm 1\}$, and $\mathcal{F}_\mathcal{W}$ is the
linear function class:
\begin{equation*}
\mathcal{F}_\mathcal{W} = \left\{ X \mapsto -\trace UU^T X \mid
U\in\R^{d\times k} \mbox{ has orthonormal columns} \right\}
\end{equation*}
\end{lemma}
\begin{proof}
This follows immediately from \citet[Theorem 8]{BartlettMe03}, although some
effort is required to translate our problem into their setting, and align
notation. To this end, we begin by defining the label, function and hypothesis
spaces $\mathcal{X}=\mathcal{Y}=\mathcal{A}$ to all be the set of rank-$1$
matrices for which the nonzero eigenvalue is no larger than $1$. Since PCA is
an unsupervised problem, there are \emph{no labels}, so we assume that $X_i =
Y_i = x_i x_i^T$ for all samples, with $x_i\sim\mathcal{D}$. Observe that we
are now treating the samples not as vectors drawn from $\mathcal{D}$, but
rather as matrices formed by taking the outer product of each sample with
itself.

For every rank-$k$ projection matrix $M$, define $f_M\left(X\right) = MX$ as
the function which projects its argument (a rank-$1$ matrix) according to $M$,
and take $F$ to be the set of all such $f_M$s. Finally, define the loss
function $\mathcal{L}\left(Y,A\right) = \trace Y-A$, and observe that since
$\norm{x}\le 1$, $Y=x_i x_i^T$ and $A = (I-M) x_i x_i^T$ for a rank-$k$
projection matrix $M$, the range of $\mathcal{L}$ is $[0,1]$. The definitions
of $F$ and $\mathcal{L}$ together recast the problem as minimizing the
\emph{compression loss} suffered by a projection matrix $M=UU^T$, with
$\mathcal{L}\left(Y,f_M(X)\right) = \trace (I-M)X$.

With these definitions in place, application of \citet[Theorem 8]{BartlettMe03}
gives the claimed result.
\end{proof}

%% file: pca/introduction/theorems/lem-saa-rademacher.tex
\medskip
\begin{lemma}\label{lem:pca-introduction:saa-rademacher}
The Rademacher complexity of the function class $\mathcal{F}_\mathcal{W}$
defined in Lemma \ref{lem:pca-introduction:saa-rademacher-bound} satisfies
$R_n\left(\mathcal{F}_\mathcal{W}\right) \le \sqrt{\frac{k}{n}}$.
\end{lemma}
\begin{proof}
First observe that if we take:
\begin{equation*}
\mathcal{F}_\mathcal{W} = \left\{ X \mapsto \trace M X \mid M \in \mathcal{W}
\right\}
\end{equation*}
then this definition of $\mathcal{F}_\mathcal{W}$ is identical to that of Lemma
\ref{lem:pca-introduction:saa-rademacher-bound} when $\mathcal{W}$ is the set
of negated rank-$k$ projection matrices. However, this definition enables us to
define function classes which are parameterized by sets other than
$\mathcal{W}$.

Define $S$ to be the set of all negative semidefinite matrices, and observe
that $S$ is closed and convex. Further define $F(X) = \frac{1}{2}\norm{X}_F^2$,
where $\norm{\cdot}_F$ is the Frobenius norm (Schatten 2-norm). By
\citet[Theorem 11]{KakadeShTe12}, $F$ is $1$-strongly convex with respect to
the Frobenius norm. We wish to define a set
$\tilde{\mathcal{W}}\supseteq\mathcal{W}$ using an equation of the form
$\tilde{\mathcal{W}} = \left\{ M \in S \mid F(M) \le W_*^2\right\}$, so that we
can apply \citet[Theorem 3]{KakadeSrTe08} to give the desired result.

Because every $M \in \mathcal{W}$ is a negated rank-$k$ projection matrix,
$F(M) = \nicefrac{k}{2}$ for all such $M$, showing that we may take $W_*^2 =
\nicefrac{k}{2}$ and have that $\tilde{\mathcal{W}} \supseteq \mathcal{W}$, and
therefore that $\mathcal{F}_{\tilde{\mathcal{W}}} \supseteq
\mathcal{F}_\mathcal{W}$, which implies that
$R_n\left(\mathcal{F}_\mathcal{W}\right) \le
R_n\left(\mathcal{F}_{\tilde{\mathcal{W}}}\right)$. The claim is then proved by
applying \citet[Theorem 3]{KakadeSrTe08}, and using this inequality as well as
the fact that $\norm{xx^T}_F \le 1$ for $x\sim\mathcal{D}$.
\end{proof}

%% file: pca/ch-warmuth.tex
\chapter{Warmuth \& Kuzmin's Algorithm}\label{ch:pca-warmuth}

\input{pca/warmuth/sec-overview}
\input{pca/warmuth/sec-objective}
\input{pca/warmuth/sec-algorithm}
\input{pca/warmuth/sec-mirror}

\paragraph{Collaborators:} The novel content of this chapter (Sections
\ref{subsec:pca-warmuth:efficiency} and \ref{sec:pca-warmuth:mirror}) was
performed jointly with Raman Arora, Karen Livescu and Nathan Srebro.

\clearpage
\input{pca/warmuth/sec-proofs}

%% file: pca/warmuth/sec-overview.tex
\section{Overview}\label{sec:pca-warmuth:overview}

The previous chapter introduced the stochastic PCA problem, and highlighted two
``basic'' algorithms for solving it. In this chapter, we will discuss a far
more principled approach, based on convex optimization, originally due to
\citet{WarmuthKu06}. Unlike the algorithms considered in the previous chapter,
the rate of convergence of this algorithm is known. However, its practical
performance is poor, primarily due to the fact that, at each iteration, a
significant amount of computation must be performed.

In order to address this shortcoming, we will present an optimization which
dramatically improves the practical performance of Warmuth and Kuzmin's
algorithm. We will also demonstrate that Warmuth and Kuzmin's algorithm is
nothing but an instance of mirror descent on a particular convex relaxation of
the PCA objective, which partially motivates our improved ``capped MSG''
algorithm of Chapter \ref{ch:pca-capped-msg}.

Warmuth and Kuzmin's algorithm was originally presented
\citep{WarmuthKu06,WarmuthKu07,WarmuthKu08} in the online setting, in which the
data examples are not drawn from an underlying unknown distribution
$\mathcal{D}$, but are instead potentially chosen adversarially. The online
setting is strictly harder than the stochastic setting, in that any good online
algorithm may be converted into a good stochastic algorithm through the use of
an online-to-batch conversion (although, as we saw in Chapter \ref{ch:svm-sbp},
dedicated stochastic algorithms can work very well, also). For this reason,
and due to our focus on the stochastic setting, all of the results in this
chapter will be presented in the stochastic setting, despite the fact that they
(and those of Chapter \ref{ch:pca-capped-msg}) apply equally well in the
more-general online setting.

This chapter will begin, in Section \ref{sec:pca-warmuth:objective}, with a
description of the convex objective which Warmuth and Kuzmin's algorithm
optimizes, and an explanation of the reasoning behind it. In Section
\ref{sec:pca-warmuth:algorithm}, their algorithm will be described in detail,
along with a novel optimization which dramatically improves its empirical
performance. The chapter will conclude, in Section
\ref{sec:pca-warmuth:mirror}, with a derivation of their algorithm as an
instance of the general Mirror Descent (MD) framework, along with a
corresponding proof of convergence. Much of the content of this chapter is due
to Warmuth and Kuzmin~\citep{WarmuthKu06,WarmuthKu08}, although Section
\ref{subsec:pca-warmuth:efficiency} and part of Section
\ref{sec:pca-warmuth:mirror} was presented in our paper at the 50th Allerton
Conference on Communication, Control and Computing \citep{AroraCoLiSr12}.

%% file: pca/warmuth/sec-objective.tex
\section{Objective}\label{sec:pca-warmuth:objective}

As we saw in Section \ref{sec:pca-introduction:objective} of Chapter
\ref{ch:pca-introduction}, one may formulate PCA as the problem of finding a
rank-$k$ projection matrix $M$ which preserves most of the variance of $\Sigma
= \expectation[x\sim\mathcal{D}]{x x^T}$. This is Problem
\ref{eq:pca-introduction:objective-m}. While this objective seeks a rank-$k$
matrix $M$ projecting onto the \emph{maximal} subspace, one could equivalently
seek a rank $d-k$ matrix $W$ projecting onto the \emph{minimal} subspace
(indeed, we used exactly this trick while proving our SAA bound in Section
\ref{sec:pca-introduction:saa}). From this matrix, the orthogonal complement
may easily be derived. This modification of Problem
\ref{eq:pca-introduction:objective-m} results in the following optimization
problem:
\begin{align}
\label{eq:pca-warmuth:unrelaxed-objective} \mbox{minimize}: &
\expectation[x\sim\mathcal{D}]{x^T W x} \\
\notag \mbox{subject to}: & \spectrum[i]{W} \in \left\{ 0, 1 \right\}, \rank W
= d-k
\end{align}
Because $W$ is a rank $d-k$ projection matrix, it must have have exactly $d-k$
eigenvalues equal to $1$, and $k$ equal to $0$. Unfortunately, this is not a
convex constraint, but if we \emph{relax} it by taking the convex hull, then
the result \emph{is} a convex optimization problem:
\begin{align}
\label{eq:pca-warmuth:convex-objective} \mbox{minimize}: &
\expectation[x\sim\mathcal{D}]{x^T W x} \\
\mbox{subject to}: & W \succeq 0, \norm{W}_2 \le \frac{1}{d-k}, \trace W = 1
\end{align}
This is precisely the relaxed PCA formulation proposed by \citet{WarmuthKu06}.
Here, $\norm{\cdot}_2$ is the spectral norm, and we have scaled both $W$ and
the objective by a factor of $d-k$ so that the eigenvalues of $W$ will sum to
$1$ (i.e. form a discrete probability distribution)---this is not strictly
necessary, but makes such quantities as the von Neumann entropy and quantum
relative entropy, which will be crucial to the algorithm description of Section
\ref{sec:pca-warmuth:algorithm} and mirror descent derivation of Section
\ref{sec:pca-warmuth:mirror}, meaningful.

\subsection{Un-relaxing a Solution}\label{subsec:pca-warmuth:unrelaxing}

While Problem \ref{eq:pca-warmuth:convex-objective}, as a convex optimization problem,
is tractable, the fact remains that, since it is a relaxation of Problem
\ref{eq:pca-warmuth:unrelaxed-objective}, its solutions will not necessarily be
solutions to the true PCA objective. In fact, this is not the case, so long as
$\Sigma = \expectation[x\sim\mathcal{D}]{xx^T}$ has distinct eigenvalues
$\sigma_1 > \sigma_2 > \dots > \sigma_d$ with corresponding eigenvectors $v_1,
v_2, \dots, v_d$. To see this, suppose that $\sum_{i=1}^k v_i^T W v_i = \alpha
> 0$ (i.e. that $W$ puts nonzero mass on the $k$ maximal eigenvalues). Then
we must have that $\sum_{i=k+1}^d v_i^T W v_i = 1 - \alpha < 1$, implying that
it is possible to ``move'' an $\alpha$-sized amount of mass from the $k$
maximal eigenvalues to the $d-k$ minimal eigenvalues, decreasing the objective
function value while continuing to satisfy the constraints. Hence, if $\Sigma$
has distinct eigenvalues, then the unique optimal $W^*$ is, aside from scaling
by $1/(d-k)$, a rank-$d-k$ projection matrix projecting onto the minimal
subspace.

In practice, we will never find the true optimal solution to the objective,
since the underlying distribution $\mathcal{D}$, and therefore $\Sigma$, is
unknown---instead, we must satisfy ourselves with suboptimal solutions which
become increasingly close to the ``truth'' as we base them on increasing
numbers of samples. The observation that Problem \ref{eq:pca-warmuth:convex-objective}
has a unique rank-$d-k$ optimum, however, motivates a simple and effective
heuristic for converting an approximate solution $W$ to the relaxed objective
of Problem \ref{eq:pca-warmuth:convex-objective} into a solution to the original
objective of Problem \ref{eq:pca-warmuth:unrelaxed-objective}---simply set the
top $d-k$ eigenvalues of $W$ to $1$, and the remainder to $0$.

\input{pca/warmuth/figures/alg-unrelaxing}
This heuristic is recommended for practical applications. However, for
theoretical purposes, we would like to have a method for converting
$\epsilon$-suboptimal solutions of the relaxed objective into equivalently
suboptimal solutions to the original objective, so that any convergence rate
which we may prove for an algorithm working on the relaxed objective will yield
an equivalent convergence result in the original PCA objective.
\citet{WarmuthKu06} present such an approach, the basis of which is the fact
that Problem \ref{eq:pca-warmuth:convex-objective} is derived from Problem
\ref{eq:pca-warmuth:unrelaxed-objective} by taking the \emph{convex hull} of
the constraints. As a result, any feasible solution $W$ to the relaxed
objective can be represented as a convex combination of rank-$d-k$ matrices for
which all nonzero eigenvalues are equal to exactly $1/(d-k)$ (i.e. projection
matrices, aside from the $1/(d-k)$ scaling of Problem
\ref{eq:pca-warmuth:convex-objective}). In fact, as is shown in \citet[Theorem
1]{WarmuthKu06}, a convex combination of \emph{at most $d$} such matrices can
be found using Algorithm \ref{alg:pca-warmuth:unrelaxing} (this is Algorithm 1
of \citet{WarmuthKu06}):
\begin{equation*}
W = \sum_i \lambda_i W_i,\ \ \lambda_i > 0,\ \ \sum_i \lambda_i = 1
\end{equation*}
One may then sample an index $i$ according to the discrete probability
distribution given by $\lambda_1, \dots, \lambda_d$, and take $(d-k) W_i$ as
the solution to the original objective. In expectation over this sampling of
$i$, $W_i$ will have the same relaxed objective function value as the original
$W$, and the objective function value in the original objective will differ
(again in expectation) only by the $d-k$ factor introduced by the different
scalings of the two objectives.

%% file: pca/warmuth/figures/alg-unrelaxing.tex
\begin{algorithm}[t]

\begin{pseudocode}
\codename $\code{unrelax}\left( d,k:\N, W:\R^{d\times d} \right)$\\
\codeline $\left(\sigma_1,v_1\right),\dots,\left(\sigma_d,v_d\right) \leftarrow \code{eig}(W)$;\\
\codeline $\code{while } \norm{\sigma}_1 > 0$\\
\codeline \>$\mathcal{I} \leftarrow \left\{ i : \sigma_i = \norm{\sigma}_1 / (d-k) \right\}$;\\
\codeline \>$\mathcal{J} \leftarrow \left\{ i : \sigma_i > 0 \right\} \textbackslash \mathcal{I}$;\\
\codeline \>$\mathcal{K} \leftarrow \mathcal{I} \cup \left\{ \code{any } d-k-\abs{\mathcal{I}} \code{ elements of } \mathcal{J} \right\}$;\\
\codeline \>$W_i \leftarrow \sum_{j\in \mathcal{K}} v_j v_j^T / (d-k)$; $\lambda_i \leftarrow \min_{\{j \in \mathcal{K}\}} \sigma_j$;\\
\codeline \>$\code{for } j \in \mathcal{K}$\\
\codeline \>\>$\sigma_j \leftarrow \sigma_j - \lambda_i$;\\
\codeline \>$i := i + 1$;\\
\codeline $\code{return } \left(\lambda_1,W_1\right),\dots,\left(\lambda_i,W_i\right)$;
\end{pseudocode}

\caption{
Routine which decomposes a $W$ which is feasible for the relaxed objective of
Problem \ref{eq:pca-warmuth:convex-objective} into a convex combination $W =
\sum_j \lambda_j W_j$ of at most $d$ rank-$d-k$ projection matrices scaled by
$1/(d-k)$. This is Algorithm 1 of \citet{WarmuthKu06}.
}

\label{alg:pca-warmuth:unrelaxing}

\end{algorithm}

%% file: pca/warmuth/sec-algorithm.tex
\section{Optimization Algorithm}\label{sec:pca-warmuth:algorithm}

\input{pca/warmuth/figures/tab-bounds}
\citet{WarmuthKu06} propose optimizing Problem \ref{eq:pca-warmuth:convex-objective}
by iteratively performing the following update:
\begin{equation}
\label{eq:pca-warmuth:update} W^{(t+1)} = \project[\textrm{RE}]{ \exp \left( \ln
W^{(t)} - \eta_t x_t x_t^{T} \right)}.
\end{equation}
These are \emph{matrix} logarithms and exponentials, which in this case amount
to element-wise logarithms and exponentials of the eigenvalues of their
argument. The projection $\mathcal{P}_{\textrm{RE}}$ onto the constraints is
performed with respect to the quantum relative entropy, a generalization of the
Kullback-Leibler divergence to the matrix setting:
\begin{equation}
\label{eq:pca-warmuth:quantum-relative-entropy} D_{KL}\left(W\Vert
W'\right)=\trace \left( W\left(\ln W-\ln W'\right) \right)
\end{equation}
Projecting with respect to this divergence turns out to be a relatively
straightforward operation (see Section \ref{subsec:pca-warmuth:projection}).

Warmuth and Kuzmin analyzed their algorithm, and derived the following bound on
the error achieved by a solution after $T$ iterations (\citep[Equation
4]{WarmuthKu08} combined with the observation that $(n-k)\log(n/(n-k)) =
(n-k)\log(1 + k/(n-k)) \le k$):
\begin{equation}
\label{eq:pca-warmuth:rate} (d-k) \expectation[x\sim\mathcal{D}]{x \bar{W} x}
\le (d-k) \inf_{W\in\mathcal{W}} \expectation[x\sim\mathcal{D}]{x W x} + 2
\sqrt{\frac{L^* k}{T}} + \frac{k}{T}
\end{equation}
Observe that the loss terms $\expectation{x \bar{W} x}$ and $\expectation{x W
x}$ have been multiplied by $d-k$---the reason for this is that feasible $W$
have eigenvalues which sum to $1$, while we are interested in the loss suffered
by rank $d-k$ projection matrices, for which the eigenvalues sum to $d-k$.
Here, $\bar{W} = (1/T)\sum_{t=1}^{T} W^{(t)}$ is the average of the iterates
found by the algorithm, $\mathcal{W}$ is the feasible region (i.e. the set of
all feasible $W$), and $L^* = ((d-k)/T) \sum_{t=1}^{T} x_t^T W^* x_t$ is the
average compression loss suffered by the optimal $W^* \in \mathcal{W}$ (i.e.
the amount of empirical variance which occupies the minimal $d-k$-dimensional
subspace). This is called an ``optimistic rate'', because if the desired level
of suboptimality $\epsilon$ is of roughly the same order as $L^*$ (i.e. the
problem is ``easy''), then the algorithm will converge at a roughly $1/T$ rate.
On more difficult problem instances, the algorithm will converge at the much
slower $1/\sqrt{T}$ rate.

The intuition which Warmuth and Kuzmin present for their update is that it can
be interpreted as a generalization of the exponentiated gradient algorithm,
which traditionally works in the vector setting, to the matrix setting. In
fact, as we will see in Section \ref{sec:pca-warmuth:mirror}, this is more than
an intuition---both the exponentiated gradient algorithm, and Warmuth and
Kuzmin's matrix algorithm, can be interpreted as instances of the mirror
descent algorithm with an entropy regularizer. This insight enables us to
analyze this algorithm by simply ``plugging in'' to known mirror descent
bounds.

\subsection{The Projection}\label{subsec:pca-warmuth:projection}

The only portion of the update of Equation \ref{eq:pca-warmuth:update} of which
at least a na\"{i}ve the implementation is not obvious is the projection onto
the constraints with respect to the quantum relative entropy. For this purpose,
we will decompose the update of Equation \ref{eq:pca-warmuth:update} into the
following two steps:
\begin{align*}
W' =& \exp \left( \ln W - \eta x x^{T} \right) \\
W =& \project[\textrm{RE}]{ W' }
\end{align*}
Here, we have temporarily simplified the notation slightly by giving the input
(formerly known as $W^{(t)}$) the same name as the output (formerly
$W^{(t+1)}$), and removing the $t$ subscripts on $\eta$ and $x$.
We are concerned with the second step, the projection, which as is shown in the
following lemma, can be represented as a relatively simple operation on the
eigenvalues of the matrix $W'$ to be projected:
\input{pca/warmuth/theorems/lem-projection}
Essentially, performing the projection reduces to finding a scaling factor $Z$
satisfying the requirements of this lemma. As is pointed out by
\citet{WarmuthKu06}, there exists a divide-and-conquer algorithm which can
perform this operation in $O(d)$ time, where $n$ is the number of eigenvalues,
and indeed (we will see in Section \ref{subsec:pca-warmuth:efficiency} why this
is important), it is simple to refine their algorithm to work in $O(m)$ time,
where $m$ is the number of \emph{distinct} eigenvalues of $W'$.

\input{pca/warmuth/figures/alg-project}
This linear-time algorithm is unnecessarily complicated, since the projection
is, \emph{by far}, not the dominant computational cost of each iteration (the
cost of each step of the update will be discussed in detail in Section
\ref{subsec:pca-warmuth:efficiency}). Hence, we present a simpler and more
intuitive algorithm for accomplishing the same goal, which is based on sorting
the eigenvalues of $W'$, and then performing a simple linear search for a
solution. Algorithm \ref{alg:pca-warmuth:project} contains pseudocode for this
algorithm. Neglecting the handling of repeated eigenvalues (which is
straightforward), it works by searching for the largest index $i$ such that
that setting the largest $d-i'$ eigenvalues of $W$ to $1/(d-i)$, and scaling
the remaining $k'$ eigenvalues so as to satisfy the normalization constraint
$\trace W = 1$ (solving for the necessary $Z$ in the process), results in all
eigenvalues of $W$ being bounded above by $1/(d-k)$. It's overall cost is $O(m
\log m)$ operations, where $m$ is the number of distinct eigenvalues of $W'$,
due to the initial sorting of the eigenvalues.

Observe that Algorithm \ref{alg:pca-warmuth:project} assumes that the
eigenvalues of $W'$ are known---it does not include an explicit
eigendecomposition step. The reason for this is that, with the most
straightforward implementation of the update, the eigendecomposition of $W$
must be known in order to take its logarithm (recall that this is the
element-wise logarithm of its eigenvalues), and the eigendecomposition of $W'$
must be known, because it is the result of a matrix exponentiation. Hence,
there is no need to explicitly find the eigenvalues during the projection, as
they have already been made available while performing the first step of the
update.

\subsection{Efficient Updates}\label{subsec:pca-warmuth:efficiency}

In contrast to the algorithms of Chapter \ref{ch:pca-introduction}, Warmuth and
Kuzmin's algorithm has the enormous advantage that a bound---indeed, a very
\emph{good} bound---is known on its convergence rate. However, this bound is
expressed in terms of the number of \emph{iterations} required to find a good
solution, and, as written, the cost of each iteration is extremely high.
Indeed, when implemented na\"{i}vely using two eigendecompositions of $d\times
d$ matrices, it is \emph{prohibitively} expensive on high-dimensional data. As
PCA is often used on extremely high-dimensional data (Google's
MapReduce~\citep{BrinPa98} is a prime example, for which there is one dimension
for each web site), and as one of the primary goals of providing stochastic
algorithms for PCA is to increase efficiency over traditional
linear-algebra-based techniques (e.g. the SAA approach of calculating a
covariance matrix, and then performing an eigendecomposition), this shortcoming
must be addressed in order for the algorithm to be considered at all practical.

\input{pca/warmuth/figures/fig-efficiency}
The key idea to improving the efficiency of the update is to observe that, at
every iteration, a rank-one positive-semidefinite matrix ($\eta x x^T$) will be
\emph{subtracted} from the current iterate before performing the projection,
and that the projection must therefore \emph{increase} the eigenvalues in
order to satisfy the constraint that the eigenvalues sum to $1$---that is, it
must be the case that $Z\le 1$. Because the eigenvalues are capped at
$1/(d-k)$, a consequence of this is that, typically, and in particular when the
desired subspace dimension $k$ is much smaller than $d$, many of the
eigenvalues of each iterate will be exactly $1/(d-k)$. Figure
\ref{fig:pca-warmuth:efficiency} experimentally illustrates this phenomenon, on
the $256$-dimensional Adult dataset (more detailed experiments may be found in
Chapter \ref{ch:pca-capped-msg}).

\input{pca/warmuth/figures/alg-rank1-update}
Exploiting this observation turns out to rely on a similar technique to that
used in the incremental algorithm of Section
\ref{sec:pca-introduction:incremental} in Chapter \ref{ch:pca-introduction}.
The key difference is that, whereas the incremental algorithm performs a rank-1
update to a low-rank matrix, we now must perform a rank-1 update to a matrix
with many eigenvalues equal to exactly $1/(d-k)$, instead of $0$. Algorithm
\ref{alg:pca-warmuth:rank1-update} demonstrates the necessary modification to
Algorithm \ref{alg:pca-introduction:rank1-update}. The reasoning behind this
algorithm is only a little more complicated than that of Section
\ref{sec:pca-introduction:incremental}. Begin by writing $W$ as:
\begin{equation*}
W = U^T \mbox{diag}(\sigma) U + \frac{1}{d-k} U^T U
\end{equation*}
Here, the columns of $U\in\R^{d\times k'}$ are the eigenvectors of $W$ with
corresponding unbound eigenvalues in the vector $\sigma$, with all other
eigenvalues being $1/(d-k)$. Defining $x_\perp = (I-U)(I-U)^T x$ as the portion
of $x$ which does not lie in the span of the columns of $U$, and taking
$r=\norm{x_\perp}$, gives that, if $r>0$ (the $r=0$ case is trivial):
\begin{align*}
W + \eta x x^T =& \left[ \begin{array}{cc} U & \frac{x_\perp}{r} \end{array}
\right] \left[ \begin{array}{cc} \mbox{diag}(\sigma) + \eta U^T x x^T U & \eta
r U U^T x \\ \eta r x^T U U^T & \frac{1}{d-k} + \eta \end{array} \right] \left[
\begin{array}{cc} U & \frac{x_\perp}{r} \end{array} \right]^T \\
& + \frac{1}{d-k} (I-U-\frac{x_\perp}{r}) (I-U - \frac{x_\perp}{r})^T
\end{align*}
We next take the eigendecomposition of the rank-$k'+1$ matrix in the above
equation:
\begin{equation*}
\left[ \begin{array}{cc} \mbox{diag}(\sigma) + \eta U^T x x^T U & \eta r U U^T
x \\ \eta r x^T U U^T & \frac{1}{d-k} + \eta \end{array} \right] = V'
\mbox{diag}\left( \sigma' \right) \left(V'\right)^T
\end{equation*}
giving that:
\begin{align*}
W + \eta x x^T =& \left( \left[ \begin{array}{cc} U & \frac{x_\perp}{r}
\end{array} \right] V' \right) \mbox{diag}\left(\sigma'\right) \left( \left[
\begin{array}{cc} U & \frac{x_\perp}{r} \end{array} \right] V' \right)^T \\
& + \frac{1}{d-k} (I-U-\frac{x_\perp}{r}) (I-U - \frac{x_\perp}{r})^T
\end{align*}
This shows that the new vector of non-capped eigenvalues is $\sigma'$, with the
corresponding eigenvectors being:
\begin{equation*}
U' = \left[ \begin{array}{cc} U & \frac{x_\perp}{r} \end{array} \right] V'
\end{equation*}
Algorithm \ref{alg:pca-warmuth:rank1-update} performs precisely these steps
(with some additional handling for the $r=0$ case).

Ultimately, the cost of finding $W' = W + \eta xx^T$, and maintaining an
eigendecomposition during this update, is the cost of the matrix multiplication
defining the new set of eigenvectors $U'$: $O\left(\left(k'\right)^2 d\right)$.
Performing the projection using Algorithm \ref{alg:pca-warmuth:project} then
requires an additional $O\left( k' \log k'\right)$ steps, since $W'$ has $k'+1$
distinct eigenvalues. The overall computational cost of performing a single
iteration of this optimized version of Warmuth and Kuzmin's algorithm is
therefore $O\left(\left(k'\right)^2 d\right)$. The memory usage at each
iteration is likewise dominated by the cost of storing the eigenvectors: $O(k'
d)$. Compared to the cost of performing and storing two rank-$d$
eigendecompositions per iteration, this is an extremely significant
improvement.

The precise cost-per-iteration, unfortunately, depends on the quantity $k'$,
which varies from iteration-to-iteration (we will occasionally refer to this
quantity as $k_t'$ to emphasize this fact), and can indeed be as large as $d-1$
(or as small as $k+1$). Hence, while this is an \emph{improvement}, its
theoretical impact is limited. In practical terms, however, $k'$ often tends to
be quite small, particularly when the algorithm is close to convergence, as can
be seen in Figure \ref{fig:pca-warmuth:efficiency}.

%% file: pca/warmuth/figures/tab-bounds.tex
\begin{table}

\begin{small}
\begin{center}
\begin{tabular}{r|cc|c}
\hline
& Computation & Memory & Convergence \\
\hline
SAA & $nd^2$ & $d^2$ & $\sqrt{\frac{k}{n}}$ \\
SAA (Coppersmith-Winograd) & $nd^{1.3727}$ & $d^2$ & $\sqrt{\frac{k}{n}}$ \\
Stochastic Power Method & $Tkd$ & $kd$ & w.p. $1$ \\
Incremental & $T k^2 d$ & $kd$ & no \\
Warmuth \& Kuzmin & $\sum_{t=1}^{T} \left(k_t'\right)^2 d$ & $\max_{t\in\{1,\dots,T\}} k_t' d$ & $\sqrt{\frac{L^* k}{T}} + \frac{k}{T}$ \\
\hline
\end{tabular}
\end{center}
\end{small}

\caption{
Summary of results from Chapter \ref{ch:pca-introduction} and Section
\ref{sec:pca-warmuth:algorithm}. All bounds are given up to constant factors.
The ``Convergence'' column contains bounds on the suboptimality of the
solution---i.e. the difference between the total variance captured by the
rank-$k$ subspace found by the algorithm, and the best rank-$k$ subspace with
respect to the data distribution $\mathcal{D}$ (the objective in Section
\ref{sec:pca-warmuth:objective} is scaled by a factor of $d-k$ relative to the
other objectives---we have corrected for this here). Warmuth and Kuzmin's
algorithm's bound is slightly better than that which we derived for SAA in
Section \ref{sec:pca-introduction:saa} of Chapter \ref{ch:pca-introduction}
because it is expressed as an optimistic rate.
}

\label{tab:warmuth:bounds}

\end{table}

%% file: pca/warmuth/theorems/lem-projection.tex
\medskip
\begin{splitlemma}{lem:pca-warmuth:projection}

Let $W'\in\R^{d\times d}$ be a symmetric matrix, with eigenvalues
$\sigma_{1}',\dots,\sigma_{d}'$ and associated eigenvectors
$v_{1}',\dots,v_{d}'$. If $W=\project[\textrm{RE}]{W'}$ projects $W'$ onto the
feasible region of Problem \ref{eq:pca-warmuth:convex-objective} with respect
to the quantum relative entropy (Equation
\ref{eq:pca-warmuth:quantum-relative-entropy}), then $W$ will be the unique
feasible matrix which has the same set of eigenvectors as $W'$, with the
associated eigenvalues $\sigma_{1},\dots,\sigma_{d}$ satisfying:
\begin{equation*}
\sigma_{i} = \min\left(\frac{1}{d-k},\frac{\sigma_{i}'}{Z}\right)
\end{equation*}
with $Z\in\R^+$ being chosen in such a way that:
\begin{equation*}
\sum_{i=1}^{d}\sigma_{i}=1
\end{equation*}
\end{splitlemma}
\begin{splitproof}
The problem of finding $W$ can be written in the form of a convex optimization
problem as:
\begin{align*}
\mbox{minimize} : &\ \ D_{KL}\left(W \Vert W'\right)\\
\mbox{subject to} : & W \succeq 0, \norm{W}_2 \le
\frac{1}{d-k}, \trace W = 1
\end{align*}
The Hessian of the objective function is $M^-1$, which is positive definite for
positive definite $M$, and all feasible $M$ are positive definite. Hence, the
objective is strongly convex. The constraints are also convex, so this problem
must have a unique solution. Letting $\sigma_{1},\dots,\sigma_{d}$ and
$v_{1},\dots,v_{d}$ be the eigenvalues and associated eigenvectors of $W$, we
may write the KKT first-order optimality conditions \citep{BoydVa04} as:
\begin{align}
\notag 0 = & \nabla D_{KL}\left(W \Vert W'\right) + \mu I - \sum_{i=1}^{d}
\alpha_{i}v_{i}v_i^T + \sum_{i=1}^d \beta_{i}v_{i}v_{i}^{T} \\
\label{eq:pca-warmuth:kkt} = & I+\ln W-\ln W' + \mu I - \sum_{i=1}^{d}
\alpha_{i}v_{i}v_i^T + \sum_{i=1}^d \beta_{i}v_{i}v_{i}^{T}
\end{align}
where $\mu$ is the Lagrange multiplier for the constraint $\trace W = 1$, and
$\alpha_i, \beta_i \ge 0$ are the Lagrange multipliers for the constraints $W
\succeq 0$ and $\norm{W}_2 \le 1/(d-k)$, respectively. The complementary
slackness conditions are that $\alpha_{i}\sigma_{i} =
\beta_{i}\left(\sigma_{i}-1/(d-k)\right) = 0$. In addition, $W$ must be
feasible.

Because every term in Equation \ref{eq:pca-warmuth:kkt} \emph{except} for $W'$ has the same
set of eigenvectors as $W$, it follows that an optimal $W$ must have the same
set of eigenvectors as $W'$, so we may take $v_i = v_i'$, and write Equation
\ref{eq:pca-warmuth:kkt} purely in terms of the eigenvalues:
\begin{equation*}
\sigma_{i} = \frac{\exp \alpha_i}{\exp \beta_i} \left( \frac{\sigma_i}{\exp\left( 1 + \mu \right)} \right)
\end{equation*}
By complementary slackness and feasibility with respect to the constraints
$0\le\sigma_i\le 1/(d-k)$, if $0\le\sigma_i'/\exp\left(1+\mu\right)\le 1/(d-k)$,
then $\sigma_i=\sigma_i'/\exp\left(1+\mu\right)$. Otherwise,
$\alpha_i$ and $\beta_i$ will be chosen so as to clip $\sigma_i$
to the active constraint:
\begin{equation*}
\sigma_i = \max\left( 0, \min\left( \frac{1}{d-k}, \frac{\sigma_i}{\exp \left( 1 + \mu \right)} \right) \right)
\end{equation*}
Because $\exp \left( 1 + \mu \right)$ is nonnegative, clipping with $0$ is
unnecessary. Primal feasibility with respect to the constraint $\trace W = 1$
gives that $\mu$ must be chosen in such a way that $\trace W = 1$, completing
the proof.
\end{splitproof}

%% file: pca/warmuth/figures/alg-project.tex
\begin{algorithm}[t]

\begin{pseudocode}
\codename $\code{project}\left( d,k,m:\N, \sigma':\R^{m}, \kappa':\N^{m} \right)$\\
\codeline $\sigma',\kappa' \leftarrow \code{sort}( \sigma',\kappa' )$;\\
\codeline $s_i \leftarrow 0$; $c_i \leftarrow 0$;\\
\codeline $\code{for } i = 1 \code{ to } m$\\
\codeline \>$s_i \leftarrow s_i + \kappa_i' \sigma_i'$; $c_i \leftarrow c_i + \kappa_i'$;\\
\codeline \>$Z \leftarrow s_i / ( 1 - ( d - c_i ) / ( d - k ) )$;\\
\codeline \>$b \leftarrow ($\\
\codeskip \>\>$\sigma_i' / Z \le 1/(d-k)$\\
\codeskip \>\>$\code{and } ( ( i \ge m ) \code{ or } ( \sigma_{i+1}' / Z \ge 1/(d-k) ) )$\\
\codeskip \>$)$;\\
\codeline \>$\code{return }Z \code{ if } b$;\\
\codeline $\code{return error}$;
\end{pseudocode}

\caption{
Routine which finds the $Z$ of Lemma \ref{lem:pca-warmuth:projection}. It takes
as parameters the dimension $d$, ``target'' subspace dimension $k$, and the
number of \emph{distinct} eigenvalues $m$ of the current iterate. The
length-$m$ arrays $\sigma'$ and $\kappa'$ contain the distinct eigenvalues and
their multiplicities, respectively, of $W'$ (with $\sum_{i=1}^{m} \kappa_i' =
d$). Line $1$ sorts $\sigma'$ and re-orders $\kappa'$ so as to match this
sorting.  The loop will be run at most $m$ times, so the computational cost is
dominated by that of the sort: $O(m \log m)$.
}

\label{alg:pca-warmuth:project}

\end{algorithm}

%% file: pca/warmuth/figures/fig-efficiency.tex
\begin{figure}

\begin{center}
\begin{tabular}{ @{} H @{} H @{} }
\includegraphics[width=0.45\textwidth]{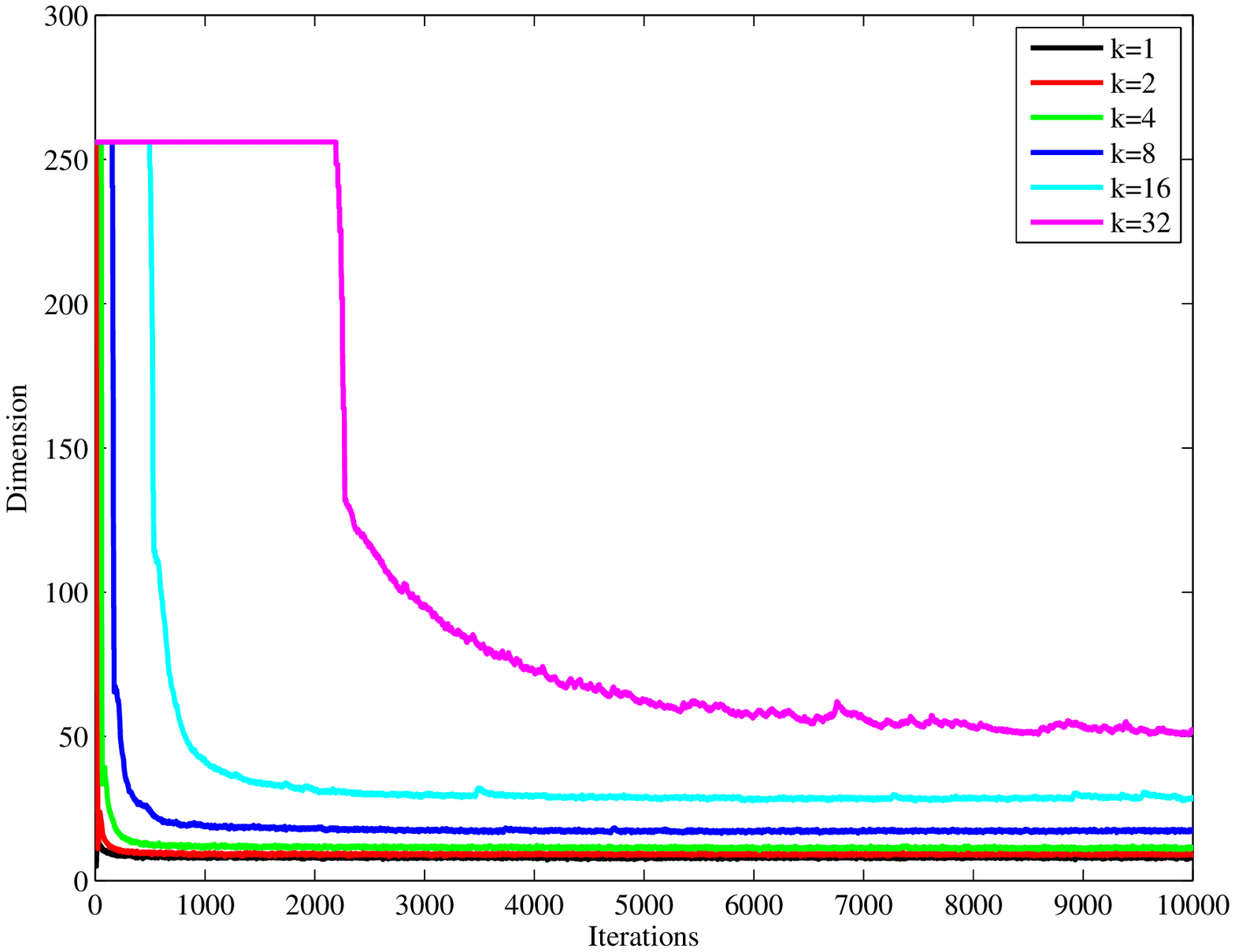} &
\includegraphics[width=0.45\textwidth]{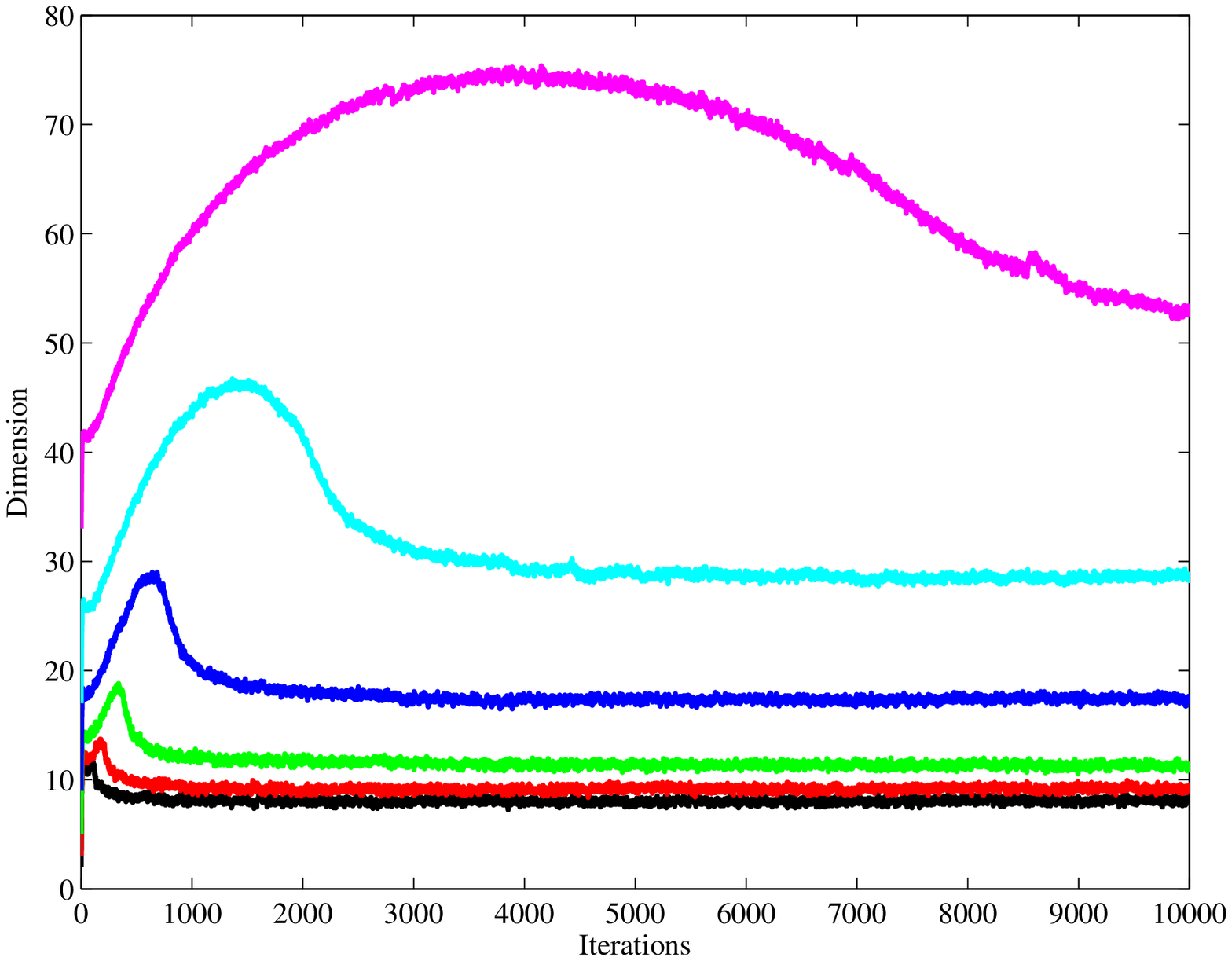} \\
\end{tabular}
\end{center}

\caption{
Plots of the number of eigenvalues $k_t'$ which are \emph{not} capped at the
upper bound of $1/(d-k)$ over $10000$ iterations of Warmuth and Kuzmin's
algorithm, on the $256$-dimensional Adult dataset, for the desired subspace
dimension $k\in\{1,2,4,8,16,32\}$. In the left hand plot, the algorithm was
started at $W^{(0)} = \left(\nicefrac{1}{d}\right)I$, while in the right plot,
it was started at a random matrix with $d-k-1$ eigenvalues equal to $1/(d-k)$,
and the remaining $k+1$ equal to $1/((k+1)(d-k))$. At each iteration, samples
were drawn uniformly at random from the dataset, and both plots are averaged
over $16$ runs.
}

\label{fig:pca-warmuth:efficiency}

\end{figure}

%% file: pca/warmuth/figures/alg-rank1-update.tex
\begin{algorithm}[t]

\begin{pseudocode}
\codename $\code{rank1-update}\left( d,k,k':\N, U:\R^{d \times k'}, \sigma:\R^{k'}, \eta \in \R, x:\R^d \right)$\\
\codeline $\hat{x} \leftarrow U^T x$; $x_\perp \leftarrow x - U \hat{x}$; $r \leftarrow \norm{x_\perp}$;\\
\codeline $\code{if }r > 0$\\
\codeline \>$V', \sigma' \leftarrow \code{eig}( [ \code{diag}( \sigma ) + \eta \hat{x} \hat{x}^T, \eta r \hat{x} ; \eta r \hat{x}^T, \eta r^2 + 1/(d-k) ] )$;\\
\codeline \>$U' \leftarrow [ U, x_\perp / r ] V'$;\\
\codeline $\code{else}$\\
\codeline \>$V', \sigma' \leftarrow \code{eig}( \code{diag}( \sigma ) + \eta \hat{x} \hat{x}^T )$;\\
\codeline \>$U' \leftarrow U V'$;\\
\codeline $\code{return } U', \sigma'$;
\end{pseudocode}

\caption{
Routine which computes an eigendecomposition of $W + \eta x x^T$ from an
eigendecomposition $W = U \code{diag}(\sigma) U^T + (1/(d-k)) (I-U)(I-U)^T$,
where $U$ is a matrix containing $k'$ orthonormal eigenvectors in its columns,
$\sigma$ is a vector of the corresponding eigenvalues, and the eigenvalues
corresponding to all eigenvectors not spanned by the columns of $U$ are exactly
$1/(d-k)$. The computational cost of this algorithm is dominated by the matrix
multiplication defining $U$ (line $4$ or $7$) costing $O((k')^2 d)$ operations.
}

\label{alg:pca-warmuth:rank1-update}

\end{algorithm}

%% file: pca/warmuth/sec-mirror.tex
\section{Interpretation as Mirror Descent}\label{sec:pca-warmuth:mirror}

\input{pca/warmuth/figures/tab-mirror-notation}
In this section, we will present the online mirror descent algorithm, give a
bound on its convergence rate, adapt it to the stochastic setting using an
online-to-batch conversion, and finally describe how it may be applied to
Problem \ref{eq:pca-warmuth:convex-objective}. Mirror Descent
\citep{NemirovskiYu83} is generally presented in the online setting, although
we are interested exclusively in the stochastic setting, and will adjust our
presentation of mirror descent, and its application to Warmuth and Kuzmin's
algorithm, accordingly. However, it should be noted that, despite our
preference for the stochastic setting, the algorithms of this section could be
presented in the online setting with little difficulty.

Whereas in the stochastic setting the learning task is essentially to fit a
hypothesis to an unknown data distribution based on some number of independent
samples drawn from this distribution, the online setting is best viewed as a
repeated game, in which the ``learner'' and ``adversary'' alternatingly choose
a hypothesis and a loss function, respectively, with the goal of the learner
being to suffer cumulative loss little higher than that achieved by the best
fixed hypothesis, and the adversary, true to its name, is free to choose loss
functions adversarially. We here base our treatment of mirror descent on that
of \citet{SrebroSrTe11}, in which, more formally, for a hypothesis space
$\mathcal{W}$ and set of candidate loss functions $\mathcal{F}$, the learner is
taken to be a function
$\mathcal{A}:\bigcup_{t\in\N}\mathcal{F}^{t}\rightarrow\mathcal{W}$---i.e.
for any sequence of $0$ or more loss functions (provided by the adversary at
previous steps), the learner outputs a hypothesis in $\mathcal{W}$. At the
$t$th step, after the learner has chosen a hypothesis
$w_{t}=\mathcal{A}\left(f_{1},f_{2},\dots,f_{t-1}\right)$, the adversary
chooses a new loss function $f_{t}\in\mathcal{F}$, and the learner suffers loss
$f_{t}\left(w_{t}\right)$. The performance of $\mathcal{A}$ is measured in
terms of the \emph{online regret}:
\begin{equation}
\label{eq:pca-warmuth:regret-definition}
R_{T}=\frac{1}{T}\sum_{t=1}^{T}f_{t}\left(w_{t}\right)-\inf_{w\in\mathcal{W}}\frac{1}{T}\sum_{t=1}^{T}f_{t}\left(w\right)
\end{equation}
The lower the regret, the closer the learner is to performing as well as the
best fixed hypothesis $w$.

While the treatment of \citet{SrebroSrTe11} is far more general, we are
interested only in the setting in which $\mathcal{W}$ is a convex subset of a
Hilbert space $\mathcal{H}$, and the loss functions $f\in\mathcal{F}$ are
convex with bounded subgradients. In order to define the mirror descent update,
we must provide a \emph{distance generating function} (d.g.f.)
$\Psi:\mathcal{H}\rightarrow\R$, which we take to be differentiable,
nonnegative and $1$-strongly convex with respect to a norm
$\norm{\cdot}_{\Psi}$:
\begin{equation}
\label{eq:pca-warmuth:strong-convexity}
\inner{\nabla\Psi\left(z\right)-\nabla\Psi\left(z'\right)}{z-z'}\ge\norm{z-z'}_{\Psi}^{2}
\end{equation}
Both the norm $\norm{\cdot}_{\Psi}$ and its dual $\norm{\cdot}_{\Psi^{*}}$ are
relevant for the analysis, but not for the specification of the algorithm
itself. The fact that mirror descent is parameterized by a distance generating
function is the essence of its generality---as we will see, choosing different
$\Psi$s results in different update rules, and yields different convergence
rates (depending on other properties of the problem under consideration). In
the vector case, the two most important distance generating functions are the
squared Euclidean norm, for which mirror descent becomes nothing but stochastic
gradient descent, and the negative Shannon entropy, which gives the
multiplicative-weights algorithm. In the matrix setting (which is our area of
interest), the analogues of these two cases are the squared Frobenius norm and
negative von Neumann entropy, respectively. There are literally infinite
possibilities, however, and the fact that they can all be unified under a
common framework is remarkable, and significantly simplifies the creation,
interpretation and analysis of new online and stochastic-gradient-type
algorithms.

We do need some additional pieces beyond $\Psi$ in order to fully specify the
mirror descent update algorithm---for one, we need its convex conjugate
$\Psi^{*}$, which we assume to be differentiable:
\begin{equation}
\label{eq:pca-warmuth:convex-conjugate}
\Psi^{*}\left(z^{*}\right)=\sup_{z}\left(\inner{z^{*}}{z}-\Psi\left(z\right)\right)
\end{equation}
and the Bregman divergence defined by $\Psi$:
\begin{equation}
\label{eq:pca-warmuth:bregman-divergence} \Delta_{\Psi}\left(z\Vert
z'\right)=\Psi\left(z\right)-\Psi\left(z'\right)-\inner{\nabla\Psi\left(z'\right)}{z-z'}
\end{equation}
Table \ref{tab:pca-warmuth:mirror-notation} summarizes the above notation and
assumptions. With these pieces in place, we may state the mirror descent
update, and give \citet{SrebroSrTe11}'s bound on its convergence rate:
\input{pca/warmuth/theorems/lem-online-bound}
We may adapt this result to the stochastic setting very straightforwardly, by
assuming that the convex loss functions $f_{t}$, instead of being chosen
adversarially, are drawn \iid from some unknown distribution with
$\expectation{f_{t}\left(w\right)}=f\left(w\right)$. In this case, taking
expectations of the above regret bound and using the definition of Equation
\ref{eq:pca-warmuth:regret-definition}:
\begin{align*}
\expectation{\frac{1}{T}\sum_{t=1}^{T}f_{t}\left(w_{t}\right)}\le &
\expectation{\inf_{w\in\mathcal{W}}\frac{1}{T}\sum_{t=1}^{T}f_{t}\left(w\right)}+\frac{2B}{\sqrt{T}}\\
\expectation{\frac{1}{T}\sum_{t=1}^{T}f\left(w_{t}\right)}\le &
\inf_{w\in\mathcal{W}}f\left(w\right)+\frac{2B}{\sqrt{T}}
\end{align*}
Let $\hat{w}$ be sampled uniformly from the set $\left\{
	w_{1},\dots,w_{T}\right\}$ \citep{CesaCoGe01}. Then:
\begin{equation}
\label{eq:pca-warmuth:stochastic-bound} \expectation{f\left(\hat{w}\right)}\le
\inf_{w\in\mathcal{W}}f\left(w\right)+\frac{2B}{\sqrt{T}}
\end{equation}
Hence, in the stochastic setting, we may find a \emph{single} hypothesis
satisfying bound of Lemma \ref{lem:pca-warmuth:online-bound} in expectation.

In the particular case of Warmuth and Kuzmin's convex relaxation of the
stochastic PCA objective (Problem \ref{eq:pca-warmuth:convex-objective}), the
hypothesis space $\mathcal{W}$ is the set of positive semidefinite matrices
satisfying the constraints. While the above has all been written in terms of
vector operations, for the purposes of defining the strong convexity of $\Psi$,
convex conjugate $\Psi^{*}$ and the Bregman divergence
$\Delta_{\Psi}\left(x\Vert x'\right)$ (Equations
\ref{eq:pca-warmuth:strong-convexity}, \ref{eq:pca-warmuth:convex-conjugate}
and \ref{eq:pca-warmuth:bregman-divergence}, respectively) we may treat the
matrices as vectors by using Frobenius inner products
$\inner{W}{W'}_{F}=\sum_{i,j}W_{i,j}W_{i,j}'=\trace WW'$ (this holds for
symmetric $W$ and $W'$). Defining $f_{t}\left(W\right)=x_t^T W x_t$ where $x_t$
is sampled \iid from the data distribution $\mathcal{D}$ suffices to express
optimization of a relaxed stochastic PCA objective as an instance of mirror
descent. It only remains to specify the distance generating function $\Psi$ and
norm $\norm{\cdot}_{\Psi}$, and verify that the conditions of Lemma
\ref{lem:pca-warmuth:online-bound} are satisfied.

\subsection{Negative von Neumann d.g.f.}

We will now show that Warmuth and Kuzmin's algorithm is nothing but mirror
descent applied to Problem \ref{eq:pca-warmuth:convex-objective} using a
shifted-and-scaled version of the negative von Neumann entropy
$-S\left(W\right)=\trace W\ln W$:
\begin{equation*}
\Psi_{vN}\left(W\right)= \alpha\left(\trace W\ln W-\beta\right)
\end{equation*}
We choose the scaling constant $\alpha$ in such a way as to insure that
$\Psi_{vN}$ is $1$-strongly convex with respect to some norm, and $\beta$ to
minimize $\Psi_{vN}$ while maintaining nonnegativity (i.e. so as to minimize
the bound of Lemma \ref{lem:pca-warmuth:online-bound}).

Before stating the mirror descent update which results from this choice of
distance generating function, we must verify that it satisfies the conditions
outlined earlier in this section. In order to show strong convexity, observe
that $-\nabla S\left(W\right)=I+\ln W$, and so: 
\begin{align*}
\MoveEqLeft \inner{\nabla\left(-S\left(W\right)\right)-\nabla\left(-S\left(W'\right)\right)}{W-W'}_{F} & \\
& =
\trace W\left(\ln W-\ln W'\right)+\trace W'\left(\ln W'-\ln W\right)\\
& =
D_{KL}\left(W\Vert W'\right)+D_{KL}\left(W'\Vert W\right)
\end{align*}
where $D_{KL}$ is the quantum relative entropy (Equation
\ref{eq:pca-warmuth:quantum-relative-entropy}). By the quantum Pinsker
inequality \citep{HiaiOhTs81,Hayashi06}, $D_{KL}\left(W\Vert
W'\right)\ge2\norm{W-W'}_{1}^{2}$, where $\norm{\cdot}_{1}$ is the Schatten
$1$-norm (i.e. the trace norm, the sum of the absolute eigenvalues of its
argument):
\begin{equation*}
\inner{\nabla\Psi_{vN}\left(W\right)-\nabla\Psi_{vN}\left(W'\right)}{W-W'}_{F}\ge
4\norm{W-W'}_{1}^{2}
\end{equation*}
Hence, $-S\left(W\right)/4$ is $1$-strongly convex, showing that we should take
$\alpha=1/4$.

In order to find $\beta$, we should note that, over the feasible
region $\mathcal{W}$ of \ref{eq:pca-warmuth:convex-objective}, the
maximum entropy (and therefore minimum negative entropy) is achieved
by the ``uniform'' matrix $W=I/d$, for which $-S\left(I/d\right)=-\ln d$.
Therefore, the largest value of $\beta$ which ensures nonnegativity
of $\Psi_{vN}$ on the feasible regions is $\beta=-\ln d$. Hence:
\begin{equation}
\label{eq:pca-warmuth:psi-vN} \Psi_{vN}\left(W\right)= \frac{1}{4}\left(\trace W\ln W+\ln d\right)
\end{equation}
What remains is now a simple matter of plugging this choice of $\Psi$ into
Lemma \ref{lem:pca-warmuth:online-bound}, although some additional work is
required in order to simplify the update and projection steps. The update and
convergence rate are given by the following lemma:
\input{pca/warmuth/theorems/lem-warmuth-bound}
Unfortunately, while this result demonstrates that Warmuth and Kuzmin's
algorithm is an instance of mirror descent, the proved convergence rate is
significantly worse than that which they proved ``from scratch'' (Equation
\ref{eq:pca-warmuth:rate}). In the first place, it is ``non-optimistic'', in
that it \emph{always} gives a $O\left( \sqrt{k(d-k)/T} \right)$ rate, while
Warmuth and Kuzmin's analysis improves towards $O( k/T )$ if $L^*$ is small.
More importantly, this mirror descent-based bound contains an additional
$\sqrt{d-k}$ factor even when $L^*$ is large.

The fact that Warmuth and Kuzmin's algorithm is nothing but mirror descent on a
particular convex relaxation of the PCA objective, however, does provide a
simple ``recipe'' for creating \emph{other} PCA optimization algorithms---we
may consider different convex relaxations, and different distance generating
functions. In Chapter \ref{ch:pca-capped-msg}, one such algorithm will be
presented, for which a straightforwardly-derived convergence rate is
non-optimistic, but does not contain the troublesome $\sqrt{d-k}$ factor.

%% file: pca/warmuth/figures/tab-mirror-notation.tex
\begin{table}

\begin{small}
\begin{center}
\begin{tabularx}{\linewidth}{lX}
\hline
& Description \\
\hline
$\mathcal{W}\subseteq\mathcal{H}$ & Feasible region (convex subset of some Hilbert space $\mathcal{H}$) \\
$\mathcal{F}$ & Set of convex loss functions with bounded subgradients \\
$\Psi$ & Distance generating function \\
$\Psi^*\left(z^*\right) = \sup_z \left( \inner{z^*}{z} - \Psi(z) \right)$ & Convex conjugate of $\Psi$ \\
$\norm{\cdot}_{\Psi}$ & Norm with respect to which $\Psi$ is $1$-strongly convex \\
$\norm{\cdot}_{\Psi^*}$ & Dual norm of $\norm{\cdot}_{\Psi}$ \\
$\Delta_{\Psi}\left(z \Vert z' \right) = \Psi(z) - \Psi(z') - \inner{\Delta\Psi(z')}{z-z'}$ & Bregman divergence derived from $\Psi$ \\
\hline
\end{tabularx}
\end{center}
\end{small}

\caption{
Summary of mirror descent notation introduced in Section \ref{sec:pca-warmuth:mirror}.
}

\label{tab:pca-warmuth:mirror-notation}

\end{table}

%% file: pca/warmuth/theorems/lem-online-bound.tex
\medskip
\begin{lemma}
\label{lem:pca-warmuth:online-bound}

In the setting described
above, define the algorithm $\mathcal{A}$ using the following update
rule:
\begin{align*}
w_{t+1}'= & \nabla\Psi^{*}\left(\nabla\Psi\left(w_{t}\right)-\eta\nabla f_{t}\left(w_{t}\right)\right)\\
w_{t+1}= & \underset{w\in\mathcal{W}}{\argmin}\Delta_{\Psi}\left(w\Vert w'\right)
\end{align*}
with the step size $\eta=\sqrt{B/T}$ where $\sup_{w\in\mathcal{W}}\Psi\left(w\right)\le B$.
Then:
\begin{equation*}
R_{T}\le2\sqrt{B/T}
\end{equation*}
provided that $\frac{1}{T}\sum_{t=1}^{T}\norm{\nabla f_{t}\left(w_{t}\right)}_{\Psi^{*}}^{2}\le1$.
\end{lemma}
\begin{proof}
This is Lemma 2 of \citet{SrebroSrTe11}.
\end{proof}

%% file: pca/warmuth/theorems/lem-warmuth-bound.tex
\medskip
\begin{splitlemma}{lem:pca-warmuth:convergence-vN}

If we
apply the mirror descent algorithm to Problem
\ref{eq:pca-warmuth:convex-objective} with the distance generating function of
Equation \ref{eq:pca-warmuth:psi-vN}, then the update equation is:
\begin{equation*}
W^{\left(t+1\right)}= \project[\textrm{RE}]{\exp\left(\ln
W^{\left(t\right)}-\eta x_{t}x_{t}^{T}\right)}
\end{equation*}
where $\project[\textrm{RE}]{\cdot}$ projects its argument onto the constraints
of Problem \ref{eq:pca-warmuth:convex-objective} with respect to the quantum
relative entropy $D_{KL}\left(W\Vert W'\right)=\trace W\left(\ln W-\ln
W'\right)$. Furthermore, if we perform $T$ iterations with step size
$\eta=\frac{2}{\sqrt{T}}\sqrt{\frac{k}{d-k}}$, then:
\begin{equation*}
\left(d-k\right)\expectation[x\sim\mathcal{D}]{x^{T}\hat{W}x}\le
\left(d-k\right)\inf_{W\in\mathcal{W}}\expectation[x\sim\mathcal{D}]{x^{T}Wx}+\sqrt{\frac{k\left(d-k\right)}{T}}
\end{equation*}
Here,
$\hat{W}$ is sampled uniformly from the set $\left\{
W^{\left(1\right)},W^{\left(2\right)},\dots,W^{\left(T\right)}\right\}
$.
\end{splitlemma}
\begin{splitproof}
We have already shown that $\Psi_{vN}$ is nonnegative and $1$-strongly convex
with respect to $\norm{\cdot}_{1}$. The convex conjugate of $\Psi_{vN}$
satisfies:
\begin{equation*}
\Psi_{vN}^{*}\left(W^{*}\right)=
\sup_{W}\left(\inner{W^{*}}{W}_{F}-\Psi_{vN}\left(W\right)\right)
\end{equation*}
differentiating the expression inside the supremum, and setting the result
equal to zero, gives that $0=W^{*}-\frac{1}{4}\left(I+\ln W\right)$, so the $W$
maximizing the supremum defining $\Psi_{vN}^{*}$ is
$W=\exp\left(4W^{*}-I\right)$, and so:
\begin{equation*}
\Psi_{vN}^{*}\left(W^{*}\right)= \frac{1}{4}\left(\trace
\exp\left(4W^{*}-I\right)-\ln d\right)
\end{equation*}
Because $\nabla\Psi_{vN}\left(W\right)=\frac{1}{4}\left(I+\ln W\right)$,
$\nabla\Psi_{vN}^{*}\left(W^{*}\right)=\exp\left(4W^{*}-I\right)$, and the
stochastic objective $f_t\left(W\right)=x_{t}^{T}Wx_{t}$ has gradient
$x_{t}x_{t}^{T}$, the mirror descent update (Lemma
\ref{lem:pca-warmuth:online-bound}) is:
\begin{equation*}
W^{\left(t+1\right)}= \project[\textrm{RE}]{\exp\left(\ln
W - 4\eta' x_t x_t^T \right)}
\end{equation*}
where $\project[\textrm{RE}]{\cdot}$ projects its argument onto $\mathcal{W}$ with
respect to the Bregman divergence:
\begin{align*}
\Delta_{\Psi_{vN}}\left(W\Vert W'\right)= &
\Psi_{vN}\left(W\right)-\Psi_{vN}\left(W'\right)-\inner{\nabla\Psi_{vN}\left(W'\right)}{W-W'}_{F}\\
= & \frac{1}{4}\trace W\ln W-\frac{1}{4}\trace W'\ln W'-\frac{1}{4}\inner{I+\ln
W'}{W-W'}_{F}\\
= & \frac{1}{4}D_{KL}\left(W\Vert W'\right)-\frac{1}{4}\trace
W+\frac{1}{4}\trace W'
\end{align*}
Because $\trace W=1$ on the feasible set and $\trace W'$ is a constant, the
last two terms may be ignored while performing the projection, showing that
$\project[\mathrm{RE}]{\cdot}$ projects its argument onto the feasible set with
respect to the quantum relative entropy. Because $\norm{\cdot}_{\Psi}$ is the
trace norm, $\norm{\cdot}_{\Psi^{*}}$ is the Schatten $\infty$-norm, and
$\norm{x} \le 1$ with probability $1$ by assumption (Section
\ref{sec:pca-introduction:objective} in Chapter \ref{ch:pca-introduction}), so
$\expectation[x\sim\mathcal{D}]{\norm{xx^{T}}_{\Psi^{*}}^{2}}=1$.  Furthermore,
the maximum value of $\Psi_{vN}$ is achieved on the corners of the matrix
simplex constraints:
\begin{align*}
\sup_{W\in\mathcal{W}}\Psi_{vN}\left(W\right)= &
\frac{1}{4}\left(\ln\frac{1}{d-k}+\ln d\right)\\
= & \frac{1}{4}\ln\frac{d}{d-k}\\
\le & \frac{1}{4}\cdot\frac{k}{d-k}
\end{align*}
so Equation \ref{eq:pca-warmuth:stochastic-bound} yields that, with the step
size $\eta'=\frac{1}{2\sqrt{T}}\sqrt{\frac{k}{d-k}}$:
\begin{equation*}
\expectation[x\sim\mathcal{D}]{x^{T}\hat{W}x}\le
\inf_{W\in\mathcal{W}} \expectation[x\sim\mathcal{D}]{x^{T}Wx}+\frac{1}{\sqrt{T}}\sqrt{\frac{k}{d-k}}
\end{equation*}
We define $\eta=4\eta'$ to complete the proof.
\end{splitproof}

%% file: pca/warmuth/sec-proofs.tex
\section{Proofs for Chapter \ref{ch:pca-warmuth}}

\begin{proofs}
\input{pca/warmuth/theorems/lem-projection}
\input{pca/warmuth/theorems/lem-warmuth-bound}
\end{proofs}

%% file: pca/ch-capped-msg.tex
\chapter{The Capped MSG Algorithm}\label{ch:pca-capped-msg}

\input{pca/capped-msg/sec-overview}
\input{pca/capped-msg/sec-msg}
\input{pca/capped-msg/sec-capping}
\input{pca/capped-msg/sec-experiments}

\paragraph{Collaborators:} The work presented in this chapter was performed
jointly with Raman Arora and Nathan Srebro.

\clearpage
\input{pca/capped-msg/sec-proofs}

%% file: pca/capped-msg/sec-overview.tex
\section{Overview}\label{sec:pca-capped-msg:overview}

In Chapter \ref{ch:pca-introduction}, two iterative techniques for performing
stochastic PCA, the stochastic power method and the incremental algorithm
(Sections \ref{sec:pca-introduction:power} and
\ref{sec:pca-introduction:incremental}, respectively), were discussed which
both perform a very small amount of computation at each iteration, but for
which no good theoretical guarantees are known. The former converges with
probability $1$, but at an unknown rate, while for the latter there exist
distributions for which it will converge to a suboptimal solution with a high
probability.

The algorithm discussed in Chapter \ref{ch:pca-warmuth}, originally due to
\citet{WarmuthKu06}, is different in both of these respects: a very good bound
on its convergence rate is known, but it suffers from a very high---in some
cases prohibitively high---cost per iteration.

Our goal in this chapter is to present an algorithm which combines the
strengths of both of these types of approaches. We begin in Section
\ref{sec:pca-capped-msg:msg} by proposing an algorithm, called Matrix
Stochastic Gradient (MSG), which yields updates very similar to the incremental
algorithm, but for which we can give a theoretical bound on how many iterations
are needed to converge to an $\epsilon$-suboptimal solution. In a sense, this
algorithm is a ``cross'' between the incremental algorithm and Warmuth and
Kuzmin's algorithm, in that its updates are similar in form to the former,
while the theoretical convergence bound on its performance is similar to
(albeit slightly worse than) that of the latter. Like Warmuth and Kuzmin's
algorithm, MSG is designed to optimize a convex relaxation of the PCA
optimization problem, and can theoretically have the same unacceptably high
computational cost per iteration.

In order to address this, in Section \ref{sec:pca-capped-msg:capped-rank} we
will introduce a variant of MSG, called ``capped MSG'', which introduces an
explicit rank constraint in the spirit of the incremental algorithm, but does
so far more safely, in that the resulting algorithm is much less likely to
``get stuck''. We believe that the capped MSG algorithm combines the speed of
the incremental algorithm with most of the reliability of MSG, yielding a
``best of both worlds'' approach.
This chapter will conclude in Section \ref{sec:pca-capped-msg:experiments} with
an experimental comparison of the algorithms introduced in this chapter with
most of those discussed in Chapters \ref{ch:pca-introduction} and
\ref{ch:pca-warmuth}.

%% file: pca/capped-msg/sec-msg.tex
\section{Matrix Stochastic Gradient (MSG)}\label{sec:pca-capped-msg:msg}

\input{pca/capped-msg/figures/tab-bounds}
One technique which one could use to optimize the $M$-based objective of
Problem \ref{eq:pca-introduction:objective-m} in Chapter
\ref{ch:pca-introduction} is projected stochastic gradient descent.  On the PCA
problem, SGD will perform the following steps at each iteration:
\begin{equation}
\label{eq:pca-capped-msg:update} M^{(t+1)} = \project{ M^{(t)} + \eta_t x_t
x_t^T }
\end{equation}
here, $x_t x_t^T$ is the gradient of $x_t^T M x_t$, $\eta_t$ is a step size,
and $\project{\cdot}$ projects its argument onto the feasible region with
respect to the Frobenius norm.

The performance of SGD is well-understood, at least on convex problems.
Unfortunately, as we have seen, while the objective function of Problem
\ref{eq:pca-introduction:objective-m} is linear (and therefore both convex and
concave), the constraints are non-convex. For this reason, performing SGD on
this objective is inadvisable---as was the case for the incremental algorithm,
it is very easy to construct a distribution $\mathcal{D}$ on which it would
``get stuck''.

Fortunately, this non-convexity problem can be addressed using the same
technique as was successful in Section \ref{sec:pca-warmuth:objective} of
Chapter \ref{ch:pca-warmuth}: we instead solve a \emph{convex relaxation} of
the problem in which the feasible region has been replaced with its convex
hull:
\begin{align}
\label{eq:pca-capped-msg:convex-objective} \mbox{maximize} : &\ \
\expectation[x\sim\mathcal{D}]{x^T M x} \\
\notag \mbox{subject to} : &\ \ M \in \R^{d\times d}, 0 \preceq M \preceq I,
\trace M = k. 
\end{align}
This objective is quite similar to that optimized by \citet{WarmuthKu06}
(compare to Problem \ref{eq:pca-warmuth:convex-objective} in Chapter
\ref{ch:pca-warmuth})---aside from scaling, the main difference is that their
objective seeks a $(d-k)$-dimensional minimal subspace, rather than a
$k$-dimensional maximal subspace.

The SGD update equation on this objective is simply Equation
\ref{eq:pca-capped-msg:update} again, with the difference being that the
projection is now performed onto the (convex) constraints of Problem
\ref{eq:pca-capped-msg:convex-objective}. We call the resulting algorithm
Matrix Stochastic Gradient (MSG). We could analyze MSG using a mirror
descent-based analysis similar to that of Section \ref{sec:pca-warmuth:mirror}
in Chapter \ref{ch:pca-warmuth} (the distance generating function would be the
squared Frobenius norm). Unlike Warmuth and Kuzmin's algorithm, however, a
direct analysis is much simpler and more straightforward:

\input{pca/capped-msg/theorems/lem-rate}
Comparing this convergence rate to that of Warmuth and Kuzmin's algorithm
(Equation \ref{eq:pca-warmuth:rate} in Chapter \ref{ch:pca-warmuth}) shows that
the only difference between the two is that Warmuth and Kuzmin's algorithm
enjoys an ``optimistic'' rate, while MSG does not---if the desired level of
suboptimality is of roughly the same order as the optimal compression loss
$L^*$ (i.e. the problem is ``easy''), then the algorithm will converge at a
roughly $1/T$ rate. On more difficult problem instances, Warmuth and Kuzmin's
algorithm will match the $1/\sqrt{T}$ rate of MSG.

In addition to the relationship to Warmuth and Kuzmin's algorithm, MSG's update
is strikingly similar to that of the incremental algorithm (see Section
\ref{sec:pca-introduction:incremental} in Chapter \ref{ch:pca-introduction}).
Both add $x x^T$ to the current iterate, and then perform a projection. The
most significant difference, from a practical standpoint, is that while the
iterates of the incremental algorithm will never have rank larger than $k$,
each MSG iterate may have rank as large as $d$ (although it will typically be
much lower). A consequence of this is that, despite the fact that Lemma
\ref{lem:pca-capped-msg:rate} can be used to bound the number of iterations
required to reach $\epsilon$-suboptimality, each iteration may be so
computationally expensive that the algorithm is still impractical. In Section
\ref{subsec:pca-capped-msg:low-rank}, we will discuss how this situation may be
addressed.


\subsection{The Projection}\label{subsec:pca-capped-msg:projection}

In order to actually perform the MSG update of Equation
\ref{eq:pca-capped-msg:update}, we must show how one projects onto the
constraints of Problem \ref{eq:pca-capped-msg:convex-objective}. As in Section
\ref{subsec:pca-warmuth:projection} of Chapter \ref{ch:pca-warmuth}, the first
step in developing an efficient algorithm for doing so is to characterize the
solution of the projection problem, which we do in the following lemma:

\input{pca/capped-msg/theorems/lem-projection}
This result shows that projecting onto the feasible region amounts to finding
the value of $S$ such that, after shifting the eigenvalues by $S$ and clipping
the results to $[0,1]$, the result is feasible.
\input{pca/capped-msg/figures/alg-project}
As was the case for the projection used by Warmuth and Kuzmin's algorithm, this
projection operates \emph{only} on the eigenvalues, which simplifies its
implementation significantly. Algorithm \ref{alg:pca-capped-msg:project}
contains pseudocode which finds $S$ from a list of eigenvalues. It is optimized
to efficiently handle repeated eigenvalues---rather than receiving the
eigenvalues in a length-$d$ list, it instead receives a length-$n$ list
containing only the \emph{distinct} eigenvalues, with $\kappa$ containing the
corresponding multiplicities.

The central idea motivating the algorithm is that, in a sorted array of
eigenvalues, all elements with indices below some threshold $i$ will be clipped
to $0$, and all of those with indices above another threshold $j$ will be
clipped to $1$. The pseudocode simply searches over all possible pairs of such
thresholds until it finds the one that works. However, it does not perform an
$O(m^2)$ search---instead, a linear search is performed by iteratively
incrementing either $i$ or $j$ at each iteration in such a way that the
eigenvalues which they index are $1$ unit apart.


\subsection{Efficient Updates}\label{subsec:pca-capped-msg:low-rank}

In Section \ref{subsec:pca-warmuth:efficiency}, we showed that the iterates
of Warmuth and Kuzmin's algorithm tend to have a large number eigenvalues
capped at $1/(d-k)$, the maximum allowed by the constraints. MSG's iterates
have a similar tendency, except that the repeated eigenvalues are at the lower
bound of $0$. The reason for this is that MSG performs a rank-$1$ update
followed by a projection onto the constraints (see Equation
\ref{eq:pca-capped-msg:update}), and because $x_t x_t^T$ is positive
semidefinite, $M' = M^{(t)} + \eta x_t x_t^T$ will have a \emph{larger} trace
than $M^{(t)}$ (i.e.  $\trace M' \ge k$). As a result, the projection, as is
shown by Lemma \ref{lem:pca-capped-msg:projection}, will \emph{subtract} a
quantity $S$ from every eigenvalue of $M'$, clipping each to $0$ if it becomes
negative. Therefore, each MSG update will increase the rank of the iterate by
at most $1$, and has the potential to decrease it, perhaps significantly. It's
very difficult to theoretically quantify how the rank of the iterates will
evolve over time, but we have observed empirically that the iterates do tend to
have relatively low rank. In Section \ref{sec:pca-capped-msg:experiments}, we
will explore this issue in greater detail experimentally.

Exploiting the low rank of the MSG iterates relies on exactly the same linear
algebra as that used for the incremental algorithm of Section
\ref{sec:pca-introduction:incremental} in Chapter \ref{ch:pca-introduction}.
Algorithm \ref{alg:pca-introduction:rank1-update} (in the same chapter) can
find the nonzero eigenvalues and associated eigenvectors of $M^{(t)} + \eta x_t
x_t^T$ from those of $M^{(t)}$ using $O((k_{t}')^2 d)$ operations, where
$k_{t}'$ is the rank of $M^{(t)}$. Notice that, because we have the
eigenvalues on-hand at every iteration, we can immediately use the projection
routine of Algorithm \ref{alg:pca-capped-msg:project} without needing to first
perform an eigendecomposition.

%% file: pca/capped-msg/figures/tab-bounds.tex
\begin{table}

\begin{small}
\begin{center}
\begin{tabular}{r|cc|c}
\hline
& Computation & Memory & Convergence \\
\hline
SAA & $nd^2$ & $d^2$ & $\sqrt{\frac{k}{n}}$ \\
SAA (Coppersmith-Winograd) & $nd^{1.3727}$ & $d^2$ & $\sqrt{\frac{k}{n}}$ \\
Stochastic Power Method & $Tkd$ & $kd$ & w.p. $1$ \\
Incremental & $T k^2 d$ & $kd$ & no \\
Warmuth \& Kuzmin & $\sum_{t=1}^{T} \left(k_t'\right)^2 d$ & $\max_{t\in\{1,\dots,T\}} k_t' d$ & $\sqrt{\frac{L^* k}{T}} + \frac{k}{T}$ \\
MSG & $\sum_{t=1}^{T} \left(k_t'\right)^2 d$ & $\max_{t\in\{1,\dots,T\}} k_t' d$ & $\sqrt{\frac{k}{T}}$ \\
Capped MSG & $T K^2 d$ & $K d$ & testable \\
\hline
\end{tabular}
\end{center}
\end{small}

\caption{
Summary of results from Chapters \ref{ch:pca-introduction} and
\ref{ch:pca-warmuth}, as well as Lemma \ref{lem:pca-capped-msg:rate}. All
bounds are given up to constant factors. The ``Convergence'' column contains
bounds on the suboptimality of the solution---i.e. the difference between the
total variance captured by the rank-$k$ subspace found by the algorithm, and
the best rank-$k$ subspace with respect to the data distribution $\mathcal{D}$
(the objective in Chapter \ref{ch:pca-warmuth} is scaled by a factor of $d-k$
relative to the other objectives---we have corrected for this here). MSG
enjoys a convergence rate which differs from that of Warmuth and Kuzmin's
algorithm only in that it is non-optimistic.
}

\label{tab:pca-capped-msg:bounds}

\end{table}

%% file: pca/capped-msg/theorems/lem-rate.tex
\medskip
\begin{splitlemma}{lem:pca-capped-msg:rate}

Suppose that we perform $T$ iterations of MSG on Problem
\ref{eq:pca-capped-msg:convex-objective}, with step size $\eta = 2 \sqrt{ \frac{k}{T} }$.
Further suppose WLOG that $\expectation[x\sim\mathcal{D}]{\norm{x}^4} \le 1$.
Then:
\begin{equation*}
\expectation[x\sim\mathcal{D}]{ x^T M^* x - x^T \bar{M} x } \le 2 \sqrt{
	\frac{k}{T} }
\end{equation*}
where $\bar{M} = \frac{1}{T} \sum_{t=1}^{T} M^{(t)}$ is the average of the
iterates, and $M^*$ is the optimum (a rank $k$ projection matrix which projects
onto the maximal $k$ eigenvectors of the second-moment matrix of
$\mathcal{D}$).
\end{splitlemma}
\begin{splitproof}
(Based on Lemma 55 of \citet{Sridharan11}, which in turn is based on
\citet{NemirovskiYu83}) Define $\tilde{M}^{(t+1)} = M^{(t)} + \eta x_t x_t^T$,
and observe that:
\begin{align*}
\MoveEqLeft \eta \left( \sum_{t=1}^T x_t^T M^* x_t - \sum_{t=1}^T x_t^T M^{(t)}
x_t \right) & \\ & = \sum_{t=1}^T \inner{ \eta x_t x_t^T }{ M^* - M^{(t)} }_F
\\ & = \frac{1}{2} \sum_{t=1}^T \left( \eta^2 \norm{x_t}^4 + \norm{ M^* -
M^{(t)} }_F^2 - \norm{ M^* - \tilde{M}^{(t+1)} }_F^2 \right) \\ & \le
\frac{1}{2} \sum_{t=1}^T \left( \eta^2 \norm{x_t}^4 + \norm{ M^* - M^{(t)}
}_F^2 - \norm{ M^* - M^{(t+1)} }_F^2 \right) \\ & \le \frac{\eta^2}{2}
\sum_{t=1}^T \norm{x_t}^4 + \frac{1}{2} \norm{ M^* - M^{(1)} }_F^2
\end{align*}
Dividing through by $\eta T$, taking expectations (observe that $x_t$ and
$M^{(t)}$ are conditionally independent), and using the bounds
$\expectation[x\sim\mathcal{D}]{ \norm{x}^4 } \le 1$ and $\norm{ M^* - M^{(1)}
}_F^2 \le 4k$ ($2\sqrt{k}$ bounds the diameter of the feasible region) gives
that:
\begin{equation*}
\expectation[x\sim\mathcal{D}]{x M^* x - x \bar{M} x} \le \frac{\eta}{2} +
\frac{2k}{\eta T}.
\end{equation*}
Choosing $\eta = 2 \sqrt{ \frac{k}{T} }$ completes the proof.
\end{splitproof}

%% file: pca/capped-msg/theorems/lem-projection.tex
\medskip
\begin{splitlemma}{lem:pca-capped-msg:projection}

Let $M'\in\R^{d\times d}$ be a symmetric matrix, with eigenvalues 
$\sigma_{1}',\dots,\sigma_{d}'$ and associated eigenvectors $v_{1}',\dots,v_{d}'$. 
If $M=\project{M'}$ projects $M'$ onto the feasible region
of Problem \ref{eq:pca-capped-msg:convex-objective} with respect to the Frobenius norm, then
$M$ will be the unique feasible matrix which has the same set of eigenvectors
as $M'$, with the associated eigenvalues $\sigma_{1},\dots,\sigma_{d}$
satisfying:
\begin{equation*}
\sigma_{i} = \max\left(0,\min\left(1,\sigma_{i}'+S\right)\right)
\end{equation*}
with $S\in\R$ being chosen in such a way that:
\begin{equation*}
\sum_{i=1}^{d}\sigma_{i}=k
\end{equation*}
\end{splitlemma}
\begin{splitproof}
This is the same proof technique as was used to prove Lemma
\ref{lem:pca-warmuth:projection} in Chapter \ref{ch:pca-warmuth}. We begin by
writing the problem of finding $M$ in the form of a convex optimization problem
as:
\begin{align*}
\mbox{minimize} : &\ \ \norm{M-M'}_F^2\\
\mbox{subject to} : &\ \ 0 \preceq M \preceq I, \trace M = k.
\end{align*}
Because the objective is strongly convex, and the constraints are convex, this
problem must have a unique solution. Letting $\sigma_{1},\dots,\sigma_{d}$ and
$v_{1},\dots,v_{d}$ be the eigenvalues and associated eigenvectors of $M$, we
may write the KKT first-order optimality conditions \citep{BoydVa04} as:
\begin{equation}
\label{eq:pca-capped-msg:kkt} 0 = M-M' + \mu I - \sum_{i=1}^{d}
\alpha_{i}v_{i}v_i^T + \sum_{i=1}^d \beta_{i}v_{i}v_{i}^{T}
\end{equation}
where $\mu$ is the Lagrange multiplier for the constraint $\trace M = k$, and
$\alpha_i, \beta_i \ge 0$ are the Lagrange multipliers for the constraints $0
\preceq M$ and $M \preceq I$, respectively.  The complementary slackness
conditions are that $\alpha_{i}\sigma_{i} = \beta_{i}\left(\sigma_{i}-1\right)
= 0$. In addition, $M$ must be feasible.

Because every term in Equation \ref{eq:pca-capped-msg:kkt} \emph{except} for $M'$ has the same
set of eigenvectors as $M$, it follows that an optimal $M$ must have the same
set of eigenvectors as $M'$, so we may take $v_i = v_i'$, and write Equation
\ref{eq:pca-capped-msg:kkt} purely in terms of the eigenvalues:
\begin{equation*}
\sigma_{i} = \sigma_{i}'-\mu+\alpha_{i}-\beta_{i}
\end{equation*}
Complementary slackness and feasibility with respect to the constraints
$0 \preceq M \preceq I$ gives that if $0 \le \sigma_{i}'-\mu\le 1$,
then $\sigma_{i}=\sigma_{i}'-\mu$. Otherwise, $\alpha_{i}$ and $\beta_{i}$
will be chosen so as to clip $\sigma_{i}$ to the active constraint:
\begin{equation*}
\sigma_i = \max\left( 0, \min\left( 1, \sigma_{i}'-\mu \right) \right)
\end{equation*}
Primal feasibility with respect to the constraint $\trace M = k$ gives that
$\mu$ must be chosen in such a way that $\trace M = k$, completing the proof.
\end{splitproof}

%% file: pca/capped-msg/figures/alg-project.tex
\begin{algorithm}[t]

\begin{pseudocode}
\codename $\code{project}\left( d,k,m:\N, \sigma':\R^{m}, \kappa':\N^{m} \right)$\\
\codeline $\sigma',\kappa' \leftarrow \code{sort}( \sigma',\kappa' )$;\\
\codeline $i \leftarrow 1$; $j \leftarrow 1$; $s_i \leftarrow 0$; $s_j \leftarrow 0$; $c_i \leftarrow 0$; $c_j \leftarrow 0$;\\
\codeline $\code{while } i \le m$\\
\codeline \>$\code{if } ( i < j )$\\
\codeline \>\>$S \leftarrow ( k - ( s_j - s_i ) - ( d - c_j ) ) / ( c_j - c_i )$;\\
\codeline \>\>$b \leftarrow ($\\
\codeskip \>\>\>$( \sigma_i' + S \ge 0 ) \code{ and } ( \sigma_{j-1}' + S \le 1 )$\\
\codeskip \>\>\>$\code{and } ( ( i \le 1 ) \code{ or } ( \sigma_{i-1}' + S \le 0 ) )$\\
\codeskip \>\>\>$\code{and } ( ( j \ge m ) \code{ or } ( \sigma_{j+1}' \ge 1 ) )$\\
\codeskip \>\>$)$;\\
\codeline \>\>$\code{return }S \code{ if } b$;\\
\codeline \>$\code{if } ( j \le m ) \code{ and } ( \sigma_j' - \sigma_i' \le 1 )$\\
\codeline \>\>$s_j \leftarrow s_j + \kappa_j' \sigma_j'$; $c_j \leftarrow c_j + \kappa_j'$; $j \leftarrow j + 1$;\\
\codeline \>$\code{else}$\\
\codeline \>\>$s_i \leftarrow s_i + \kappa_i' \sigma_i'$; $c_i \leftarrow c_i + \kappa_i'$; $i \leftarrow i + 1$;\\
\codeline $\code{return error}$;
\end{pseudocode}

\caption{
Routine which finds the $S$ of Lemma \ref{lem:pca-capped-msg:projection}. It takes as
parameters the dimension $d$, ``target'' subspace dimension $k$, and the number
of \emph{distinct} eigenvalues $m$ of the current iterate. The length-$m$
arrays $\sigma'$ and $\kappa'$ contain the distinct eigenvalues and their
multiplicities, respectively, of $M'$ (with $\sum_{i=1}^{m} \kappa_i' = d$).
Line $1$ sorts $\sigma'$ and re-orders $\kappa'$ so as to match this sorting.
The loop will be run at most $2m$ times (once for each possible increment to
$i$ or $j$ on lines $12$--$15$), so the computational cost is dominated by that
of the sort: $O(m \log m)$.
}

\label{alg:pca-capped-msg:project}

\end{algorithm}

%% file: pca/capped-msg/sec-capping.tex
\section{Capped MSG}\label{sec:pca-capped-msg:capped-rank}

While, as was observed in Section \ref{sec:pca-capped-msg:msg}, MSG's iterates
will tend to have ranks $k_t'$ smaller than $d$, they will nevertheless also be
larger than $k$, and may occasionally be much larger. For this reason, in
practice, we recommend adding a hard constraint $K$ on the rank of the
iterates:
\begin{align}
\label{eq:pca-capped-msg:capped-convex-objective} \mbox{maximize} : & \ \
\expectation[x\sim\mathcal{D}]{x^T M x} \\
\notag \mbox{subject to} : &\ \ M \in \R^{d\times d}, 0 \preceq M \preceq I \\
\notag  &\ \ \trace M = k, \rank M \le K
\end{align}
We will refer to SGD on this objective as the ``capped MSG'' algorithm. For
$K<d$, this objective is non-convex, so the convergence result of Lemma
\ref{lem:pca-capped-msg:rate} no longer applies. However, in practical terms,
varying the parameter $K$ enables us to smoothly transition between a regime
with a low cost-per-iteration, but no convergence result (similar to the
incremental algorithm), to one in which the convergence rate is known, but the
cost-per-iteration may be very high (similar to Warmuth and Kuzmin's
algorithm).

\subsection{The Projection}

The implementation of capped MSG is very similar to that of ``vanilla''
MSG---the only change is in how one performs the projection. Similar reasoning
to that which was used in the proof of Lemma
\ref{lem:pca-capped-msg:projection} shows that if $M^{(t+1)} = \project{ M' }$
with $M' = M^{(t)} + \eta x_t x_t^T$, then $M^{(t)}$ and $M'$ are
simultaneously diagonalizable, and therefore we can consider only how the
projection acts on the eigenvalues. Hence, if we let $\sigma'$ be the vector of
the eigenvalues of $M'$, and suppose that there are more than $K$ such
eigenvalues, then there is a size-$K$ subset of $\sigma'$ such that applying
Algorithm \ref{alg:pca-capped-msg:project} to this set gives the projected
eigenvalues, with the projected eigenvectors being those corresponding to the
$K$ eigenvalues which we ``keep''.

Since we perform only a rank-$1$ update at every iteration, the matrix $M'$
will have a rank of at most $K+1$. Hence, there are at most $K+1$ possible
size-$K$ subsets of the eigenvalues which we must consider. We call Algorithm
\ref{alg:pca-capped-msg:project} for each such subset, and then calculate the
Frobenius norm between each of the resulting projections, and the original
matrix $M'$. That which is closest to $M'$ is the projected matrix which we
seek.

The cost of calculating these Frobenius norms is insignificant ($O(K)$ time for
each), and Algorithm \ref{alg:pca-capped-msg:project} costs $O(K \log K)$
operations, so the total computational cost of the capped MSG projection is
$O(K^2 \log K)$ operations. The projection therefore has no effect on the
asymptotic runtime of the capped MSG algorithm, because Algorithm
\ref{alg:pca-introduction:rank1-update}, which we also must perform at every
iteration, requires $O(K^2 d)$ operations.

\subsection{Convergence}

It is important to recognize that the capped MSG algorithm is very different
from one obvious alternative: using the incremental algorithm of Section
\ref{sec:pca-introduction:incremental} to find a maximal variance
$K$-dimensional subspace, and then taking the top $k$ directions. Both
approaches have the advantage that there is more ``room'' in which to place
candidate directions. However, the capped MSG algorithm is still only searching
for a $k$-dimensional subspace---this is the constraint $\trace M \le k$. As a
result, it should have less of a tendency to get ``stuck'', since as it becomes
more confident in its current candidate, the trace of $M$ will become
increasingly concentrated on the top $k$ directions, eventually freeing up the
remaining $K-k$ directions for further exploration. This property is shown more
rigorously in the following lemma:

\input{pca/capped-msg/theorems/lem-capping}
The facts that the optimal solution to problem
\ref{eq:pca-capped-msg:capped-convex-objective} will have rank $k$, and that
there are no local optima, means that we expect (although we have not proved)
that the capped MSG algorithm will always converge \emph{eventually} when
$K>k$, although it might spend a very large amount of time exploring flat
regions of the objective.

More practically, this observation makes it possible for one to easily check
for convergence: if the capped MSG algorithm appears to be oscillating around a
solution of rank $K$, then it must be stuck. While one could wait until a
fortuitous sequence of samples causes it to escape, one could accelerate the
process by simply increasing the upper bound $K$, after which the algorithm
should continue to make progress. Conversely, once it has converged to a
rank-$k$ solution, it has found the global optimum, and the algorithm may be
terminated.

%% file: pca/capped-msg/theorems/lem-capping.tex
\medskip
\begin{splitlemma}{lem:pca-capped-msg:capping}

Consider the capped-MSG algorithm with $K>k$, and suppose that the true second
moment matrix $\Sigma=\expectation[x\sim\mathcal{D}]{xx^T}$ has no repeated
eigenvalues. Then Problem \ref{eq:pca-capped-msg:capped-convex-objective},
despite not being a convex optimization problem (for $K<d$), has no local
optima, and its unique global optimum is the rank-$k$ matrix projecting onto
the span of the top-$k$ eigenvectors of $\Sigma$.
\end{splitlemma}
\begin{splitproof}
We'll prove this by considering two cases: first, that in which $M$ has
rank-$k$, and next, that it which it has rank higher than $k$.

\paragraph{Case 1:} suppose that $M$ has exactly $k$ nonzero eigenvalues, but that the
corresponding eigenvectors do not span the maximal subspace. Let
$v_1,\dots,v_k$ be the eigenvectors of $M$, and observe that, by the assumption
that $M$ is not optimal, there must exist a $v_{k+1}$ of norm $1$ which is
orthogonal to $v_1,\dots,v_k$ such that the matrix $\tilde{M}$ projecting onto
the maximal rank-$k$ subspace in the span of $v_1,\dots,v_{k+1}$ has strictly
larger objective function value that $M$. Both $M$ and $\tilde{M}$ reside in
the convex set:
\begin{equation*}
\mathcal{S} = \left\{ M : VV^T M = M \wedge 0 \preceq M \preceq I \wedge \trace
M = k \right\}
\end{equation*}
Here, $V$ is the matrix containing $v_1,\dots,v_{k+1}$ in its columns, so that
$\mathcal{S}$ is essentially what the feasible region of the convex non-capped
MSG objective would be if it were restricted to the span of
$v_1,\dots,v_{k+1}$. Furthermore, $\mathcal{S}$ is a subset of the feasible
region of Problem \ref{eq:pca-capped-msg:capped-convex-objective}. Since $M$ is
not optimal on this convex subset of the feasible region, it cannot be a local
optimum.

\paragraph{Case 2:} Suppose that $M$ has more than $k$ nonzero eigenvalues, and
let $i=\argmin_{\sigma_i > 0} v_i^T \Sigma v_i$. Then, by the same reasoning as
was used in Section \ref{subsec:pca-warmuth:unrelaxing} of Chapter
\ref{ch:pca-warmuth}, we may move a $\sigma_i$-sized amount of the mass onto
the $\sigma_j$s with $\{j:\sigma_j>0\wedge j\ne i\}$ without decreasing the
objective function value. Furthermore, we may do this \emph{continuously}. The
rank-$k$ solution which we find via this procedure is, by case 1, not a local
optimum, and can only be the global optimum if this transfer of mass
\emph{increases} the objective function value, since $\Sigma$ has no repeated
eigenvalues. Hence, $M$ itself is not a local optimum, but is at worst a saddle
point.
\end{splitproof}

%% file: pca/capped-msg/sec-experiments.tex
\section{Experiments}\label{sec:pca-capped-msg:experiments}

In this section, we report the results of experiments on simulated data
designed to explore some of the ``edge cases'' which may cause difficulty for
stochastic PCA algorithms (Section \ref{subsec:pca-capped-msg:simulated}), as
well as a comparison on the real-world MNIST dataset (Section
\ref{subsec:pca-capped-msg:real}). The experiments in this section summarize
the empirical performance not only of the algorithms introduced in this
chapter, but also most of those discussed in Chapters \ref{ch:pca-introduction}
and \ref{ch:pca-warmuth}.

In addition to MSG and capped MSG, we also compare to the incremental algorithm
described in Section~\ref{sec:pca-introduction:incremental} of Chapter
\ref{ch:pca-introduction}~\citep{AroraCoLiSr12}, the online PCA algorithm of
Chapter \ref{ch:pca-warmuth}~\citep{WarmuthKu06,WarmuthKu08}, and, in the MNIST
experiments, a Grassmannian SGD algorithm which is nothing but a
full-information variant of the recently proposed GROUSE algorithm
\citep{BalzanoNoRe10}, and can be regarded as a refinement of the SGD algorithm
of Section \ref{sec:pca-introduction:power} in Chapter \ref{ch:pca-introduction}.

In order to compare Warmuth and Kuzmin's algorithm, MSG, capped MSG and the
incremental algorithm in terms of runtime, we calculate the dominant term in
the computational complexity: $\sum_{t=1}^{T} (k_t')^2$ (the true computational
cost is this quantity, multiplied by $d$, which is the same for all
algorithms). The incremental algorithm has the property that $k_t' \le k$,
while $k_t' \le K$ for the capped MSG algorithm, but in order to ensure a fair
comparison, we measure the \emph{actual} ranks, instead of merely using these
bounds.

All algorithms except for the incremental algorithm and Grassmannian SGD store
an internal state matrix of rank $k_t'>k$. As we observed in Sections
\ref{sec:pca-capped-msg:capped-rank}, if the covariance matrix of $\mathcal{D}$
has no repeated eigenvalues, then the optimal solution of Problems
\ref{eq:pca-capped-msg:convex-objective} and
\ref{eq:pca-capped-msg:capped-convex-objective} will be rank-$k$, while if the
distribution's covariance matrix does have repeated eigenvalues, there there is
no point in attempting to distinguish between them. Hence, we compute
suboptimalities based on the largest $k$ eigenvalues of the state matrix
$M^{(t)}$ (smallest, for Warmuth \& Kuzmin's algorithm, which searches for a
$(d-k)$-dimensional minimal subspace)---this is the same procedure as we
recommended in Section \ref{subsec:pca-warmuth:unrelaxing} of Chapter
\ref{ch:pca-warmuth}.


\subsection{Simulated Data}\label{subsec:pca-capped-msg:simulated}

\input{pca/capped-msg/figures/fig-simulated-iterations}
\input{pca/capped-msg/figures/fig-simulated-runtime}
Our first round of experiments is designed to explore both the raw performance
of our algorithms (MSG and capped MSG) on distributions which we believe to be
particularly bad for them. To this end, we generated data from known
$32$-dimensional distributions with diagonal covariance matrices $\Sigma^{(k)}
= \code{diag}(\sigma^{(k)}/\norm{\sigma^{(k)}})$, where $\sigma$ is the average
of a ``smooth'' portion and a discontinuous portion:
\begin{equation*}
\sigma^{(k)}_i = \frac{1}{2}\left( \frac{(1.1)^{-i}}{\sum_{j=1}^{32}
(1.1)^{-j}} + \frac{\mathbf{1}_{ i \le k }}{k} \right)
\end{equation*}
Observe that $\Sigma^{(k)}$ has a smoothly-decaying set of eigenvalues, except
that there is a large gap between the $k$th and $(k+1)$th eigenvalues. We
experimented with $k\in\{1,2,4\}$, where $k$ is both the desired subspace
dimension used by each algorithm, and also the parameter defining
$\Sigma^{(k)}$. Hence, these are examples of relatively ``easy'' problems for
the algorithm of \citet{WarmuthKu08}, in that the compression loss suffered by
the best $k$-dimensional subspace is relatively small. As was mentioned in
Section \ref{sec:pca-capped-msg:msg}, this is the regime in which their upper bound on the
convergence rate of their algorithm is superior to that of MSG
(Lemma \ref{lem:pca-capped-msg:rate}).
%

In addition to varying $k$, we also experimented with two different
distributions of the same covariance $\Sigma^{(k)}$. The first is simply a
Gaussian distribution with covariance matrix $\Sigma^{(k)}$. The second is
meant to be particularly hard for algorithms (such as the incremental algorithm
of Section \ref{sec:pca-introduction:incremental} and capped MSG algorithm of Section
\ref{sec:pca-capped-msg:capped-rank}) which place a hard cap on the rank of
their iterates: the distribution samples the $i$th standard unit basis vector
$e_i$ with probability $\sqrt{\Sigma_{ii}}$.  We refer to this as the
``orthogonal distribution'', since it is a discrete distribution over $32$
orthogonal vectors.
Interestingly, while this is a ``hard'' distribution for the incremental and
capped MSG algorithms, it is particularly easy for SAA, since, as we saw in
Section \ref{sec:pca-introduction:saa}, this algorithm will perform extremely
well when $\mathcal{D}$ is supported on $d$ orthogonal directions. The reason
for this difference is that the two ``capped'' algorithms have a type of
\emph{memory}, and will tend to ``forget'' directions supported on infrequent
but large-magnitude samples.

All of the compared algorithms except the incremental algorithm have a
step-size parameter. In these experiments, we ran each algorithm for the
decreasing step sizes $\eta_t = c/\sqrt{t}$ for $c\in 2^{-12:5}$, and created
plots for the best choice of $c$, in terms of the average suboptimality over
the run (on the real-data experiments of the following section, we use a
validation-based approach to choosing the step-size).

Figure \ref{fig:pca-capped-msg:simulated-iterations} contains plots of
individual runs of MSG, capped MSG with $K=k+1$, the incremental algorithm, and
Warmuth and Kuzmin's algorithm, all based not only on the same sequence of
samples drawn from $\mathcal{D}$. On the Gaussian data distribution (top row),
all of the algorithms performed roughly comparably, with the incremental
algorithm seemingly being the best-performer, followed by Warmuth \& Kuzmin's
algorithm.  On the orthogonal distribution (bottom row), the behavior of the
algorithms changed markedly, with the incremental algorithm getting stuck for
$k\in\{1,4\}$, and the others intermittently plateauing at intermediate
solutions before beginning to again converge rapidly towards the optimum. This
behavior is to be expected on the capped MSG algorithm, due to the fact that
the dimension of the subspace stored at each iterate is constrained. However,
it is somewhat surprising that MSG and Warmuth \& Kuzmin's algorithm behaved
similarly.

In Figure \ref{fig:pca-capped-msg:simulated-runtime}, we look in greater depth
at the results for the orthogonal distribution with $k=4$. We can see from the
left-hand plot that both MSG and Warmuth \& Kuzmin's algorithm maintain
subspaces of roughly dimension $15$. For reference, the $k_t'$ for both
algorithms tended to be roughly $10$ on the Gaussian distribution with $k=4$ or
orthogonal distribution with $k=1$, and roughly $3$ for the Gaussian
distribution with $k=1$. The middle plot shows how the set of nonzero
eigenvalues of the MSG iterates evolves over time, from which we can see that
many of the extra ranks are ``wasted'' on very small eigenvalues, corresponding
to directions which leave the state matrix only a handful of iterations after
they enter. This indicates that constraining $k_t'$, as is done by capped MSG,
may indeed be safe, and lead to significant speedups. Indeed, as is shown in
the right-hand plot, the low cost-per-iteration of capped MSG causes it to find
good solutions using significantly less computation than the others.


\subsection{Real Data}\label{subsec:pca-capped-msg:real}

\input{pca/capped-msg/figures/fig-real}
Our second set of experiments measure the performance of various stochastic
approximation algorithms for PCA in terms of the population objective, both as
a function of number of iterations as well as the estimated computational
complexity, which, as in Section \ref{subsec:pca-capped-msg:simulated}, we compute from the
representation size of the internal state maintained by each algorithm. Since
we cannot evaluate the true population objective, we estimate it by evaluating
on a held-out test set. We use 40\% of samples in the dataset for training,
20\% for validation, and 40\% for testing.

These experiments were performed on the MNIST dataset, which consists of
$70,000$ binary images of handwritten digits of size $28 \times 28$, resulting
in a dimension of $784$. We normalized the data by mean centering the feature
vectors and scaling each feature by the product of its standard deviation and
the data dimension, so that each feature vector has zero mean and unit norm in
expectation. 
The results are averaged over $100$ random splits into train-validation-test
sets. We are interested in learning a maximum variance subspace of dimension
$k\in{1,4,8}$ in a single ``pass'' over the training sample.
For all algorithms requiring a step size (e.g. all except incremental), we
tried the decreasing sequence of step sizes $\eta_t = \frac{c}{\sqrt{t}}$ for
each $c \in \{2^{-20}, 2^{-19}, \ldots, 2^6\}$, chose the $c$ which minimized
the average validation suboptimality, and reported the suboptimality
experienced for this choice of $c$ on the test set.

Figure \ref{fig:pca-capped-msg:real} plots suboptimality as a function of the
number of samples processed (iterations), and also as a function of the
estimated runtime (computed as in Section
\ref{subsec:pca-capped-msg:simulated}).
The incremental algorithm makes the most progress per iteration and is also the
fastest of all algorithms.
MSG is comparable to the incremental algorithm in terms of the the progress
made per iteration. However, its runtime is a worse than the incremental
because it will often keep a slightly larger representation (of dimension
$k_t'$) than the incremental algorithm. The capped MSG variant (with $K=k+1$)
is significantly faster---almost as fast as the incremental algorithm, while, as
we saw in the previous section, being less prone to getting stuck.
Warmuth \& Kuzmin's algorithm fares well with $k=1$, but its performance drops
for higher $k$. Inspection of the underlying data shows that, in the
$k\in\{4,8\}$ experiments, it also tends to have a larger $k_t'$ than MSG in
these experiments, and therefore has a higher cost-per-iteration.
GROUSE performs better than Warmuth \& Kuzmin's algorithm, but fares much worse
when compared with the MSG and capped MSG.

The most important ``message'' of these experiments is that the capped MSG
algorithm, despite its similarity to the incremental algorithm, has far less of
a tendency to get stuck, even when $K=k+1$, while still performing nearly as
well even in those cases in which the incremental algorithm converges rapidly.
Hence, it is a good ``safe'' alternative to the incremental algorithm.

%% file: pca/capped-msg/figures/fig-simulated-iterations.tex
\begin{figure}

\begin{center}
\begin{tabular}{ @{} L @{} H @{} H @{} }
& \large{Gaussian} & \large{Orthogonal} \\
\rotatebox{90}{\scriptsize{Suboptimality}} &
\includegraphics[width=0.45\textwidth]{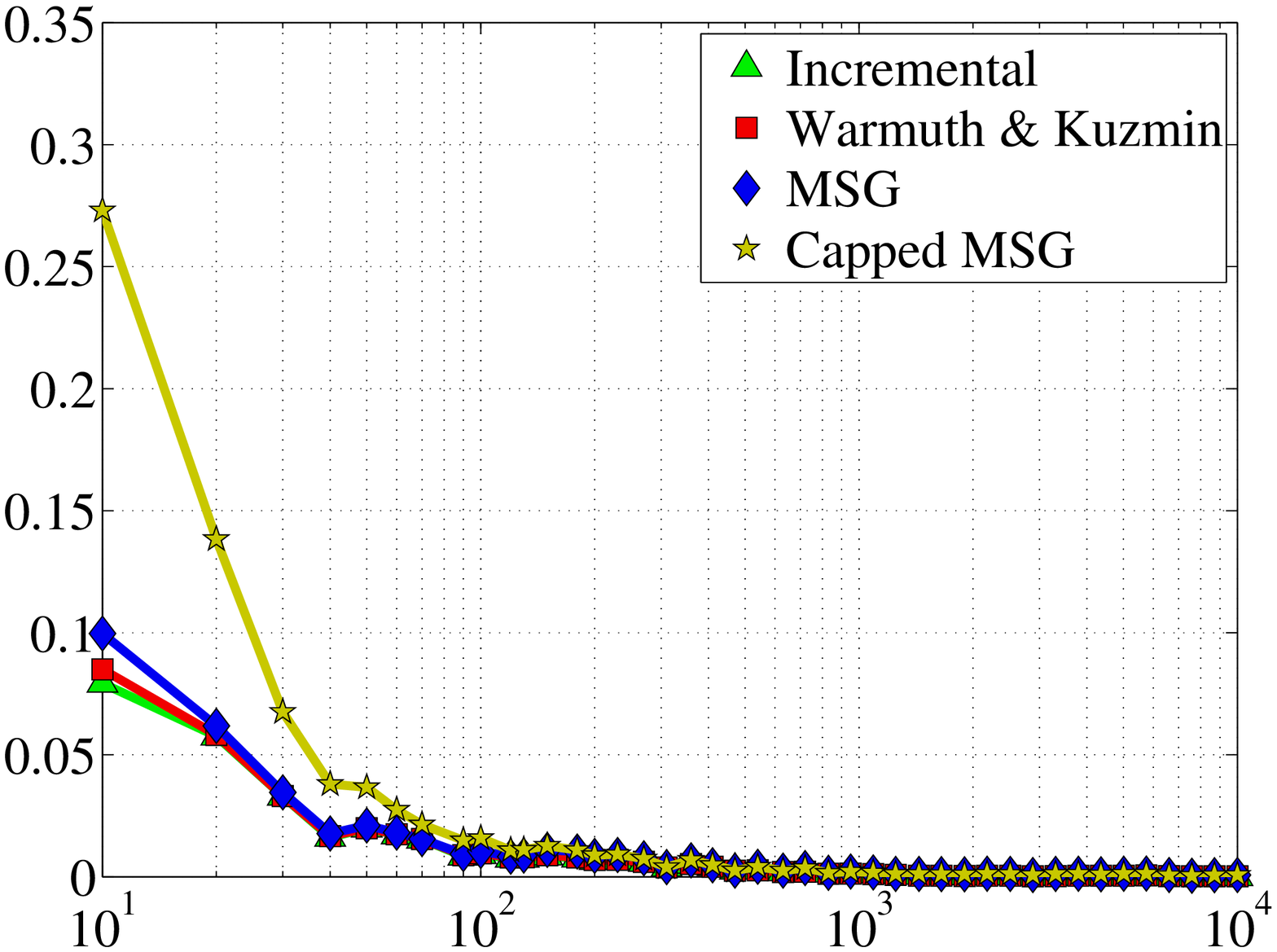} &
\includegraphics[width=0.45\textwidth]{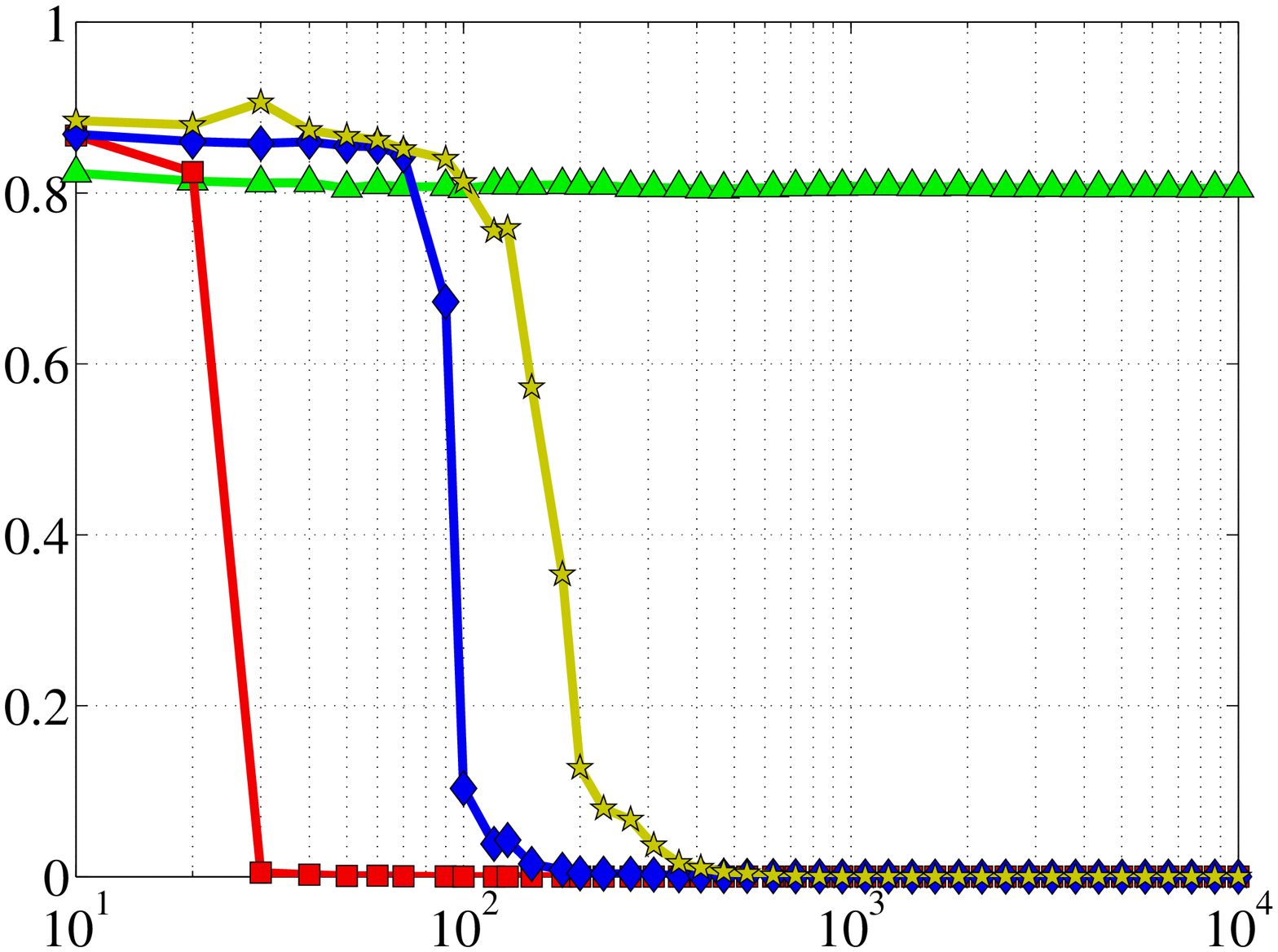} \\
\rotatebox{90}{\scriptsize{Suboptimality}} &
\includegraphics[width=0.45\textwidth]{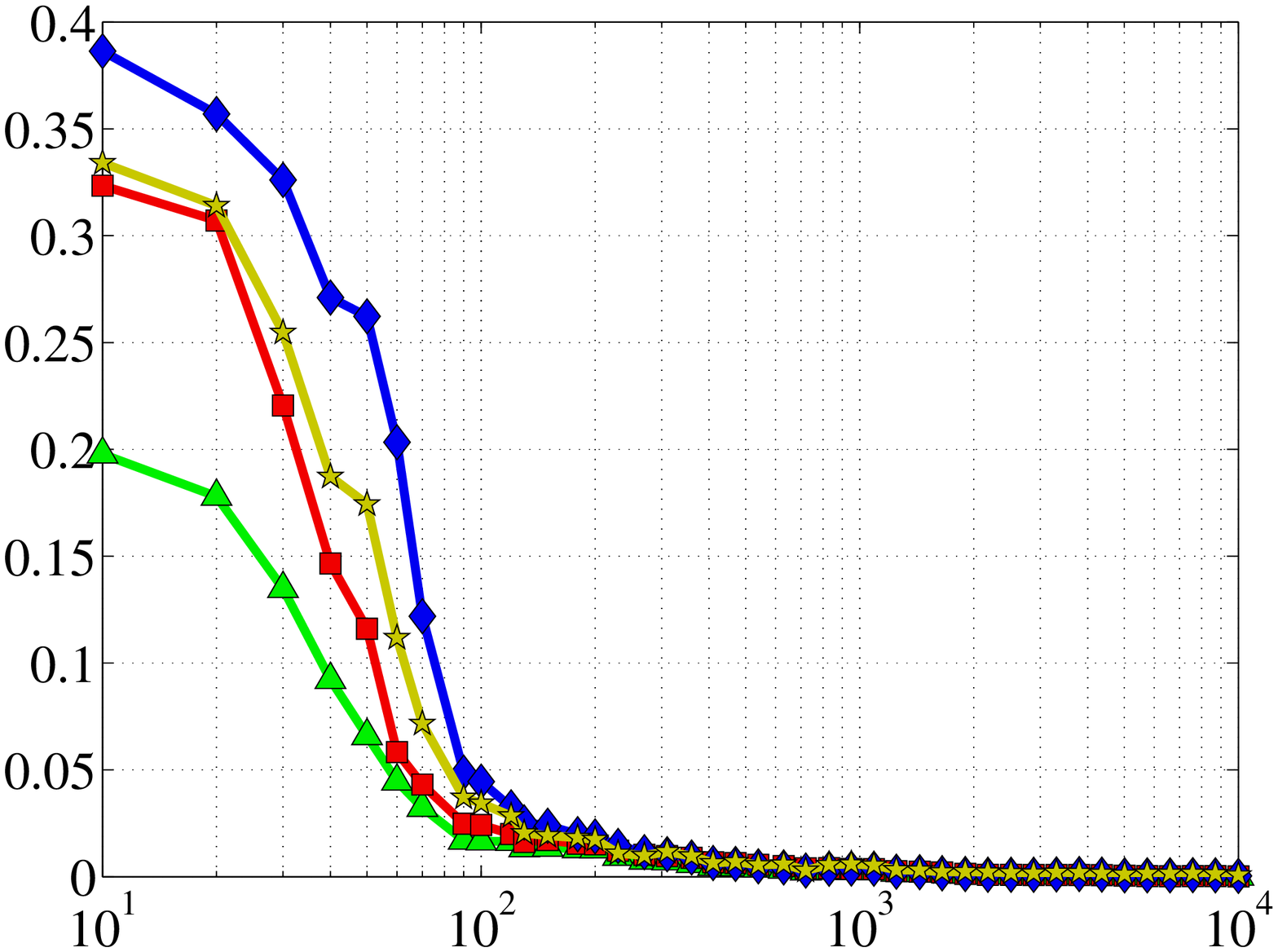} &
\includegraphics[width=0.45\textwidth]{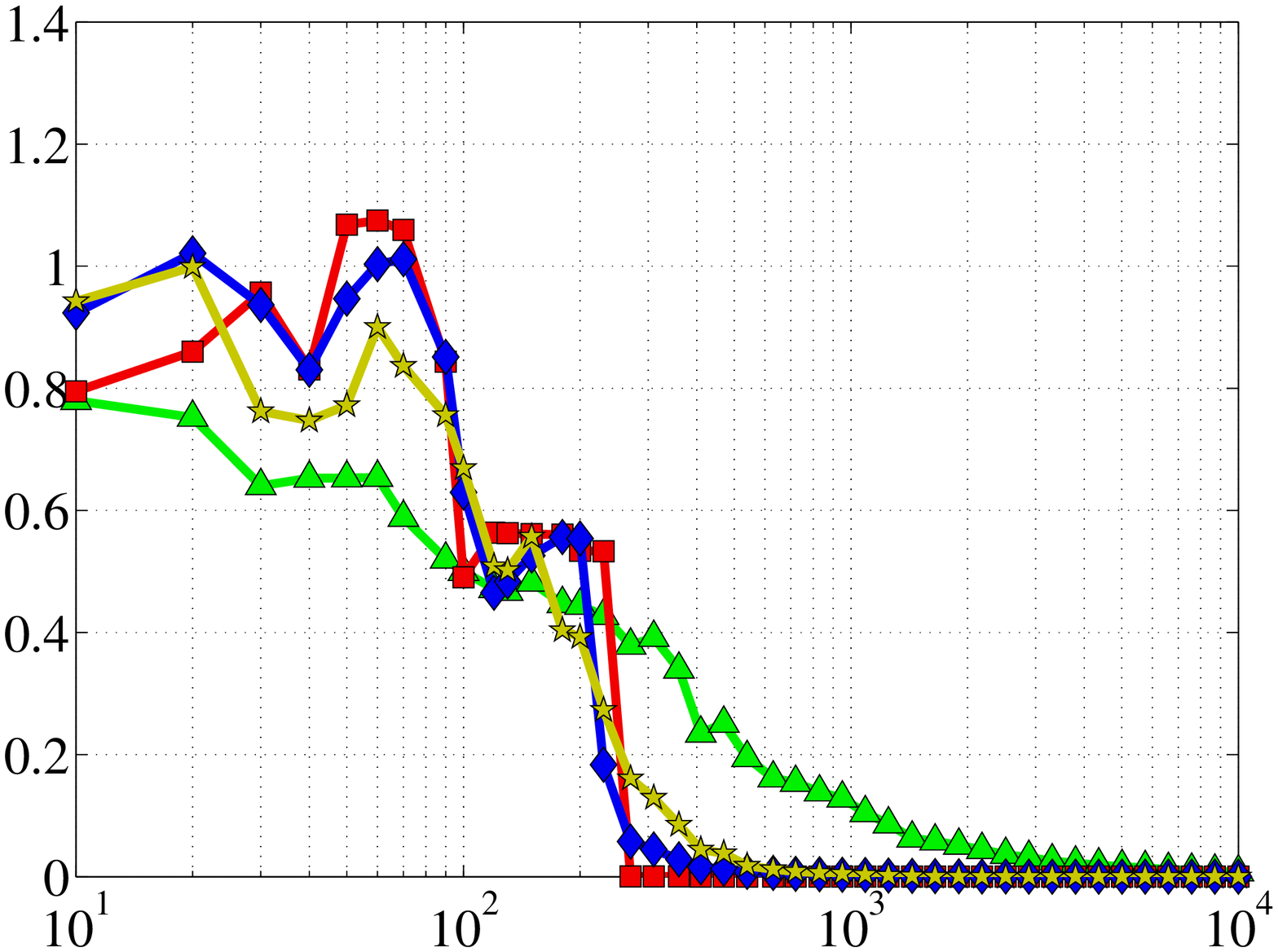} \\
\rotatebox{90}{\scriptsize{Suboptimality}} &
\includegraphics[width=0.45\textwidth]{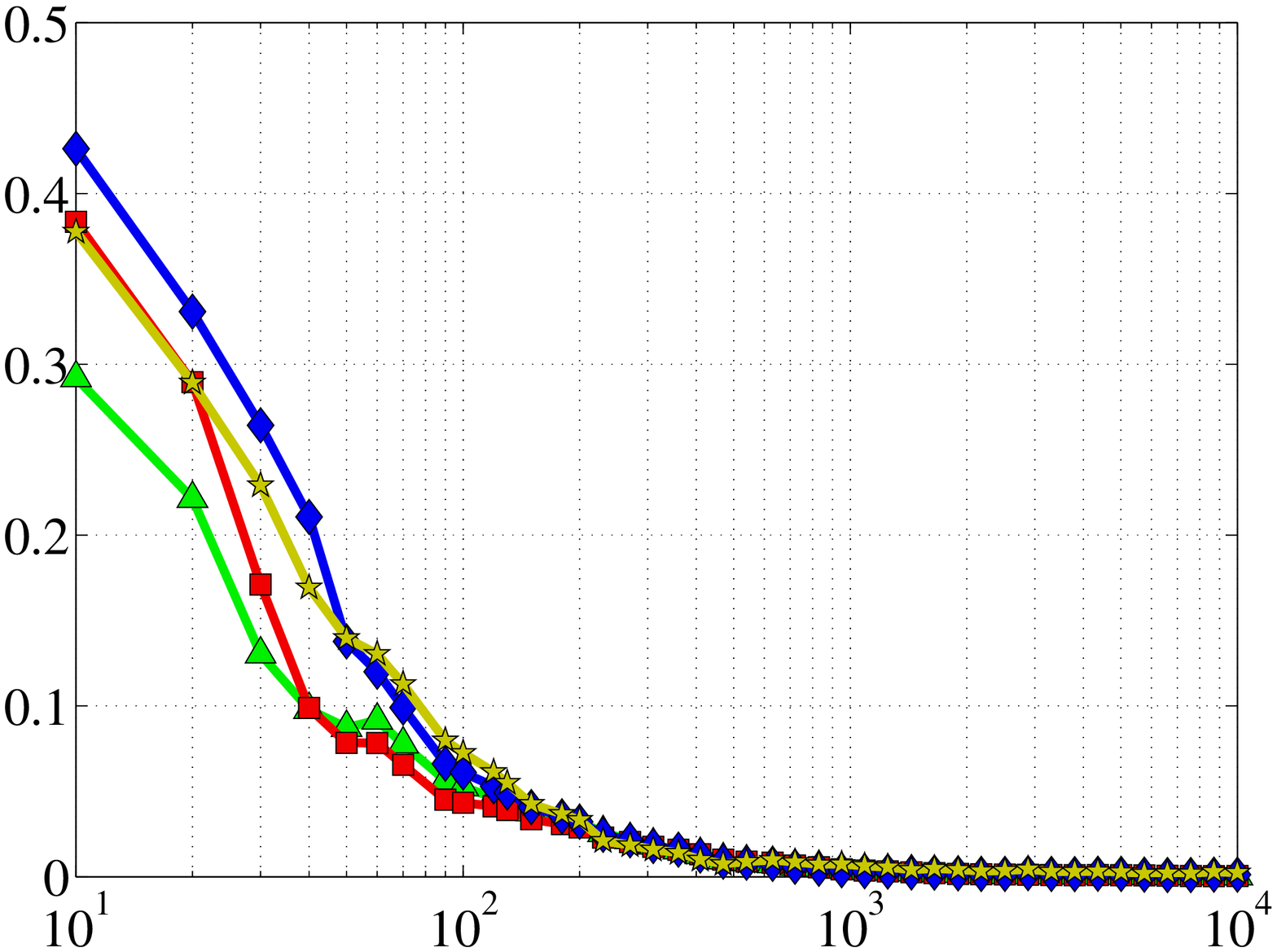} &
\includegraphics[width=0.45\textwidth]{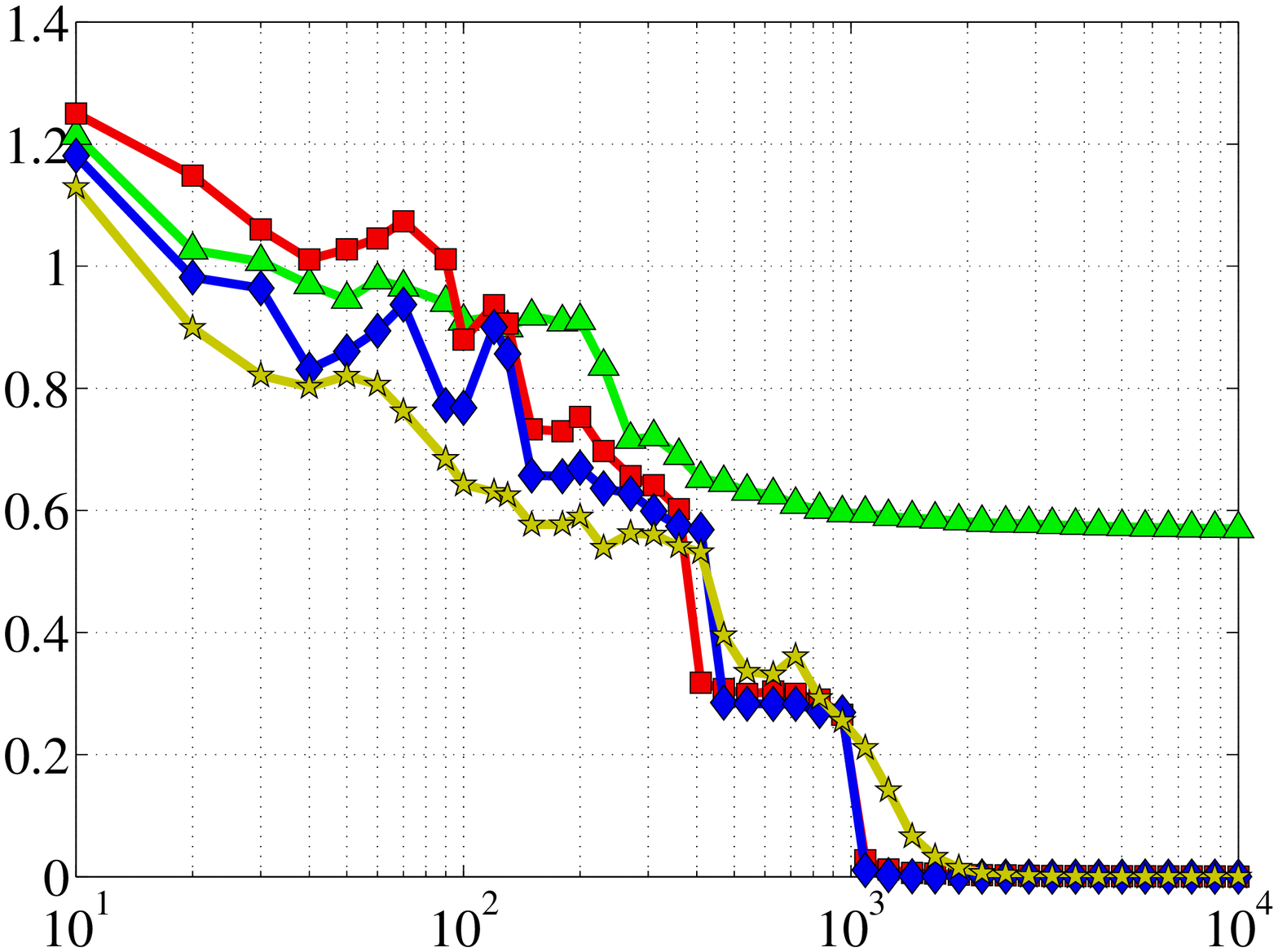} \\
& \scriptsize{Iterations} & \scriptsize{Iterations}
\end{tabular}
\end{center}

\caption{
Plots of suboptimality (vertical axis) versus iteration count (horizontal axis)
on simulated data. Each row of plots corresponds to a different choice of the
parameter $k$, which is both the subspace dimension sought by the algorithms,
and the parameter to the covariance of the data distribution $\Sigma^{(k)}$.
The first row has $k=1$, the second $k=2$ and the third $k=4$.  In the first
column of plots, the data distribution is Gaussian, while in the second column
it is the ``orthogonal distribution'' described in Section
\ref{subsec:pca-capped-msg:simulated}.
}

\label{fig:pca-capped-msg:simulated-iterations}

\end{figure}

%% file: pca/capped-msg/figures/fig-simulated-runtime.tex
\begin{figure}

\begin{center}
\begin{tabular}{ @{} L @{} H @{} H @{} }
& \large{Gaussian} & \large{Orthogonal} \\
\rotatebox{90}{\scriptsize{$k_t'$}} &
\includegraphics[width=0.45\textwidth]{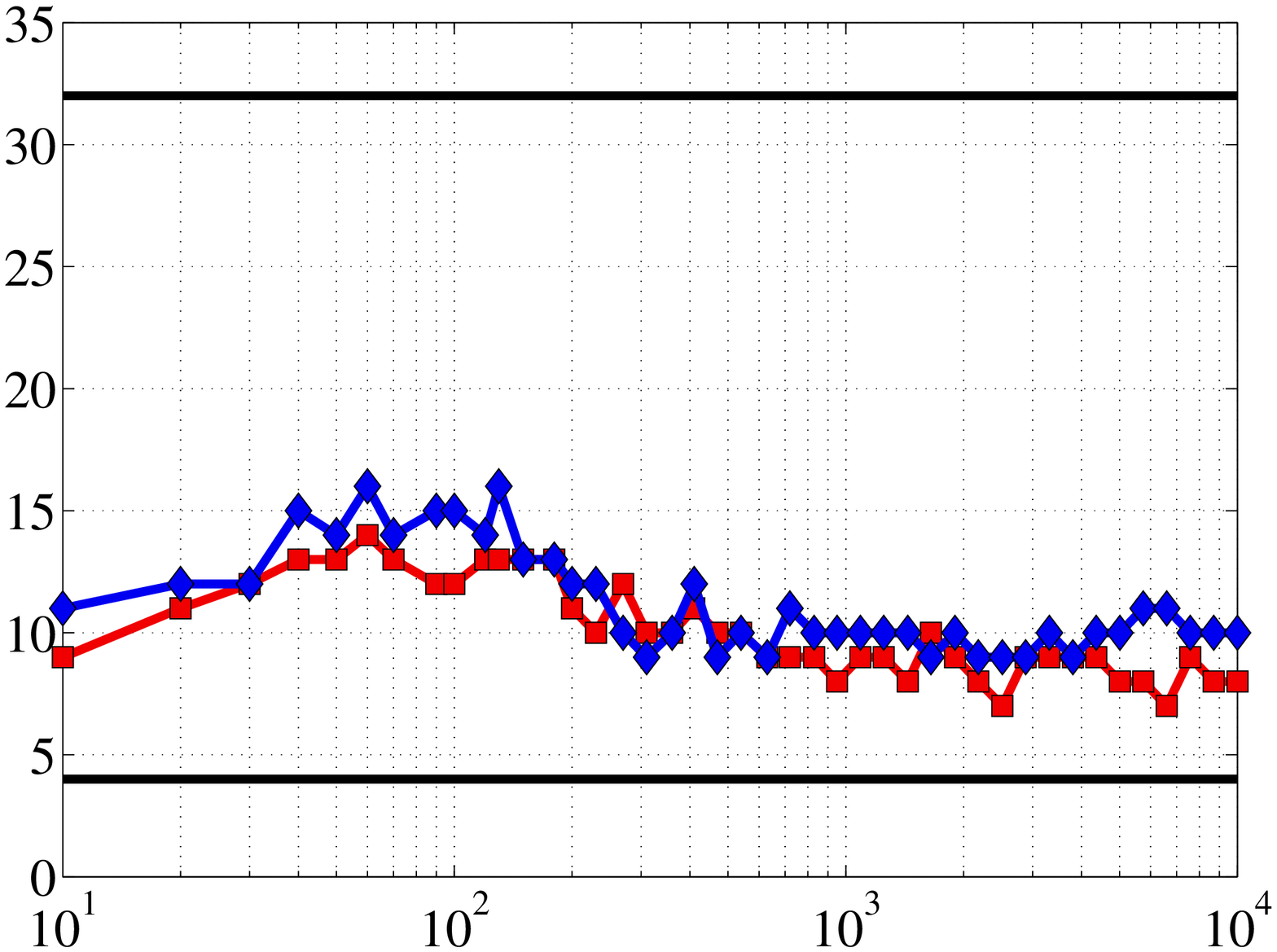} &
\includegraphics[width=0.45\textwidth]{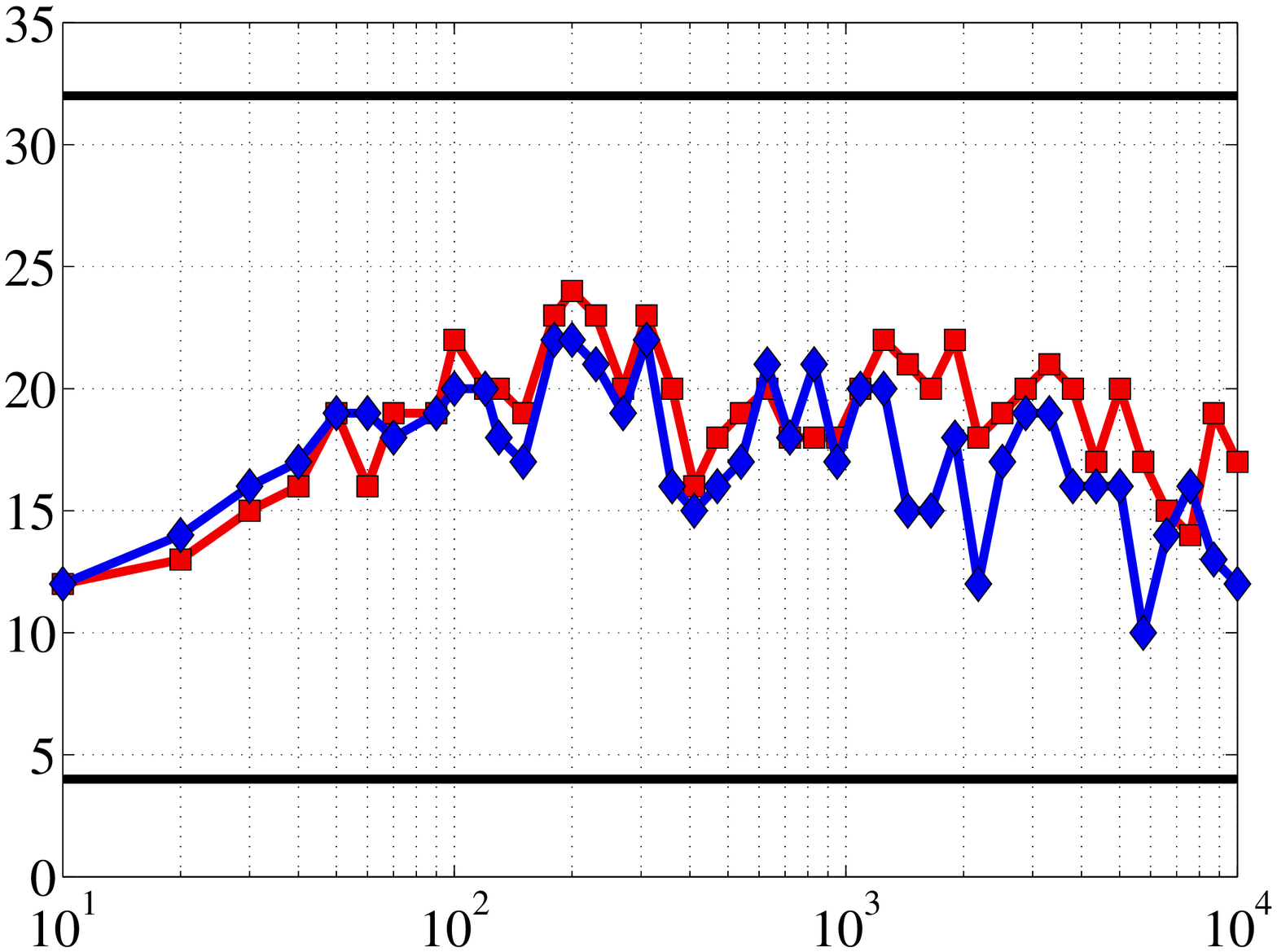} \\
& \scriptsize{Iterations} & \scriptsize{Iterations} \\
\rotatebox{90}{\scriptsize{Eigenvalue}} &
\includegraphics[width=0.45\textwidth]{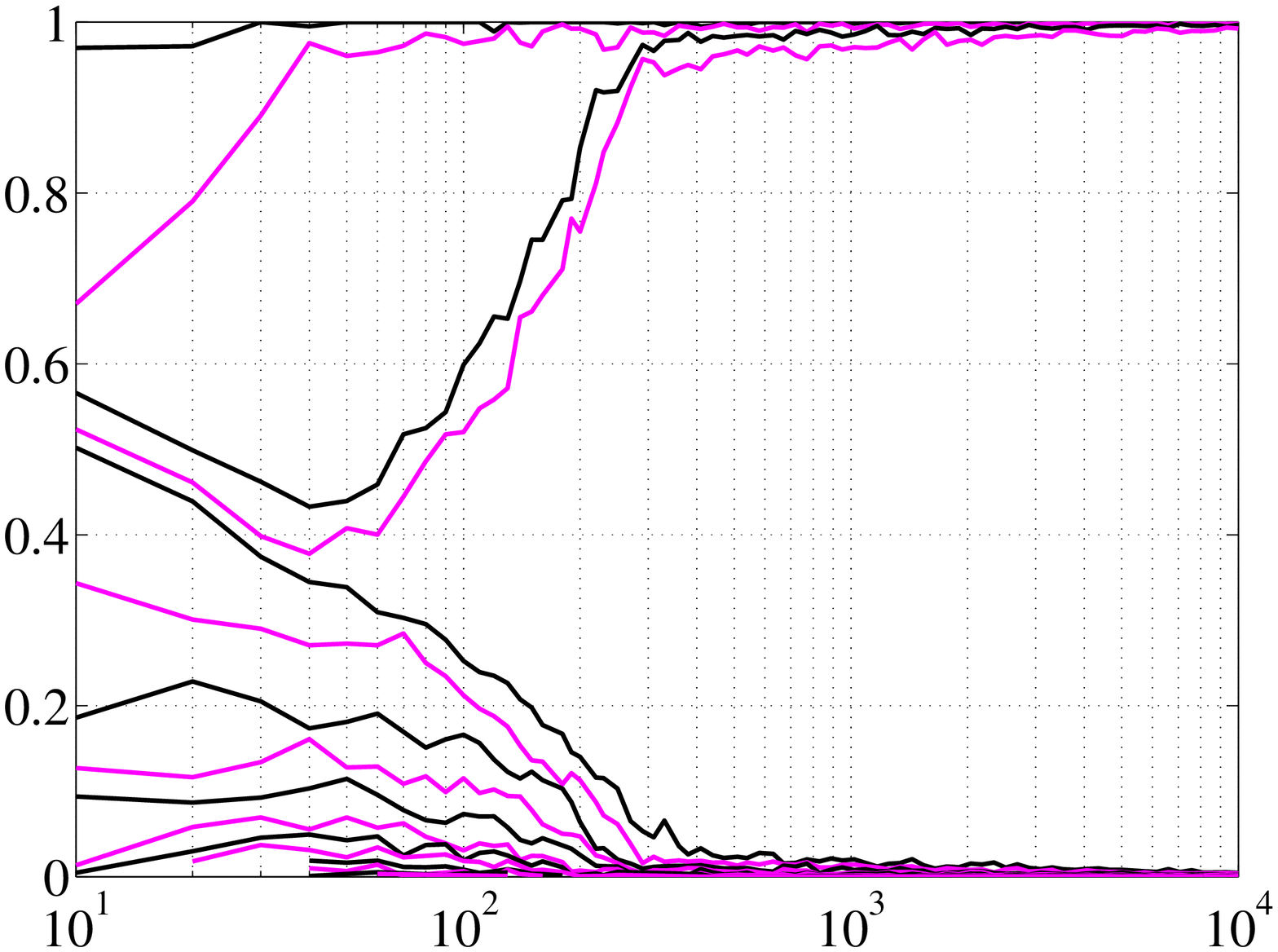} &
\includegraphics[width=0.45\textwidth]{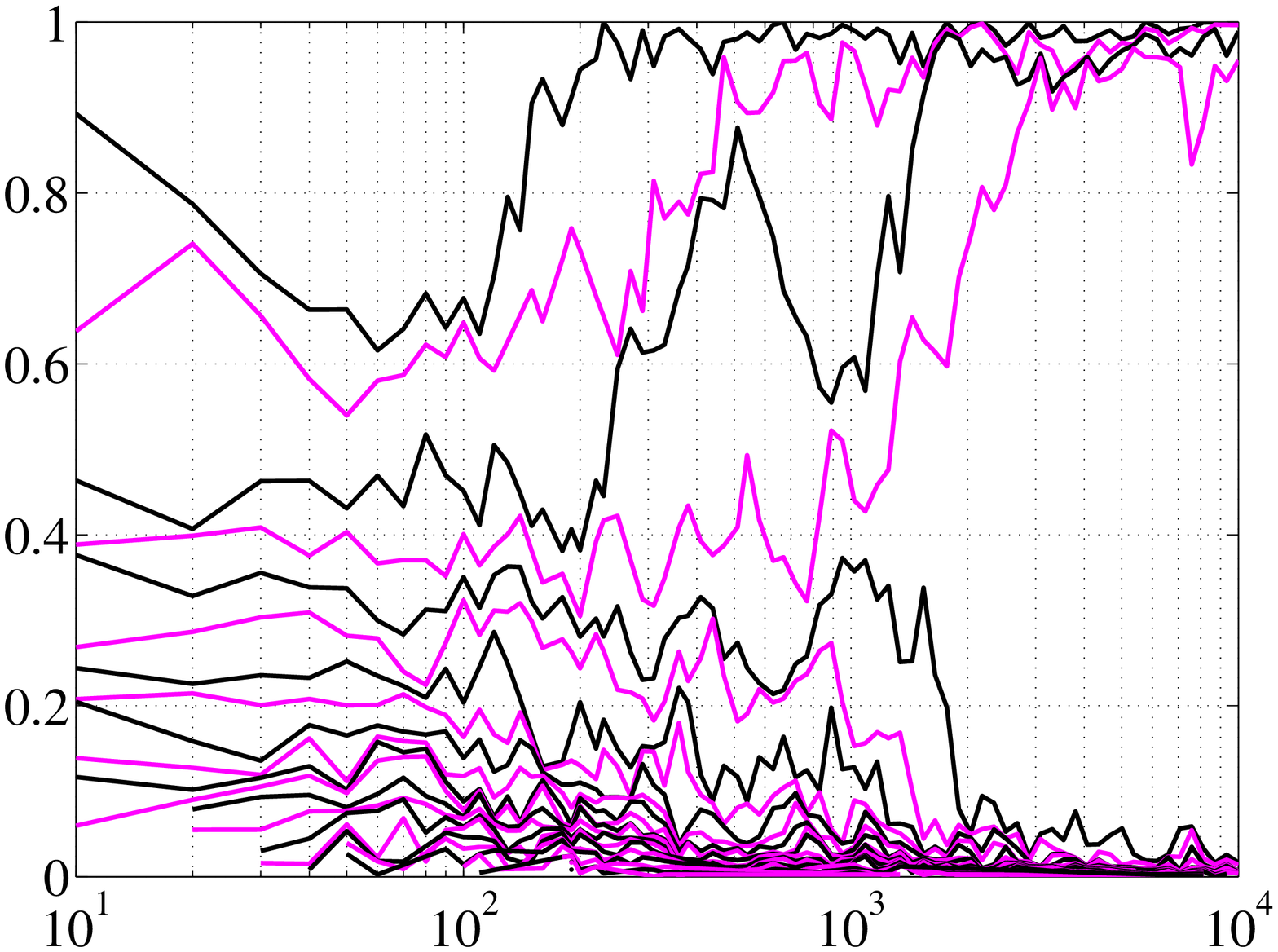} \\
& \scriptsize{Iterations} & \scriptsize{Iterations} \\
\rotatebox{90}{\scriptsize{Suboptimality}} &
\includegraphics[width=0.45\textwidth]{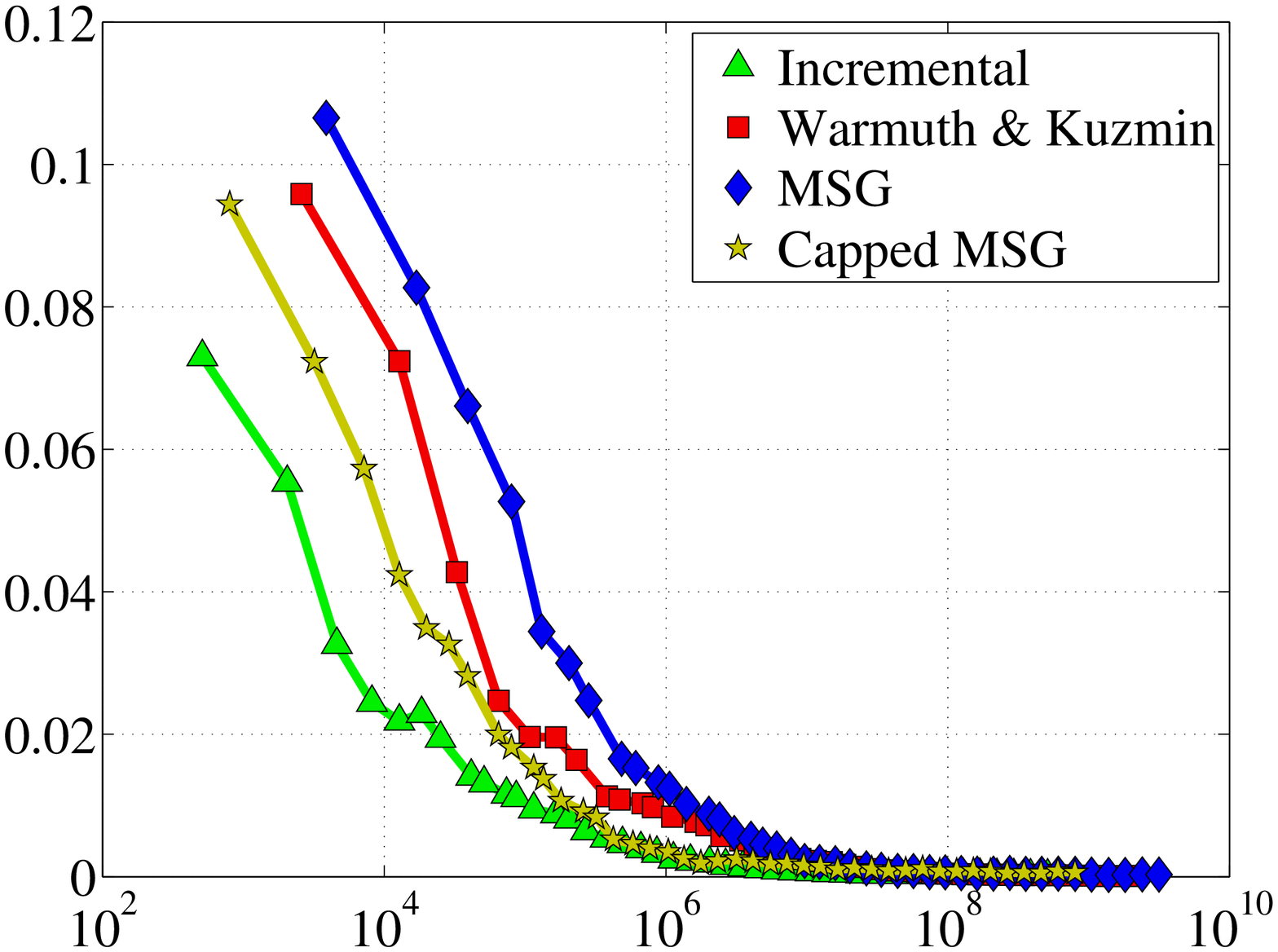} &
\includegraphics[width=0.45\textwidth]{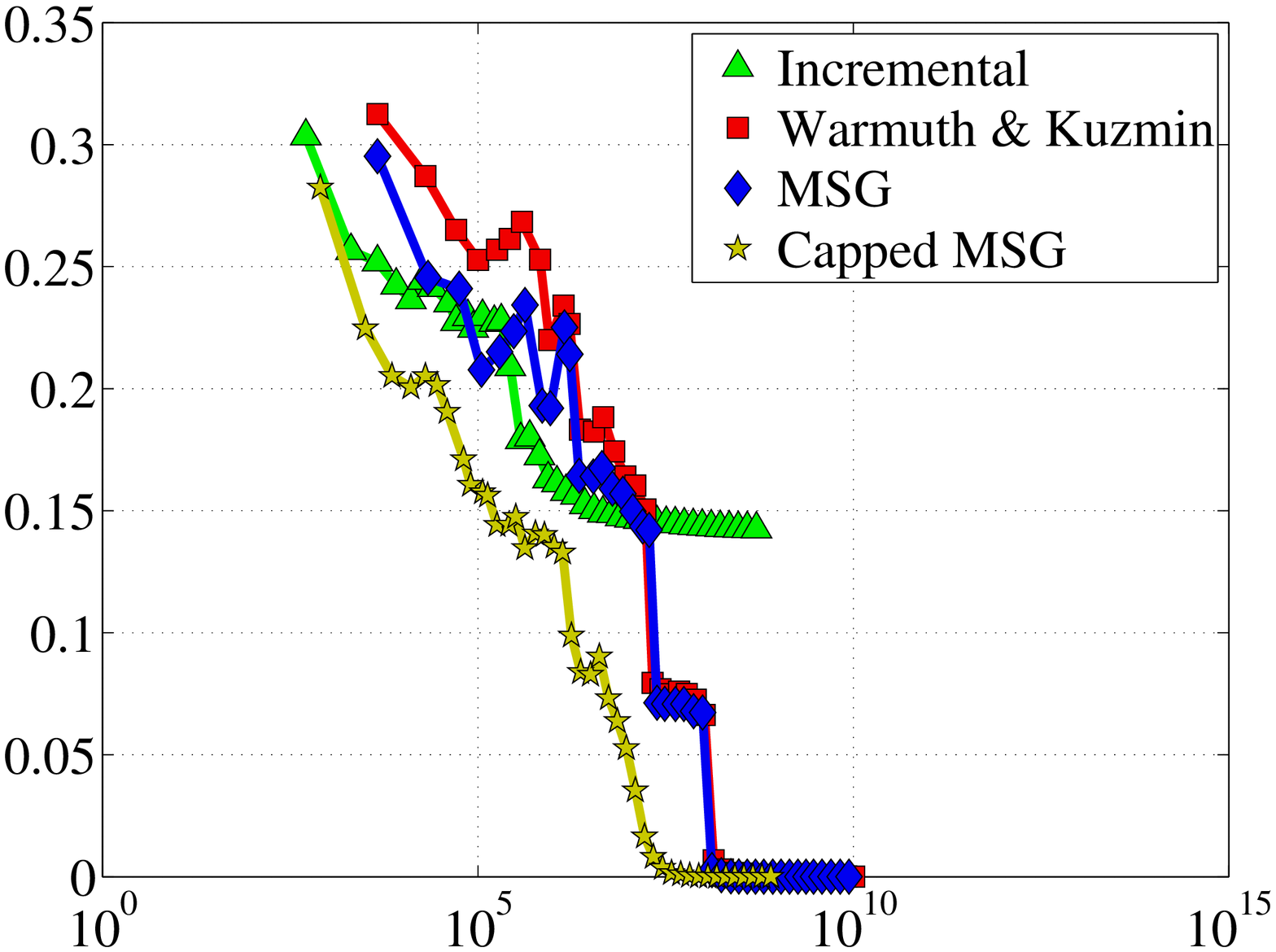} \\
& \scriptsize{Est. runtime} & \scriptsize{Est. runtime}
\end{tabular}
\end{center}

\caption{
In-depth look at simulated-data experiments with $k=4$. First row: the ranks
$k_t'$ of the iterates found by MSG. Middle row: the eigenvalues of the
iterates $M^{(t)}$ found by MSG. Bottom row: suboptimality as a function of
estimated runtime $\sum_{s=1}^{t} (k_s')^2$.  In the first column of plots, the
data distribution is Gaussian, while in the second column it is the
``orthogonal distribution'' described in Section
\ref{subsec:pca-capped-msg:simulated}.
}

\label{fig:pca-capped-msg:simulated-runtime}

\end{figure}

%% file: pca/capped-msg/figures/fig-real.tex
\begin{figure}

\begin{center}
\begin{tabular}{ @{} L @{} H @{} H @{} }
\rotatebox{90}{\scriptsize{Suboptimality}} &
\includegraphics[width=0.45\textwidth]{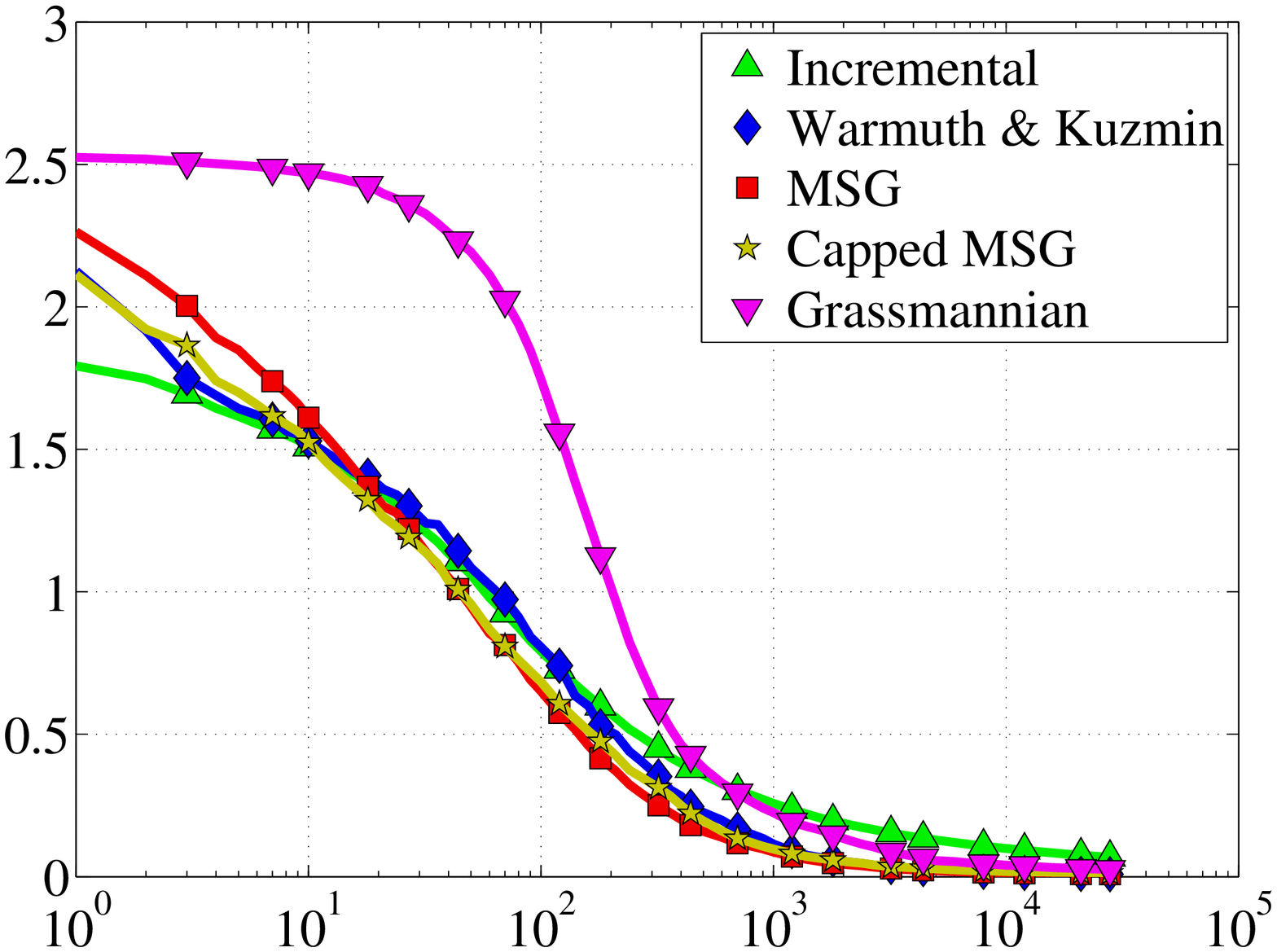} &
\includegraphics[width=0.45\textwidth]{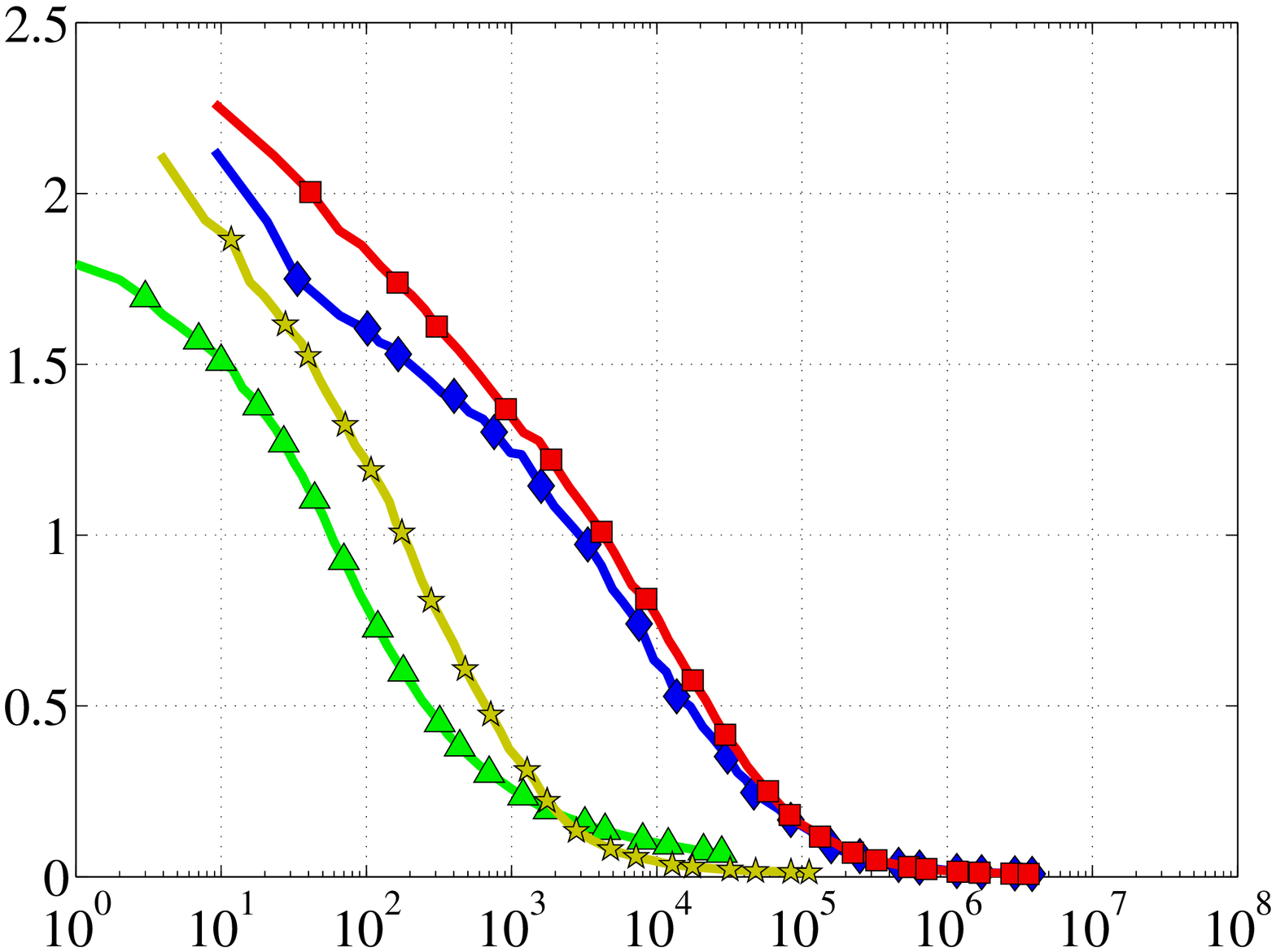} \\
\rotatebox{90}{\scriptsize{Suboptimality}} &
\includegraphics[width=0.45\textwidth]{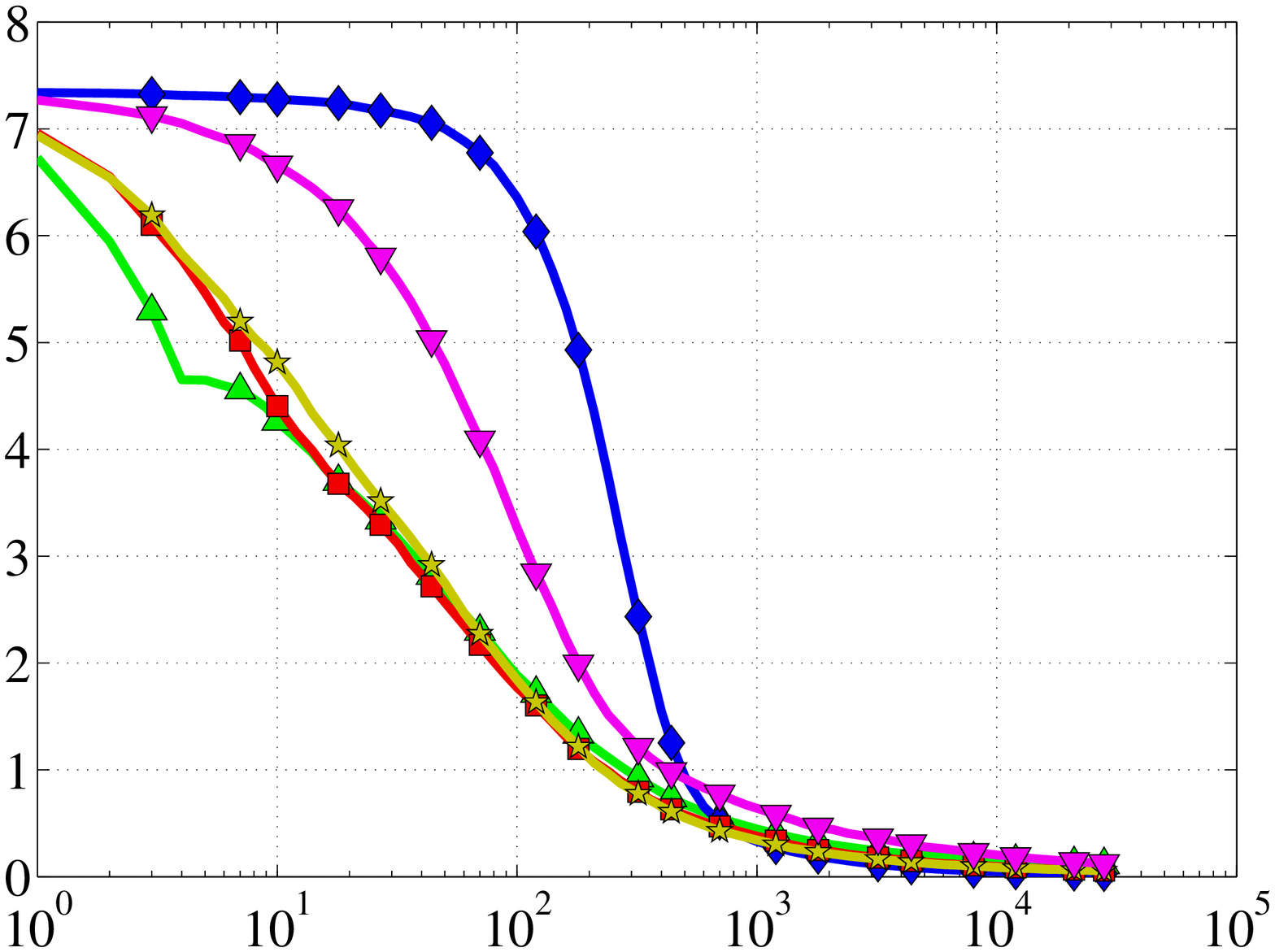} &
\includegraphics[width=0.45\textwidth]{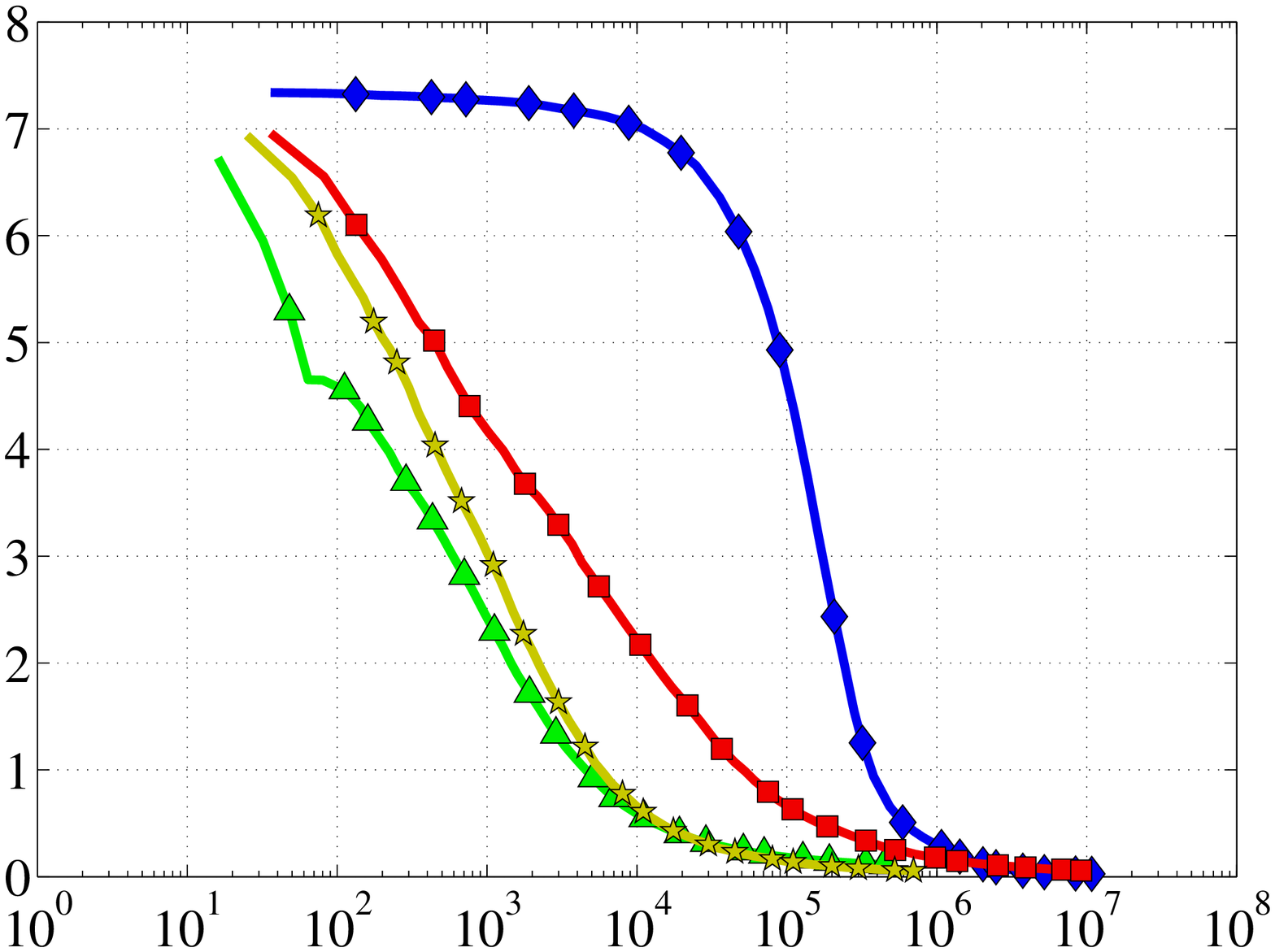} \\
\rotatebox{90}{\scriptsize{Suboptimality}} &
\includegraphics[width=0.45\textwidth]{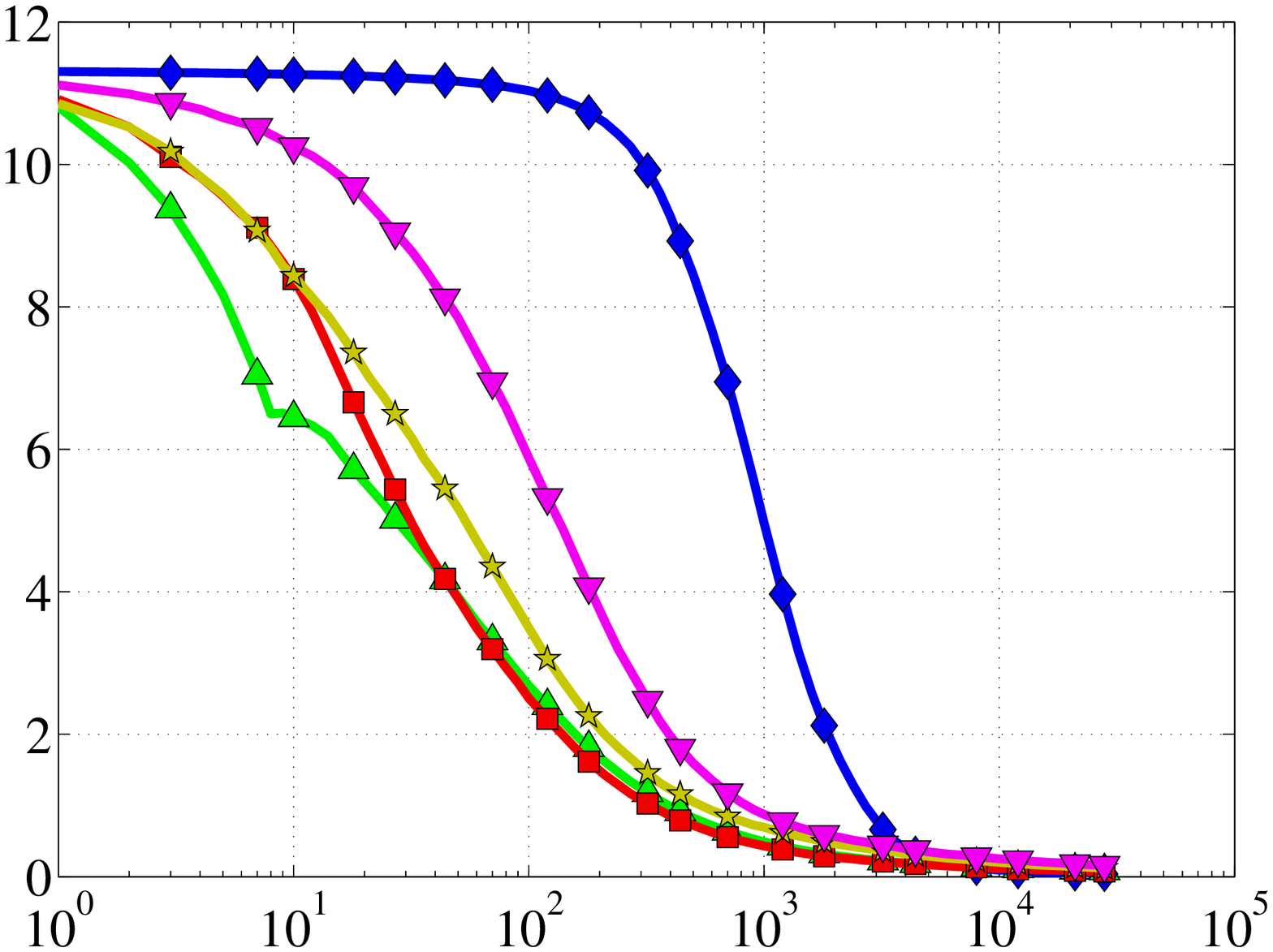} &
\includegraphics[width=0.45\textwidth]{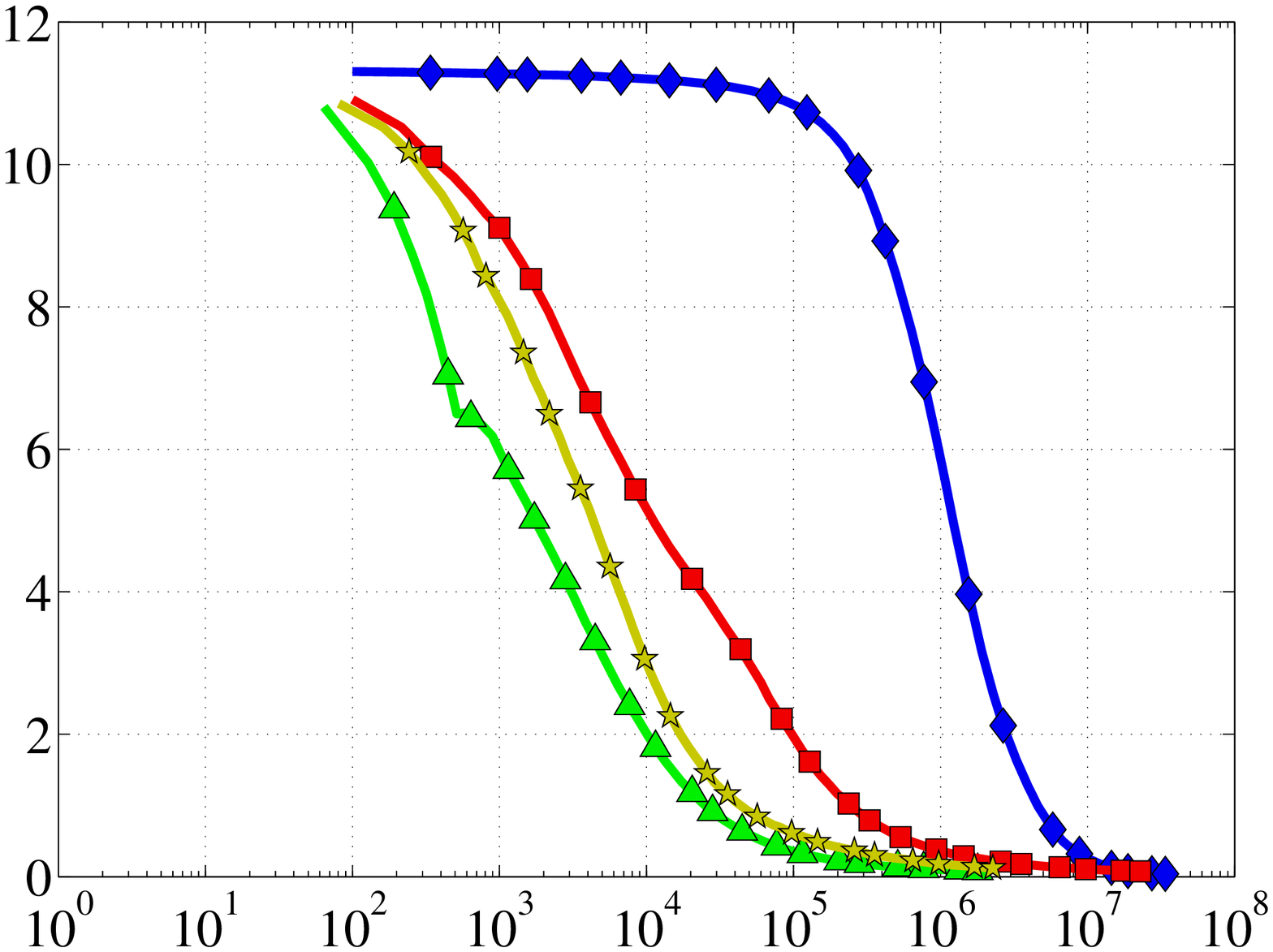} \\
& \scriptsize{Iterations} & \scriptsize{Est. Runtime}
\end{tabular}
\end{center}

\caption{
Plots of suboptimality (vertical axis) versus the iteration count or estimated
runtime (horizontal axis) on the MNIST dataset. Each row of plots corresponds
to a different choice of the parameter $k$, with the first row having $k=1$,
the second $k=4$ and the third $k=8$. The first column of plots shows
suboptimality as a function of iteration count, while the second column shows
suboptimality as a function of estimated runtime $\sum_{s=1}^{t} (k_s')^2$.
}

\label{fig:pca-capped-msg:real}

\end{figure}

%% file: pca/capped-msg/sec-proofs.tex
\section{Proofs for Chapter \ref{ch:pca-capped-msg}}

\begin{proofs}
\input{pca/capped-msg/theorems/lem-rate}
\input{pca/capped-msg/theorems/lem-projection}
\input{pca/capped-msg/theorems/lem-capping}
\end{proofs}

%% file: front-back-matter/bibliography.tex

\label{app:bibliography} 

\manualmark
\markboth{\spacedlowsmallcaps{\bibname}}{\spacedlowsmallcaps{\bibname}} 
\refstepcounter{dummy}

\addtocontents{toc}{\protect\vspace{\beforebibskip}} 
\addcontentsline{toc}{chapter}{\tocEntry{\bibname}}

\bibliographystyle{plainnat}

\bibliography{main}